\documentclass[aos]{imsart}
\usepackage{amsthm,amsmath,amsfonts,amssymb,stackengine,comment}
\usepackage{amsfonts,mathrsfs}
\usepackage{color}
\usepackage{enumerate}
\usepackage[textwidth=8em,textsize=scriptsize]{todonotes}
\usepackage{booktabs,float}
\usepackage{graphicx}
\usepackage{multirow}
\usepackage{float}
\usepackage{enumerate}
\usepackage{mathtools}
\usepackage{subfigure}
\usepackage{algorithm}
\usepackage{algorithmicx}
\usepackage[noend]{algpseudocode}
\usepackage{xr}
\RequirePackage[authoryear]{natbib} %% uncomment this for author-year bibliography
\RequirePackage[colorlinks,citecolor=blue,urlcolor=blue]{hyperref}

\startlocaldefs
\numberwithin{equation}{section}
\theoremstyle{plain}
\newtheorem{theorem}{Theorem}[section]
\newtheorem{lemma}{Lemma}[section]
\theoremstyle{remark}
\newtheorem{definition}{Definition}
\newtheorem{assumption}{Assumption}

\newtheorem{corollary}{Corollary}

\newtheorem{Remark}{\bf Remark}
\newtheorem{NS}{NS}

\def\argmin{\mbox{argmin}}
\def\argmax{\mbox{argmax}}
\def\Pdim{\text{Pdim}}

\def\Erm{\mathrm{E}}

\endlocaldefs

\begin{document}
	
\begin{frontmatter}
		\title{Deep Sufficient Representation Learning \\ via Mutual Information}
		%\title{A sample article title with some additional note\thanksref{t1}}
		\runtitle{
Deep Sufficient Representation Learning}
		%\thankstext{T1}{A sample additional note to the title.}
		
\begin{aug}
			%%%%%%%%%%%%%%%%%%%%%%%%%%%%%%%%%%%%%%%%%%%%%%
			%%Only one address is permitted per author. %%
			%%Only division, organization and e-mail is %%
			%%included in the address.                  %%
			%%Additional information can be included in %%
			%%the Acknowledgments section if necessary. %%
			%%%%%%%%%%%%%%%%%%%%%%%%%%%%%%%%%%%%%%%%%%%%%%
			\author[A]{\fnms{Siming} \snm{Zheng}\ead[label=e1,mark]{simingzheng@cuhk.edu.hk}}
			\author[C]{\fnms{Yuanyuan} \snm{Lin}\ead[label=e2,mark]{ylin@sta.cuhk.edu.hk}}
			\and
			\author[D]{\fnms{Jian} \snm{Huang}\ead[label=e3,mark]{j.huang@polyu.edu.hk}}
			%%%%%%%%%%%%%%%%%%%%%%%%%%%%%%%%%%%%%%%%%%%%%%
			%% Addresses                                %%		%%%%%%%%%%%%%%%%%%%%%%%%%%%%%%%%%%%%%%%%%%%%%%
\address[A]{Department of Statistics, The Chinese University of Hong Kong,  Hong Kong SAR  \printead{e1}}							
\address[C]{Department of Statistics, The Chinese University of Hong Kong,  Hong Kong SAR  \printead{e2}}			
\address[D]{Department of Applied Mathematics, The Hong Kong Polytechnic University, Hong Kong SAR
  \printead{e3}}	
\end{aug}
		
\begin{abstract}
We propose a mutual information-based sufficient representation learning (MSRL) approach, which uses the variational formulation of mutual information and leverages the approximation power of deep neural networks.  MSRL learns a sufficient representation with the maximum mutual information with the response and a user-selected distribution.  It can easily handle multi-dimensional continuous or categorical response variables.  MSRL is shown to be consistent in the sense that the conditional probability density function of the response variable given the learned representation converges to the conditional probability density function of the response variable given the predictor.  Non-asymptotic error bounds for MSRL are also established under suitable conditions.  To establish the error bounds, we derive a generalized Dudley's inequality for an order-two U-process indexed by deep neural networks, which may be of independent interest.
We discuss how to determine the intrinsic dimension of the underlying data distribution.  Moreover, we evaluate the performance of MSRL via extensive numerical experiments and real data analysis and demonstrate that MSRL outperforms some existing nonlinear sufficient dimension reduction methods.

\medskip\noindent
\textbf{Keywords:}
Deep neural networks; non-asymptotic error bounds; nonparametric estimation;  representation learning; sufficient dimension reduction
\end{abstract}

\iffalse		
		\begin{keyword}[class=MSC2020]
			\kwd[Primary ]{62G05}
			\kwd{62G08}
			\kwd[; secondary ]{68T07}
		\end{keyword}
		
		\begin{keyword}
			\kwd{Circumventing the curse of dimensionality}
			\kwd{Deep neural networks}
			\kwd{Heavy-tailed error}
			\kwd{Non-asymptotic error bound}
			\kwd{Low-dimensional manifolds}
		\end{keyword}
\fi		
	\end{frontmatter}

\section{Introduction}
Learning a good representation of a predictor $X$ for predicting a response $Y$ is a fundamental problem in statistics and machine learning.  An effective data representation is crucial to the success for downstream modeling and analysis tasks.  In this paper, we propose a nonparametric method for learning a sufficient representation of $X$. We formulate an objective function based on  a variational form of the mutual information and use deep neural networks to approximate the representation function.  We  shall refer to our proposed method as {Mutual-information-based Sufficient Representation Learning} (MSRL).

MSRL is inspired by the idea underlying sufficient dimension reduction \citep[SDR,][]{Li:1991, Cook:1998}.  The goal of SDR is to learn a linear representation of $X$ such that the response $Y$ is conditionally independent of $X$ given this representation, that is, the representation is sufficient.
SDR is a semiparametric method in the sense that only the representation to be estimated is assumed to take a linear form, but no parametric form is assumed for the conditional distribution of $Y$ given $X$.  \cite{Li:1991} introduced the sliced inverse regression (SIR) for estimating the sufficient direction under the assumption that the predictor satisfies a linear expectation property.  Subsequently, various related methods have been proposed  under similar assumptions, including sliced
average variance estimation \citep[SAVE,][]{Cook:Weisberg:1991}, principal Hessian directions \citep{Li:1992}, minimum average variance estimation
\citep{Xia:Tong:Li:Zhu:2002}, and cumulative slicing estimation \citep{Zhu:Zhu:Feng:2010}, among others.  We refer to \cite{Cook:2007,Cook:2018} and \cite{Ma:Zhu:2013} and the references therein for thorough reviews of the SDR methods.

In recent years, several studies have developed nonlinear SDR methods using the kernel tricks.  \cite{Lee2013aos} proposed the generalized sliced inverse regression (GSIR) and generalized sliced average variance (GSAVE)  methods based on the reproducing kernel Hilbert space (RKHS) theory.  Other kernel-based nonlinear SDR methods have been developed by  \cite{Li:Artemiou:Li:2011}, \cite{Li:Song:2017},  and \cite{Shin:Wu:Zhang:Liu:2017}, among others.  For these kernel-based methods, an $n\times n$ Gram matrix is involved in the estimation procedure, where $n$ is the sample size. Hence, these methods are computationally expensive when $n$ is relatively large.

MSRL uses an objective function based on the mutual information for learning a sufficient representation $R(X)$ so that $Y$ and $X$ are conditionally independent given $R(X)$.  It takes advantage of the strong approximation power of deep neural networks and model the representation $R(X)$ using deep neural networks to extract nonlinear features in a nonparametric fashion.  To facilitate the implementation of the proposed method, we consider a dual form of the mutual information criterion. Then the estimation of $R(X)$ can be done using the adversarial training procedure \citep{goodfellow2014generative}.  MSRL can easily deal with the case when $Y$ is a multi-dimensional response or a categorical variable.  MSRL is easy-to-implement and computationally efficient, even for very large $n$.

%MSRL is related to
The information bottleneck (IB) method \citep{tishby:Pereira:Bialek:1999, Tishby:Zaslavsky:2015} is also based on the mutual information.  The goal of IB is to find a maximally compressed mapping of the
%input/
predictor
%vector
that can preserve as much information on the
%output/
response
%variable
as possible.
To achieve this, IB maximizes the mutual information between $Y$ and a parametrically modeled representation $R(X)$, under a constraint that the mutual information between $X$ and $R(X)$ is smaller than a pre-specified nonnegative constant.  Different from the IB method and motivated by a result from the optimal transport theory, we introduce a new constraint  to enforce the  representation to follow a user-specified distribution.  Making use of the variational formulation of the $f$-divergence, we formulate MSRL
%estimation %procedure
%problem
as a minimax optimization problem and compute the estimated representation
%which can be solved
using the adversarial learning techniques \citep{goodfellow2014generative, nowozin2016f}.

Our contributions are as follows.
\begin{itemize}
	\item We propose MSRL, a new nonparametric sufficient representation learning method based on $f$-mutual information and deep  neural networks. We formulate the problem as that of maximizing the mutual information between the response and the representation of the predictor. Computationally, we use a dual form of the mutual information and parameterize the representation using deep neural networks. MSRL can directly handle multi-dimensional continuous response and can be easily modified to deal with categorical response variables.
	
	\item We show that the learned representation is consistent in the sense that the conditional probability density function (p.d.f) of $Y$ given the  learned representation converges to the conditional p.d.f of $Y$ given $X$. We also study the convergence properties by establishing the non-asymptotic error bounds  under suitable conditions.  To establish the non-asymptotic error bounds, we derive a generalized Dudley's inequality for an order-two U-process indexed by deep neural network functions using Rademacher chaos. This inequality may be of independent interest in other problems where U-processes arise.
	
	\item We evaluate the performance of MSRL via extensive numerical experiments and real data analysis and demonstrate that MSRL outperforms some existing nonlinear sufficient dimension reduction methods. We also discuss the extension to accommodate the cases involving categorical variables and provide a feasible algorithm to determine the intrinsic dimension of the sufficient representation.
\end{itemize}

The rest of the paper is organized as follows.  We first give some preliminary information-theoretic background and describe our proposed MSRL method  in Section \ref{Info-BackGroud-EstMethod}.  In Section \ref{Theoretical_Results}, we establish the consistency property of MSRL and study its non-asymptotic error bounds under reasonable conditions.  We provide a high-level description of the proofs for the theoretical results in Section \ref{high_level_description} and briefly discuss the extension to handle categorical responses or predictor in Section \ref{Extensions}.  We describe a training algorithm and a way to determine the intrinsic dimension of the representation in Section \ref{Implementation}.  In Section \ref{Experiments} we conduct numerical experiments to evaluate the performance of MSRL.  Concluding remarks are given in Section \ref{Discussion}.
	
	\section{$f$-mutual information and sufficient representations}
To learn a sufficient representation $R$, we use mutual information to characterize conditional independence.  Let $f:\mathbb{R}^+\to\mathbb{R}$ be a differentiable and convex function with $f(1)=0$.  Let $P_1$ and $P_2$ be two probability measures supported on a set $\mathcal{Z}$ and with p.d.f's $p_1(\cdot)$ and $p_2(\cdot)$, respectively.  The $f$-divergence \citep{ali1966general} %Csiszar1967}
between $P_1$ and $P_2$ is defined as
\[
\mathbb{D}_f(P_1~||~P_2) :=\int_{\mathcal{Z}}p_2(z)f\left(\frac{p_1(z)}{p_2(z)}\right)dz.
\]
It can be easily verified that $\mathbb{D}_f(P_1~||~P_2)\ge0$ and the equality holds if $p_1\equiv p_2$. Commonly-used $f$-functions include $f(t)=t\log t$, $f(t)=-(t+1)\log(t+1/2)+t\log t$, and $f(t)=(t-1)^2$, corresponding to  the well-known Kullback-Leibler (KL),
Jensen-Shanon (JS), and $\chi^2$ divergences.

Throughout the paper, we let $(X, Y)$
 be a pair of $d_{X}$-vector of
covariate and $d_{Y}$-vector of response, where   $X\in\mathcal{X}=[0,1]^{d_{X}}$ and $Y\in\mathcal{Y}=[0,1]^{d_{Y}}.$
The probability measures induced by $(X,Y)$, $X$ and $Y$ are respectively denoted by $P_{XY},P_{X}$ and $P_{Y}$, and their corresponding p.d.f's are $p_{XY}, p_{X}$ and $p_{Y}$.  In addition, $P_{X}P_{Y}$ denotes the product measure.  Using these notation, the \textit{$f$-mutual information} between $X$ and $Y$ is defined as
\[
\mathbb{I}_f(X;Y) :=\mathbb{D}_f(P_{XY}~||~P_XP_Y)=\int_{\mathcal{X}\times\mathcal{Y}}p_X(x)p_Y(y)f\left(\frac{p_{XY}(x,y)}{p_X(x)p_Y(y)}\right)dxdy,
\]
and $\mathbb{I}_f(X;Y)=0$ if and only if $X$ and $Y$ are independent, when $f(t)$ is strictly convex at $t=1$ \citep{Liese:Vajda:2006,Esposito2020RobustGV}.
Thus, the  \textit{$f$-mutual information} $\mathbb{I}_f(X;Y)$ characterises  the independence between $X$ and $Y$. We now state a fundamental data-processing inequality for $f$-mutual information in Lemma \ref{f_Data_Processing_Inequality} below, which serves as the key theoretical basis of the proposed method.
\begin{lemma}\label{f_Data_Processing_Inequality}
	For any measurable function $g:\mathcal{X}\to\mathbb{R}^d$, we have
	\begin{equation}\label{InformationLossInequality}
		\mathbb{I}_f(X;Y)\ge \mathbb{I}_f(g(X);Y).
	\end{equation}
	If all the relevant p.d.f's and conditional p.d.f's are positive, and $f$ is strictly convex, the equality  in (\ref{InformationLossInequality})
	holds if and only if (iff) $X$ and $Y$ are conditionally independent given $g(X)$, that is,
	\begin{equation}\label{sufficiency_characterization}
		\mathbb{I}_f(X;Y)=\mathbb{I}_f(g(X);Y)\ \ \Leftrightarrow \ \  Y\perp \!\!\! \perp X~|~g(X).
	\end{equation}
\end{lemma}
The data-processing inequality in (\ref{InformationLossInequality})  can be also found in \cite{Polyanskiy:Wu:2017} and \cite{Igal2019}.  We provide a proof of Lemma \ref{f_Data_Processing_Inequality} in Section \ref{appA} in
the supplementary material.
%Appendix \ref{appA}.
In what follows, if $X$ and $Y$ are independent given a representation $R(X)$, i.e., $Y\perp \!\!\! \perp X|R(X)$, the representation $R$ is said to be \textit{sufficient}.  By Lemma \ref{f_Data_Processing_Inequality}, we can characterize a sufficient representation $V:\mathcal{X}\to\mathbb{R}^{d_{V}}$ with $d_{V}\in\mathbb{N}^+$ and $d_{V}\le d_X$ as a minimizer
\begin{eqnarray*}
	V&\in&\argmin_{v~:~\textup{dim}(v)=d_V}\{\mathbb{I}_f(X;Y)-\mathbb{I}_f(v(X);Y)\}\\
	&=&\argmin_{v~:~\textup{dim}(v)=d_V}\{-\mathbb{I}_f(v(X);Y)\},
\end{eqnarray*}
where $\textup{dim}(T)=k$ means $T:\mathcal{X}\to\mathbb{R}^k$ for a transformation $T$ of $X$.

For a sufficient representation $V$ and any injective transformation $T$, $T\circ V$  is also a sufficient representation by the basic property of conditional probability.  Based on a result from the optimal transport theory given in Lemma \ref{OT_map}, under some regularity conditions, there exists an injective map $T^*:\mathbb{R}^{d_{V}}\to[0,1]^{d_{V}}$ depending on $V$ such that
\[
T^*(V(X))\sim  \gamma_{U}\equiv \textup{Uniform}[0,1]^{d_{V}},
\]
where $\textup{Uniform}[0,1]^{d_{V}}$ denotes the uniform distribution on $[0,1]^{d_{V}}$.  Letting
%\[
$R_V := T^*\circ V:\ \ \mathcal{X}\to[0,1]^{d_{V}}, $
%\]
we have
\[
Y \perp \!\!\! \perp X~|~R_V(X),~~R_V(X)\sim \textup{Uniform}[0,1]^{d_V},
\]
which implies that $R_V(X)$ is a sufficient representation.  We refer to it as  a {\it $d_V$-dimensional uniformized sufficient representation (USR)}.

In addition,  define
\[
d_0 :=\min\{\textup{dim}(R):\ R~\text{is a USR}\}.
\]
We call $d_0$ the \textit{intrinsic dimension} and denote a $d_0$-dimensional USR by $R_0$. Namely,
\begin{equation}\label{underlying_Representation}
	Y \perp \!\!\! \perp X~|~R_0(X),~~R_0(X)\sim \textup{Uniform}[0,1]^{d_0}.	
\end{equation}
Our goal is to learn such a  $d_0$-dimensional USR $R_0$.  A sufficient and necessary condition to ensure the existence of such a $R_0$ is that there exists a sufficient representation $R$ with $\textup{dim}(R)=d_0$ having the second order moment and absolute continuous distribution. This condition is mild as the sufficient representation condition is also required for all the SDR methods.  Moment condition and absolute continuity condition are also found in the SDR literature, such as the conditions for Theorem 4 in \cite{Li:Artemiou:Li:2011} and Assumptions 1 and 2 in \cite{Shin:Wu:Zhang:Liu:2017}.  In what follows, we assume that a $d_0$-dimensional USR $R_0$ exists.

\begin{Remark}
	The reference distribution $\gamma_{U}$ is user-specified. We choose the uniform distribution for simplicity.  Other distributions can also be considered, such as the multivariate Gaussian distribution and the uniform distribution on a sphere.
\end{Remark}

\subsection{Population objective function based on mutual informaiton}
In the rest of this paper, we focus on the case when $f(t)=t\log t$, that is, $\mathbb{D}_{f}$ is the KL-divergence and $\mathbb{I}_{f}$ is the mutual information. The KL-divergence and mutual information, denoted by $\mathbb{D}_{\textup{KL}}$ and $\mathbb{I}_{\textup{KL}}$ respectively, have wide applications in statistics \citep{Nguyen2010,Berrett:Samworth:2019}.
To
%learn
estimate $R_0$ in (\ref{underlying_Representation}), by the definition of $R_0$, we note that
\[
R_0\in\argmin_{R~:~\textup{dim}(R)=d_0}\{-\mathbb{I}_{\textup{KL}}(R(X);Y)\}\text{ and }\mathbb{D}_{\textup{KL}}(P_{R_0}~||~\gamma_{U})=0,
\]
where $P_{T}$ denotes the probability measure induced by $T(X)$ for any transformation $T$ of $X$. These properties plus the nonnegativity of KL divergence motivate us to consider minimizing the
%following
population-level objective function:
\begin{align}
	\label{Lobja}
	\mathbb{L}(R;\lambda) :=-\mathbb{I}_{\textup{KL}}(Y;R(X)) + \lambda \mathbb{D}_{\textup{KL}}(P_{R}~||~\gamma_{U}),
\end{align}
where $\textup{dim}(R)=d_0$ and $\lambda$ is a nonnegative constant.  Heuristically, minimizing $-\mathbb{I}_{\textup{KL}}(Y;R(X))$ enforces the conditional independence of $X$ and $Y$ given $R(X)$, and the regularization term $\lambda \mathbb{D}_{\textup{KL}}(P_{R}~||~\gamma_{U})$ with $\lambda\textgreater0$ encourages $R(X)$ to be distributed as the reference distribution $\textup{Uniform}[0,1]^{d_0}$.  The following theorem  shows that the target $R_0$ is a minimizer of $\mathbb{L}(R;\lambda)$ over all measurable functions $R$ satisfying $\textup{dim}(R)=d_0$.
\begin{theorem}\label{Representation_Learnable_Theorem}
	The uniformized sufficient representation (USR) $R_0$  in (\ref{underlying_Representation}) satisfies that
	\[
	R_0\in\argmin_{R~:~\textup{dim}(R)=d_0}\mathbb{L}(R;\lambda)
	\]
	for any $\lambda\ge0$.
\end{theorem}
A
%simple
proof of Theorem \ref{Representation_Learnable_Theorem} is given in
Section \ref{appB} in the supplementary material.
%Appendix \ref{appB}.
 To facilitate the estimation of $R_0$ based on $\mathbb{L}(R;\lambda)$, we use a variational formulation of $f$-divergence \citep[Lemma 1 of][]{Nguyen2010}.
 % stated in the next lemma.
\begin{lemma}\label{variational_representation}
	Suppose that two distributions $P_1$ and $P_2$ are supported on the same set
	$\mathcal{Z}\subseteq \mathbb{R}^d$. For any differentiable and convex function $f:\mathbb{R}^+\to\mathbb{R}$ with $f(1)=0$,  the $f$-divergence between $P_1$ and $P_2$
	\begin{eqnarray}\label{varep}
		\mathbb{D}_f(P_1~||~P_2)=\sup_{D:\ \mathbb{R}^d\to \textup{dom}(f^*)} E_{Z\sim P_1}D(Z)-E_{Z\sim P_2}f^*(D(Z)),
	\end{eqnarray}
	where $f^*(t)=\sup_{x\in\mathbb{R}}\{tx-f(x)\}$ is the Fenchel conjugate of $f$, and the maximum in (\ref{varep}) is attained at $D(z)=f'(p_1(z)/p_2(z))$.  Here, $f'$ denotes the first derivative of $f$.
\end{lemma}
When $f(t)=t\log t$, its Fenchel conjugate $f^*(t)=e^{t-1}$.  It then follows from Lemma \ref{variational_representation} that
\begin{eqnarray}
	\mathbb{D}_{\textup{KL}}(P_1~||~P_2)&=&\sup_{D:\ \mathbb{R}^d\to \mathbb{R}} E_{Z\sim P_1}D(Z)-E_{Z\sim P_2}e^{D(Z)}+1.
	\label{variational_representation_KL}
\end{eqnarray}
The sample version of the variational expression in (\ref{variational_representation_KL}) can be easily obtained by replacing the expectations by their empirical analogues.

By (\ref{variational_representation_KL}), we can write (\ref{Lobja}) as
\begin{align}
	\label{Lobja-dual}
	\mathbb{L}(R;\lambda) =&-\sup\limits_{D:\ \mathbb{R}^{d_Y+d_0}\to \mathbb{R}}
	\left\{ E_{(X, Y) \sim P_{XY}} D(Y,R(X))
	-E_{(X, Y)\sim P_XP_Y}e^{D(Y,R(X))}\right\} \nonumber \\
	&+\lambda\sup\limits_{Q:\ \mathbb{R}^{d_0}\to \mathbb{R}}
	\left\{E_{X\sim P_X} Q(R(X))-E_{U\sim P_U} e^{Q(U)}\right\}.
\end{align}
Hence, using the expression in (\ref{Lobja-dual}), we can obtain the sample version of $\mathbb{L}(R;\lambda)$ by approximating the function classes of $R$, $D$, and $Q$, and using the empirical analogues for the expectations.  We introduce the neural network function classes used to approximate the function classes of $R$, $D$, and $Q$ in the next subsection.

\subsection{Deep neural networks}\label{DNN_description}
The class of feedforward neural networks (FNNs) $\mathcal{F}$ consists of functions $F_{\boldsymbol{\theta}}:\mathbb{R}^{d_{\textup{in}}}\to\mathbb{R}^{d_{\textup{out}}}$ with input dimensionality $\textup{dim}_{\textup{in}}(\mathcal{F})=d_{\textup{in}}$, output dimensionality $\textup{dim}_{\textup{out}}(\mathcal{F})=d_{\textup{out}}$, weight and bias parameters $\boldsymbol{\theta}$, depth $\mathcal{L}$, width $\mathcal{W}$, size $\mathcal{S}$, number of neurons $\mathcal{U}$, and $F_{\boldsymbol{\theta}}$ satisfying $\|F_{\boldsymbol{\theta}}\|_\infty\le \mathcal{B}$ for some $0\le \mathcal{B}\le\infty$, where $\|\cdot\|_{\infty}$ is the sup-norm on some specific domain.  Specifically,  the architecture of $F_{\boldsymbol{\theta}}$ can be expressed as:
\[
F_{\boldsymbol{\theta}}(x)=\mathbb{L}_{\mathcal{L}}\circ\sigma\circ\mathbb{L}_{\mathcal{L}-1}\circ\sigma\circ\cdots \circ \sigma \circ \mathbb{L}_{1}\circ\sigma\circ\mathbb{L}_{0}(x),\ \ x\in\mathbb{R}^{d_{\textup{in}}},
\]
where $\mathbb{L}_{i}(z)=\Theta_ix+b_i,z\in\mathbb{R}^{k_{i}}$ with weight matrix $\Theta_i\in \mathbb{R}^{k_{i+1}\times k_i}$ and bias vector $b_i\in \mathbb{R}^{k_{i+1}},~i=0,1,\ldots,\mathcal{L}$, and $\sigma$ is the componentwise linear rectified linear unit (ReLU) activation function.  For $i=1,2,\ldots,\mathcal{L},~k_{i}$ is the number of neurons (width) in the $i$th hidden layer; then $k_0=d_{\textup{in}}$ and $k_{\mathcal{L}+1}=d_{\textup{out}}$.  For this network, we say that it has $\mathcal{L}$ hidden layers when considering the layers containing the activation function, or $(\mathcal{L}+1)$ layers in total when considering the layers involving affine transformations. We write the parameters
of the network as $\boldsymbol{\theta}=(\Theta_0,\Theta_1,\ldots,\Theta_{\mathcal{L}},b_1,b_2,\ldots,b_{\mathcal{L}})$. The width parameter
$\mathcal{W}=\max\{k_i,i=1,\ldots,\mathcal{L}\}$ is the maximum width of hidden layers; the number of neurons $\mathcal{U}$ is defined as the number of neurons in $f_{\boldsymbol{\theta}}$, i.e., $\mathcal{U}=\sum_{i=1}^{\mathcal{L}}k_i$; the size $\mathcal{S}$ is the total number of parameters in the network, i.e., $\mathcal{S}=\sum_{i=0}^{\mathcal{L}}{k_{i+1} \times(k_i + 1)}$. The multi-layer structure implies that
\[
\mathcal{S}\le  \mathcal{W}(d_{\textup{in}}+1)+(\mathcal{W}^2+\mathcal{W})(\mathcal{L}-1)+(\mathcal{W}+1)d_{\textup{out}}.
\]
Note that the number of hidden layers $\mathcal{L}$ may depend on $n,d_{\textup{in}}$, and $\mathcal{W},\mathcal{U},\mathcal{S}$ may depend on $n,d_{\textup{in}}$ and $d_{\textup{out}}$, but we suppress the dependence for notational simplicity.  We simply refer to all
ReLU-activated FNNs using the ReLU activation function as ReLU networks.

We use three FNNs $\mathcal{R},\mathcal{D}$, and $\mathcal{Q}$ with  parameters given below to approximate the functions $R$, $D$, and $Q$ in (\ref{Lobja-dual}), respectively:
\begin{eqnarray*}
	\mathcal{R}&:&\textup{weight and bias parameter}~
	\boldsymbol{\theta},\textup{dim}_{\textup{in}}(\mathcal{R})=d_X,\textup{dim}_{\textup{out}}(\mathcal{R})=d_0,\mathcal{L}_{\mathcal{R}},\mathcal{W}_{\mathcal{R}},\mathcal{S}_{\mathcal{R}},\mathcal{B}_{\mathcal{R}}; \\\mathcal{D}&:&\textup{weight and bias parameter}~\boldsymbol{\phi},\textup{dim}_{\textup{in}}(\mathcal{D})=d_Y+d_0,\textup{dim}_{\textup{out}}(\mathcal{D})=1,\mathcal{L}_{\mathcal{D}},\mathcal{W}_{\mathcal{D}},\mathcal{S}_{\mathcal{D}},\mathcal{B}_{\mathcal{D}}; \\
	\mathcal{Q}&:&\textup{weight and bias parameter}~\boldsymbol{\psi},\textup{dim}_{\textup{in}}(\mathcal{Q})=d_0,\textup{dim}_{\textup{out}}(\mathcal{Q})=1,\mathcal{L}_{\mathcal{Q}},\mathcal{W}_{\mathcal{Q}},\mathcal{S}_{\mathcal{Q}},\mathcal{B}_{\mathcal{Q}}.
\end{eqnarray*}
In what follows, $\mathcal{R}$ is called the representer network, $\mathcal{D}$ the MI discriminator network, and $\mathcal{Q}$ the push-forward network.

\subsection{Empirical objective function}
In practice, the distribution of $(X,Y)$ is usually unknown and only a random sample $\{(X_i,Y_i),i=1,\ldots, n\}$ that are i.i.d. copies of $(X,Y)$ is available. Let $U_1,\ldots,U_n$ be i.i.d. observations sampled from $\textup{Uniform}[0,1]^{d_0}$.  The empirical objective function for the proposed MSRL estimator is an empirical version of $\mathbb{L}(R;\lambda)$, defined as
\begin{eqnarray*}
	\mathbb{L}^{\textup{net}}_n(R;\lambda):= -\sup\limits_{D\in\mathcal{D}}\mathbb{L}_n^{\textup{MI}}(D,R)
	+\lambda\sup\limits_{Q\in\mathcal{Q}}\mathbb{L}_n^{\textup{Push}}(Q,R),
\end{eqnarray*}
where
\begin{align}
	\label{MIdef}
	\mathbb{L}_n^{\textup{MI}}(D,R):=\frac{1}{n}\sum_{i=1}^{n}D(Y_i,R(X_i))-\frac{1}{n(n-1)}\sum_{i\neq j}e^{D(Y_i,R(X_j))},
\end{align}
and
\begin{align}\label{Pushdef}
	\mathbb{L}_n^{\textup{Push}}(Q,R)
	:=\frac{1}{n}\sum_{i=1}^{n}Q(R(X_i))-\frac{1}{n}\sum_{i=1}^ne^{Q(U_i)}.
\end{align}
Then the MSRL representation is
\[
\hat{R}_n^{\lambda}\in\argmin_{R\in \mathcal{R}}\mathbb{L}^{\textup{net}}_n(R;\lambda).
\]
Below, we use an example to illustrate how MSRL differs from two existing nonlinear SDR methods: generalized sliced inverse regression (GSIR) and generalized sliced average variance estimation (GSAVE).  We consider the model
\begin{equation}\label{toy_model}
	Y=\textup{sign}[2\sin(\beta_1^\top X)+\epsilon_1]\cdot\log [|\sin(\beta_2^\top X)+c+\epsilon_2|],
\end{equation}
where $\beta_1,\beta_2$ are two parameter vectors generated from the uniform distribution on the ten-dimensional sphere, $X\sim N(0,\boldsymbol{\textup{I}}_{10})$, $\epsilon_1,\epsilon_2\sim N(0,0.25)$ independently, and $c=5$.  Intuitively, a better representation should be able to better predict the sign of the response.
We learn the two-dimensional representations, apply them to 1000 testing data points, and then plot these transformed data along with their response sign.  We use two reference distributions for this experiment, $\gamma_{U}\sim\textup{Uniform[0,1]}$ or $\gamma_{S}\sim\sin(Z)$ with $Z\sim N(0,1)$.  The training details are given in
Section \ref{Network_Struc} in the supplementary material.
%Appendix \ref{Network_Struc}
and the results are displayed in Figures \ref{Feature_Plotting} and \ref{Density_Drawing}.

From Figure \ref{Feature_Plotting}, compared to GSIR and GSAVE, the learned MSRL features are more predictive for the response sign as the transformed data based on the learned MSRL features behave more similarly to the linearly separable data.  This partially implies the sufficiency of the learned representation through MSRL.  The results in Figure \ref{Density_Drawing} show that the learned representation through MSRL has the desired reference distribution.
\begin{figure}[htbp]
	\centering
	\subfigure[MSRL with $\gamma_U$]{
		\begin{minipage}[t]{0.4\linewidth}
			\centering
			\includegraphics[width=2.5in]{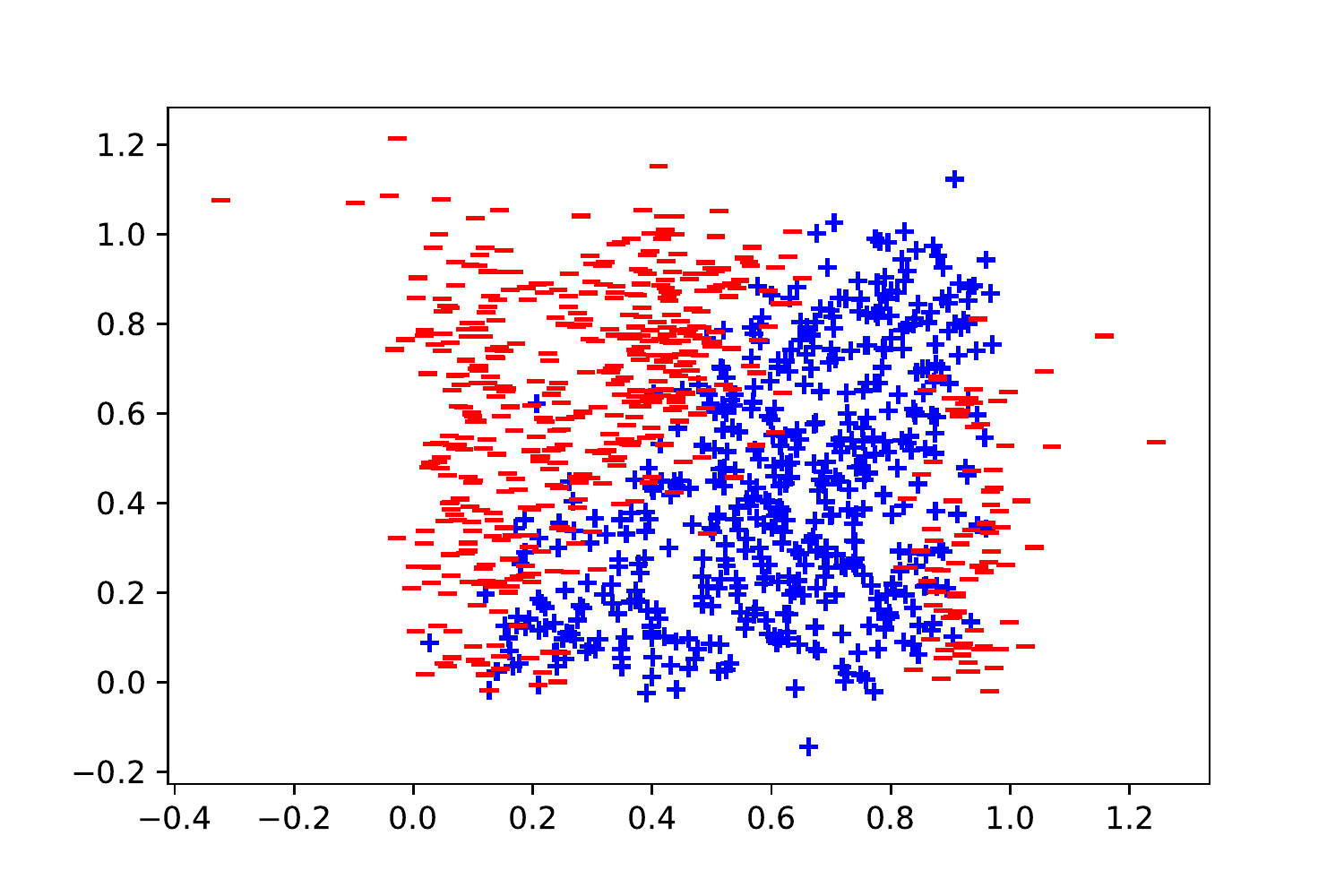}
		\end{minipage}
	}
	\quad
	\subfigure[MSRL with $\gamma_S$]{
		\begin{minipage}[t]{0.4\linewidth}
			\centering
			\includegraphics[width=2.5in]{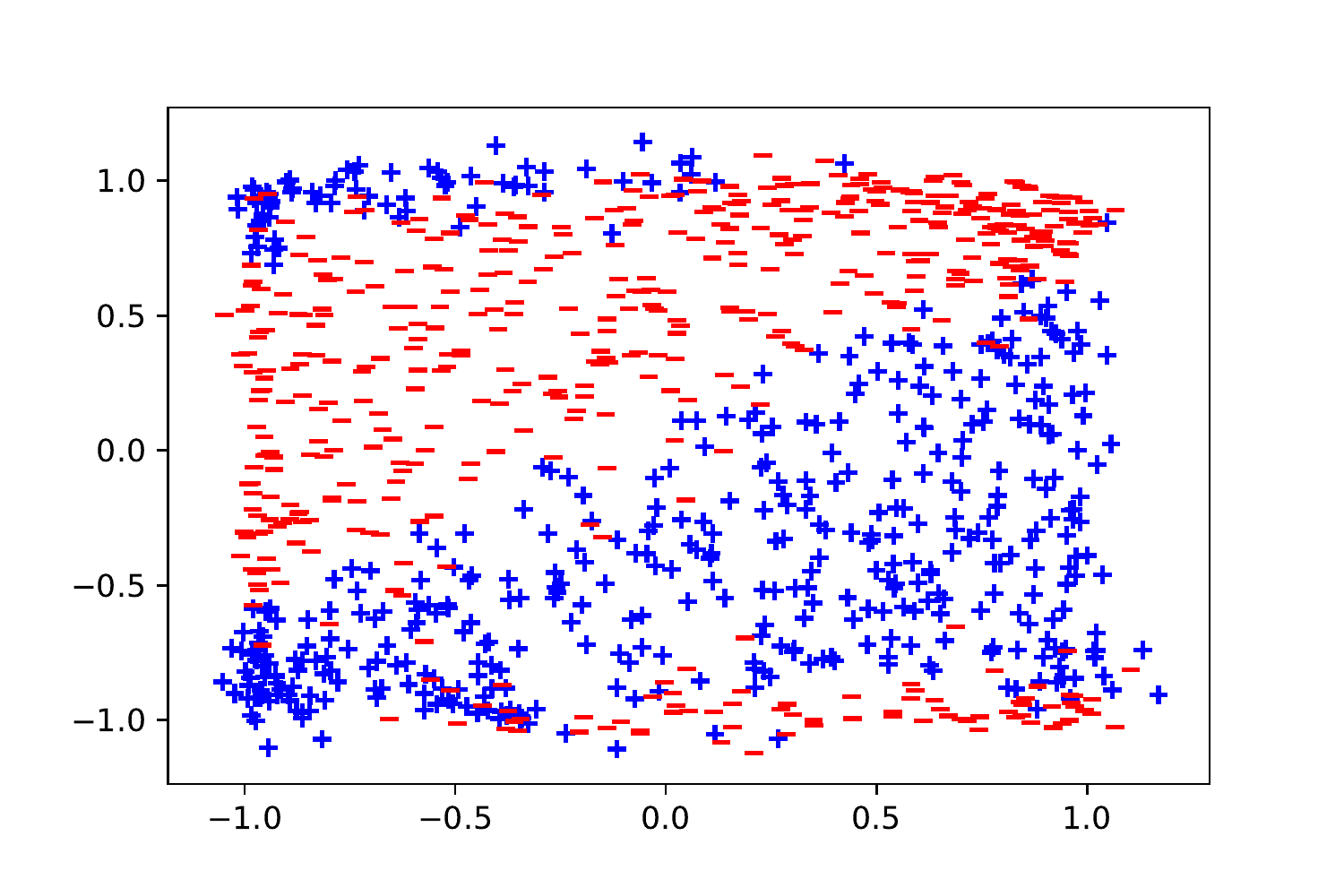}
		\end{minipage}
	}
	\quad
	\subfigure[GSIR]{
		\begin{minipage}[t]{0.4\linewidth}
			\centering
			\includegraphics[width=2.5in]{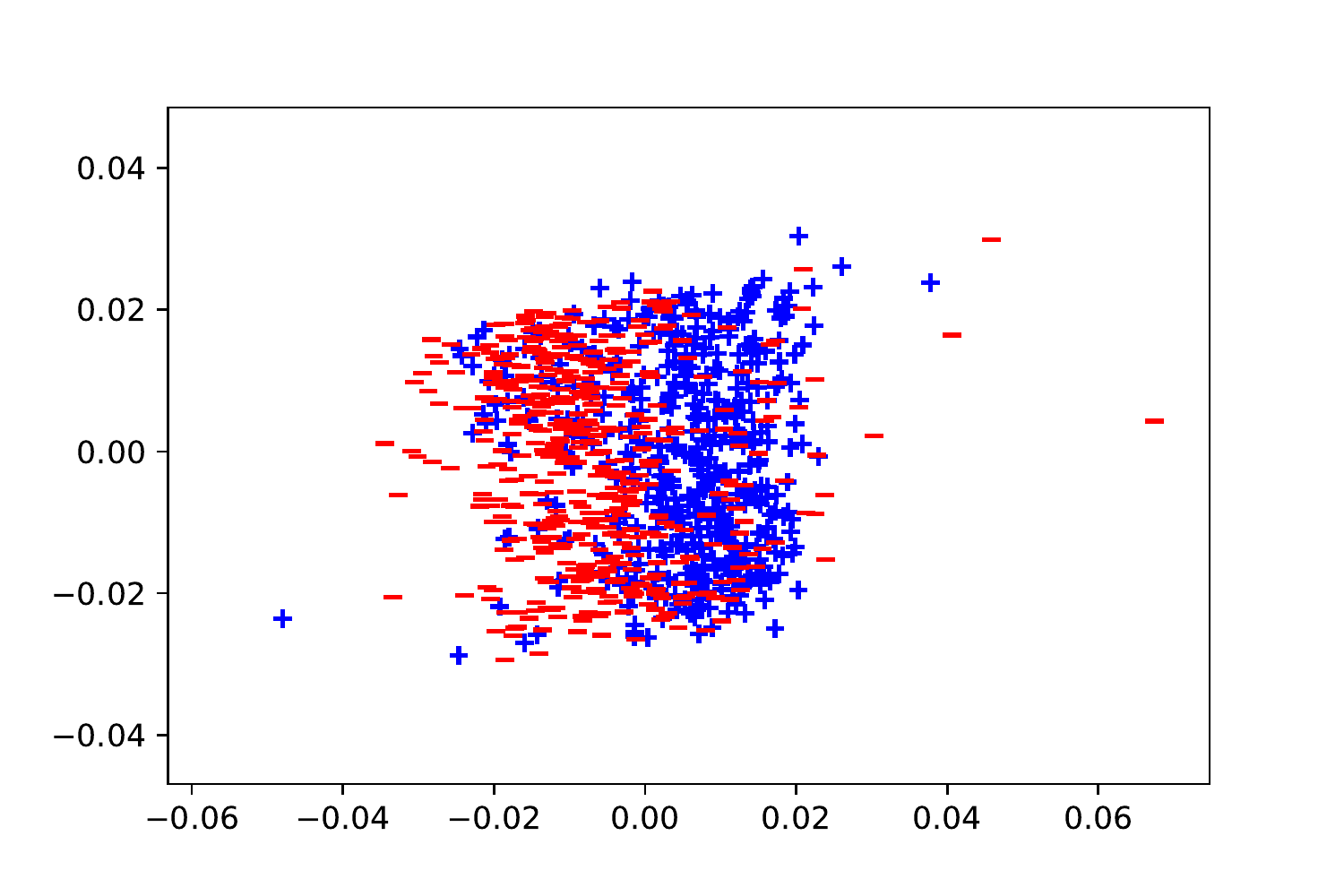}
		\end{minipage}
	}
	\quad
	\subfigure[GSAVE]{
		\begin{minipage}[t]{0.4\linewidth}
			\centering
			\includegraphics[width=2.5in]{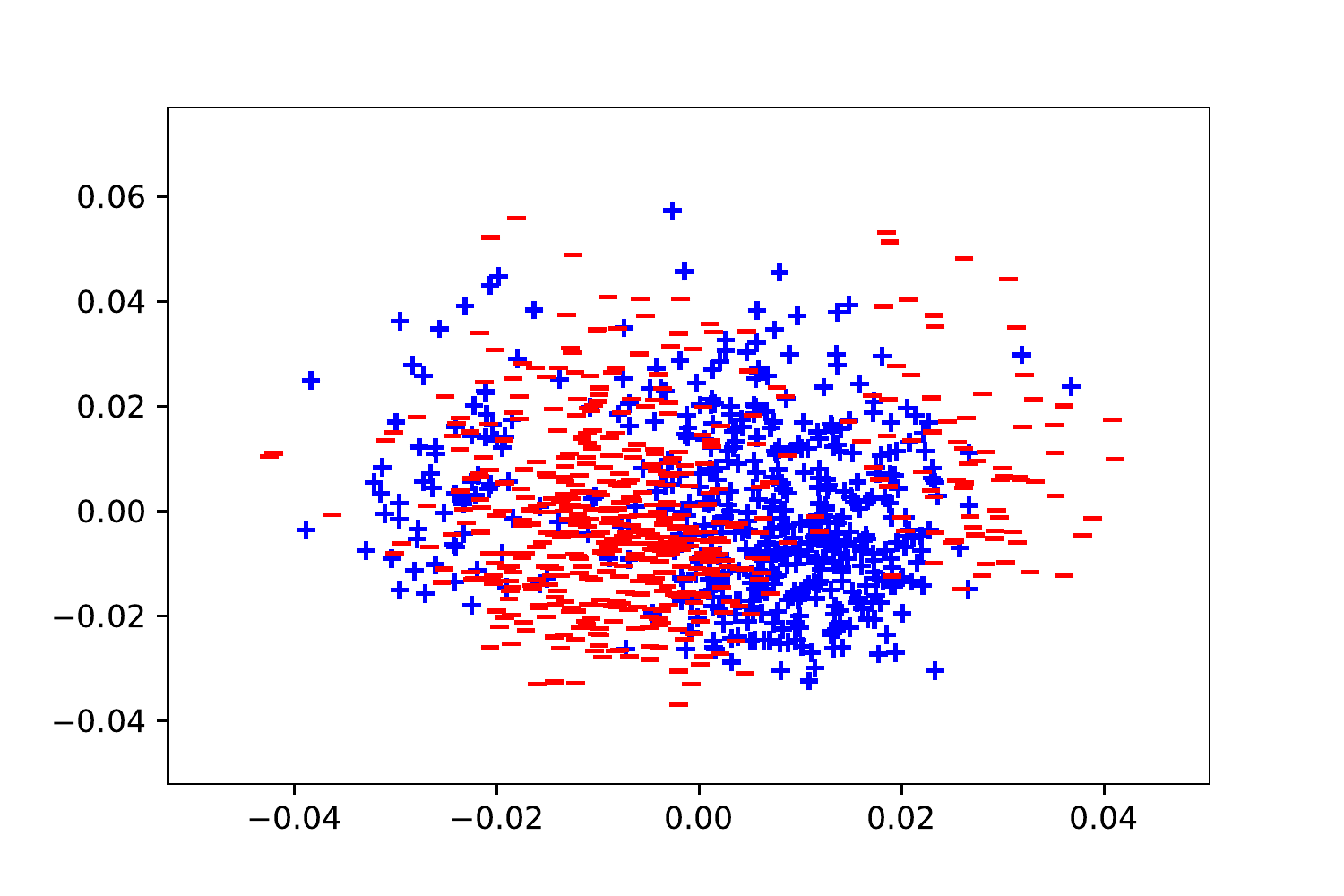}
		\end{minipage}
	}
	\centering
	\caption{The plots of a synthetic dataset from (\ref{toy_model}) based on the learned two-dimensional representations. Red ``--" indicates that the response of the data point is negative and blue ``+" indicates that the response of the data point is positive.
	}\label{Feature_Plotting}
\end{figure}
\begin{figure}[htbp]
	\centering
	\subfigure[MSRL with $\gamma_U$]{
		\begin{minipage}[t]{0.4\linewidth}
			\centering
			\includegraphics[width=2.5in]{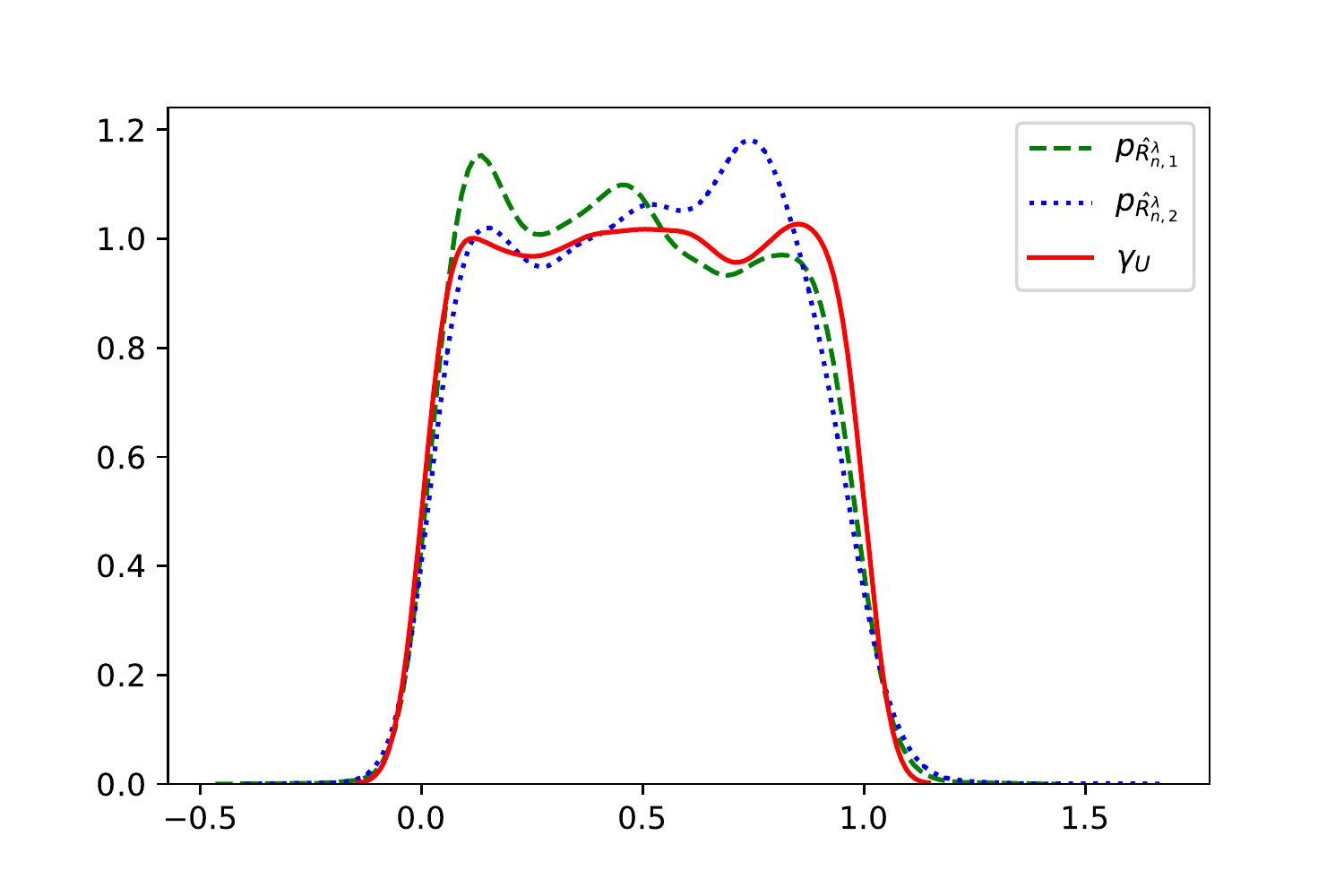}
		\end{minipage}
	}
	\quad
	\subfigure[MSRL with $\gamma_S$]{
		\begin{minipage}[t]{0.4\linewidth}
			\centering
			\includegraphics[width=2.5in]{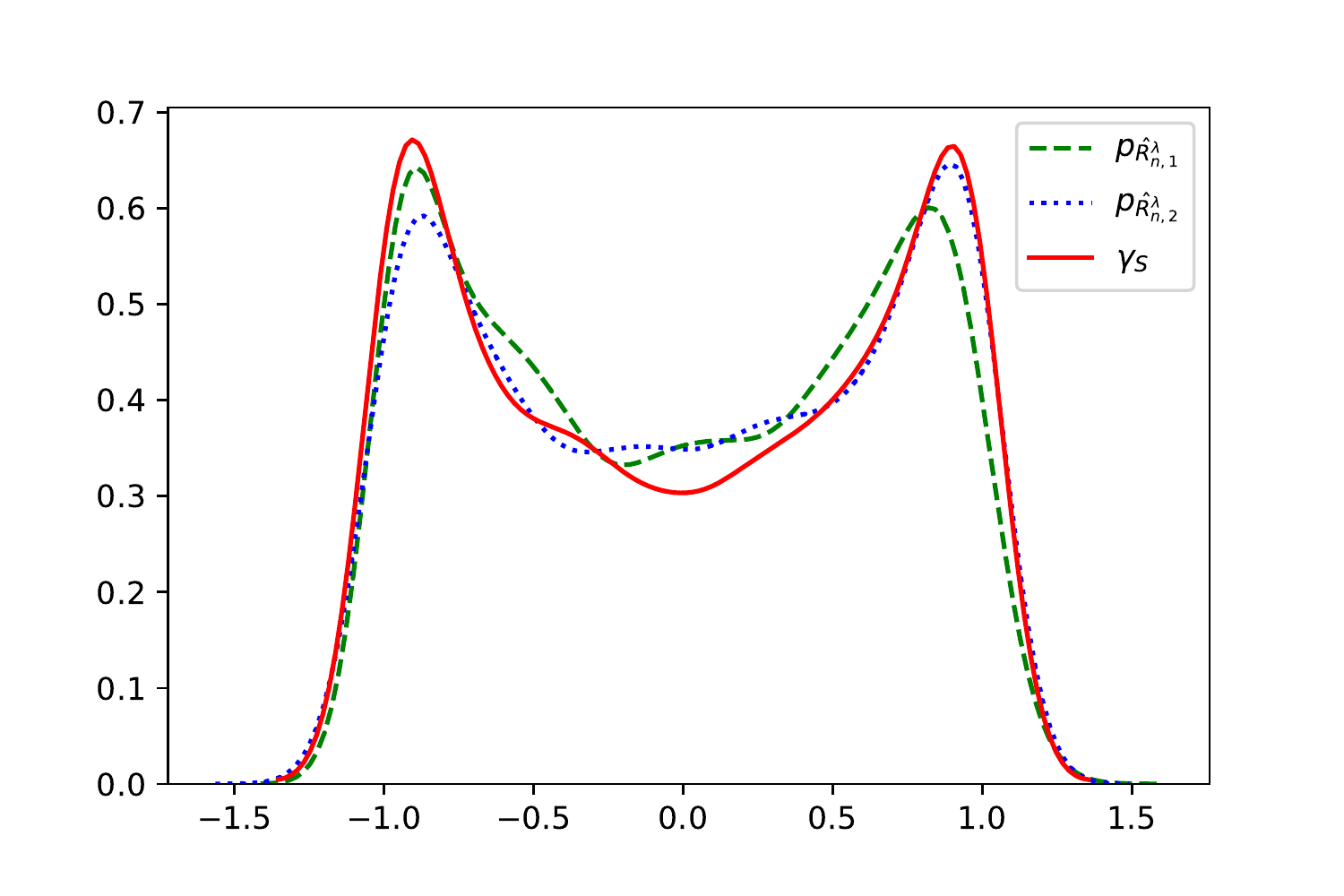}
		\end{minipage}
	}
	\centering
	\caption{The kernel density estimates for the components of the MSRL estimator and the corresponding reference distribution.}\label{Density_Drawing}
\end{figure}

\section{Theoretical results}\label{Theoretical_Results}
In this section, we establish consistency results and non-asymptotic error bounds for the proposed MSRL method.  The intrinsic dimension $d_0$ is  assumed known for simplicity.  Our theoretical analysis focus on continuous type random variables $X$ and $Y$.  Extension to the categorical case is discussed in Section \ref{Extensions}.

\subsection{Consistency}
For any $d\in\mathbb{N}^+$, we denote the class of continuous functions and the class of Lipschitz functions by
\begin{eqnarray*}
	\mathcal{F}_{\textup{C}}([0,1]^d)&=&\{f:[0,1]^d\to\mathbb{R},~ f \text{ continuous}\}
\end{eqnarray*}
and
\begin{eqnarray*}
	\mathcal{F}_{\textup{Lip}}([0,1]^d,B)&=&\{f:[0,1]^d\to\mathbb{R},~|f(x)-f(y)|\le B\|x-y\|_2\}
\end{eqnarray*}
for a constant $B\textgreater0$.

For any transformation $T$ of $X$, let $p_T$ and $p_{Y|T}$ denote the p.d.f of $T(X)$ and the conditional p.d.f of $Y$ given $T(X)$.  We need the following regularity conditions in establishing a consistency result.

\begin{assumption}\label{Weak_Assumption_on_TrueR}
	The target USR $R_0=(R_{0,1}, \ldots,R_{0,d_0})$, where $R_{0,j}\in\mathcal{F}_{\textup{C}}([0,1]^{d_X})$ and $0\le R_{0,j}(x) \le1$ for any $x\in[0,1]^{d_X}$ and $j=1,\ldots,d_0$.  The conditional p.d.f $p_{Y|R_0}(y|r)\in\mathcal{F}_{\textup{C}}([0,1]^{d_Y+d_0})$ and $1/B_{1}\le \inf_{y,r}p_{Y|R_0}(y|r)\le\sup_{y,r}p_{Y|R_0}(y|r)\le B_{1}$, where $B_{1}$ is a constant satisfying $1\le B_{1}\le\infty$.
\end{assumption}
\begin{assumption}\label{Weak_Assumption_on_NetworkR}
	For every $R\in\mathcal{R}$, $p_{Y|R}(y|r)\in\mathcal{F}_{\textup{Lip}}([0,1]^{d_Y+d_0},B_{1})$ and $1/B_{1}\le \inf_{y,r}p_{Y|R}(y|r)\le\sup_{y,r}p_{Y|R}(y|r)\le B_{1}$. Analogously, $p_{R}(r)\in\mathcal{F}_{\textup{Lip}}([0,1]^{d_0},B_{1})$ and $1/B_{1}\le \inf_{r}p_{R}(r)\le\sup_{r}p_{R}(r)\le B_{1}$.
\end{assumption}

For ease of notation, we define three functions $W(\cdot)$, $L(\cdot)$, and $S(\cdot)$ on $\mathbb{N}^+$ that are needed to specify the required neural network structures: for any $d\in \mathbb{N}^+$,
\begin{align*}
	W(d)&:=24d\left\lceil n^{\frac{d}{2(2+d)}} /\log n\right\rceil, \nonumber \\
	L(d)&:=20d\lceil\log n\rceil,\nonumber\\
	S(d)&:=C_1d^{2}\left\lceil n^{\frac{d}{(d+2\beta)}}/\log n\right\rceil,
\end{align*}
for a positive constant $C_1$, where $\lceil a\rceil$ for any $a\in\mathbb{R}$ denotes the smallest integer no less than $a$. We specify the network structures as follows.

\begin{NS}\label{Structure_on_NetworkR}
	The representer network $\mathcal{R}$ has parameters: depth $\mathcal{L}_{\mathcal{R}}=L(d_X)+3$, width $\mathcal{W}_{\mathcal{R}}=d_0W(d_X)$, $\mathcal{S}_{\mathcal{R}}\le4d_0S(d_X)$. And for every $R\in\mathcal{R}$, $0\le R(x)\le1$ for any $x\in[0,1]^{d_X}$, where the inequality holds in a componentwise sense.
\end{NS}

\begin{NS}\label{Structure_on_NetworkD}
	The discriminator network $\mathcal{D}$ has parameters: depth $\mathcal{L}_{\mathcal{D}}=L(d_Y+d_0)$, width $\mathcal{W}_{\mathcal{D}}=W(d_Y+d_0)$, $\mathcal{S}_{\mathcal{D}}\le S(d_Y+d_0)$, and $\mathcal{B}_\mathcal{D}\ge\max\{1,2\log B_{1}\}$, where $B_{1}$ is defined in Assumptions \ref{Weak_Assumption_on_TrueR} \& \ref{Weak_Assumption_on_NetworkR}.
\end{NS}
\begin{NS}\label{Structure_on_NetworkQ}
	%	Pushing-forward
	The discriminator network $\mathcal{Q}$ has parameters: depth $\mathcal{L}_{\mathcal{Q}}=L(d_0)$, width $\mathcal{W}_{\mathcal{Q}}=W(d_0)$, $\mathcal{S}_{\mathcal{Q}}\le S(d_0)$, and $\mathcal{B}_\mathcal{Q}\ge\max\{1,\log B_{1}\}$, where $B_{1}$ is given in Assumptions \ref{Weak_Assumption_on_TrueR} \& \ref{Weak_Assumption_on_NetworkR}.
\end{NS}
\begin{theorem}%[Consistency] %Weak Convergence]
	\label{Weak_Consistency_Results0}
	Suppose Assumptions \ref{Weak_Assumption_on_TrueR} \& \ref{Weak_Assumption_on_NetworkR} hold and the network parameters of $\mathcal{R},\mathcal{D}$, and $\mathcal{Q}$ satisfy (NS\ref{Structure_on_NetworkR}), (NS\ref{Structure_on_NetworkD}) and (NS\ref{Structure_on_NetworkQ}) respectively, then
	\[
	E\{\mathbb{L}(\hat{R}_n^{\lambda};\lambda)-\mathbb{L}(R_0;\lambda)\}\to 0,\text{ as } n\to\infty.
	\]
\end{theorem}

The proof of Theorem \ref{Weak_Consistency_Results0} is deferred to
Section \ref{appB} in the supplementary material.
%Appendix \ref{appB}.

\begin{corollary}[Consistency] %Weak Convergence]
	\label{Weak_Consistency_Results}
	Suppose Assumptions \ref{Weak_Assumption_on_TrueR} \& \ref{Weak_Assumption_on_NetworkR} hold and the network parameters of $\mathcal{R},\mathcal{D}$, and $\mathcal{Q}$ satisfy (NS\ref{Structure_on_NetworkR}), (NS\ref{Structure_on_NetworkD}) and (NS\ref{Structure_on_NetworkQ}) respectively, then
	\begin{equation}\label{weak_Results}
		E_{X}\|p_{Y|X}-p_{Y|\hat{R}_n^{\lambda}}\|_{L_{1}}\to 0,~~E\|p_{\hat{R}_n^{\lambda}}-1\|_{L_1}\to0,\ \ \text{ as } n\to\infty,
	\end{equation}
	where $\|p_{Y|X}-p_{Y|R}\|_{L_{1}}=\int_{\mathcal{Y}}|p_{Y|X}(y|X=x)-p_{Y|R}(y|R(X)=R(x))|dy$ is a function of $x$ and
	$\|p_{R}-1\|_{L_1}=\int_{[0,1]^{d_0}}|p_{R}(r)-1|dr$.
\end{corollary}

Corollary \ref{Weak_Consistency_Results}
%indicates
shows that the conditional p.d.f
%$p_{Y|\hat{R}_n^{\lambda}}$
of $Y$ given the the learned representation $\hat{R}_n^{\lambda}$  converges to the conditional p.d.f %$p_{Y|X}$
of $Y$ given $X$, which implies  the asymptotic sufficiency of $\hat{R}_n^{\lambda}$.  The slow convergence or consistency  in Theorem \ref{Weak_Consistency_Results} is mainly due to the fact that the class of continuous function  is much larger than a bounded and smooth function class. The difficulties in the approximation of  a very general continuous function class are responsible for  the arbitrarily slow convergence rate.

\subsection{Error bound} %Fast Convergence}
%The consistency established in  Theorem \ref{Weak_Consistency_Results} could be of very %slow convergence rate.
To derive error bounds with faster convergence rate,  we impose some slightly stronger smoothness assumptions on the relevant function classes.  In this subsection, we establish non-asymptotic error bound  and faster convergence rates under the H\"older continuity assumptions.   We  first give the definition of  a  H\"older  class below.
\begin{definition}[H\"older  class]\label{HolderClass}
	A H\"older class $\mathcal{H}^\beta([0,1]^d,B)$ with  $\beta=k+a$,  $k\in \mathbb{N}^+$,  $a\in(0,1]$ and a finite constant $B>0$,
	is a function class consisting of function $f:[0,1]^d\to\mathbb{R}$ satisfying
	\[
	\max\limits_{\|\boldsymbol{\alpha}\|_1\le k}\|\partial^{\boldsymbol{\alpha}}f\|_{\infty}\le B,\ \ \max\limits_{\|\boldsymbol{\alpha}\|_1= k}\max\limits_{x\neq y}\frac{|\partial^{\boldsymbol{\alpha}}f(x)-\partial^{\boldsymbol{\alpha}}f(y)|}{\|x-y\|_2^a}\le B,
	\]
	where $\|\boldsymbol{\alpha}\|_1=\sum_{i=1}^{d}\alpha_i$ and  $\partial^{\boldsymbol{\alpha}}=\partial^{\alpha_1}\partial^{\alpha_2}\cdots\partial^{\alpha_d}$ for $\boldsymbol{\alpha}=(\alpha_1,\alpha_2,\ldots,\alpha_d)\in \mathbb{N}^{+d}$.
\end{definition}
In this subsection, other than Assumptions \ref{Weak_Assumption_on_TrueR} \& \ref{Weak_Assumption_on_NetworkR}, we make
the following H\"older continuity assumptions for the target and related density functions.  We use the same smoothness parameter $\beta$ for all function classes for simplicity and when different function classes have various smoothness parameters, we can use the smallest smoothness parameter in the analysis.  In the following Assumptions \ref{Strong_Assumption_on_TrueR} \& \ref{Strong_Assumption_on_NetworkR}, the smoothness parameter $\beta\textgreater0$ and $B_{2}$ is a constant satisfying $1\le B_{2}\le\infty$.

\begin{assumption}\label{Strong_Assumption_on_TrueR}
	Each component of  the target USR $R_0$  belongs to the H\"older  class $ \mathcal{H}^\beta([0,1]^{d_X},B_{2})$
	i.e.,
	$R_{0,i}(x) \in\mathcal{H}^\beta([0,1]^{d_X},B_{2})$. And  the conditional p.d.f $p_{Y|R_0}(y|r)\in\mathcal{H}^\beta([0,1]^{d_Y+d_0},B_{2})$ and $1/B_{2}\le \inf_{y,r}p_{Y|R_0}(y|r)$. \end{assumption}

\begin{assumption}\label{Strong_Assumption_on_NetworkR}
	For every $R\in\mathcal{R}$, $p_{Y|R}(y|r)\in\mathcal{H}^\beta([0,1]^{d_Y+d_0},B_{2})$, and $1/B_{2}\le \inf_{y,r}p_{Y|R}(y|r)$. In addition, $p_{R}(r)\in\mathcal{H}^\beta([0,1]^{d_0},B_{2})$ and $1/B_{2}\le \inf_{r}p_{R}(r)$.		
\end{assumption}
A straightforward but tedious calculation yields that
\begin{equation}\label{log_holder_class}
	\log p_{Y|R_0}(y|r)\in \mathcal{H}^\beta([0,1]^{d_Y+d_0},2dc_2(\lfloor\beta \rfloor)B^{c_1(\lfloor\beta \rfloor)+c_2(\lfloor\beta \rfloor)}_{2}),
\end{equation}
where $c_1(s)=2^{s-1},c_2(s)=c_1(s)[c_2(s-1)+1]$ for any $s\in\mathbb{N}^+$ and $c_2(0)=1$, and $\lfloor a\rfloor$ for any $a\in\mathbb{R}$ is the greatest integer smaller than $a$. Similar results also hold for $\log p_{Y|R}(y|r),\log p_{Y}(y)$ and $\log p_{R}(r)$.

To describe the network structures needed to achieve a  faster convergence rate, we define the following three functions $\tilde{W}(\cdot)$, $\tilde{L}(\cdot)$, and $\tilde{S}(\cdot)$ on $\mathbb{N}^+$: for any $d\in \mathbb{N}^+$,
\begin{align*}
	\tilde{W}(d)&:=114(\lfloor\beta \rfloor+1)^2d^{\lfloor\beta \rfloor+1}, \nonumber \\
	%\]
	%\[
	\tilde{L}(d)&:=21(\lfloor\beta \rfloor+1)^2\left\lceil n^{\frac{d}{2(d+2\beta)}}\log_2\left(8n^{\frac{d}{2(d+2\beta)}}\right)\right\rceil,\nonumber\\
	\tilde{S}(d)&:=C_2(\lfloor\beta \rfloor+1)^6d^{2\lfloor\beta \rfloor+2}\left\lceil n^{\frac{d}{2(d+2\beta)}}\log_2^{-3}n\right\rceil,
\end{align*}
for a positive constant $C_2$.

\begin{NS}\label{Strong_Structure_on_NetworkR}
	Representer network $\mathcal{R}$ has parameters: depth $\mathcal{L}_{\mathcal{R}}=\tilde{L}(d_X)+3$, width $\mathcal{W}_{\mathcal{R}}=d_0\tilde{W}(d_X)$, $\mathcal{S}_{\mathcal{R}}\le4d_0\tilde{S}(d_X)$. And for every $R\in\mathcal{R}$, $0\le R(x)\le1$ for any $x\in[0,1]^{d_X}$, where the inequality holds in a componentwise sense.
\end{NS}

\begin{NS}\label{Strong_Structure_on_NetworkD}
	The discriminator network $\mathcal{D}$ has parameters: depth $\mathcal{L}_{\mathcal{D}}=\tilde{L}(d_Y+d_0)$, width $\mathcal{W}_{\mathcal{D}}=\tilde{W}(d_Y+d_0)$, $\mathcal{S}_{\mathcal{D}}\le \tilde{S}(d_Y+d_0)$, and $\mathcal{B}_\mathcal{D}\ge\max\{1,2\log B_{2}\}$, where $B_{2}$ is defined in Assumptions \ref{Strong_Assumption_on_TrueR} \& \ref{Strong_Assumption_on_NetworkR}.
\end{NS}

\begin{NS}\label{Strong_Structure_on_NetworkQ}
	The discriminator network $\mathcal{Q}$ has parameters: depth $\mathcal{L}_{\mathcal{Q}}=\tilde{L}(d_0)$, width $\mathcal{W}_{\mathcal{Q}}=\tilde{W}(d_0)$, $\mathcal{S}_{\mathcal{Q}}\le \tilde{S}(d_0)$, and $\mathcal{B}_\mathcal{Q}\ge\max\{1,\log B_{2}\}$, where $B_{2}$ is defined in Assumptions \ref{Strong_Assumption_on_TrueR} \& \ref{Strong_Assumption_on_NetworkR}.
\end{NS}

Let
\begin{align}
	\gamma (d,\beta) & = \{(d_Y+d_0)\vee d_X\}^{\lfloor\beta \rfloor+1}, \label{gamma_def}\\
	\xi_n(d,\beta)& = n^{-\frac{\beta(\beta\wedge1)}{2(2\beta+d_X)}\wedge\frac{\beta}{2\{2\beta+(d_Y+d_0)\vee d_X\}}} \label{xi_def},
\end{align}
where $a\vee b=\max(a,b)$, $a\wedge b=\min(a,b)$ for any $a,b\in\mathbb{R}$.

\begin{theorem}[Non-asymptotic risk bounds]
	\label{Non_Asymp_Results0}
	Suppose Assumptions \ref{Strong_Assumption_on_TrueR} \& \ref{Strong_Assumption_on_NetworkR} hold and the network parameters of $\mathcal{R},\mathcal{D}$, and $\mathcal{Q}$ satisfy (NS\ref{Strong_Structure_on_NetworkR}), (NS\ref{Strong_Structure_on_NetworkD}) and (NS\ref{Strong_Structure_on_NetworkQ}) respectively. Then,
	\begin{equation}\label{excess_risk_bound}
		E\{\mathbb{L}(\hat{R}_n^{\lambda};\lambda)-\mathbb{L}(R_0;\lambda)\}
		\le C(\lambda \vee2)\gamma^2(d,\beta) \xi_n^2(d,\beta),
	\end{equation}
	where $C$ is a constant only depending on the constants $B_2$ and $\beta$ in Assumptions \ref{Strong_Assumption_on_TrueR}
	\& \ref{Strong_Assumption_on_NetworkR}, and $\gamma (d,\beta),\xi_n(d,\beta)$ are defined in (\ref{gamma_def}) and (\ref{xi_def}) respectively.
\end{theorem}

\begin{corollary}[Non-asymptotic error bounds]
	\label{Non_Asymp_Results}
	Under the conditions of Theorem \ref{Non_Asymp_Results0}, we have
	\begin{equation}\label{conpdf_R_hat_conRate}
		E_{X}\|p_{Y|X}-p_{Y|\hat{R}_n^{\lambda}}\|_{L_{1}}\le C_{1\lambda}
		\gamma(d,\beta) \xi_n(d, \beta)
	\end{equation}
	and
	\begin{equation}\label{pdf_R_hat_conRate}
		E\|p_{\hat{R}_n^{\lambda}}-1\|_{L_1}\
		\le
		C_{2\lambda} \gamma(d,\beta) \xi_n(d,\beta),
	\end{equation}
	where $C_{1\lambda}=C(\lambda \vee2)^{1/2}$, $C_{2\lambda}=C\{(\lambda \vee2)/\lambda\}^{1/2}$. Here $C$ is a constant only depending on the constants $B_{2}$ and $\beta$ in Assumptions \ref{Strong_Assumption_on_TrueR}
	and  \ref{Strong_Assumption_on_NetworkR}.
\end{corollary}

The excess error bound (\ref{excess_risk_bound}) follows from the error decomposition inequality in Lemma \ref{Error Decomposition} and
\[
\Delta_j\lesssim n^{-\frac{\beta}{2\beta+(d_Y+d_0)\vee d_X}} \text{ for }j=1,2,3,4,\text{ and }\Delta_5\lesssim n^{-\frac{\beta(\beta\wedge1)}{2\beta+d_X}},
\]
where $a\lesssim b$ means $a$ is smaller than $b$ up to a constant factor independent of $n$.
The error bounds (\ref{conpdf_R_hat_conRate}) and (\ref{pdf_R_hat_conRate}) follow from the excess error bound and Pinsker's inequality.  We defer the
%detailed
proofs of Theorem \ref{Non_Asymp_Results} and Corollary \ref{Non_Asymp_Results} to
Section \ref{appB} in the supplementary material.
%Appendix \ref{appB}.
An appealing feature in Theorem \ref{Non_Asymp_Results} and Corollary \ref{Non_Asymp_Results} is that the prefactors in the error bounds depend on the relevant dimensions polynomially.

\begin{Remark}
	The cubic support assumption $\mathcal{X}\times\mathcal{Y}=[0,1]^{d_{X}}\times[0,1]^{d_{Y}}$ is made for technical convenience.  It can be relaxed to the unbounded cases.  For the unbounded case,  under the assumption that $P(\|Z\|_\infty\ge c_1\log n)\le c_2/n$ for some positive constants $c_1,c_2$ and $Z=(Y,X)$, the consistency result  in Theorem \ref{Weak_Consistency_Results} and error bounds in Theorem \ref{Non_Asymp_Results} can be derived, at a small price of an additional polynomial $\log n$ factor in the convergence rate.
\end{Remark}

\begin{Remark}
	By Corollary \ref{Non_Asymp_Results}, we see that if $\lambda\textgreater2$, $C_{1\lambda}=C\lambda^{1/2}$ and $C_{2\lambda}=C$, where $C$ is defined in Corollary \ref{Non_Asymp_Results}.  This implies that $\lambda$ with a large value would slow the convergence rate for the asymptotic sufficiency of $\hat{R}_n^{\lambda}$ in (\ref{conpdf_R_hat_conRate}) and does not impact the pushing-forward convergence rate in (\ref{pdf_R_hat_conRate}).  In contrast, if $\lambda\le2$, $C_{1\lambda}=\sqrt{2}C$ and $C_{2\lambda}=C(2/\lambda)^{1/2}$.  Thus $\lambda$ with a small value would not impact the convergence rate for the asymptotic sufficiency but deteriorate the pushing-forward convergence rate.  Based on these observations, we can use a constant value for $\lambda$.  In particular, we use $\lambda=2$ in our numerical experiments below, which leads to satisfactory performance of MSRL.
\end{Remark}

\section{A high-level description of the proofs}\label{high_level_description}
In this section, we provide a description of the main ideas and technical results needed for
proving the error analysis results stated in Section \ref{Theoretical_Results}.
% Theorems \ref{Weak_Consistency_Results}.

\subsection{Error decomposition}
We first give two basic lemmas that provide a basis for the error analysis.
The first lemma gives an error decomposition for bounding the excess risk and the second lemma
bounds the $L_1$ distance between two density functions in terms of the excess risk.

%We now present an error decomposition for the excess risk in the next lemma, which  plays an %important role in our theoretical analysis.

For any measurable functions $R,D$ and $Q$, let
\[
\mathbb{L}^{\textup{MI}}(D,R)
=\Erm\mathbb{L}_n^{\textup{MI}}(D,R),~~\mathbb{L}^{\textup{Push}}(Q,R)=
\Erm \mathbb{L}_n^{\textup{Push}}(Q,R),
\]
where $\mathbb{L}_n^{\textup{MI}}$ and $\mathbb{L}_n^{\textup{Push}}$ are defined in (\ref{MIdef}) and (\ref{Pushdef}) respectively.

\begin{lemma}[Error Decomposition]\label{Error Decomposition}
	Let $\hat{R}_n^{\lambda}\in\argmin_{R\in \mathcal{R}}\mathbb{L}^{\textup{net}}_n(R;\lambda)$,
	and let $R^*\in\argmin_{R\in \mathcal{R}}\mathbb{L}(R;\lambda)$ be the best representation in the neural network class $\mathcal{R}$.  Then,
	for any fixed $\lambda\ge0$,
	\[
	\Erm\{\mathbb{L}(\hat{R}_n^{\lambda};\lambda)-\mathbb{L}(R_0;\lambda)\}\le \lambda \Delta_1+\Delta_2+2\lambda \Delta_3+2\Delta_4+\Delta_5,
	\]
	where
	\begin{align*}
		\Delta_1&:=E\left\{\sup\limits_{Q}\mathbb{L}^{\textup{Push}}(Q,\hat{R}_n^{\lambda})-\sup\limits_{Q\in\mathcal{Q}}\mathbb{L}^{\textup{Push}}(Q,\hat{R}_n^{\lambda})\right\},\\
		\Delta_2&:=\sup\limits_{D}\mathbb{L}^{\textup{MI}}(D,R^*)-\sup\limits_{D\in\mathcal{D}}\mathbb{L}^{\textup{MI}}(D,R^*),\\
		\Delta_3&:=E\left\{\sup\limits_{Q\in\mathcal{Q},R\in\mathcal{R}}|\mathbb{L}_n^{\textup{Push}}(Q,R)-\mathbb{L}^{\textup{Push}}(Q,R)|\right\},  \\
		\Delta_4&:=E\left\{\sup\limits_{D\in\mathcal{D},R\in\mathcal{R}}|\mathbb{L}_n^{\textup{MI}}(D,R)
		-\mathbb{L}^{\textup{MI}}(D,R)|\right\},\\
		\Delta_5&:=\mathbb{L}(R^*;\lambda)-\mathbb{L}(R_0;\lambda).
	\end{align*}
\end{lemma}
The proof of Lemma \ref{Error Decomposition} is given in
%Appendix \ref{appB}.
Section \ref{appB} in
the supplementary material.
The error decomposition in  Lemma  \ref{Error Decomposition} is a key step in the proof of Theorem \ref{Weak_Consistency_Results}, as we only need to bound each term on the right-hand side of the error decomposition inequality in Lemma \ref{Error Decomposition}.

The terms $\Delta_1,\Delta_2$ and $\Delta_5$ are the approximation errors, which can be %properly handled
controlled
using the approximation results with deep neural networks \citep{Shen2020CICP,jiao2021deep}.
The terms $\Delta_3$ and $\Delta_4$ are stochastic errors.
Specifically, $\Delta_3$  involves  empirical processes and  can be bounded through some classic empirical process results such as Dudley's inequality; see
Lemma \ref{Classic Dudleys Chaining} in
%Appendix
Section \ref{appC} in
the supplementary material.
However, the term $\Delta_4$ involves an order-two U-process, and the classic Dudley's inequality is no longer
%directly
applicable. We will derive an inequality to bound this term below.

The second basic lemma bounds the $L_1$ distances of density functions in terms of the excess risk.
\begin{lemma}
	\label{pinsker}
	It holds that
	\begin{align}
		\mathbb{L}(R;\lambda)-\mathbb{L}(R_0;\lambda)	\ge \frac{1}{2}(E_{X}\|p_{Y|X}-p_{Y|R}\|_{L_{1}}^2
		+\lambda\|p_{R}-1\|_{L_1}^2). \label{Pinsker_induced}
	\end{align}
\end{lemma}

Lemma \ref{pinsker} can be easily verified as follows.
By the definition of the objective function $\mathbb{L}$ given in (\ref{Lobja}),
for any $R\in\mathcal{R}$,
\begin{align*}
	\mathbb{L}(R;\lambda)-\mathbb{L}(R_0;\lambda)
	%	=\mathbb{I}_{\textup{KL}}(Y;X)+\mathbb{L}(R;\lambda)\\
	=\mathbb{I}_{\textup{KL}}(Y;X)-\mathbb{I}_{\textup{KL}}(Y;R(X)) + \lambda \mathbb{D}_{\textup{KL}}(P_{R}~||~\gamma_{U}).
\end{align*}
Now since $\mathbb{I}_{\textup{KL}}(Y;X) =\mathbb{I}_{\textup{KL}}(Y;(X,R(X))),$
we have
\begin{eqnarray}
	\mathbb{L}(R;\lambda)-\mathbb{L}(R_0;\lambda)\nonumber
	&=&\mathbb{I}_{\textup{KL}}(Y;(X,R(X)))-\mathbb{I}_{\textup{KL}}(Y;R(X)) + \lambda \mathbb{D}_{\textup{KL}}(P_{R}~||~\gamma_{U})\nonumber\\
	&=&\mathbb{I}_{\textup{KL}}(Y;X|R(X))+\lambda \mathbb{D}_{\textup{KL}}(P_{R}~||~\gamma_{U}), \label{con_MI_equality}
\end{eqnarray}
where  (\ref{con_MI_equality}) follows from the chain rule for conditional mutual information \citep[Theorem 2.5.2 in][]{cover2006elements}, i.e.,  $\mathbb{I}_{\textup{KL}}(Y;(X,R(X)))
=\mathbb{I}_{\textup{KL}}(Y;R(X))+\mathbb{I}_{\textup{KL}}(Y;X|R(X))$.  Therefore, Lemma \ref{pinsker} follows directly from  Pinsker's inequality \citep{npe2008}.

In view of (\ref{Pinsker_induced}), to prove (\ref{weak_Results}), it suffices to show
\[
E\{\mathbb{L}(\hat{R}_n^{\lambda};\lambda)-\mathbb{L}(R_0;\lambda)\}\to0,\text{ as } n\to\infty.
\]

\subsection{A maximal inequality for U-process of order  2}
To bound $\Delta_4$, we derive
%an extended
a generalized Dudley's inequality for an order-two U-process using Rademacher chaos. This inequality stated in the next lemma may be of independent interest.
% which could be of independent interest; see (\ref{Dudley_Rad_Chaos}) in Appendix \ref{appC}.

\begin{lemma}[%Symmetrization
Maximal inequality for U-Process of Order 2]\label{Symmetrization}
	Let $\mathcal{H}$ be a class of symmetric and real-valued functions on $\mathcal{Z}^2$ and assume that $\sup\limits_{h\in\mathcal{H}}\|h\|_{\infty}\le b$ for some constant $b$. Define
	\[
	\mathcal{U}_n^{\textup{Deg}}(h)=\frac{1}{n(n-1)}\sum_{1\le i\neq j\le n}\left\{h(Z_i,Z_j)-\tilde{h}(Z_i)-\tilde{h}(Z_j)+Eh(Z_i,Z_j)\right\},
	\]
	where $\tilde{h}(z)=Eh(z,Z)$ for any $z\in\mathcal{Z}$, then there exists a universal constant $C$ such that
	\begin{equation}\label{Dudley_Rad_Chaos}
		E\sup\limits_{h\in\mathcal{H}}\mathcal{U}_n^{\textup{Deg}}(h)\le E\mathcal{S}_n^{\textup{Deg}}(h)\le\inf\limits_{0\le\delta\le b}\left(\delta+\frac{24C\int_{\delta/4}^{b}\log(\mathcal{N}_{2n(n-1)}(u,\|\cdot\|_{\infty},\mathcal{H}))du}{n}\right),
	\end{equation}
	where
	\[
	\mathcal{S}_n^{\textup{Deg}}(\mathcal{H})=E_{\epsilon}\sup\limits_{h\in\mathcal{H}}\frac{1}{n(n-1)}\sum_{1\le i\neq j\le n}\epsilon_i\epsilon_j\left\{h(Z_i,Z_j)-h(Z'_i,Z_j)-h(Z_i,Z'_j)+h(Z'_i,Z'_j)\right\}
	\]
	and $\mathcal{N}_{2n(n-1)}(u,\|\cdot\|_{\infty},\mathcal{H})$ is the uniform covering number of $\mathcal{H}$ \citep{anthony_bartlett_1999}.
%defined in the beginning of Appendix
See also
Section \ref{appB} in the supplementary material.
\end{lemma}

%Through scaling,
We note that some similar entropy bounds but with different metrics
have been obtained in the existing literature,
see Lemma 8 of \cite{FukumizuNIPS2009} and Theorem 2 of \cite{Ying:Campbell:2010}.  These entropy bounds are similar to a special case for the inequality (\ref{Dudley_Rad_Chaos}), which is
\begin{equation}\label{Dudley_integral_bound}
	E\sup\limits_{h\in\mathcal{H}}\mathcal{U}_n^{\textup{Deg}}(h)\le E\mathcal{S}_n^{\textup{Deg}}(h)\le C\int_{0}^{b}\log(\mathcal{N}_{2n(n-1)}(u,\|\cdot\|_{\infty},\mathcal{H}))du/n,
\end{equation}
where $C$ is a universal constant.  Obviously, the integral upper bound in (\ref{Dudley_Rad_Chaos}) can be tighter than that in (\ref{Dudley_integral_bound}).

A distinctive feature in (\ref{Dudley_Rad_Chaos}) is the uniform covering number we use, which is naturally related to  neural networks .
%and may have
%many potential
%other applications in the theoretical analysis of DNNs.
%For example, using
By Theorem 12.2 of \cite{anthony_bartlett_1999}, we have for $n~\textgreater~\Pdim(\mathcal{H})$,
\[
\mathcal{N}_{2n(n-1)}(u,\|\cdot\|_{\infty},\mathcal{H})\le\left(\frac{4bn(n-1)}{u\Pdim(\mathcal{H})}\right)^{\Pdim(\mathcal{H})},
\]
where $\Pdim(\mathcal{H})$ is the pseudo dimension of $\mathcal{H}$ \citep{bartlett2019nearly}, see also Section \ref{appB} in the supplementary material.
%and defined in the beginning of Appendix \ref{appB}.
By Theorem 6 in \citet{bartlett2019nearly}, if $\mathcal{H}$ is a ReLU network with depth $\mathcal{L}$ and size $\mathcal{S}$, there exists a universal constant $C_1$, such that $\Pdim(\mathcal{H})\le C_1\mathcal{S}\mathcal{L}\log \mathcal{S}$.  Using these results and our inequality (\ref{Dudley_Rad_Chaos}), if $\mathcal{H}=\mathcal{D}$ satisfying (NS\ref{Strong_Structure_on_NetworkD}), we have
\[
E\sup\limits_{h\in\mathcal{H}}\mathcal{U}_n^{\textup{Deg}}(h)\le E\mathcal{S}_n^{\textup{Deg}}(h)\lesssim \frac{\Pdim(\mathcal{H})\log n}{n}\lesssim n^{-\frac{2\beta}{2\beta+d_Y+d_0}}.
\]

\section{Extensions to categorical response variables}
\label{Extensions}
We focus on continuous response and predictors in
the previous sections and the theoretical analysis also relies on  the continuous-distributed assumption. When there are categorical variables,  we can easily modify our method to accommodate this case.   We first assume only the response $Y$ is  categorical variable with finite categories.  Namely,  $Y\in\{1,2,\ldots,K\}$ for some $K\in\mathbb{N}^+$.  Note that
\begin{align*}
	\mathbb{I}_{\textup{KL}}(Y;X)&=\int_{\mathcal{X}\times\mathcal{Y}}p_{XY}(x,y)\log \left(\frac{p_{XY}(x,y)}{p_X(x)p_Y(y)}\right)dxdy\\
	&=\sum_{k=1}^{K}p_k\int_{\mathcal{X}}p_{X|Y}(x|Y=k)\log \left(\frac{p_{X|Y}(x|Y=k)}{p_X(x)}\right)dx\\
	&=\sum_{k=1}^{K}p_k \mathbb{D}_{\textup{KL}}(P_{X|Y=k}~||~P_X),
\end{align*}
where $p_k=P(Y=k)$ for $k=1,2,\ldots,K$.  In view of this,  we consider the following empirical objective function
\begin{eqnarray}
	\mathbb{L}^{\textup{net}}_{n,c}(R;\lambda)&=&-\sum_{k=1}^{K}\hat{p}_k\sup\limits_{D_k\in\mathcal{D}_k}\left\{\frac{1}{n_k}\sum_{i:Y_i=k}^{n}D_k(R(X_i))-\frac{1}{n}\sum_{i=1}^ne^{D_k(R(X_i))}\right\}\nonumber\\
	&&\hskip 3cm +\lambda\sup\limits_{Q\in\mathcal{Q}}\left\{\frac{1}{n}\sum_{i=1}^{n}Q(R(X_i))-\frac{1}{n}\sum_{i=1}^ne^{Q(U_i)}\right\},\label{categorical_loss}
\end{eqnarray}
and its minimizer $\hat{R}_n^{\lambda}\in\argmin_{R\in \mathcal{R}}\mathbb{L}^{\textup{net}}_{n,c}(R;\lambda)$ is defined as a USR estimate, where $\hat{p}_k=n_k/n,n_k = \sum_{i=1}^n\textup{I}\{Y_i=k\}$ for $k=1,2,\ldots,K$, and $\mathcal{R}$, $\mathcal{Q}$, and $\mathcal{D}_k$,  $k=1,2,\ldots,K$, are some neural network function classes.

Let $\hat{p}=(\hat{p}_1,\hat{p}_2,\ldots,\hat{p}_K)^\top$ and let $\tilde{Y}_i\in\mathbb{R}^K$ be the one-hot vector with the $k$th entry being $1$ when $Y_i=k$. And $\mathcal{D}$ is a $K$-output neural network function class but each output does not share any weights or bias.  Namely,  $D=(D_1,D_2,\ldots,D_K)^\top$ for any $D\in\mathcal{D}$ and $D_1,D_2,\ldots,D_K$ do not share any parameter. Then,  (\ref{categorical_loss}) can be recast as
\begin{eqnarray*}
	\mathbb{L}^{\textup{net}}_{n,c}(R;\lambda)&=&-\sup\limits_{D\in\mathcal{D}}\left\{\frac{1}{n}\sum_{i=1}^{n}D(R(X_i))^\top \tilde{Y}_i-\frac{1}{n}\sum_{i=1}^n \hat{p}^\top e^{D(R(X_i))}\right\}\\
	&&\hskip 3cm +\lambda\sup\limits_{Q\in\mathcal{Q}}\left\{\frac{1}{n}\sum_{i=1}^{n}Q(R(X_i))-\frac{1}{n}\sum_{i=1}^ne^{Q(U_i)}\right\},
\end{eqnarray*}
where $e^{D(\cdot)}=(e^{D_1(\cdot)},e^{D_2(\cdot)},\ldots,e^{D_K(\cdot)})^\top$.  This formulation is in terms of one single neural network and allows direct application of the MSRL training algorithm \ref{Training_Algorithm}.

Note that we can further extend the above approach to allow more general categorical cases. For instance, we assume $Y=(Y_{D},Y_{C})$ and $X=(X_{D},X_{C})$, where the subscripts $D$ and $C$ indicate the categorical and continuous components of a random vector, respectively.  Moreover, it is assumed that $Y_{D}\in\{1,2,\ldots,K\}$ and $X_{D}\in\{1,2,\ldots,J\}$ for some $J\in\mathbb{N}^+$.
Our goal is to learn a representation $R$ such that $Y\perp \!\!\! \perp X|(X_{D},R(X_{C}))$.
Analogous to  (\ref{categorical_loss}),  we consider the following empirical objective function
\begin{eqnarray*}
	&&\mathbb{L}^{\textup{net}}_{n,c}(R;\lambda)\\
	&=&-\sum_{j=1}^{J}\sum_{k=1}^{K}\hat{p}_{jk}\sup\limits_{D_{jk}\in\mathcal{D}_{jk}}\left\{\frac{1}{n_{jk}}\sum_{i:Y_{D,i}=k,X_{D,i}=j}^{n}D_{jk}(Y_{C,i},R(X_{C,i}))-\frac{1}{n}\sum_{i=1}^ne^{D_{jk}(Y_{C,i},R(X_{C,i}))}\right\}\nonumber\\
	&&+\lambda\sup\limits_{Q\in\mathcal{Q}}\left\{\frac{1}{n}\sum_{i=1}^{n}Q(R(X_{C,i}))-\frac{1}{n}\sum_{i=1}^ne^{Q(U_i)}\right\},
\end{eqnarray*}
where $\hat{p}_{jk}=n_{jk}/n,n_{jk} = \sum_{i=1}^n\textup{I}\{Y_{D,i}=k,X_{D,i}=j\}$, $j=1,2,\ldots,J$, $k=1,2,\ldots,K$, and $\mathcal{R}$, $\mathcal{Q}$, and $\mathcal{D}_{jk}$ for $j=1,2,\ldots,J$, and $k=1,2,\ldots,K$, are some neural network function classes.

Our theoretical analysis in section \ref{Theoretical_Results} can be applied to study the convergence properties of the MSRL for the aforementioned categorical cases without further difficulties. We omit the details.

\section{Implementation}\label{Implementation} In this section,
we present a MSRL training algorithm and discuss some practical issues.  We rewrite $R,D,Q$ as $R_{\boldsymbol{\theta}},D_{\boldsymbol{\phi}},Q_{\boldsymbol{\psi}}$ to underline their weight and bias parameters $\boldsymbol{\theta},\boldsymbol{\phi}$ and $\boldsymbol{\psi}$, respectively.  The training procedure is summarized in Algorithm \ref{Training_Algorithm}.

When the batch size $m$ is relatively large, such as $m=256$ or $512$, to reduce the computational burden, we approximate the U-process term in Algorithm \ref{Training_Algorithm} using permuted samples.  The early-stopping (ES) criterion is applied to avoid the over-fitting phenomena, and the distance covariance or distance correlation \citep{szekely2007measuring} is used as the early-stopping criterion.

\subsection{Determination of the intrinsic dimension}\label{d_Determination}
The determination of the intrinsic dimension $d_0$ is crucial in real implementation.  In practice, if taking different downstream  tasks into account, different determination methods may be adopted. For example, when  the predictive performance is of more concern, we may determine $d_0$ by the cross-validation in terms of APE.  In sections \ref{rd1}-\ref{rd2}, we determine $d_0$ as 5 and  $d_0$ as 10  respectively by this approach.

In this section, we propose a dichotomy-based cross-validation method to  determine the intrinsic dimension $d_0$.  For $k=1,2,\ldots,d_X$, let
\begin{equation}\label{ture_Rk}
	R_{0}^{(k)}\in\argmax_{R~:~\textup{dim}(R)=k}\mathbb{I}_{\textup{KL}}(Y;R(X)).
\end{equation}
Then we have
\begin{equation}\label{MI_inequality1}
	\mathbb{I}_{\textup{KL}}(Y;R_{0}^{(k)}(X))\le \mathbb{I}_{\textup{KL}}(Y;(0, R_{0}^{(k)}(X)))\le \mathbb{I}_{\textup{KL}}(Y;R_{0}^{(k+1)}(X)),
\end{equation}
for $k=1,2,\ldots,d_X-1$, where the first inequality follows from the data-processing inequality (\ref{InformationLossInequality}) and the second inequality is by the definition in (\ref{ture_Rk}).  Notably, when $k\ge d_0$,
\begin{equation}\label{MI_inequality2}
	\mathbb{I}_{\textup{KL}}(Y;R_{0}^{(k)}(X))=\mathbb{I}_{\textup{KL}}(Y;R_0(X)).
\end{equation}
This is because
\[
\mathbb{I}_{\textup{KL}}(Y;R_{0}^{(k)}(X))\le\mathbb{I}_{\textup{KL}}(Y;R_0(X))~\text{for any}~k,
\]
by the definition of USR for $R_0$, and
\[
\mathbb{I}_{\textup{KL}}(Y;R_0(X))=\mathbb{I}_{\textup{KL}}(Y;R_{0}^{(d_0)}(X))\le\mathbb{I}_{\textup{KL}}(Y;R_{0}^{(k)}(X)),~\text{for any}~k\ge d_0.
\]
Applying MSRL with intrinsic dimension input $k$, we obtain the minimizer
\begin{equation}\label{R_train}
	\hat{R}_{k,n}^{\lambda}\in\argmin_{R\in \mathcal{R},  \textup{dim}_{\textup{out}}(\mathcal{R})=k}\mathbb{L}^{\textup{net}}_n(R;\lambda)
\end{equation}
as an estimate of $R_{0}^{(k)}$.  Let
\begin{equation}\label{D_train}
	\hat{D}_{k,n}^{\lambda}\in\argmax_{D\in \mathcal{D}}\left\{\frac{1}{n}\sum_{i=1}^{n}D(Y_i,\hat{R}_{k,n}^{\lambda}(X_i))-\frac{1}{n(n-1)}\sum_{i\neq j}e^{D(Y_i,\hat{R}_{k,n}^{\lambda}(X_j))}\right\}.
\end{equation}
Under some mild conditions, $\mathbb{I}_{\textup{KL}}(Y;R_{0}^{(k)}(X))$ can be consistently estimated by
\begin{equation}\label{MI_est}
	\hat{\mathbb{I}}_{\textup{KL}}^{k,n}=\frac{1}{n}\sum_{i=1}^{n}\hat{D}_{k,n}^{\lambda}(Y_i,\hat{R}_{k,n}^{\lambda}(X_i))-\frac{1}{n(n-1)}\sum_{i\neq j}e^{\hat{D}_{k,n}^{\lambda}(Y_i,\hat{R}_{k,n}^{\lambda}(X_j))}+1.
\end{equation}
Intuitively, $\hat{\mathbb{I}}_{\textup{KL}}^{\lfloor d_X/2\rfloor,n}\ll \hat{\mathbb{I}}_{\textup{KL}}^{d_X,n}$ implies $d_{0}\in (\lfloor d_X/2\rfloor,d_X]$;  otherwise,  $d_{0}\in [1,\lfloor d_X/2\rfloor]$.  This motivates a dichotomy-based cross-validation selection method as follows.  Given the data $\{(Y_i,X_i)\}_{i=1}^n$ and tolerance level $\eta$, we partition the sample into folds $I_t(t = 1,\ldots,T)$ and denote the complement of $I_t$ by $I^c_t$.  For any nonnegative integer $k\le d_X$, we obtain the mutual information estimate $\hat{\mathbb{I}}_{\textup{KL}}^{k,n}=\frac{1}{T}\sum_{t=1}^{T}\hat{\mathbb{I}}_{\textup{KL}}^{k,n,t}$ of $\mathbb{I}_{\textup{KL}}(Y;R_{0}^{(k)}(X))$ as below.  For $t = 1,\ldots,T$,
\[
\hat{\mathbb{I}}_{\textup{KL}}^{k,n,t}=\frac{1}{|I_t|}\sum_{i\in I_t}\hat{D}_{k,n}^{\lambda}(Y_i,\hat{R}_{k,n}^{\lambda}(X_i))-\frac{1}{|I_t|(|I_t|-1)}\sum_{i,j\in I_t,i\neq j}e^{\hat{D}_{k,n}^{\lambda}(Y_i,\hat{R}_{k,n}^{\lambda}(X_j))}+1,
\]
where $\hat{R}_{k,n,t}^{\lambda}$ and $\hat{D}_{k,n,t}^{\lambda}$ are trained based on $I^c_t$ as we have done in (\ref{R_train}) and (\ref{D_train}).

Let $\textup{UD}=d_X$, $\textup{LD}=1$, $\widehat{\textup{MI}}=\hat{\mathbb{I}}_{\textup{KL}}^{\textup{UD},n}$, and $u=\lfloor (\textup{UD}+\textup{LD})/2\rfloor$.  If $|\hat{\mathbb{I}}_{\textup{KL}}^{u,n}-\widehat{\textup{MI}}|/\widehat{\textup{MI}}\le\eta$, we update $\textup{UD}=u,\widehat{\textup{MI}}= \hat{\mathbb{I}}_{\textup{KL}}^{u,n}$ and otherwise update $\textup{LD}=u+1$.  We continue this procedure until $\textup{UD}=\textup{LD}$ and the estimate of the intrinsic dimension is the ultimate $\textup{UD}$.

In practice, instead of starting from $d_X$, one may start from certain upper bound $d_{U}$ for $d_0$.  The proposed method significantly reduces the computational complexity since it only screens $\lceil \log_2 d_U\rceil$  alternative values for $d_0$. We provide a simple simulation example to
illustrate the proposed method in Subsection \ref{Intrinsic_dim_est}.

\begin{algorithm}[t]
	\caption{MSRL Training Procedure}\label{Training_Algorithm}
	\hspace*{0.02in} {\bf Input:}
	Data $\{(Y_i,X_i)\}_{i=1}^n$, tuning parameter $\lambda$, the intrinsic dimension $d_0$ and batch size $m$,\\
	\hspace*{0.02in} {\bf Output:}
	Estimated USR $\hat{R}_n^{\lambda}$, MI discriminator $\hat{D}_n^{\lambda}$, pushing-forward discriminator $\hat{Q}_n^{\lambda}$.
	\begin{center}
		\begin{algorithmic}[1]
			\State Sample i.i.d. $U_1,\ldots,U_n$ from $\textup{Uniform}[0,1]^{d_0}$.
			\While{not converged}
			\State Draw $m$ minibatch samples $\{(Y_{b,j},X_{b,j},U_{b,j})\}_{j=1}^m$ from $\{(Y_i,X_i,U_i)\}_{i=1}^n$.
			\State Fix $R_{\boldsymbol{\theta}}$, update $D_{\boldsymbol{\phi}}$ by ascending its stochastic gradient:
			\[
			\nabla_{\boldsymbol{\phi}}\left\{\frac{1}{m}\sum_{j=1}^{m}D_{\boldsymbol{\phi}}(Y_{b,j},R_{\boldsymbol{\theta}}(X_{b,j}))-\frac{1}{m(m-1)}\sum_{1\le j\neq k\le m}e^{D_{\boldsymbol{\phi}}(Y_{b,j},R_{\boldsymbol{\theta}}(X_{b,k}))}\right\},
			\]
			and update $Q_{\boldsymbol{\psi}}$ by ascending its stochastic gradient:
			\[
			\nabla_{\boldsymbol{\psi}}\left\{\frac{1}{m}\sum_{j=1}^{m}Q_{\boldsymbol{\psi}}(R_{\boldsymbol{\theta}}(X_{b,j}))-\frac{1}{m}\sum_{j=1}^me^{Q_{\boldsymbol{\psi}}(U_{b,j})}\right\}.
			\]
			\State Fix $D_{\boldsymbol{\phi}}$ and $Q_{\boldsymbol{\psi}}$, update $R_{\boldsymbol{\theta}}$ by descending its stochastic gradient
			\begin{eqnarray*}
				&-&\nabla_{\boldsymbol{\theta}}\left\{\frac{1}{m}\sum_{j=1}^{m}D_{\boldsymbol{\phi}}(Y_{b,j},R_{\boldsymbol{\theta}}(X_{b,j}))-\frac{1}{m(m-1)}\sum_{1\le j\neq k\le m}e^{D_{\boldsymbol{\phi}}(Y_{b,j},R_{\boldsymbol{\theta}}(X_{b,k}))}\right\}\\
				&+&\nabla_{\boldsymbol{\theta}}\left\{\frac{1}{m}\sum_{j=1}^{m}Q_{\boldsymbol{\psi}}(R_{\boldsymbol{\theta}}(X_{b,j}))-\frac{1}{m}\sum_{j=1}^me^{Q_{\boldsymbol{\psi}}(U_{b,j})}\right\}.
			\end{eqnarray*}
			\EndWhile
			\State \Return $\hat{R}_n^{\lambda}=R_{\hat{\boldsymbol{\theta}}},\hat{D}_n^{\lambda}=D_{\hat{\boldsymbol{\phi}}}$, and $\hat{Q}_n^{\lambda}=Q_{\hat{\boldsymbol{\psi}}}$, where $\hat{\boldsymbol{\theta}},\hat{\boldsymbol{\phi}}$ and $\hat{\boldsymbol{\psi}}$ are the ultimate weight and bias parameters estimates.
		\end{algorithmic}
	\end{center}
\end{algorithm}

\section{Experiments}\label{Experiments}
We conduct several simulation studies to examine the
%numerical
performance of the proposed method. For comparison, we compute two commonly-used linear SDR methods: sliced inverse regression (SIR) and sliced
average variance estimation (SAVE), and two nonlinear SDR methods: generalized sliced inverse regression and generalized sliced average variance estimation (GSIR and GSAVE), as the competitors.  In Subsection \ref{Simulation_Study}, we first adopt the conventional practice in many SDR existing literature to assume the intrinsic dimension is known, see such as \cite{Zhu:Miao:Peng:2006} and \cite{Lee2013aos}.

After learning a representation $\hat{R}$ on the training data $\{(Y_{j},X_{j})\}_{j=1}^{n}$, we apply it to the testing data $\{(Y_{n+j},X_{n+j})\}_{j=1}^{m}$ and obtain the transformed data $\{(Y_{n+j},\hat{R}(X_{n+j}))\}_{j=1}^{m}$. Based on $\{(Y_{n+j},\hat{R}(X_{n+j}))\}_{j=1}^{m}$, we then calculate the empirical distance correlation (DC) of $Y$ and $R(X)$.  To assess the predictive power of the learned representation, we fit a linear regression between $Y$ and $R(X)$ by using the transformed training data $\{(Y_{j},\hat{R}(X_{j}))\}_{j=1}^{n}$ and subsequently employ the fitted linear model $\widehat{LR}$ to the testing data $\{(Y_{n+j},\hat{R}(X_{n+j}))\}_{j=1}^{m}$.  The average prediction error (APE) is calculated by $[(1/m)\sum_{j=1}^{m}\{Y_{n+j}-\widehat{LR}(\hat{R}(X_{n+j}))\}^2]^{1/2}$.

\subsection{Simulation Study}\label{Simulation_Study}
In the simulation studies, we generate 6000 data points from each model below:
\[
\begin{cases}
	\textup{\uppercase\expandafter{\romannumeral1}}:&Y = 0.5X_1+X_2+\epsilon; \\
	\textup{\uppercase\expandafter{\romannumeral2}}:&Y = (X^2_1 + X^2_2)^{1/2}\log ((X^2_1 + X^2_2)^{1/2})+\epsilon; \\
	\textup{\uppercase\expandafter{\romannumeral3}}:&Y = (X_1+X_2)^2/(1+\exp(X_1))+\epsilon; \\
	\textup{\uppercase\expandafter{\romannumeral4}}:&Y = \sin(\pi(X_1 + X_2)/10)+X_1^2+\epsilon,
\end{cases}
\]
where $\epsilon\perp \!\!\! \perp X=(X_1,X_2,\ldots,X_p)^\top,\epsilon\sim N(0,0.25)$. Four distributional settings for the predictor $X$ are considered:
\[
\begin{cases}
	\textup{(\romannumeral1)}:&X\sim \textup{Uniform}[-2,2]^{p}; \\
	\textup{(\romannumeral2)}:&X\sim N(\boldsymbol{0},\boldsymbol{\textup{I}}_{p}); \\
	\textup{(\romannumeral3)}:&X\sim\frac{1}{4}N(-\boldsymbol{2}_{p},\boldsymbol{\textup{I}}_{p})+\frac{1}{2}\textup{Uniform}[-2,2]^{p}+\frac{1}{4}N(\boldsymbol{2}_{p},\boldsymbol{\textup{I}}_{p});\\
	\textup{(\romannumeral4)}:&X\sim N(\boldsymbol{0},0.5\cdot\boldsymbol{1}_{p}\cdot\boldsymbol{1}_{p}^\top+0.5\cdot\boldsymbol{\textup{I}}_{p}).
\end{cases}
\]
Notably, model complexity increases from model \textup{\uppercase\expandafter{\romannumeral1}} to model \textup{\uppercase\expandafter{\romannumeral4}}, and from scenario \textup{(\romannumeral1)} to scenario \textup{(\romannumeral4)}.  Full details of neural network architectures and hyperparameter choices for all the methods are given in
Section \ref{Descrip_Implement_Sim}if  the supplementary material.
%Appendix \ref{Descrip_Implement_Sim}.

The simulations are implemented in Python and the results for $p=10$ are summarized in Table \ref{T1} from which we have several observations.  First, our proposed MSRL outperforms all competitors in terms of larger distance correlation in almost all cases, implying that the features extracted by our method are more relevant to the outcome $Y$.  Second, as expected, SIR and SAVE perform  the best for  model \textup{\uppercase\expandafter{\romannumeral1}} which is relative simple model,   but suffer from serious DC and APE loss for nonlinear models. GSIR and GSAVE show the best performance in terms of APE when the underlying distribution of the predictor $X$ is simple. Among all these cases, our MSRL performs  comparably well.  Third,  the performance of our MSRL is superior to  other competitors in terms of APE when  the underlying distributions of $X$ are sophisticated.  Additional simulations results for $p=30$ are provided  in
Section \ref{Additional_Simulation} in the supplementary material.
% Appendix \ref{Additional_Simulation}.
 Overall, the proposed MSRL could be good supplement to other existing competitors in learning sufficient representation.

\begin{table}[htbp]
	\caption{\label{T1} Distance correlation (DC), average prediction errors (APE), and their standard errors (based on 6-fold validation and best hyperparameters selected by validation method) for $p=10$}
	\centering
	 \resizebox{\textwidth}{!}{
	\begin{tabular}{cccccccccc}
			\hline
			\hline
			&        & \multicolumn{2}{c}{Model \textup{\uppercase\expandafter{\romannumeral1}}} & \multicolumn{2}{c}{Model \textup{\uppercase\expandafter{\romannumeral2}}} & \multicolumn{2}{c}{Model \textup{\uppercase\expandafter{\romannumeral3}}} & \multicolumn{2}{c}{Model \textup{\uppercase\expandafter{\romannumeral4}}} \\
			\hline
			& Method & DC           & APE          & DC           & APE          & DC           & APE          & DC           & APE          \\
			\hline
			\multirow{5}{*}{S(\romannumeral1)} & MSRL  & .96(.00)   & 0.27(.01)   & .88(.01)   & 0.27(.01)   & .95(.01)   & 0.44(.03)   & .95(.00)   & 0.28(.01)   \\
			& SIR    & .96(.00)   & \textbf{0.25}(.01)   & .00(.00)   & 0.79(.01)   & .20(.02)   & 1.70(.05)   & .09(.01)   & 1.22(.02)   \\
			& SAVE   & .96(.00)   & \textbf{0.25}(.01)   & .00(.00)   & 0.79(.01)   & .20(.02)   & 1.70(.05)   & .09(.01)   & 1.22(.02)   \\
			& GSIR   & .96(.00)   & 0.27(.01)   & .89(.01)   & \textbf{0.26}(.01)   & .76(.02)   & \textbf{0.36}(.04)   & .71(.01)   & \textbf{0.27}(.01)   \\
			& GSAVE  & .96(.00)   & 0.26(.01)   & .86(.00)   & 0.30(.00)   & .66(.02)   & 0.71(.02)   & .66(.02)   & \textbf{0.27}(.01)   \\
			\hline
			\multirow{5}{*}{S(\romannumeral2)} & MSRL  & .95(.00)   & 0.26(.01)   & .90(.01)   & 0.28(.01)   & .91(.01)   & 0.49(.05)   & .94(.01)   & 0.35(.01)   \\
			& SIR    & .95(.00)   & \textbf{0.25}(.01)   & .00(.00)   & 0.90(.03)   & .15(.02)   & 1.51(.12)   & .07(.01)   & 1.43(.08)   \\
			& SAVE   & .95(.00)   & \textbf{0.25}(.01)   & .00(.00)   & 0.90(.03)   & .15(.02)   & 1.51(.12)   & .07(.01)   & 1.43(.08)   \\
			& GSIR   & .94(.01)   & 0.28(.01)   & .91(.01)   & \textbf{0.26}(.01)   & .78(.03)   & \textbf{0.38}(.02)   & .79(.04)   & \textbf{0.30}(.01)   \\
			& GSAVE  & .94(.01)   & 0.28(.01)   & .90(.01)   & 0.29(.01)   & .62(.03)   & 0.75(.04)   & .63(.03)   & 0.32(.02)   \\
			\hline
			\multirow{5}{*}{S(\romannumeral3)} & MSRL  & .99(.00)   & 0.27(.01)   & .97(.00)   & \textbf{0.33}(.01)   & .95(.01)   & \textbf{1.89}(.24)   & .98(.00)   & \textbf{0.50}(.06)   \\
			& SIR    & .99(.00)   & \textbf{0.25}(.01)   & .00(.00)   & 2.04(.05)   & .32(.02)   & 6.31(.22)   & .03(.01)   & 3.59(.11)   \\
			& SAVE   & .99(.00)   & \textbf{0.25}(.01)   & .00(.00)   & 2.04(.05)   & .32(.02)   & 6.33(.21)   & .03(.00)   & 3.59(.10)   \\
			& GSIR   & .85(.03)   & 0.94(.10)   & .92(.07)   & 0.55(.20)   & .61(.03)   & 3.31(.24)   & .72(.05)   & 0.71(.27)   \\
			& GSAVE  & .98(.00)   & 0.35(.03)   & .91(.01)   & 0.60(.05)   & .70(.02)   & 2.09(.27)   & .74(.01)   & 1.10(.11)   \\
			\hline
			\multirow{5}{*}{S(\romannumeral4)} & MSRL  & .96(.00)   & 0.26(.00)   & .92(.00)   & \textbf{0.29}(.01)   & .93(.00)   & \textbf{0.77}(.08)   & .95(.01)   & \textbf{0.34}(.02)   \\
			& SIR    & .97(.00)   & \textbf{0.25}(.01)   & .00(.00)   & 1.01(.03)   & .21(.02)   & 2.47(.23)   & .07(.02)   & 1.47(.03)   \\
			& SAVE   & .97(.00)   & \textbf{0.25}(.01)   & .00(.00)   & 1.01(.03)   & .21(.02)   & 2.47(.22)   & .07(.02)   & 1.47(.04)   \\
			& GSIR   & .95(.01)   & 0.31(.02)   & .91(.01)   & 0.30(.01)   & .60(.05)   & 1.22(.37)   & .73(.03)   & 0.37(.07)   \\
			& GSAVE  & .95(.01)   & 0.31(.02)   & .85(.01)   & 0.39(.01)   & .61(.03)   & 1.27(.21)   & .66(.02)   & 0.58(.04)\\
			\hline
		\end{tabular}}
\end{table}

\subsection{The superconductivity dataset}\label{rd1}
In this subsection, we apply MSDL to the superconductivity dataset, which is available at \textit{https://archive.ics.uci.edu/ml/datasets/Superconductivty+Data}.
In this dataset, there are $81$ features of $21,263$ superconductors and the response variable is the critical temperature; see a relevant paper \cite{Hamidieh2018} for the feature explanations. The goal of this analysis is to predict the critical temperature based on these features. The MSRL training procedure in Algorithm \ref{Training_Algorithm} is implemented with $d_0=5,10$ and $20$, respectively, and the reported results are based 5-fold validation. Additional implementation details are given in
%Appendix
Section \ref{Descrip_Implement_Real} in the supplementary material. The distance correlation, average prediction errors, and their standard errors for all methods are reported in Table \ref{Superconductivty}. It can be seen that MSRL outperforms other methods in terms of both empirical DC and prediction accuracy.
In this example, SIR performs better than  GSIR and GSAVE, which may be due to the curse of dimensionality for the kernel-based approaches and the relatively small effective sample size.  The implementation details for GSIR and GSAVE are given in
%Appendix
Section \ref{Descrip_Implement_Real} in the supplementary material.

\begin{table}
	\caption{\label{Superconductivty}Distance correlation (DC), average prediction errors (APE), and their standard errors (based on 5-fold validation and best hyperparameters selected by validation method) on Superconductivity dataset}
	\centering
	\begin{tabular}{c|cc|cc|cc}
		\hline
		\hline
		d        & \multicolumn{2}{c}{5}    & \multicolumn{2}{c}{10}   & \multicolumn{2}{c}{20}   \\
		\hline
		criterion & DC         & APE         & DC         & APE         & DC         & APE         \\
		\hline
		MSRL    & \textbf{0.80}(0.01) & \textbf{14.47}(0.39) & \textbf{0.82}(0.01) & \textbf{14.37}(0.36) & \textbf{0.82}(0.00) & \textbf{14.11}(0.18) \\
		SIR       & 0.33(0.01) & 17.86(0.41) & 0.23(0.01) & 17.67(0.25) & 0.16(0.00) & 17.64(0.26) \\
		SAVE      & 0.00(0.00) & 34.23(0.34) & 0.00(0.00) & 34.18(0.38) & 0.01(0.00) & 34.39(0.50) \\
		GSIR      & 0.45(0.04) & 21.44(2.28) & 0.37(0.03) & 20.87(2.30) & 0.39(0.06) & 20.86(2.23) \\
		GSAVE     & 0.47(0.03) & 29.89(1.44) & 0.45(0.03) & 29.51(1.09) & 0.44(0.03) & 29.54(1.22)\\
		\hline
	\end{tabular}
\end{table}

In Figure \ref{SuperConduct_drawing}, we plot the critical temperature against each component of the learned representation with $d_0=5$ for MSRL, GSIR and GSAVE. The plots for SIR and SAVE are given in Section \ref{extra-plots} in the supplementary material.  We fit a quadratic model to check how the components impact the response. From Figure \ref{SuperConduct_drawing}, we see that the critical temperature is nearly linear in three components of the MSRL estimate  and among other SDR methods, only SIR and GSIR show that the critical temperature is nearly linear in their corresponding first dimension reduction direction.  This may account for why MSRL has the high empirical DC and low prediction error when using a linear model.
\begin{figure}[htbp]
	\centering
	\begin{minipage}[t]{0.9\linewidth}
		\centering
		\rotatebox{90}{\quad\tiny MSRL}
		\includegraphics[width=0.10\textheight]{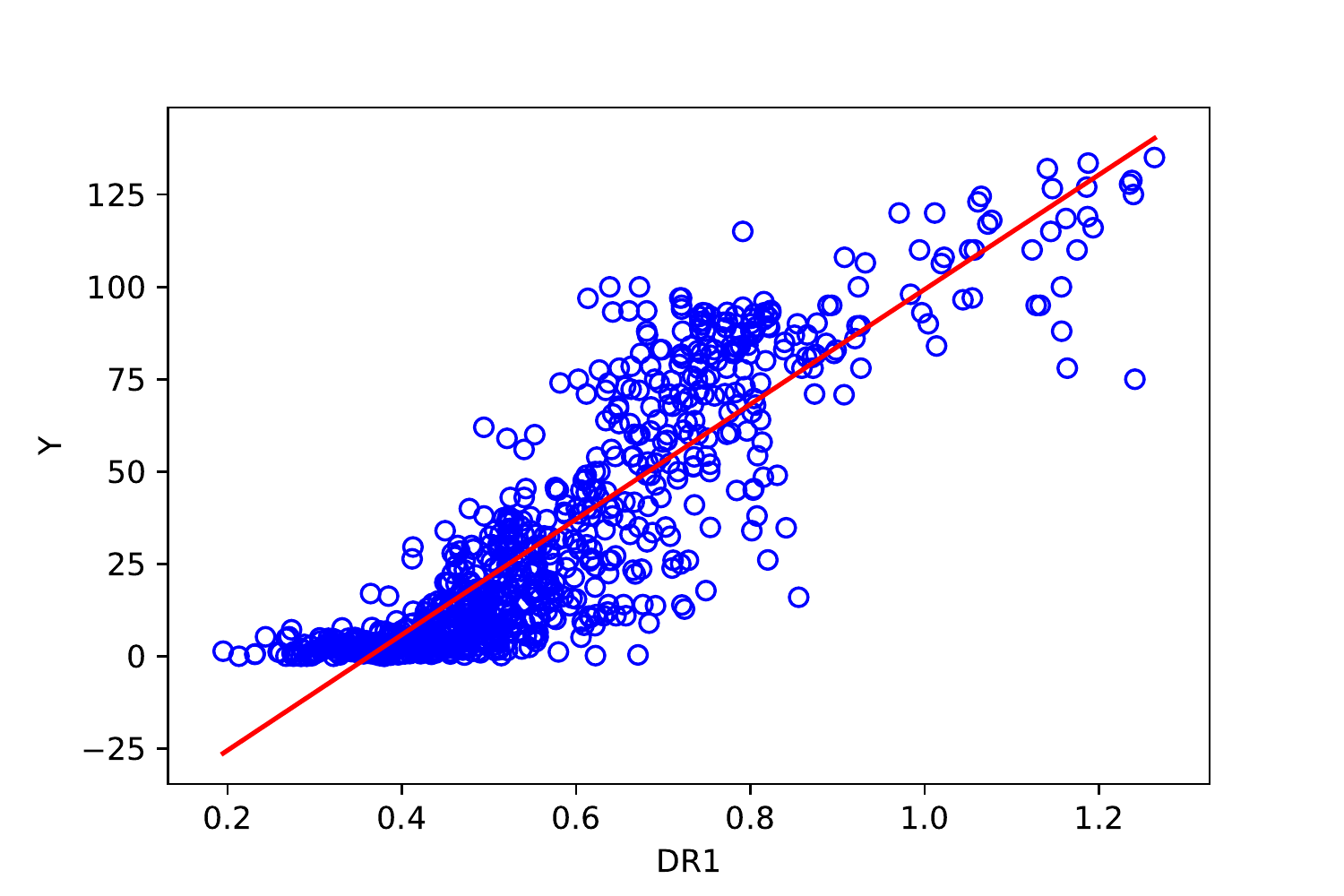} \includegraphics[width=0.10\textheight]{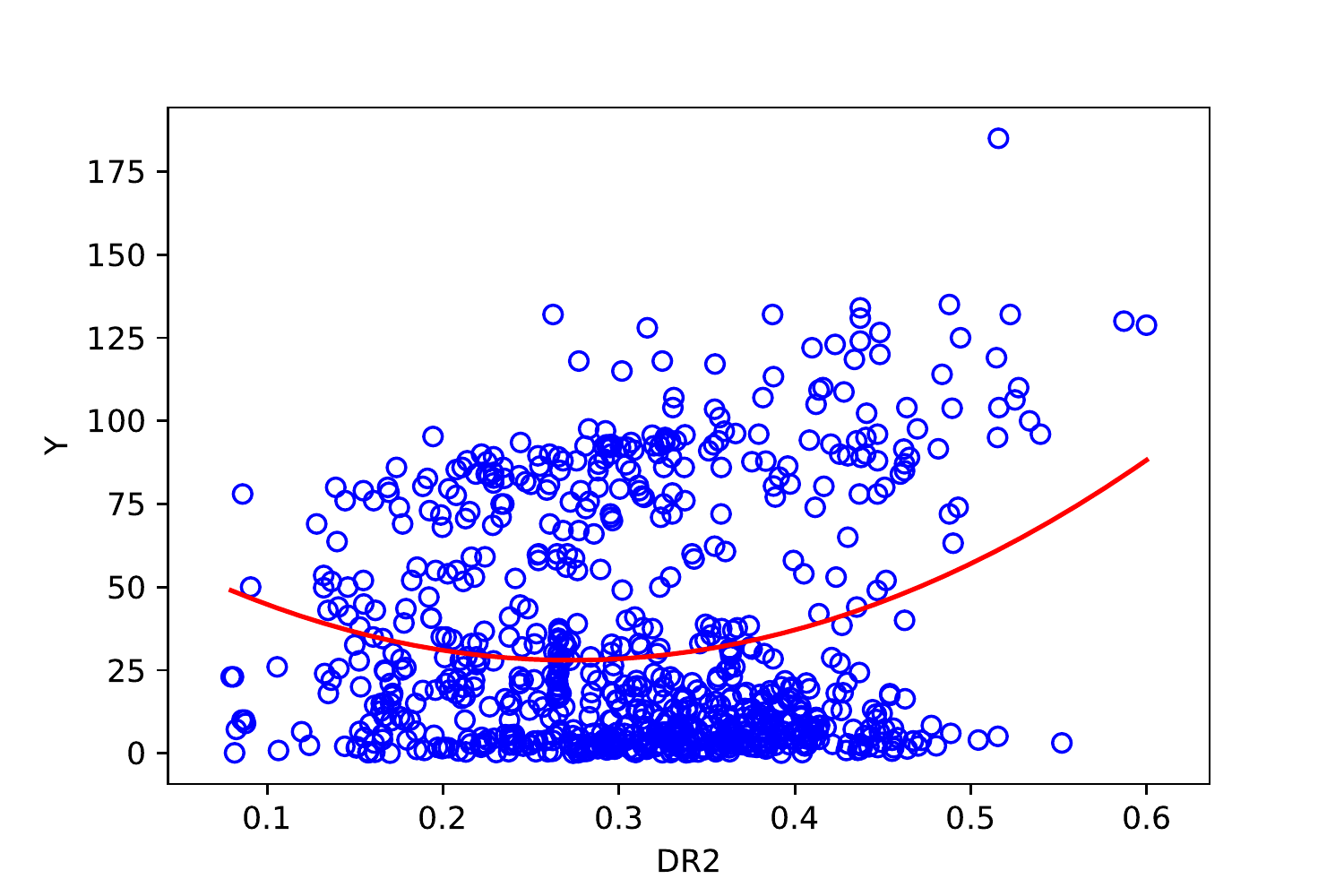}
		\includegraphics[width=0.10\textheight]{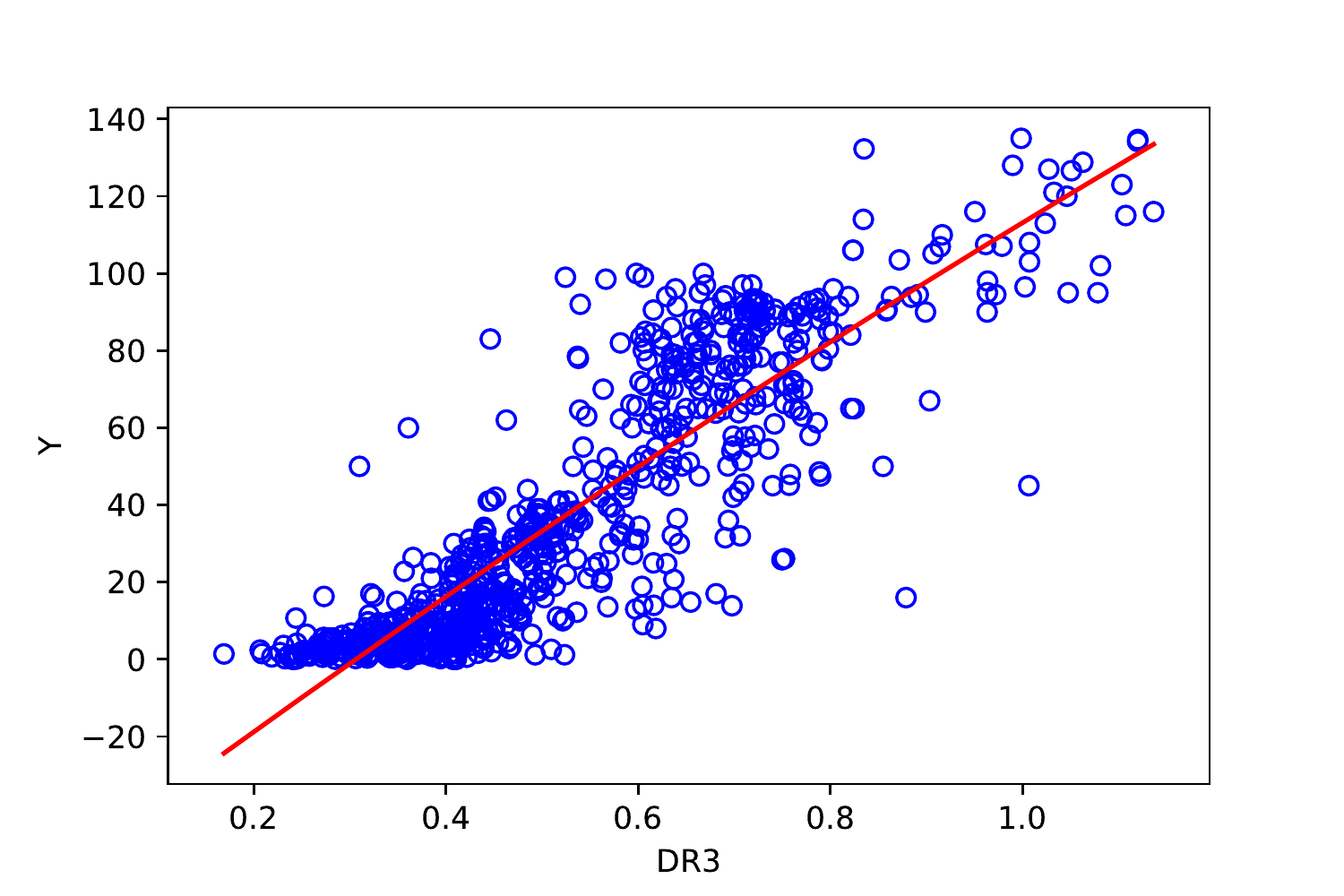} \includegraphics[width=0.10\textheight]{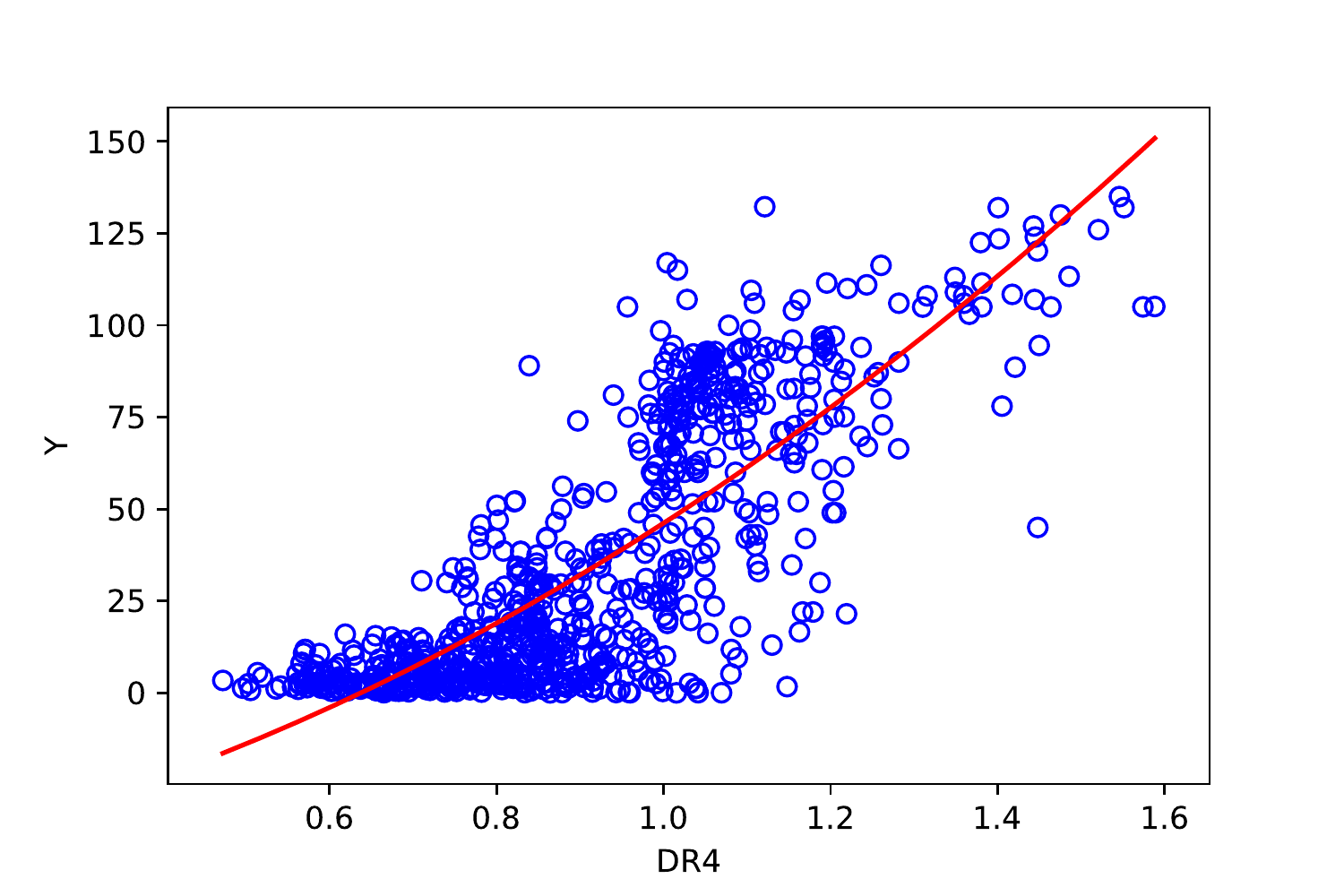}
		\includegraphics[width=0.10\textheight]{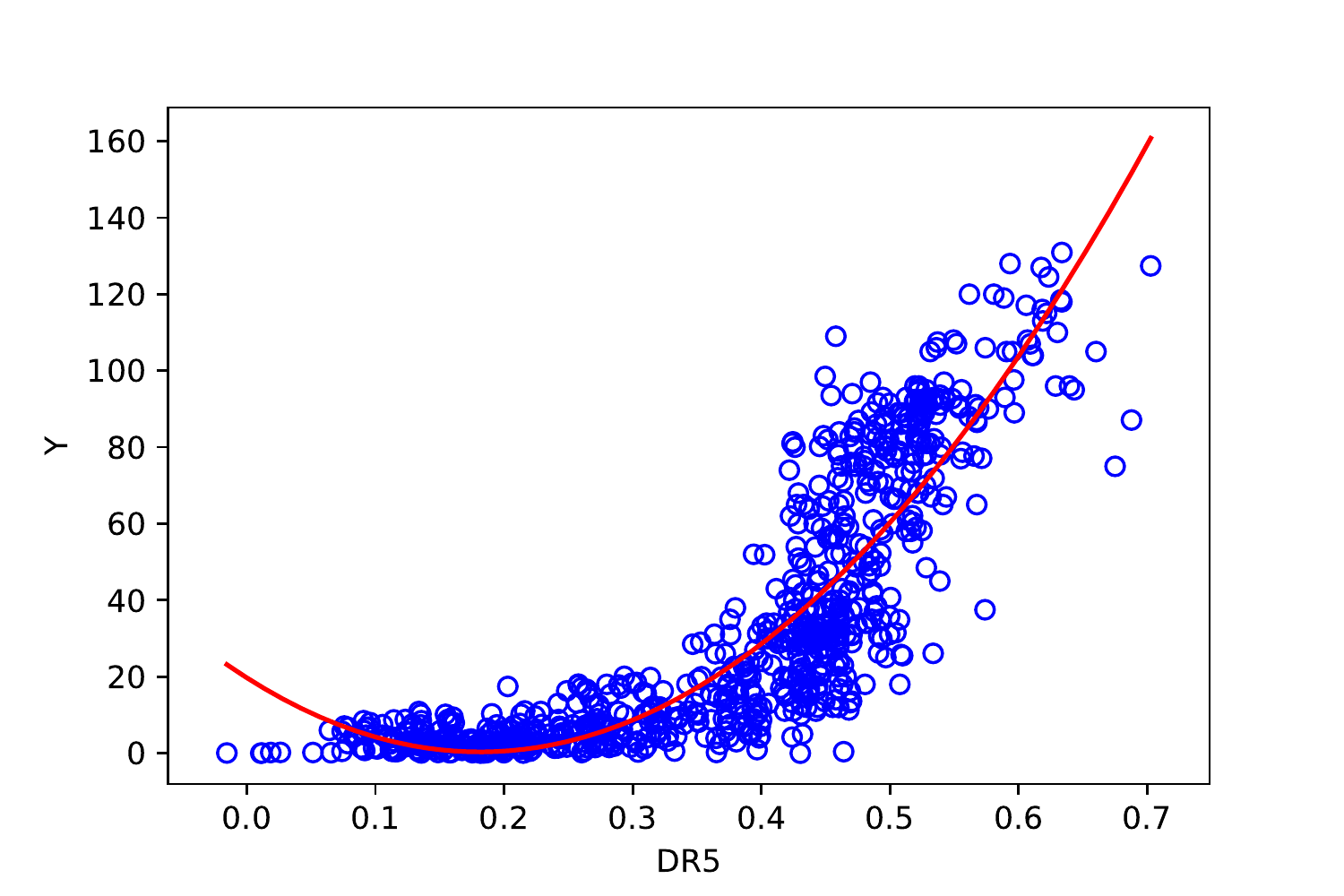}
	\end{minipage}
	\\
	\begin{minipage}[t]{0.9\linewidth}
		\centering
		\rotatebox{90}{\quad\tiny GSIR}
		\includegraphics[width=0.10\textheight]{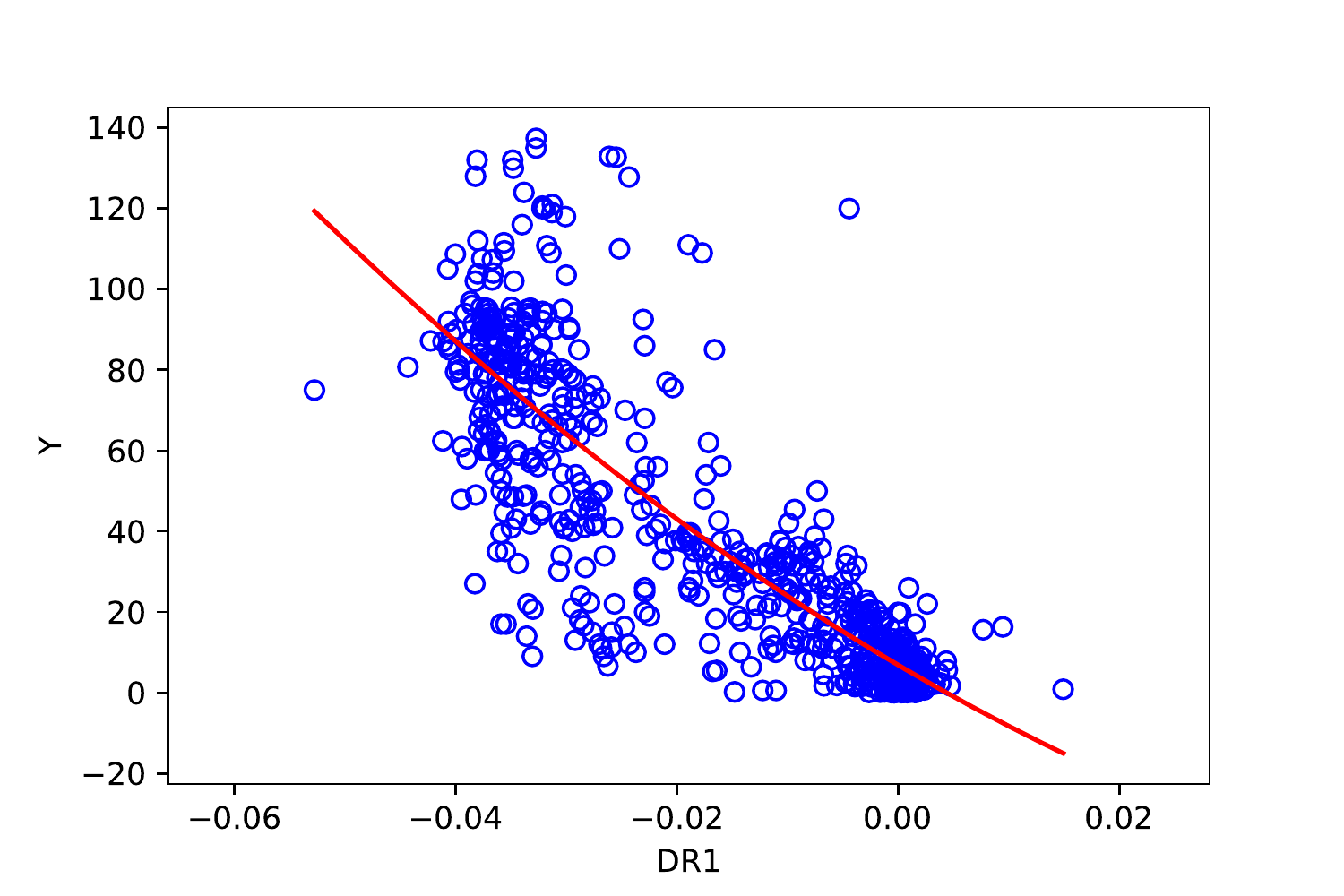} \includegraphics[width=0.10\textheight]{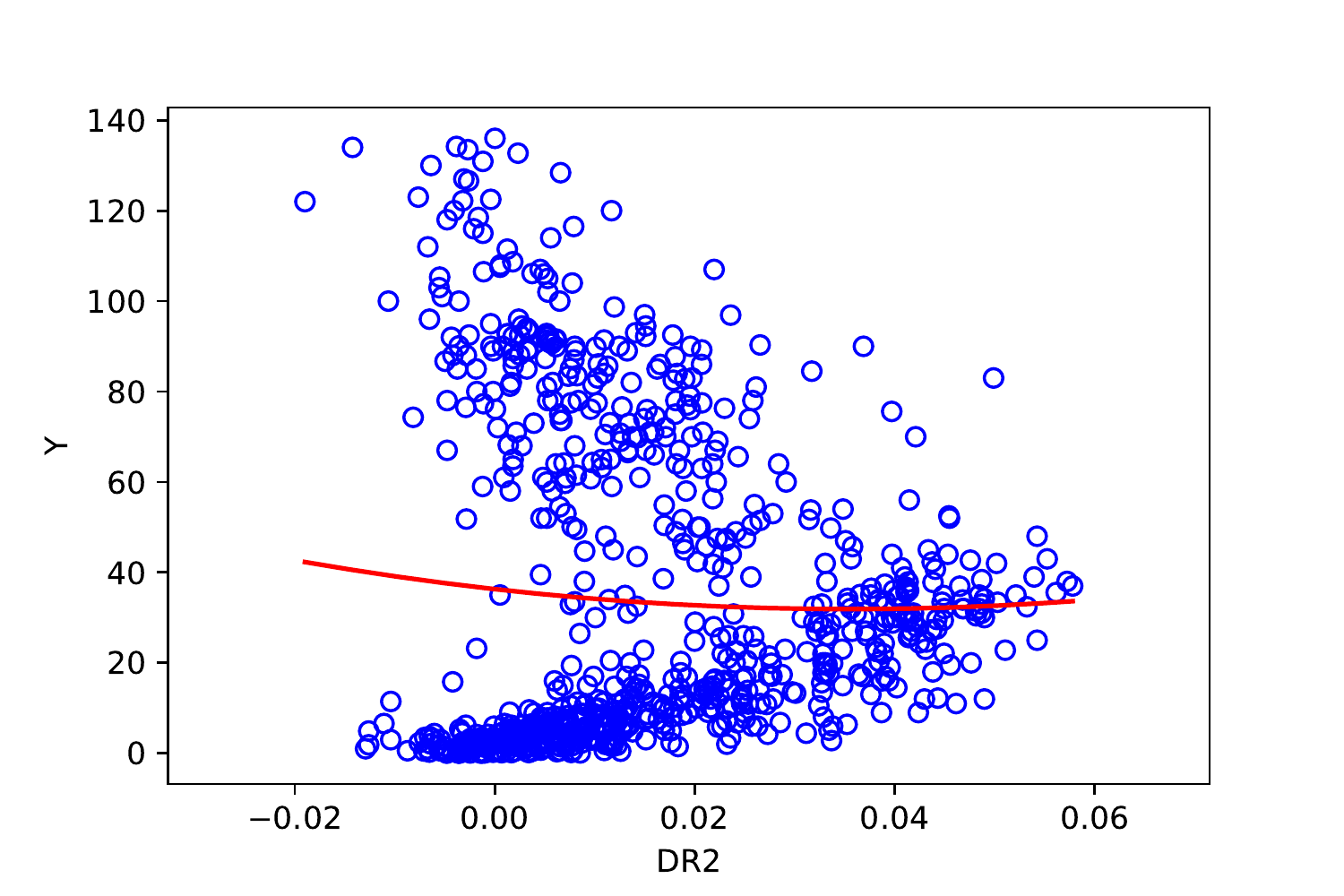}
		\includegraphics[width=0.10\textheight]{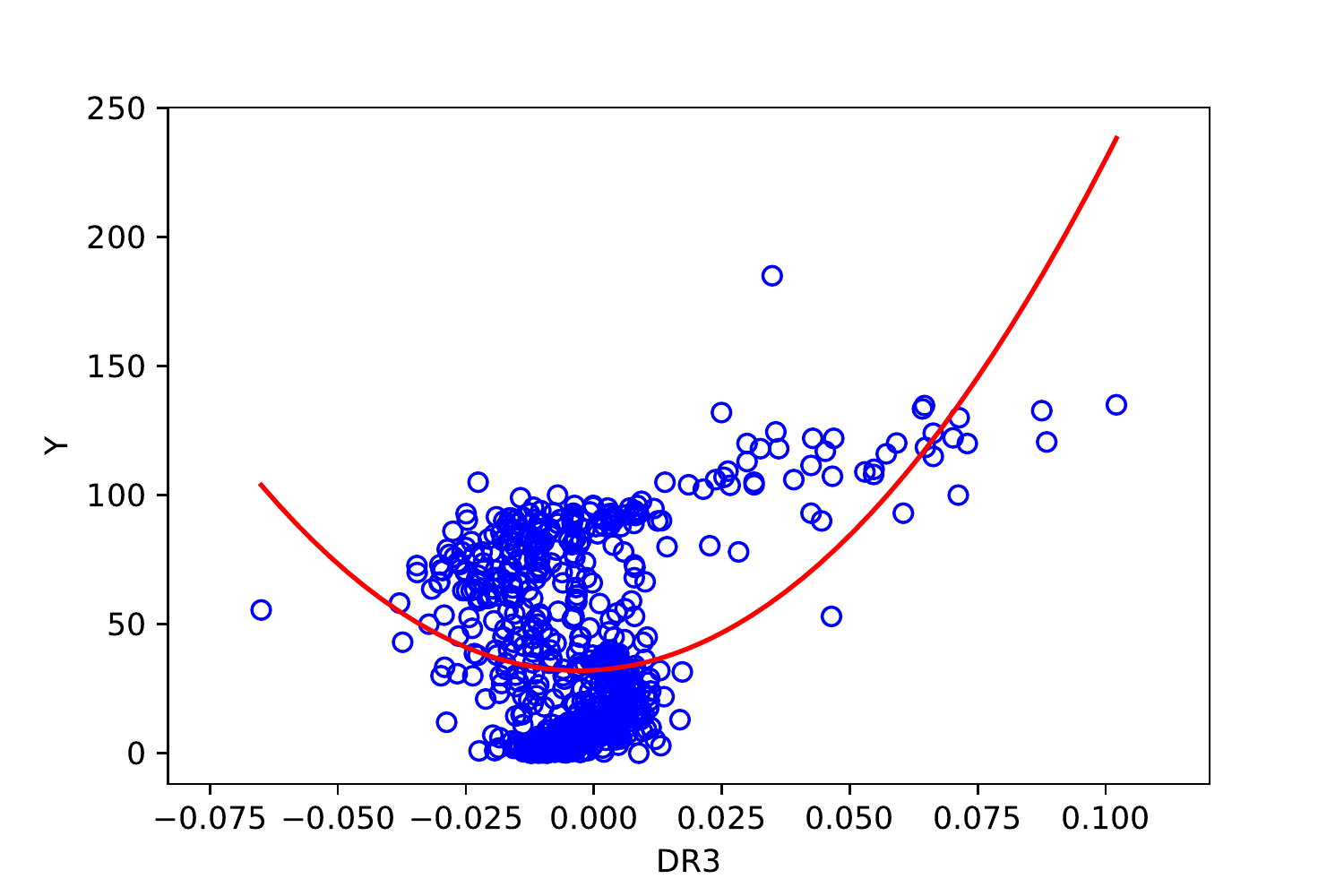} \includegraphics[width=0.10\textheight]{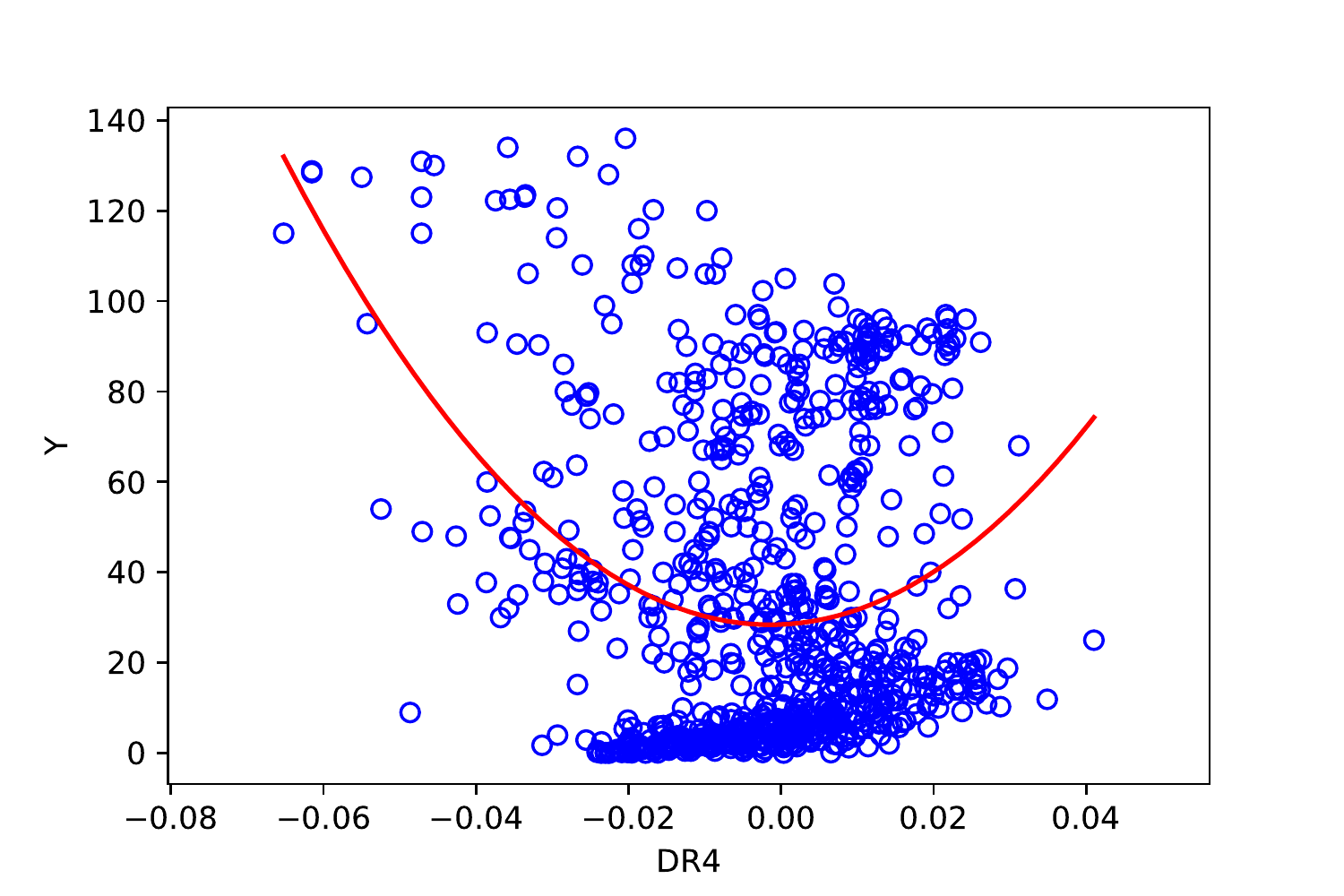}
		\includegraphics[width=0.10\textheight]{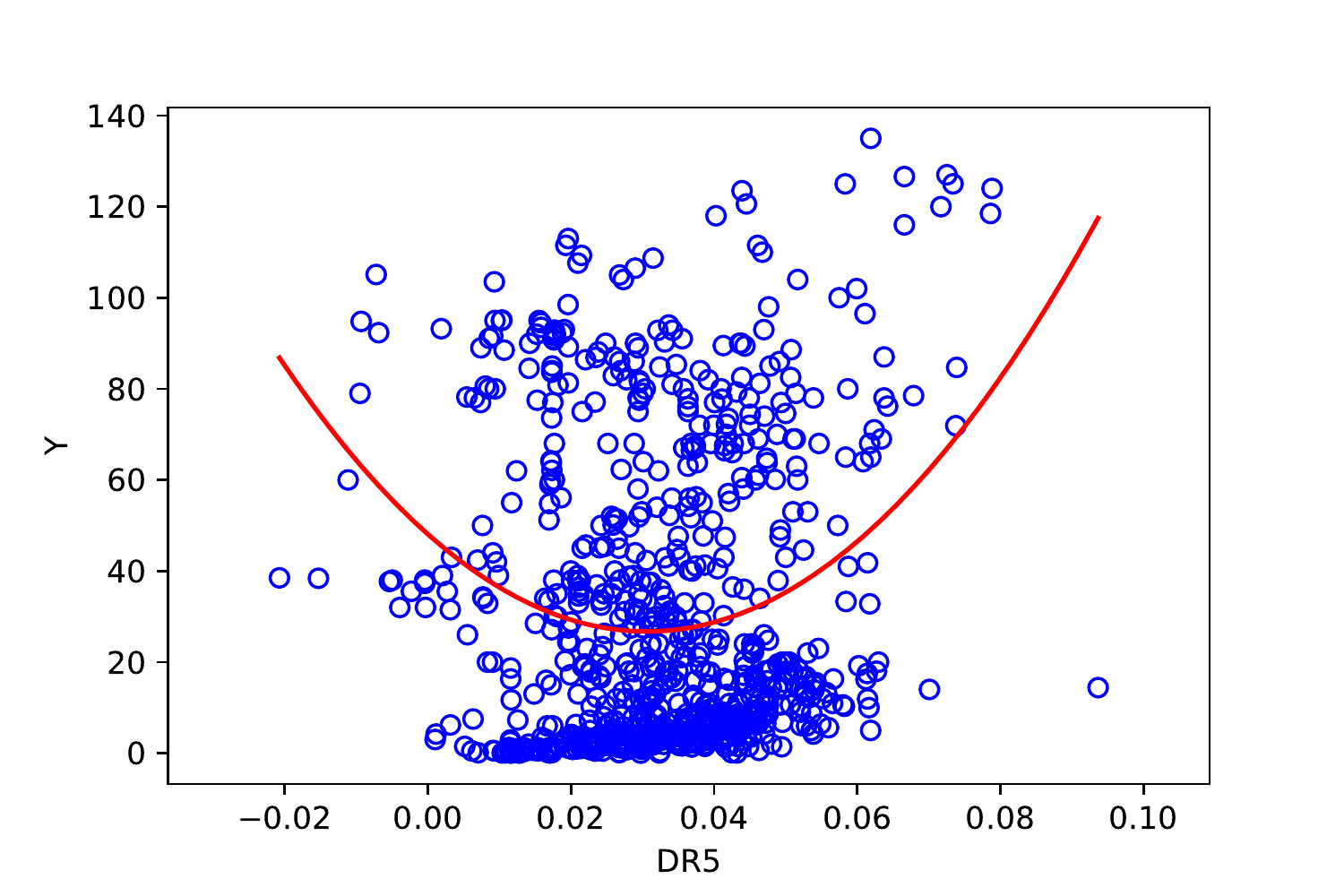}
	\end{minipage}
	\\
	\begin{minipage}[t]{0.9\linewidth}
		\centering
		\rotatebox{90}{\quad\tiny GSAVE}
		\includegraphics[width=0.10\textheight]{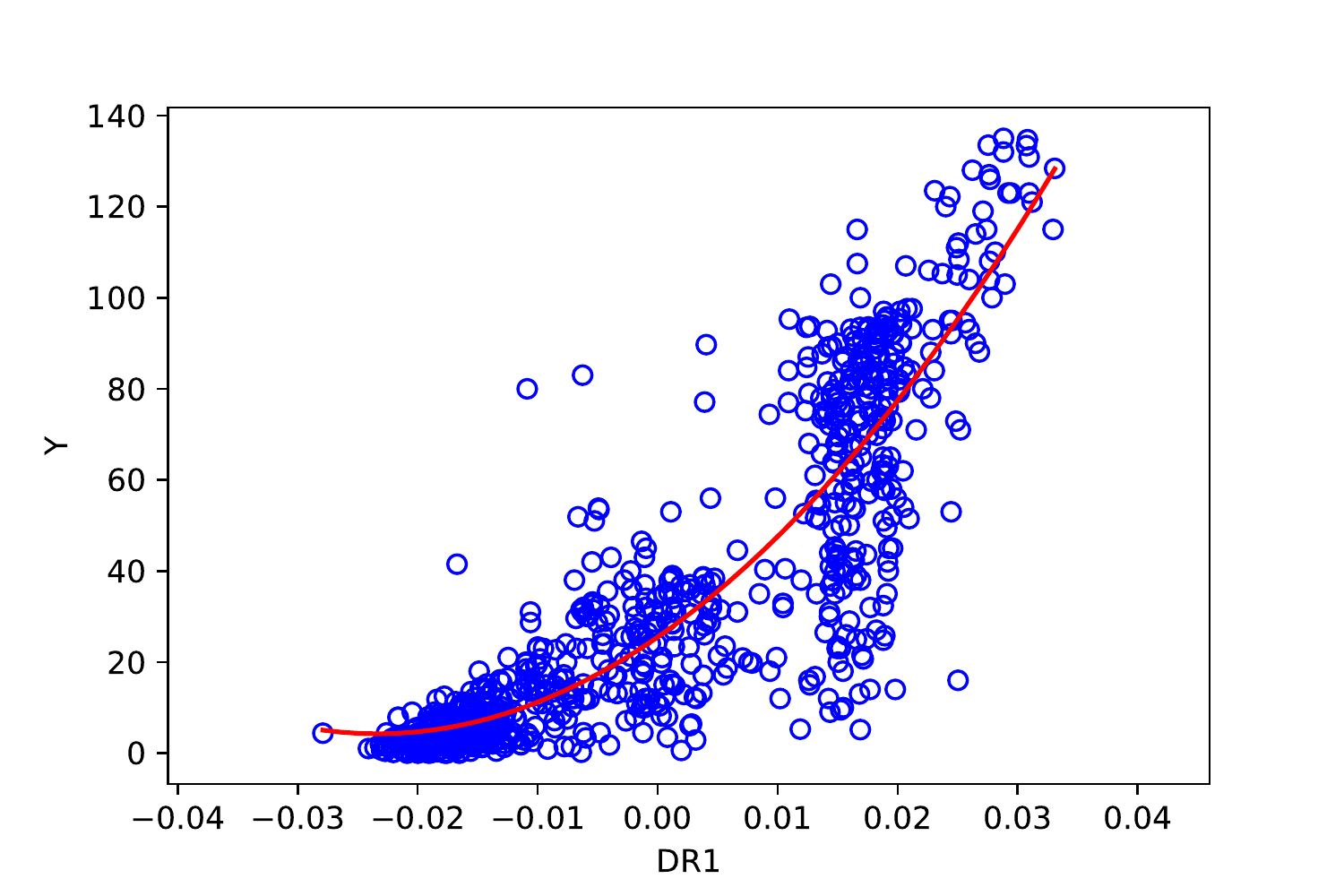} \includegraphics[width=0.10\textheight]{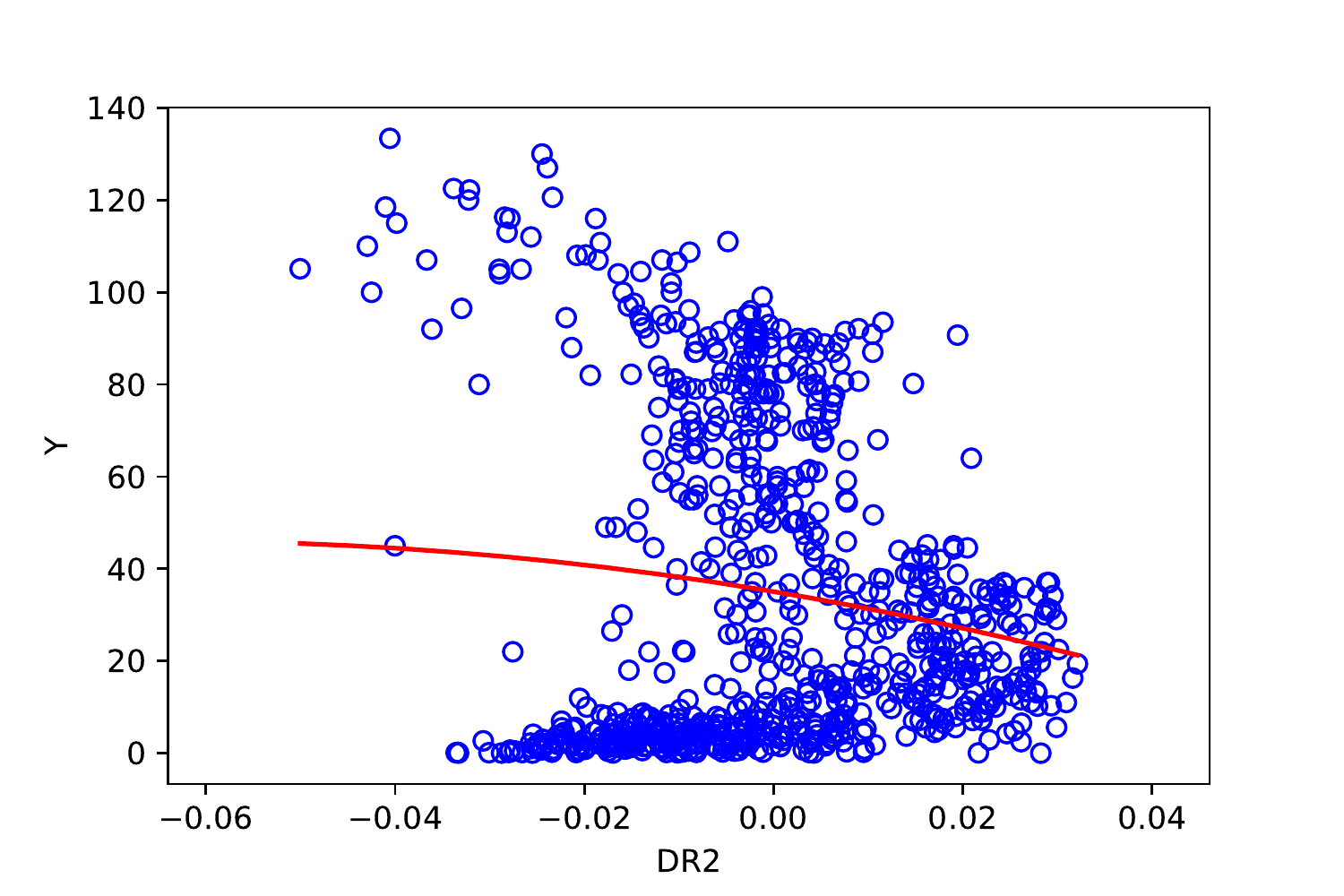}
		\includegraphics[width=0.10\textheight]{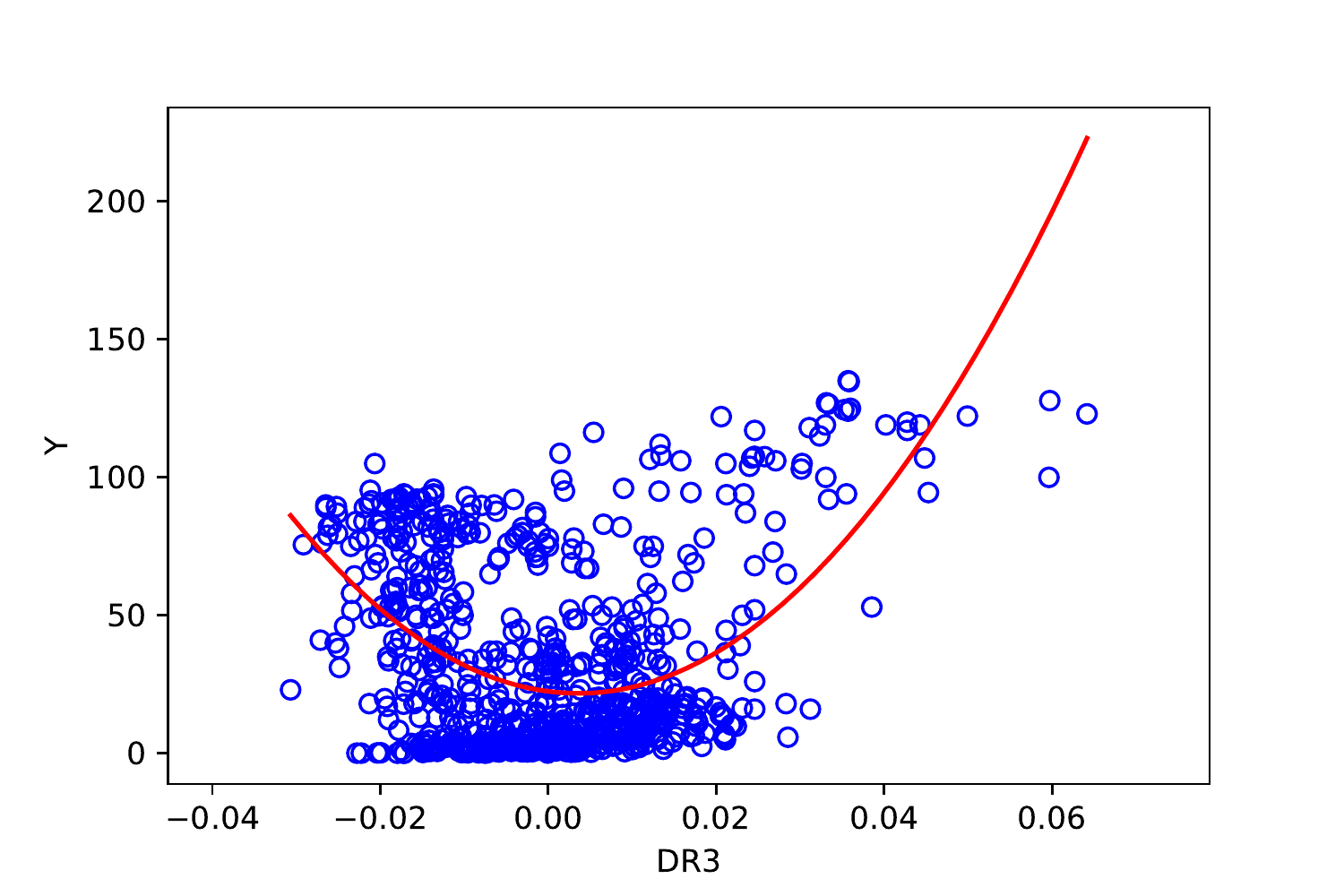} \includegraphics[width=0.10\textheight]{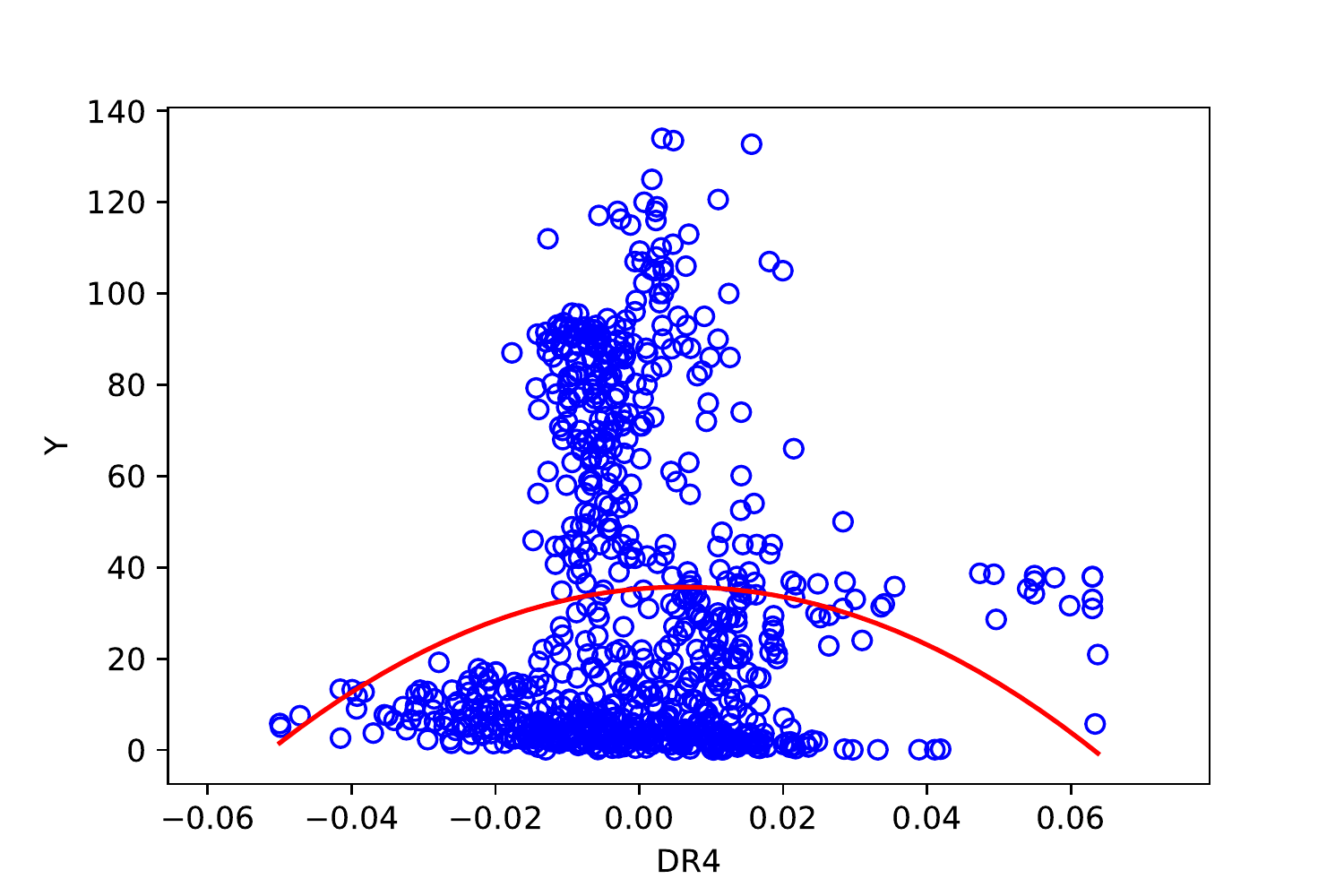}
		\includegraphics[width=0.10\textheight]{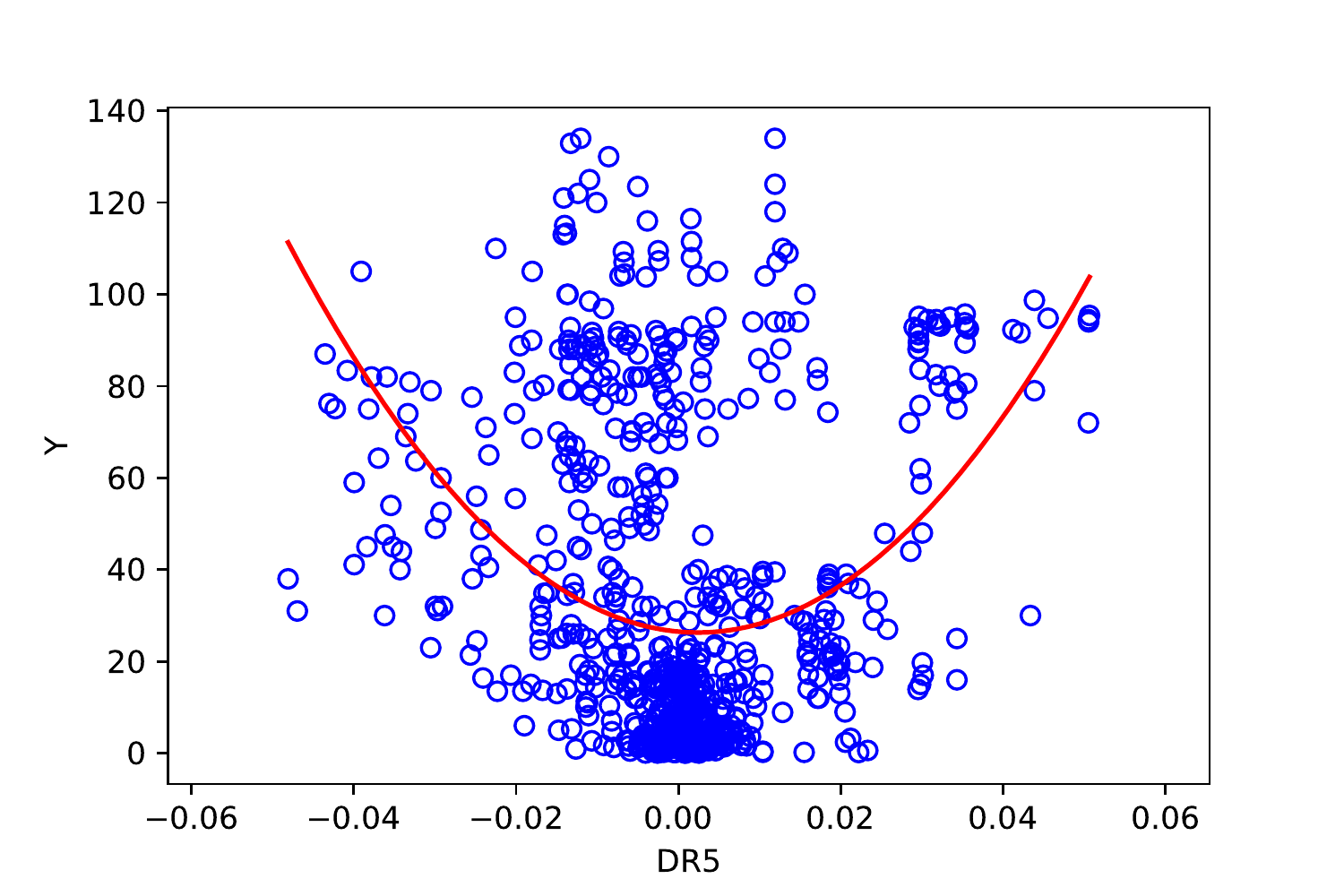}
	\end{minipage}
	\caption{\label{SuperConduct_drawing}The plots of the critical temperature against each component of the learned representation by MSRL, GSIR and GSAVE for the superconductivity dataset.  The red line is a fitting quadratic model.}
\end{figure}

\subsection{The Pole-Telecommunication dataset}\label{rd2}
We next analyze the Pole-Telecommunication dataset  \citep{Weiss1995}, which is available at \textit{https://www.dcc.fc.up.pt/~ltorgo/Regression/DataSets.html}. In this dataset, there are 15,000 instances with $48$ predictors and a continuous distance-based response. Among the $48$ predictors, $22$ of them are non-informative,  meaning that they have the same value for all subjects. So we delete all these non-informative features.  Other implementation details are similar to those in section \ref{rd1}, see Section \ref{Descrip_Implement_Real} in the supplementary material, except that we take $d_0=2,5,10,20$ in this example.  The distance correlation, average prediction errors, and their standard errors for all methods are reported in Table \ref{Pole}.  One can see that MSRL produces better results.  The high empirical DC is supporting evidence of the high quality of the learned representation.  Different from the previous example, in this example, GSIR and GSAVE outperform the linear SDR methods and GSIR performs the best  among the four SDR competitors.

\begin{table}
	\caption{\label{Pole} Distance correlation (DC), average prediction errors (APE), and their standard errors (based on five-fold validation and best hyperparameters selected by validation method) on Pole-Telecommunication dataset}
	\centering
	\resizebox{\textwidth}{!}{
	\begin{tabular}{c|c @{\hspace{0.1cm}}c|c @{\hspace{0.1cm}} c|c @{\hspace{0.1cm}} c|c @{\hspace{0.1cm}} l}
		\hline
		\hline
		d        & \multicolumn{2}{c}{2}  & \multicolumn{2}{c}{5}  & \multicolumn{2}{c}{10} & \multicolumn{2}{c}{20}  \\
		\hline
		criterion & DC       & APE         & DC       & APE         & DC       & APE         & DC       & APE          \\
		\hline
		MSRL    & \textbf{.94}(.02) & \textbf{9.21}(1.87)  & \textbf{.97}(.00) & \textbf{5.43}(1.05)  & \textbf{.96}(.01) & \textbf{4.69}(0.84)  & \textbf{.96}(.01) & \textbf{4.18}(0.79)   \\
		SIR                           & .33(.02) & 30.52(0.26) & .21(.02) & 30.54(0.26) & .15(.01) & 30.53(0.27) & .10(.02) & 30.52(0.26)  \\
		SAVE                          & .00(.00) & 41.75(0.42) & .00(.00) & 41.81(0.41) & .00(.00) & 41.87(0.45) & .10(.01) & 53.77(11.99) \\
		GSIR                          & .61(.03) & 15.24(0.23) & .37(.06) & 15.14(0.20) & .26(.06) & 15.15(0.20) & .26(.06) & 15.15(0.21)  \\
		GSAVE                         & .58(.02) & 19.93(0.30) & .40(.05) & 19.91(0.31) & .25(.02) & 19.91(0.31) & .17(.01) & 19.90(0.32) \\
		\hline
	\end{tabular}}
\end{table}

To visualize the predictive power of the learned representations, we classify the data into two classes: one has the response less than 50 and the rest constitutes another class.  In Figure \ref{Poldata_drawing}, we plot the transformed data based on each pair of the components of the learned representation with $d_0=5$.  From Figure \ref{Poldata_drawing}, we see that each pair of features from the MSRL estimate are informative for this classification task, and among other SDR methods, only the first dimension reduction directions of GSIR and GSAVE are predictive for the classification. The plots for SIR and SAVE are included in the Section \ref{extra-plots} in the supplementary material.
This indicates that in this example the representation learned by MSRL has more predictive power.
\begin{figure}[htbp]
\centering
\begin{minipage}[t]{0.9\linewidth}
	\parbox[c][0.5cm]{0.18cm}{\rotatebox{90}{\tiny MSRL}}
	\begin{minipage}[t]{\linewidth}
		\includegraphics[width=0.10\textheight]{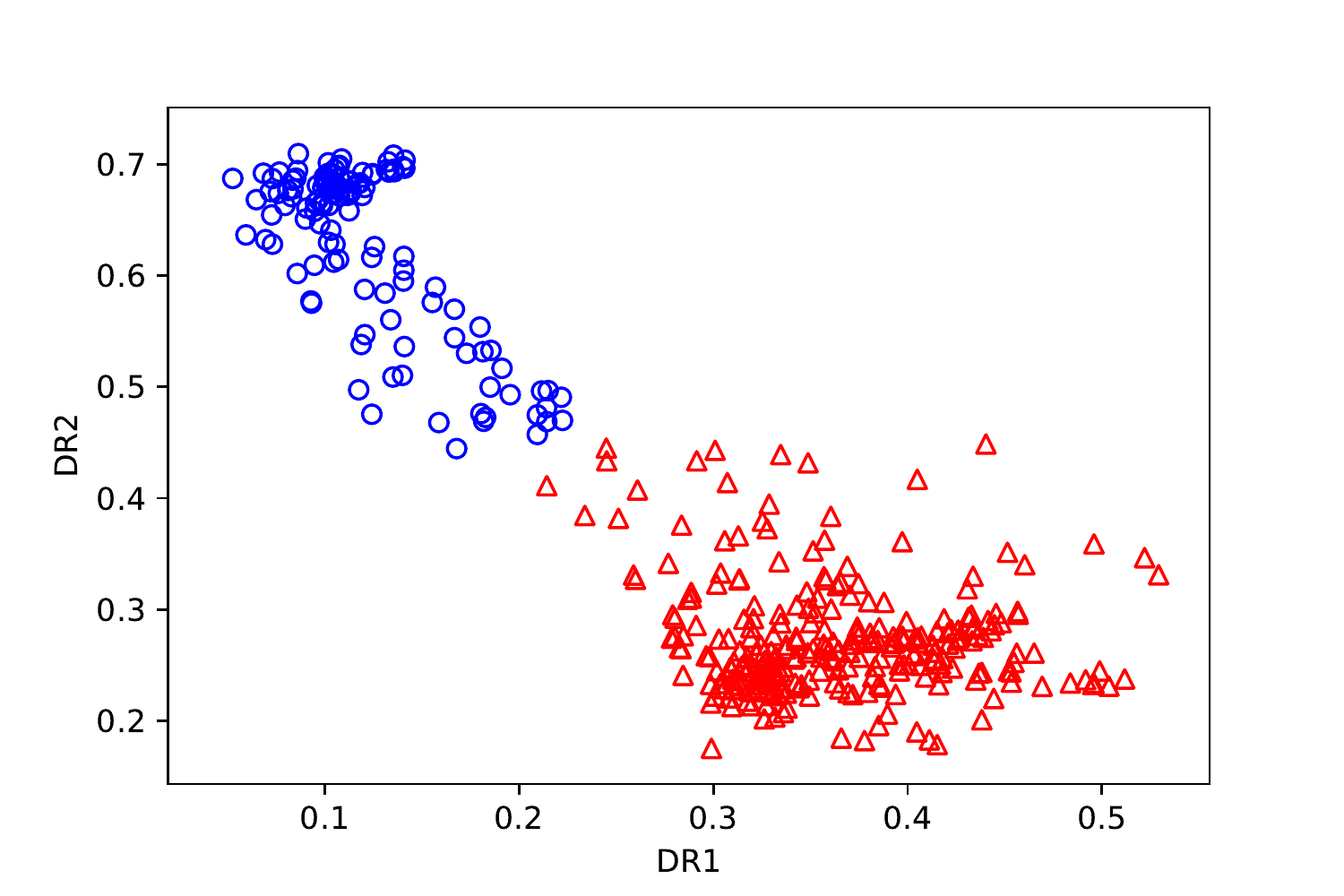}
		\includegraphics[width=0.10\textheight]{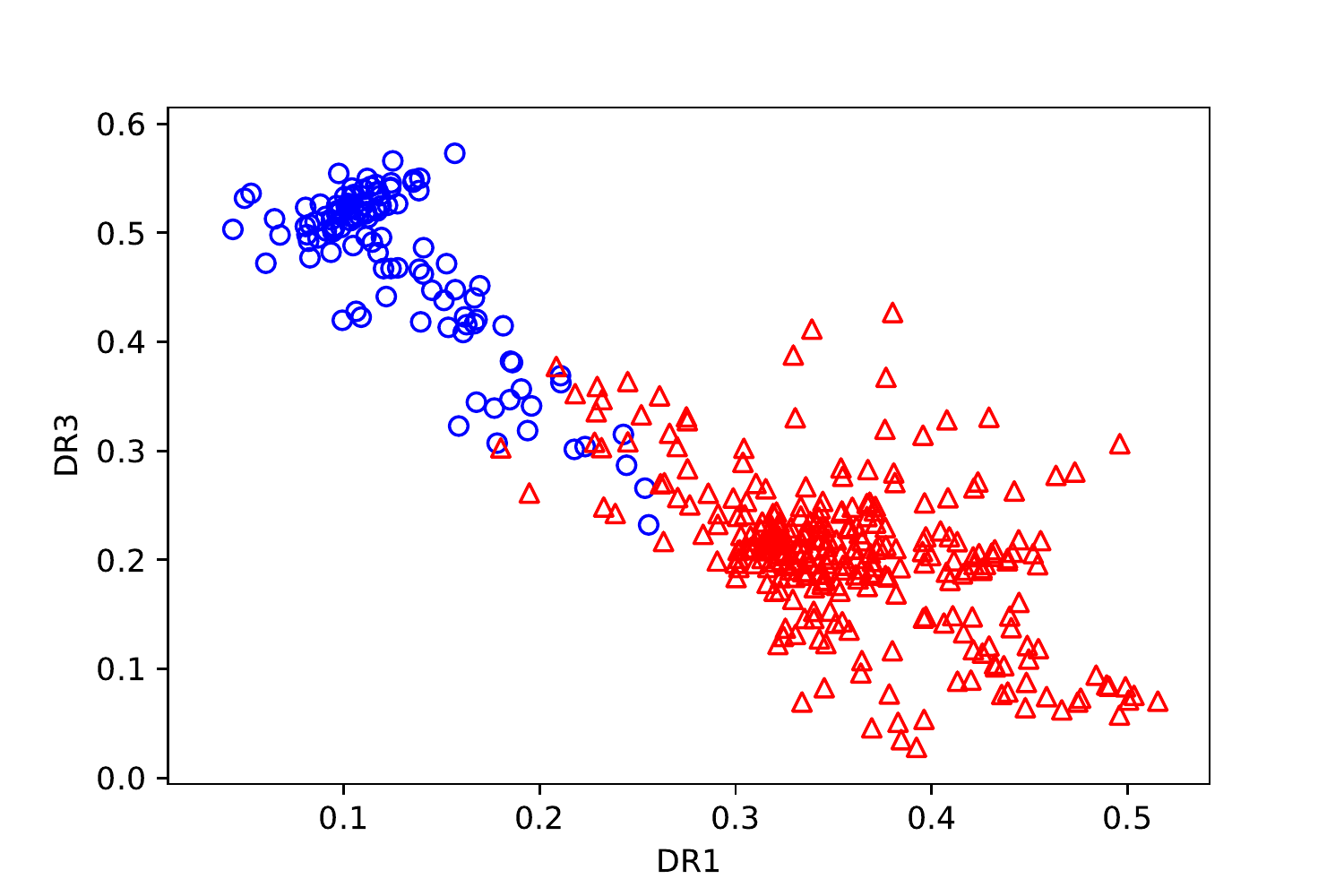}
		\includegraphics[width=0.10\textheight]{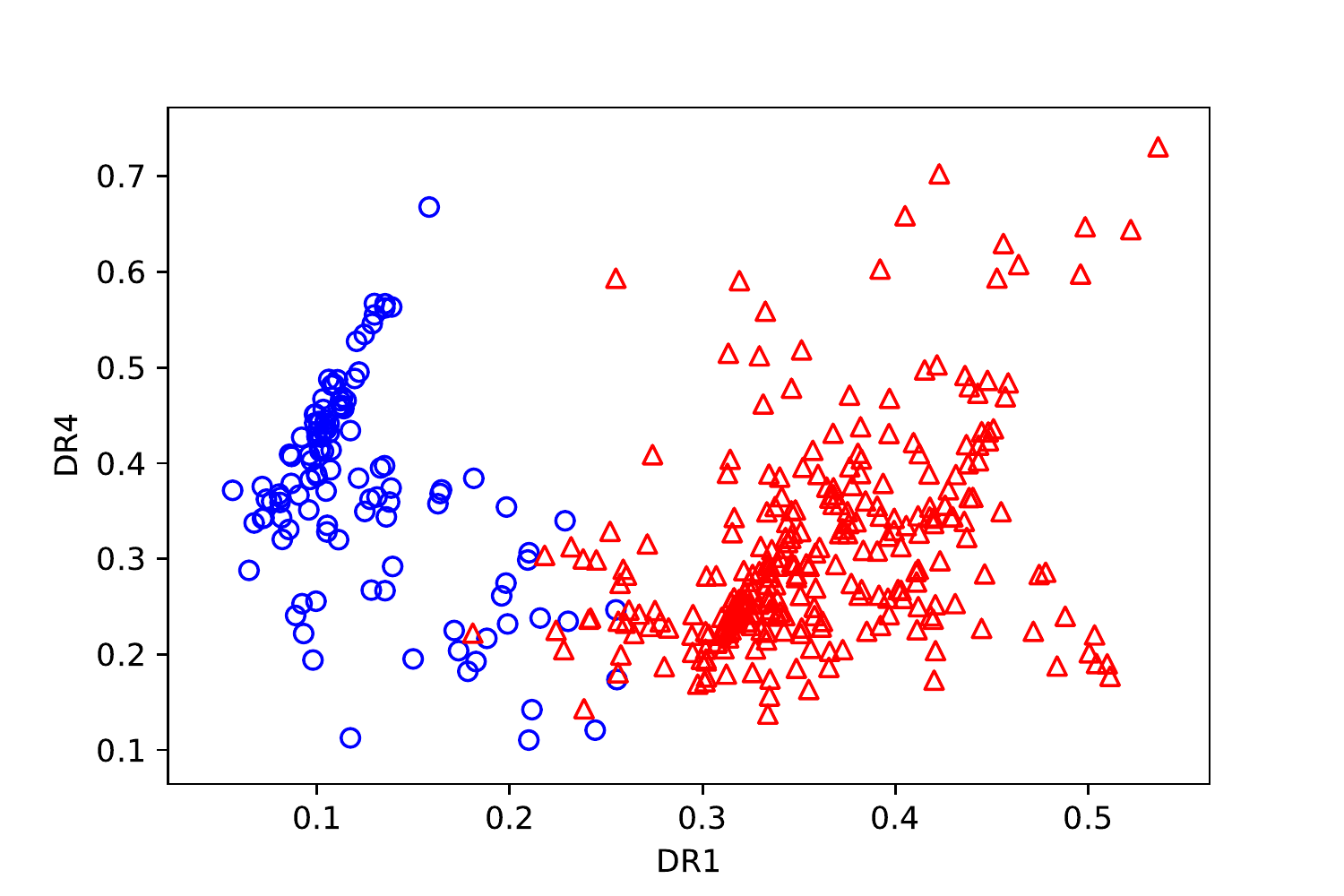}
		\includegraphics[width=0.10\textheight]{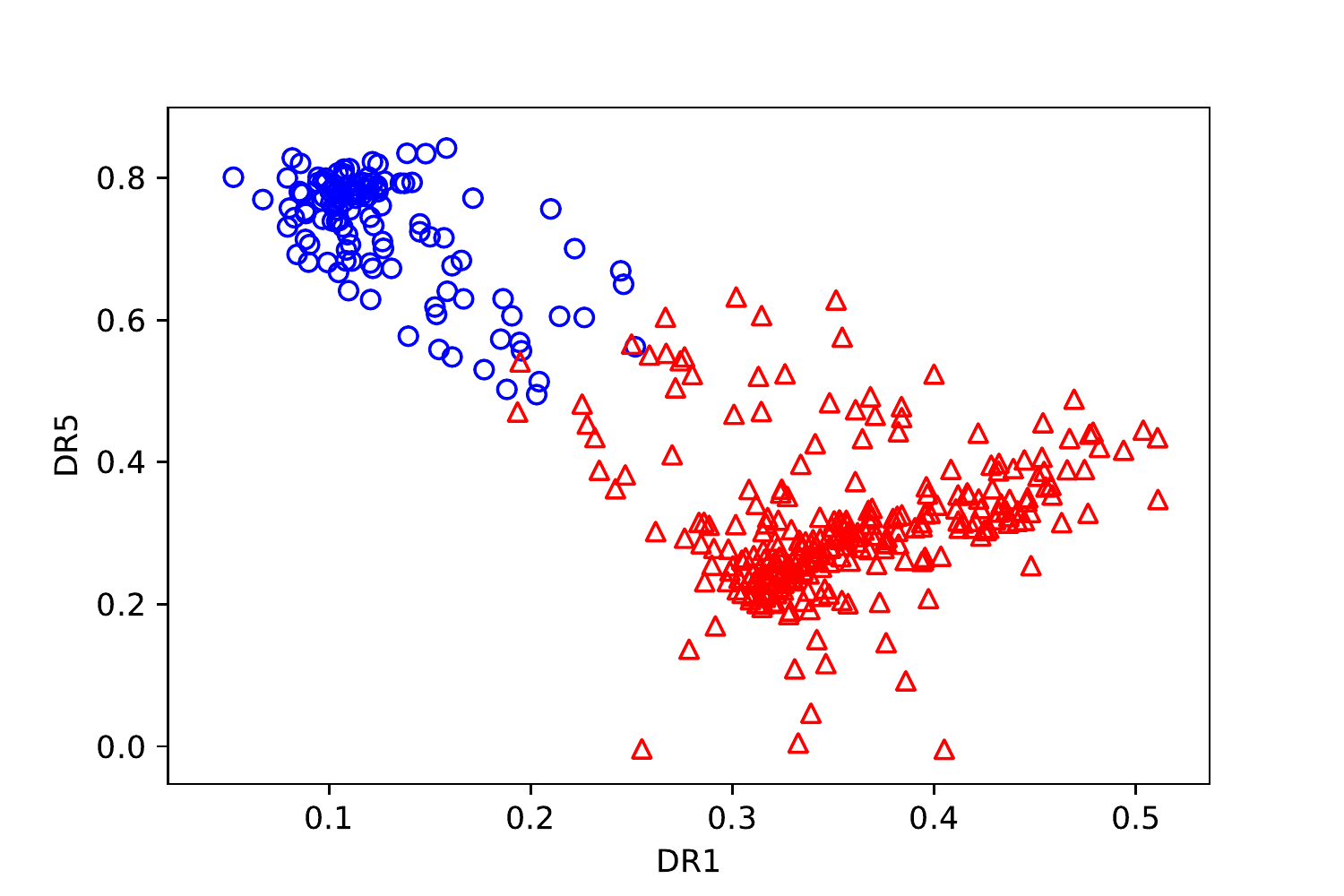}
		\includegraphics[width=0.10\textheight]{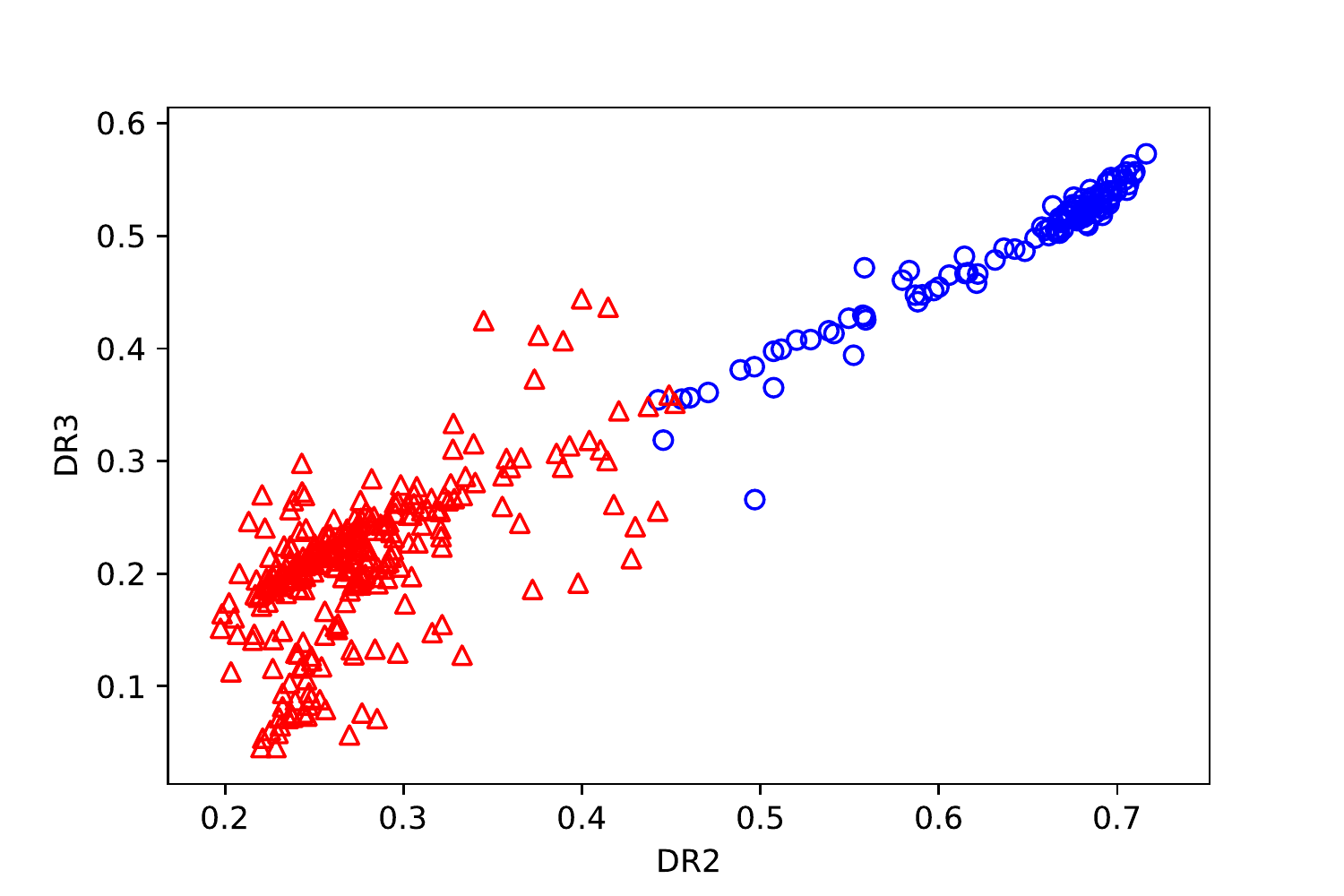}\\
		\includegraphics[width=0.10\textheight]{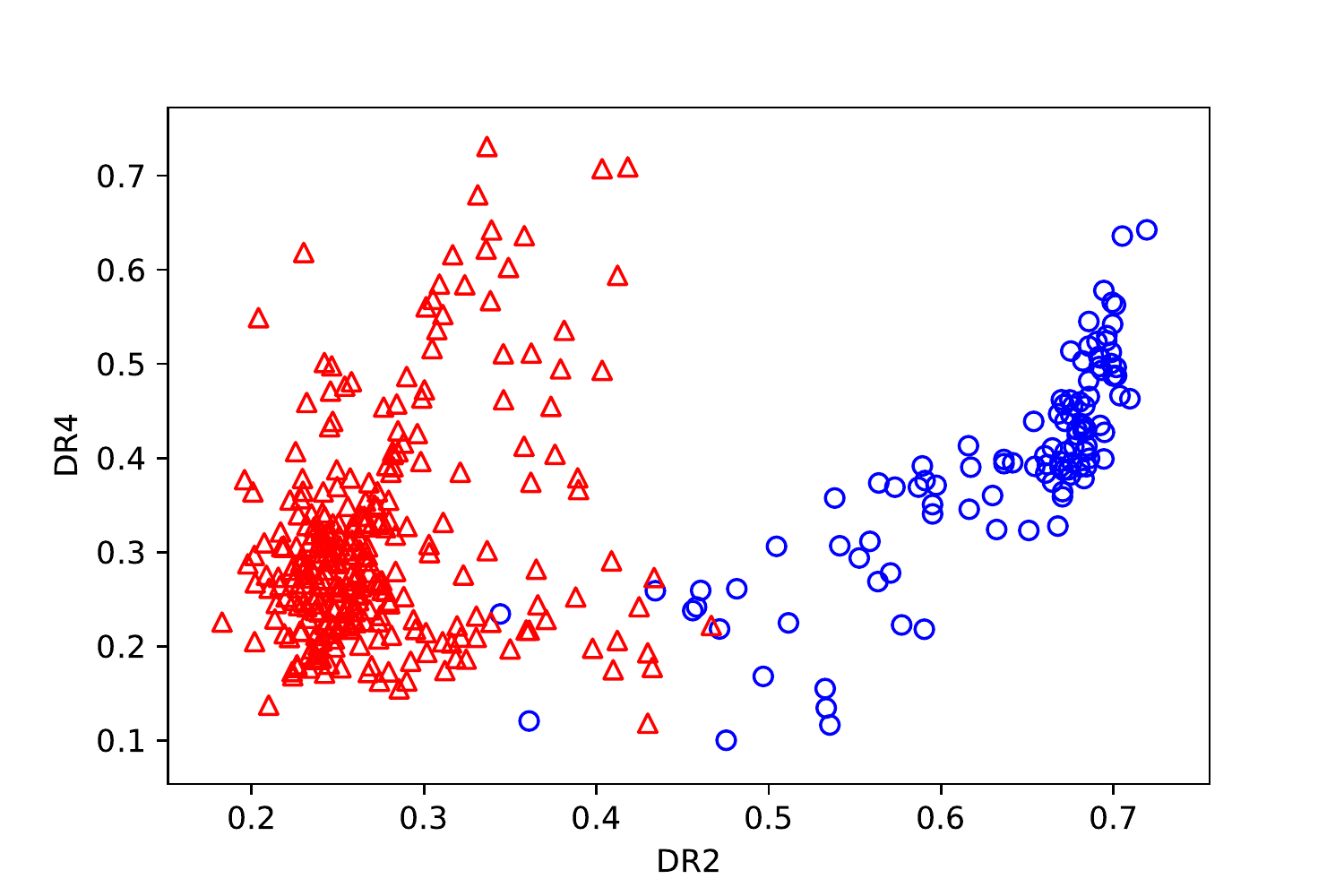}
		\includegraphics[width=0.10\textheight]{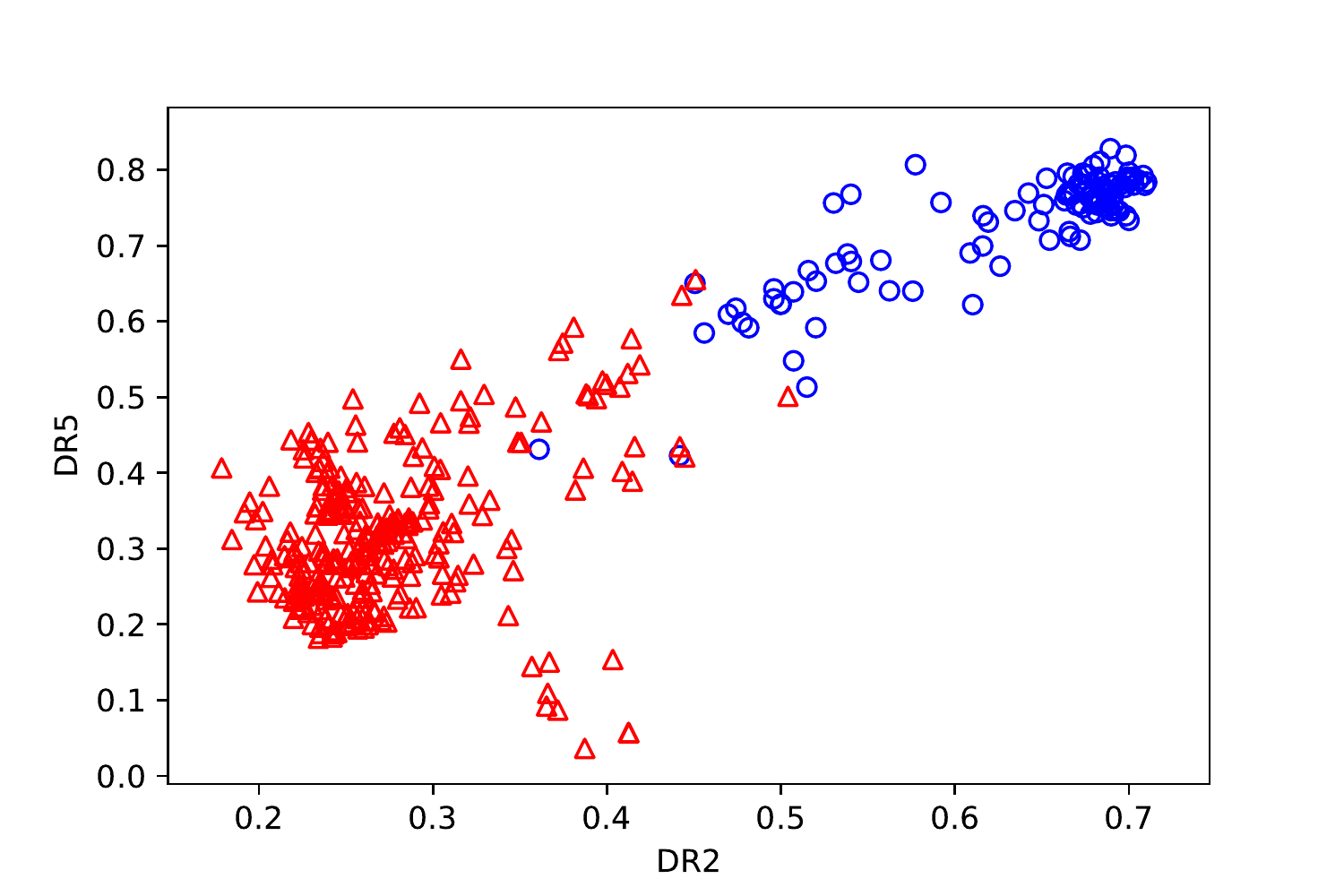}
		\includegraphics[width=0.10\textheight]{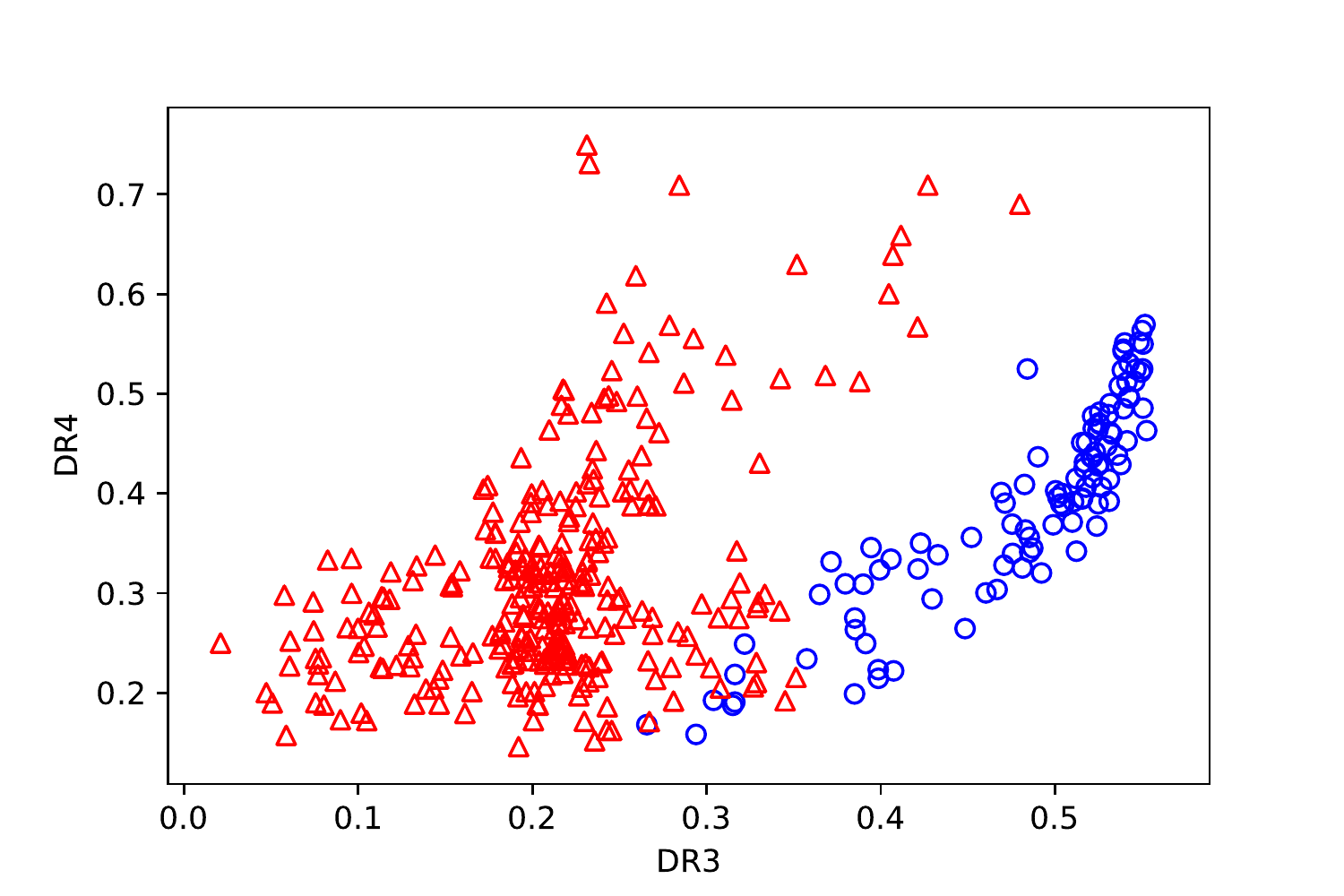}
		\includegraphics[width=0.10\textheight]{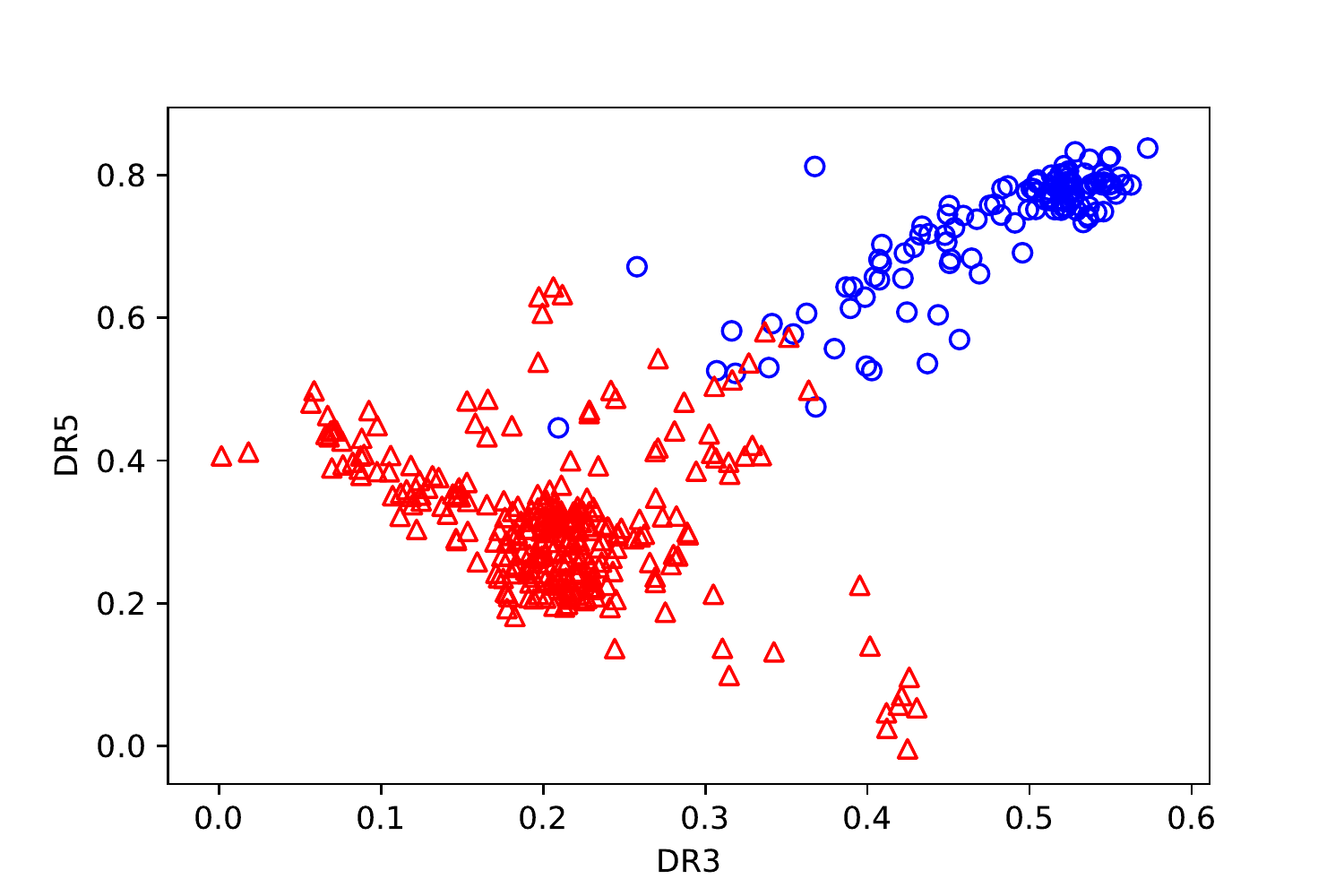}
		\includegraphics[width=0.10\textheight]{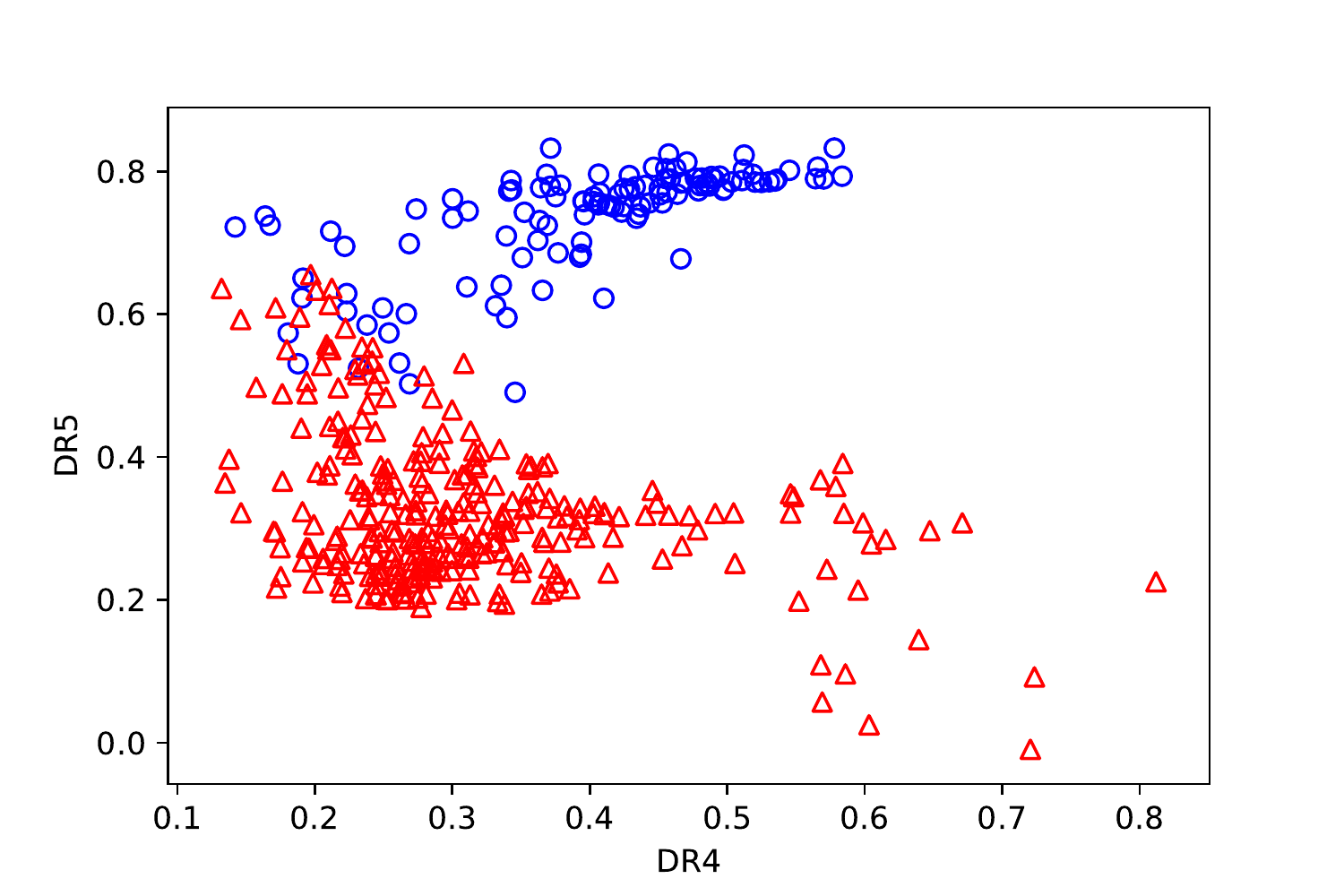}
	\end{minipage}
\end{minipage}	
	\\
	\begin{minipage}[t]{0.9\linewidth}
		\parbox[c][0.5cm]{0.18cm}{\rotatebox{90}{\tiny GSIR}}
		\begin{minipage}[t]{\linewidth}
			\includegraphics[width=0.10\textheight]{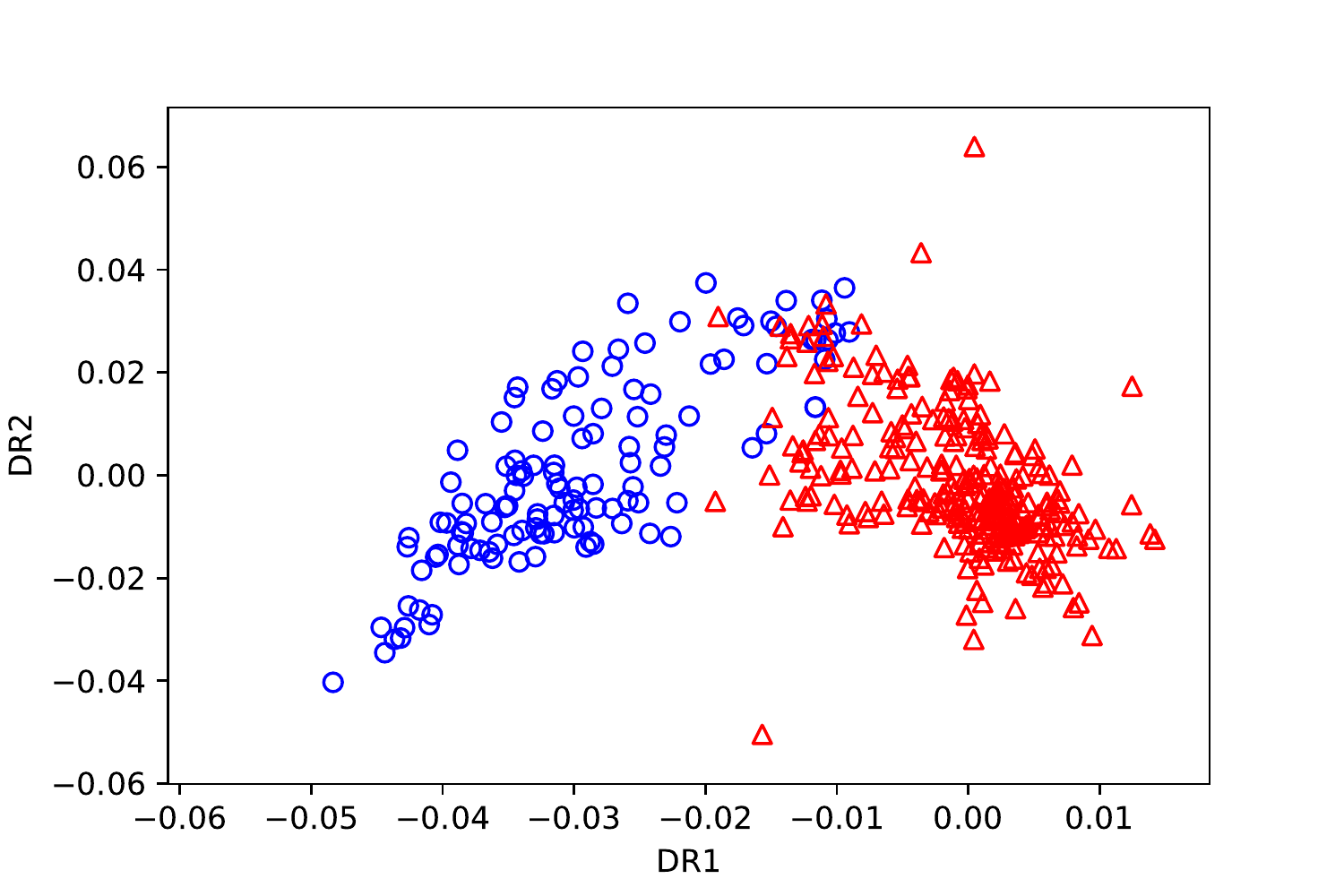}
			\includegraphics[width=0.10\textheight]{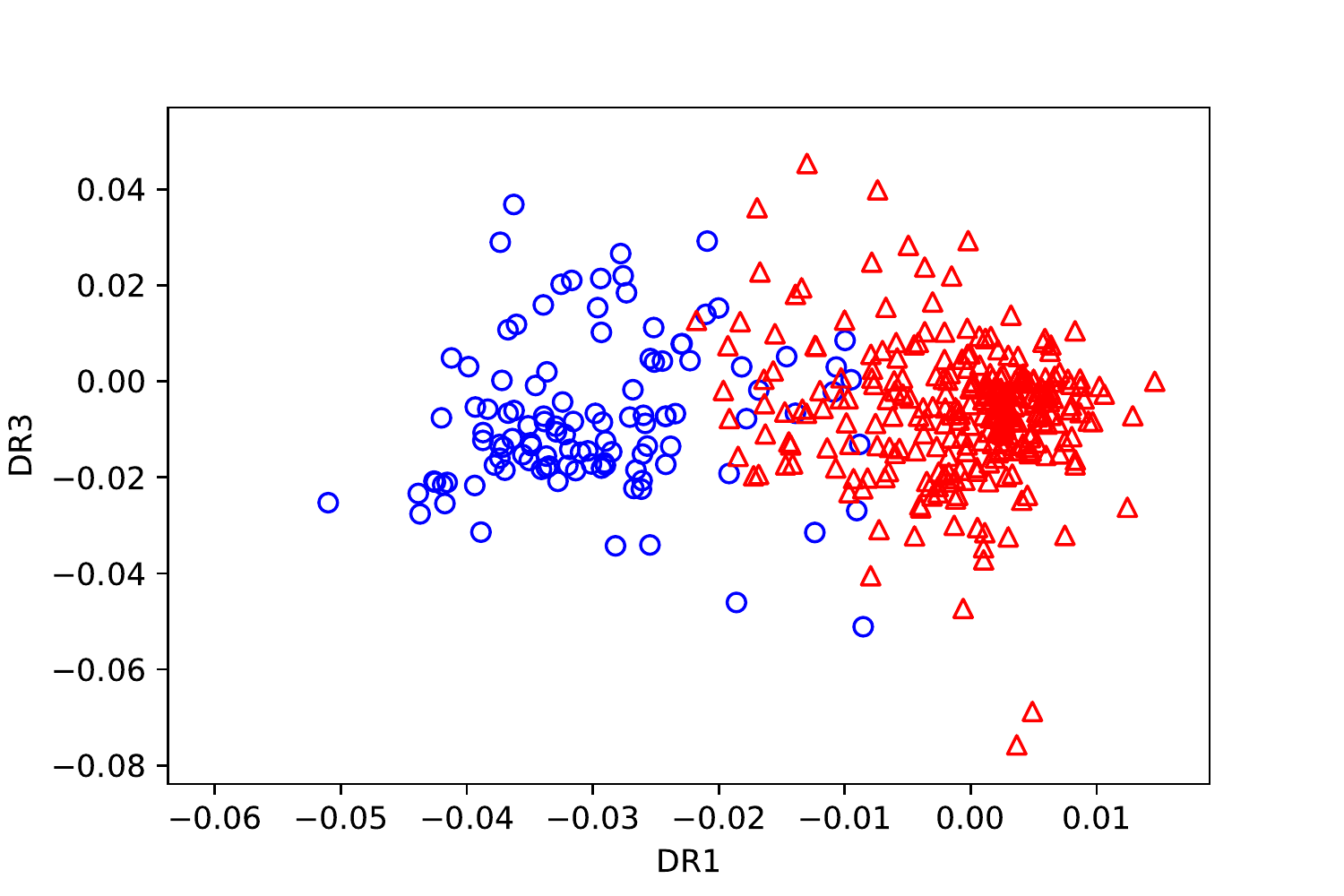}
			\includegraphics[width=0.10\textheight]{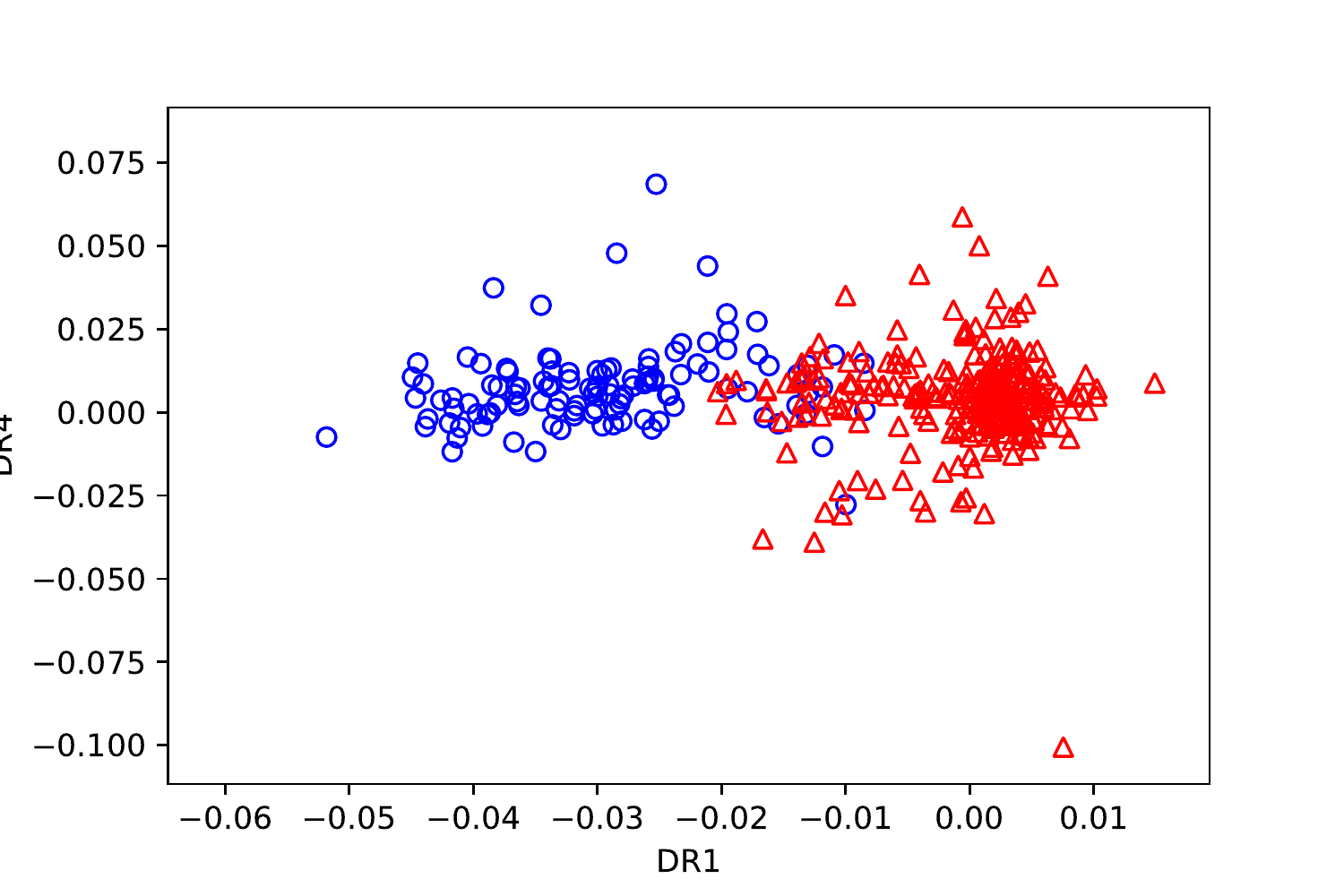}
			\includegraphics[width=0.10\textheight]{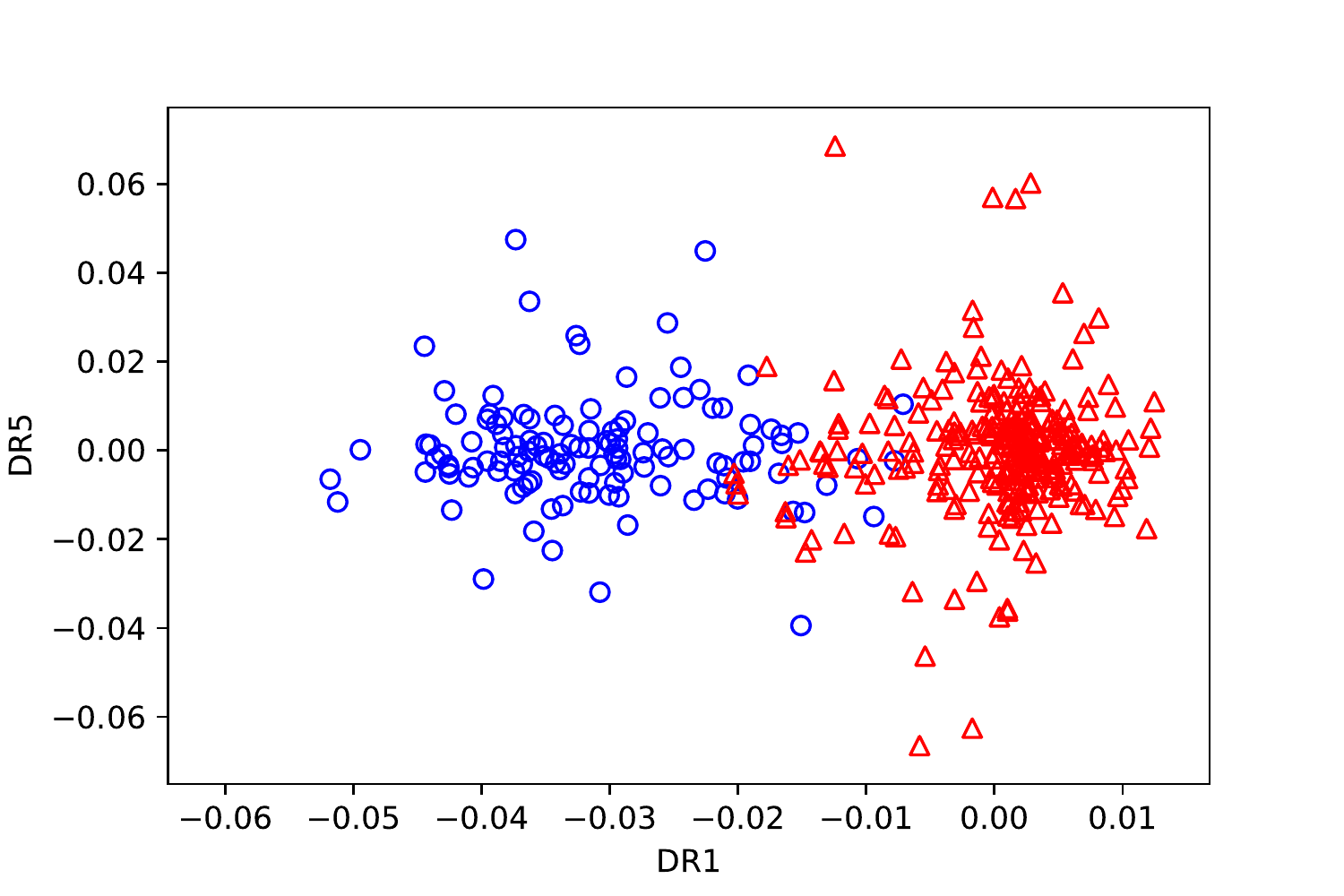}
			\includegraphics[width=0.10\textheight]{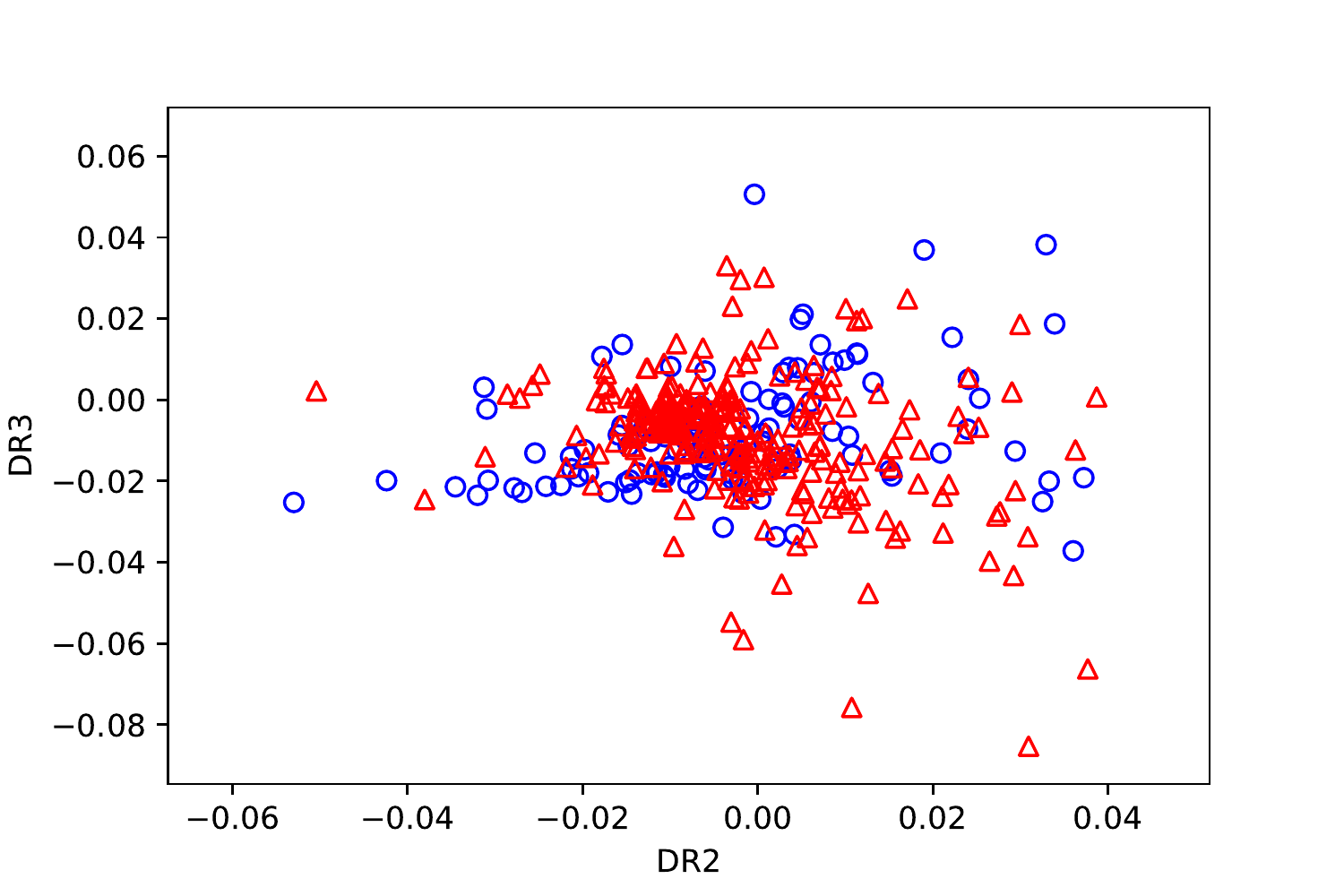}\\
			\includegraphics[width=0.10\textheight]{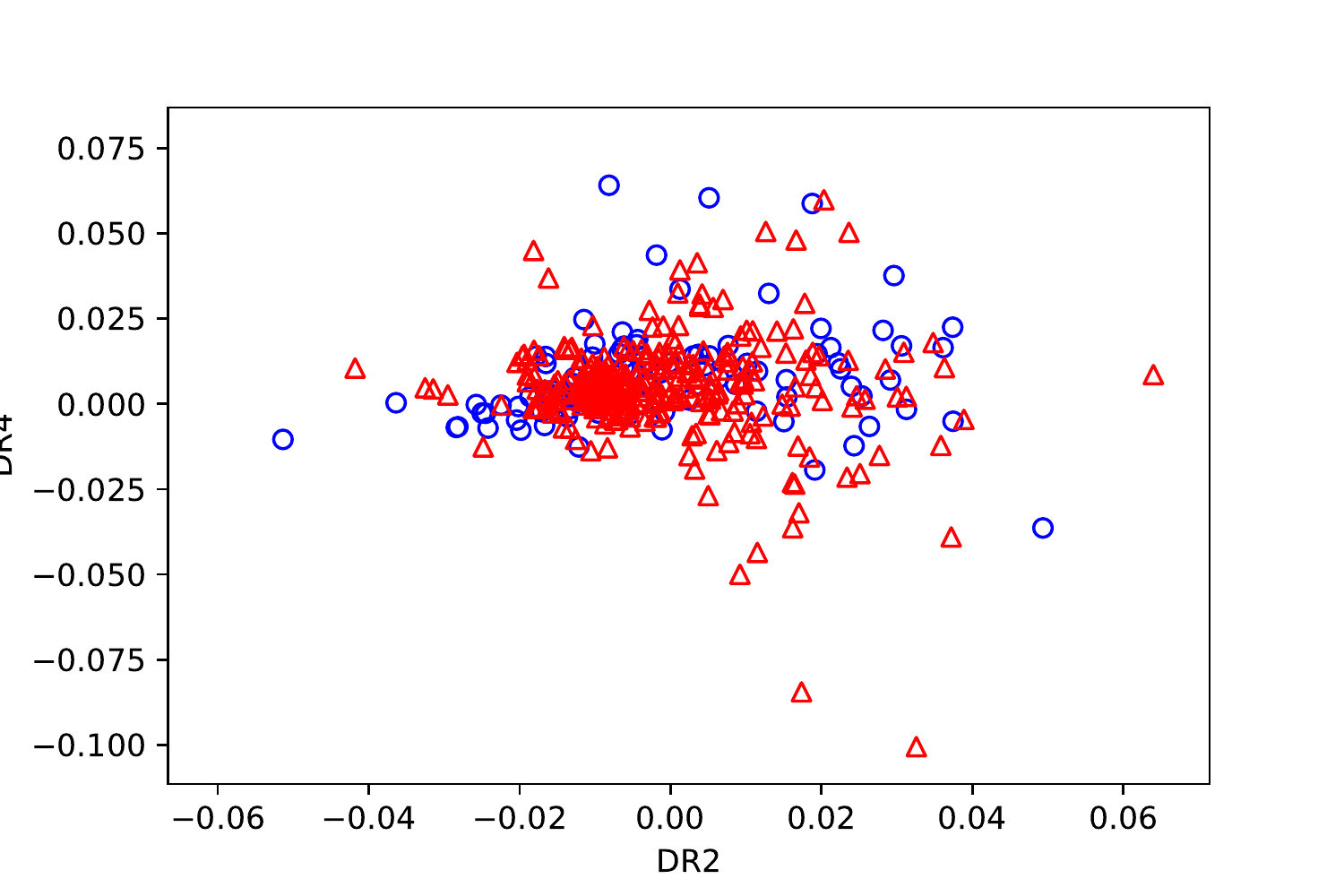}
			\includegraphics[width=0.10\textheight]{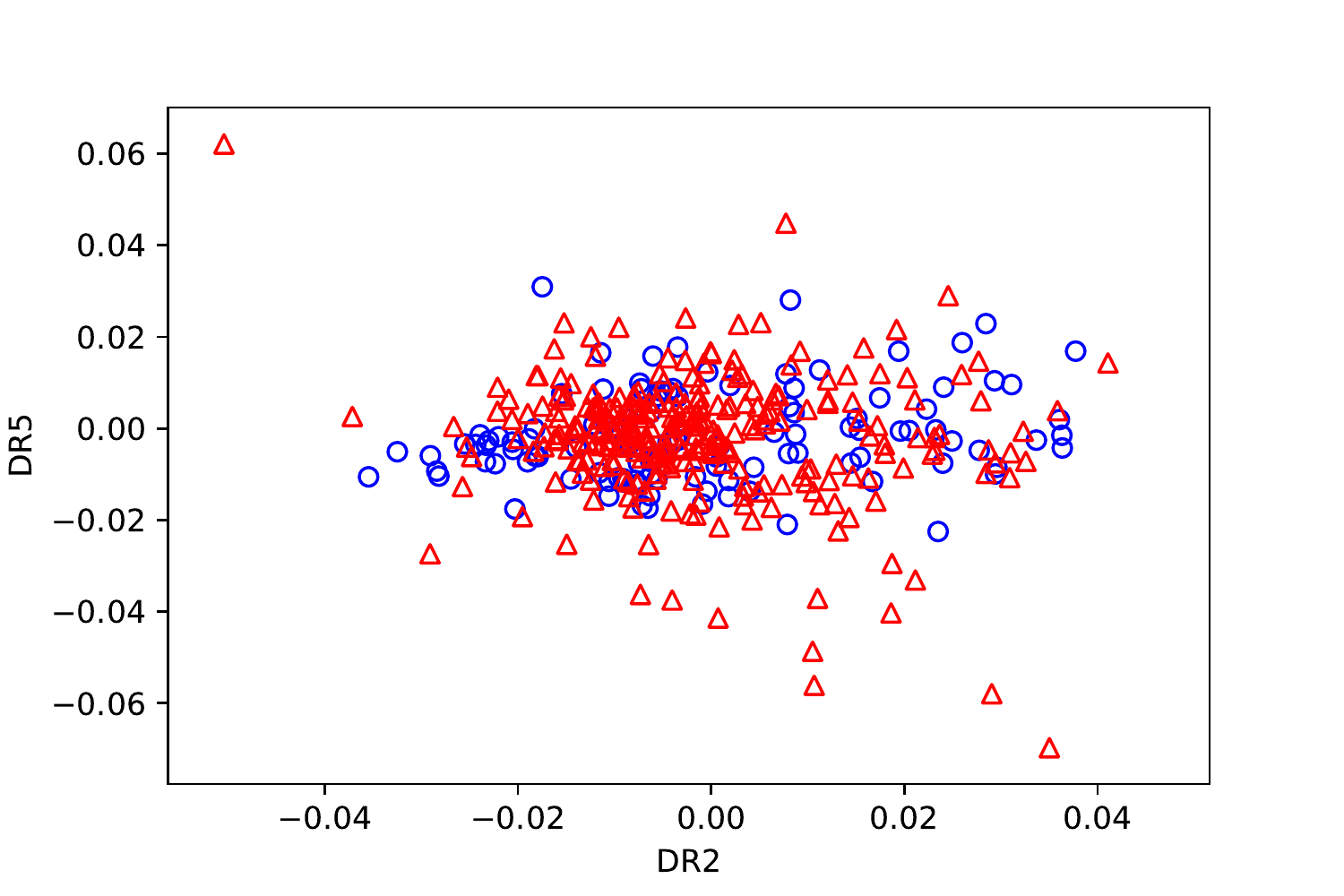}
			\includegraphics[width=0.10\textheight]{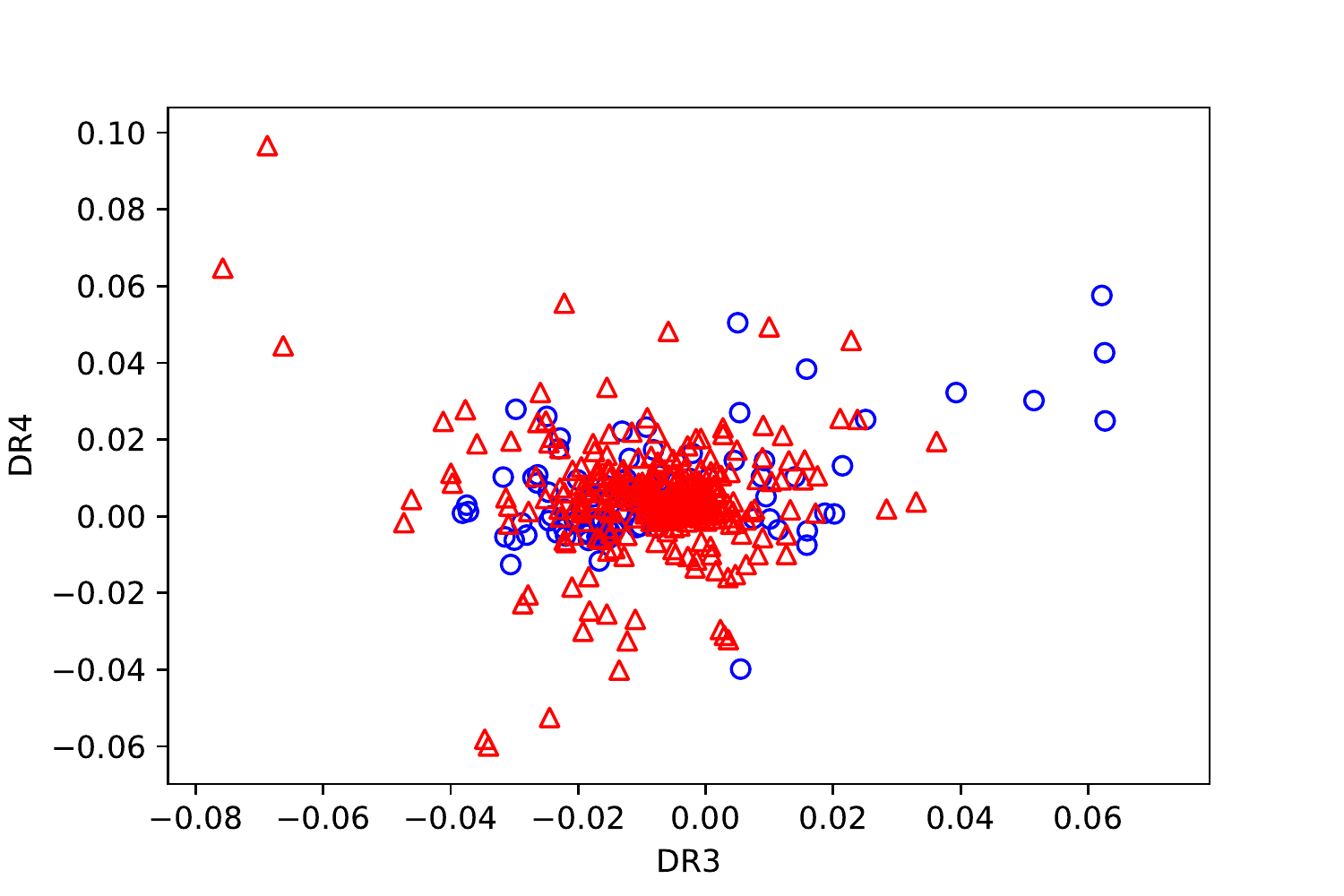}
			\includegraphics[width=0.10\textheight]{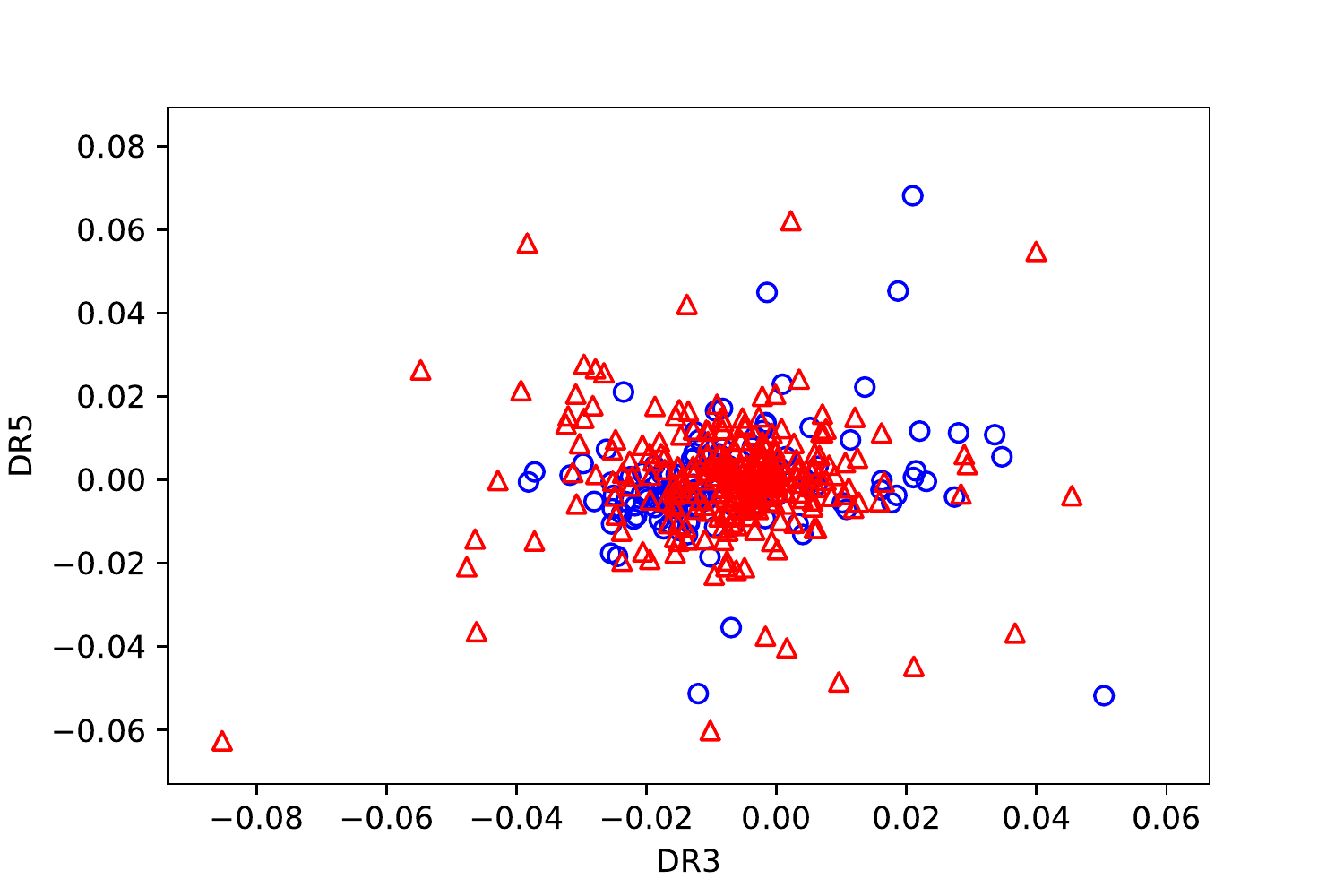}
			\includegraphics[width=0.10\textheight]{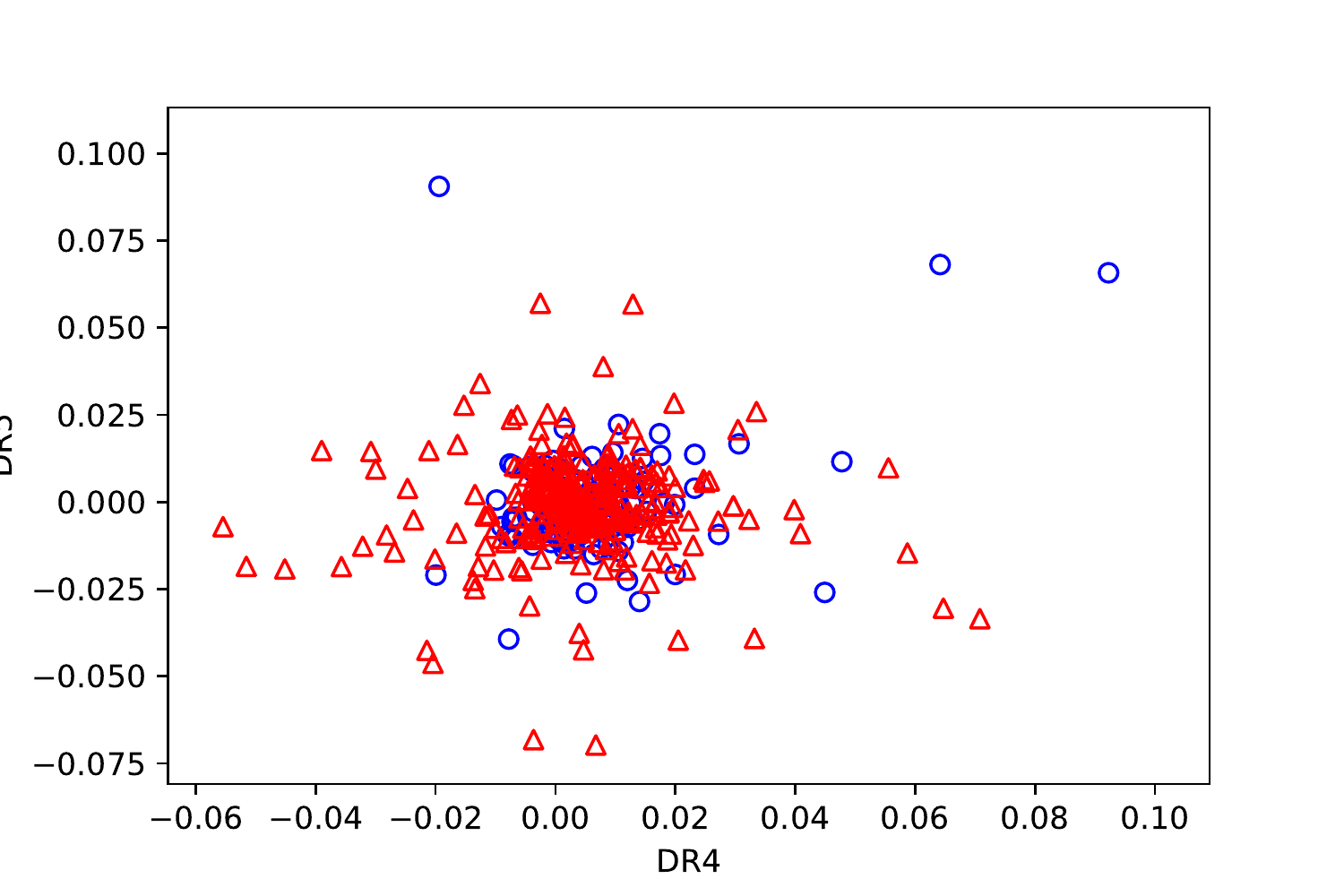}
		\end{minipage}
	\end{minipage}
	\\
	\begin{minipage}[t]{0.9\linewidth}
		\parbox[c][0.5cm]{0.18cm}{\rotatebox{90}{\tiny GSAVE}}
		\begin{minipage}[t]{\linewidth}
			\includegraphics[width=0.10\textheight]{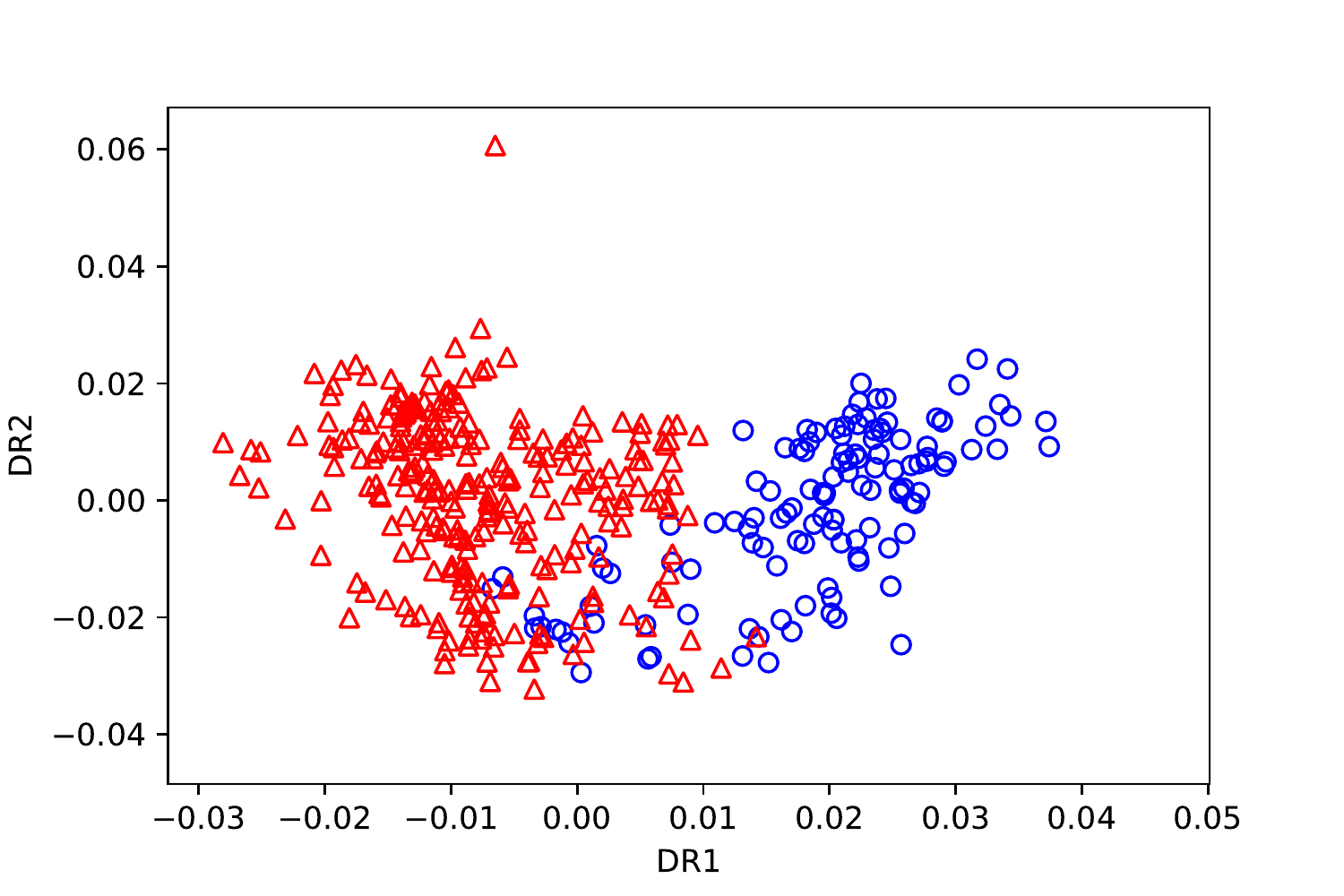}
			\includegraphics[width=0.10\textheight]{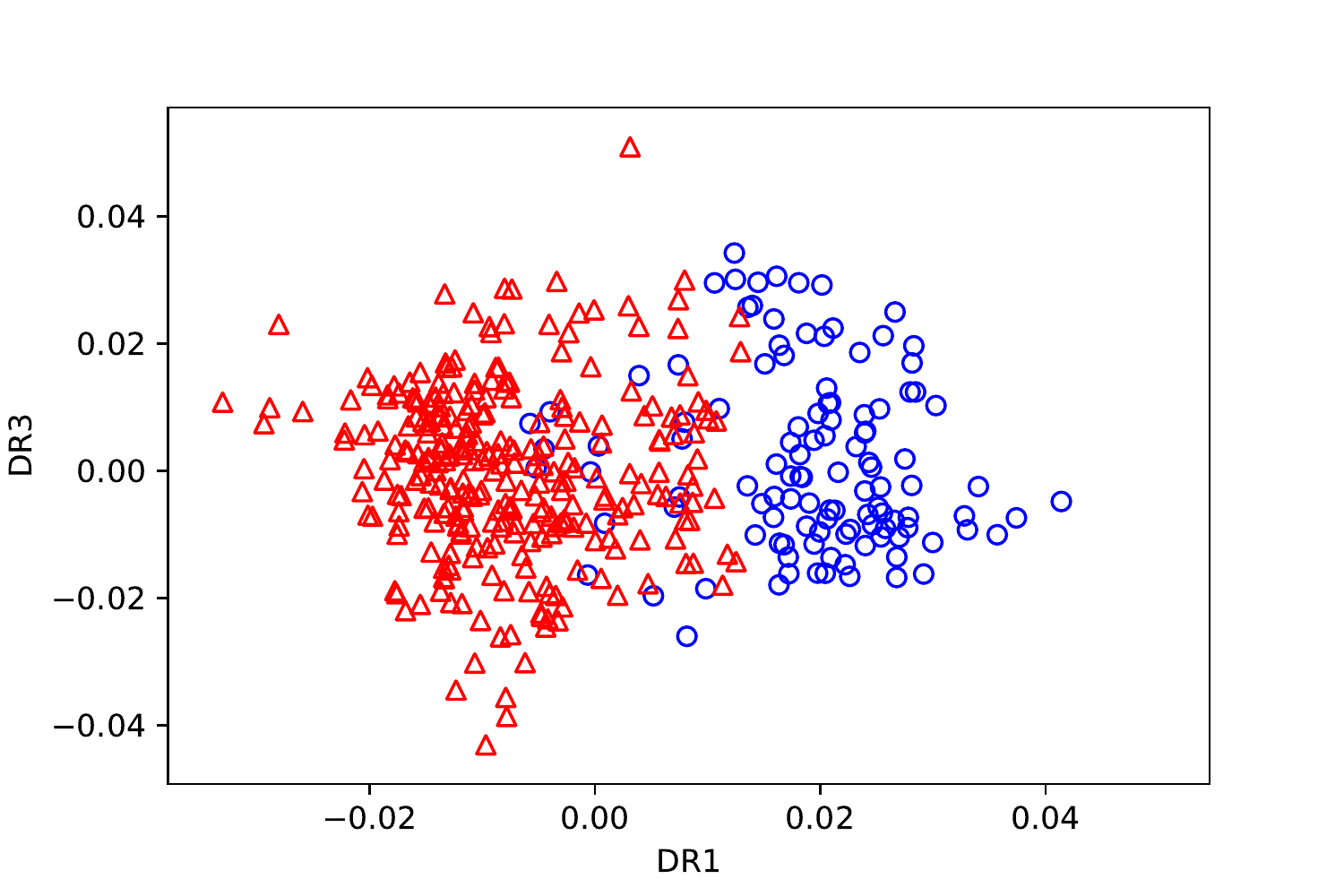}
			\includegraphics[width=0.10\textheight]{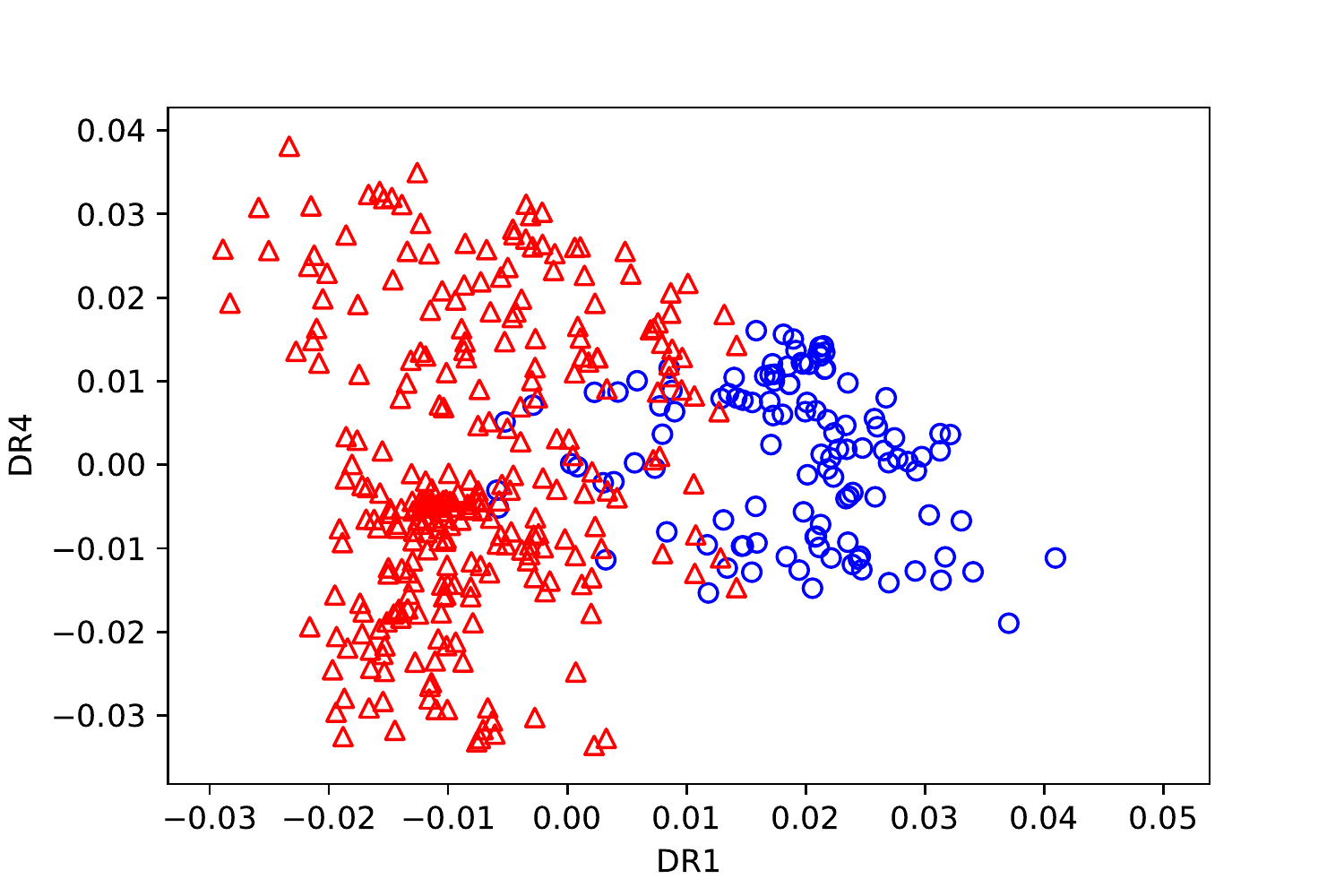}
			\includegraphics[width=0.10\textheight]{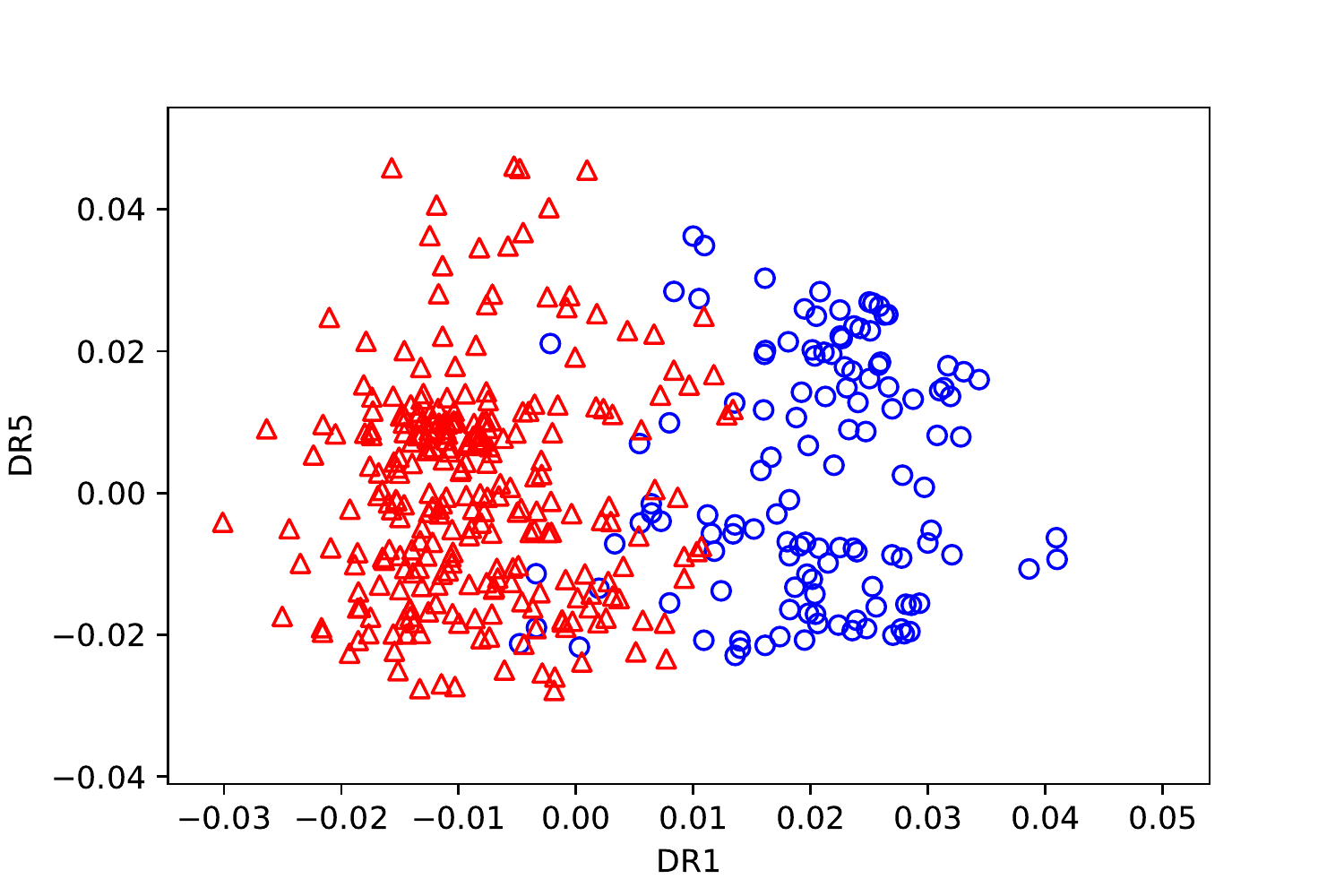}
			\includegraphics[width=0.10\textheight]{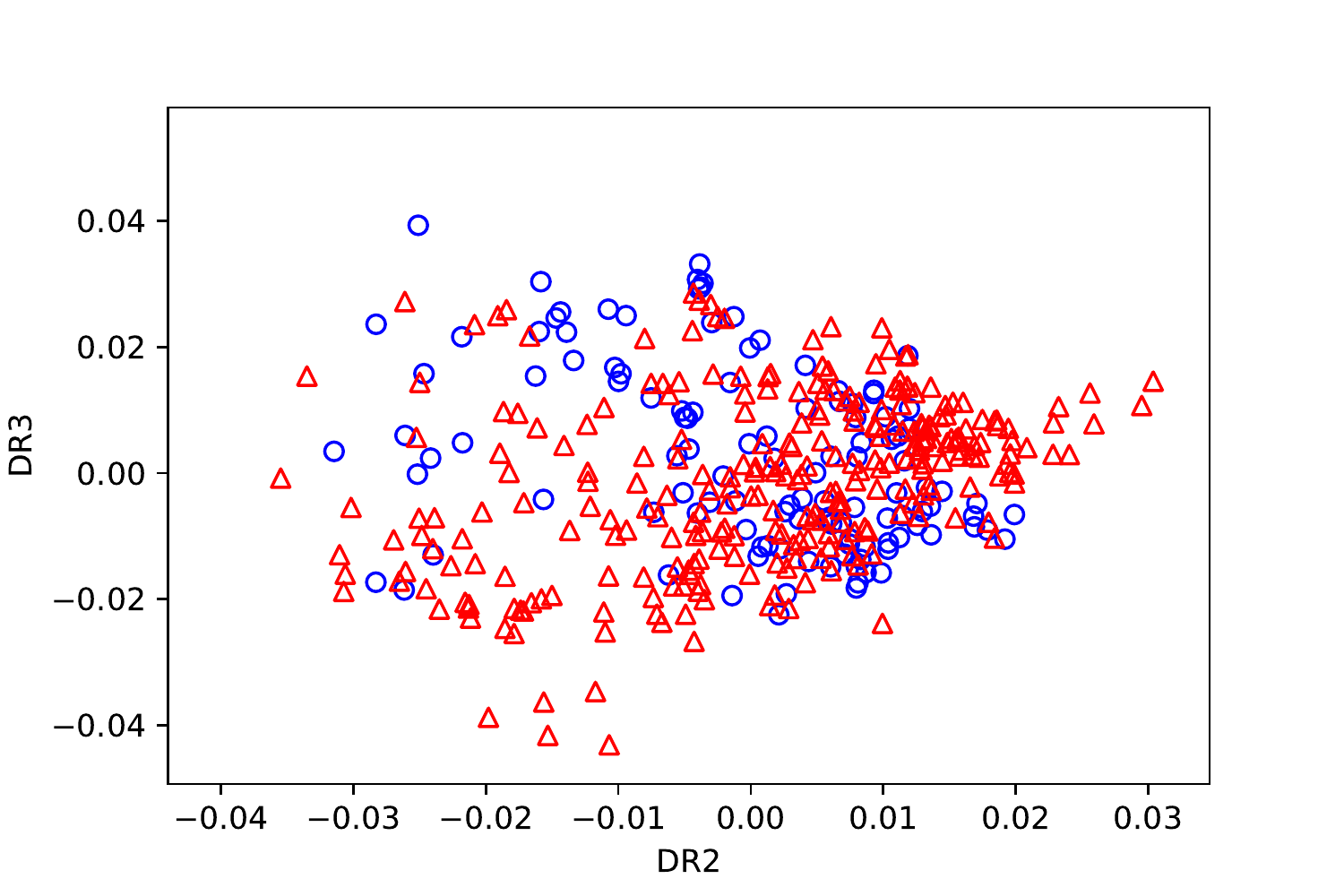}\\
			\includegraphics[width=0.10\textheight]{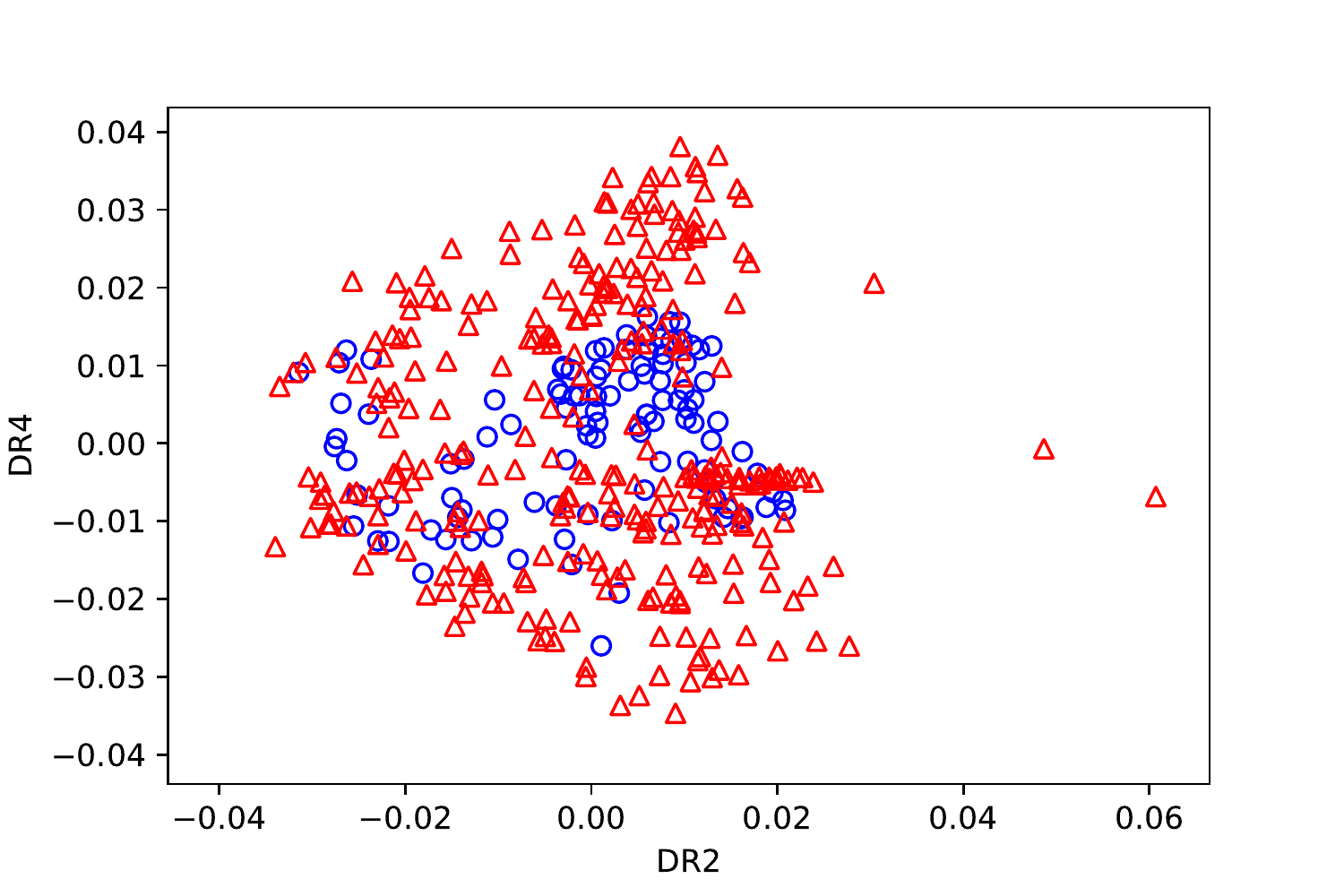}
			\includegraphics[width=0.10\textheight]{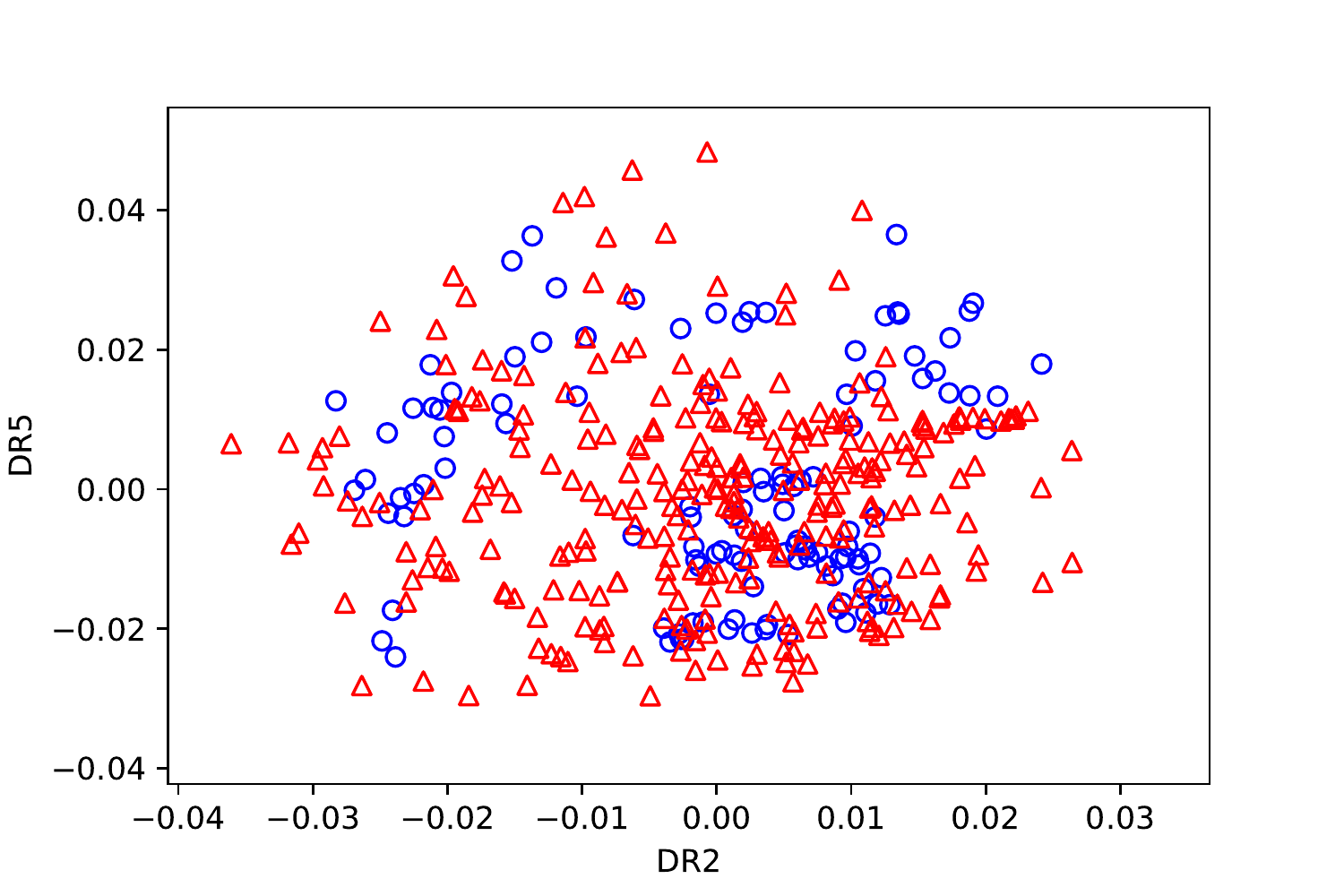}
			\includegraphics[width=0.10\textheight]{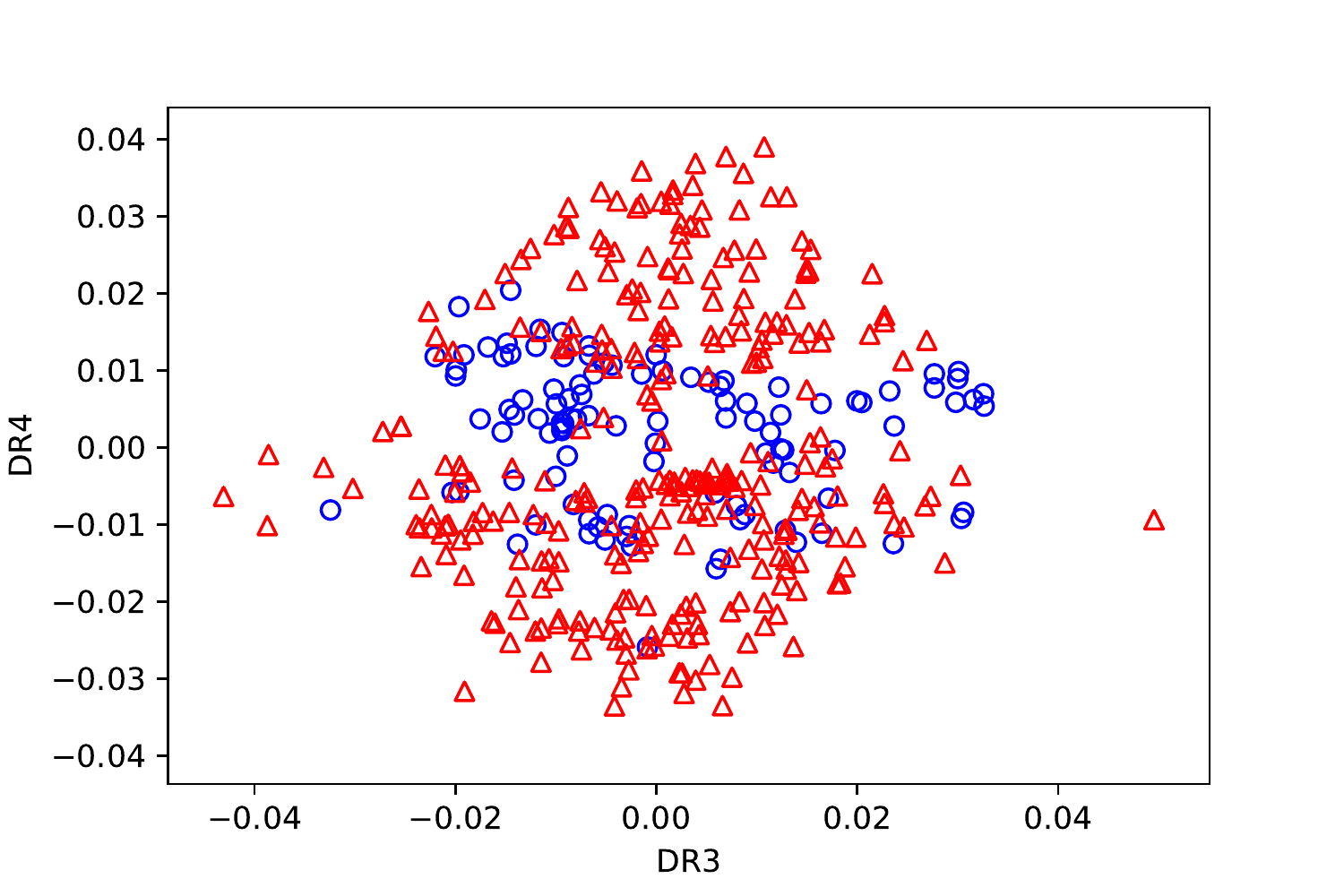}
			\includegraphics[width=0.10\textheight]{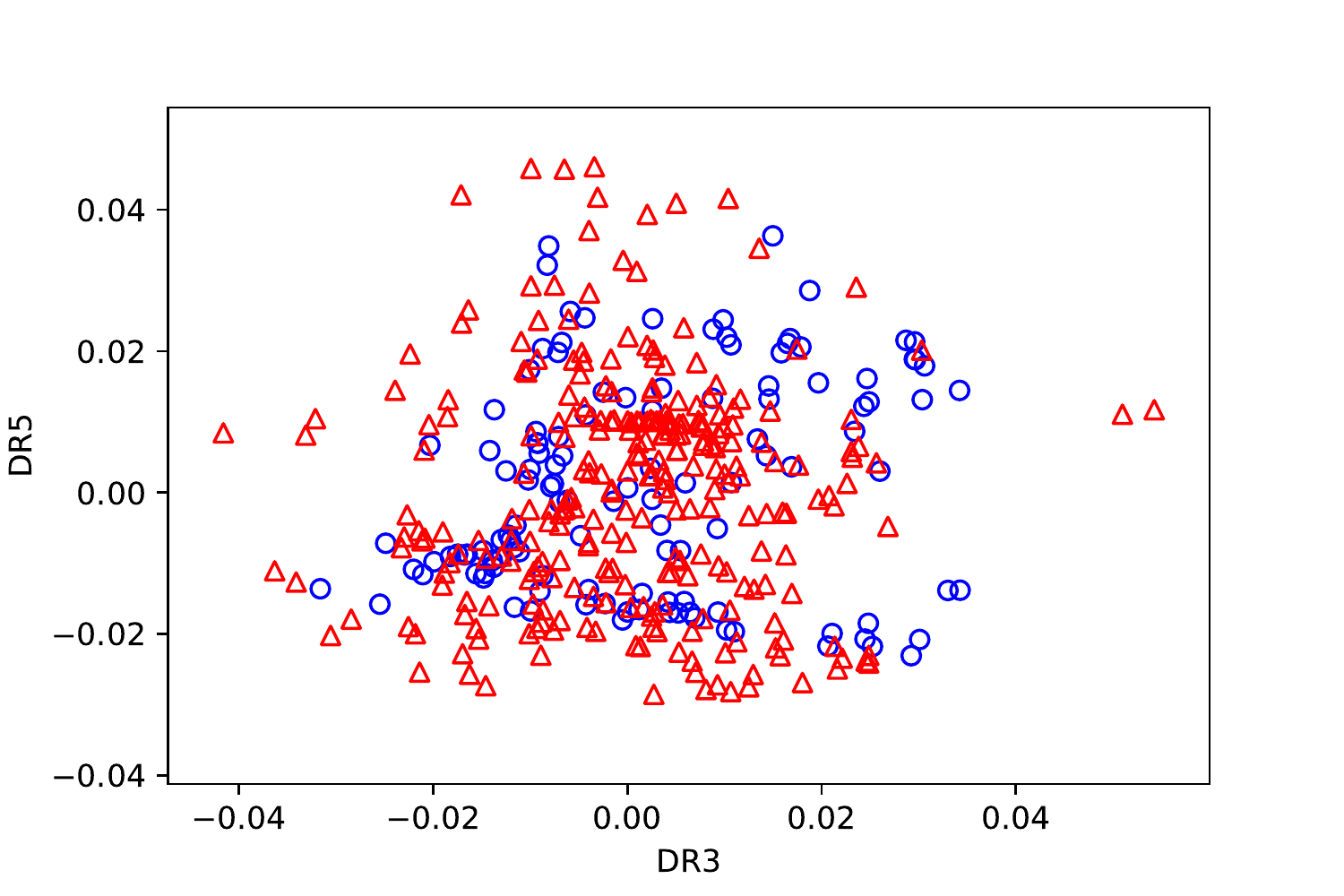}
			\includegraphics[width=0.10\textheight]{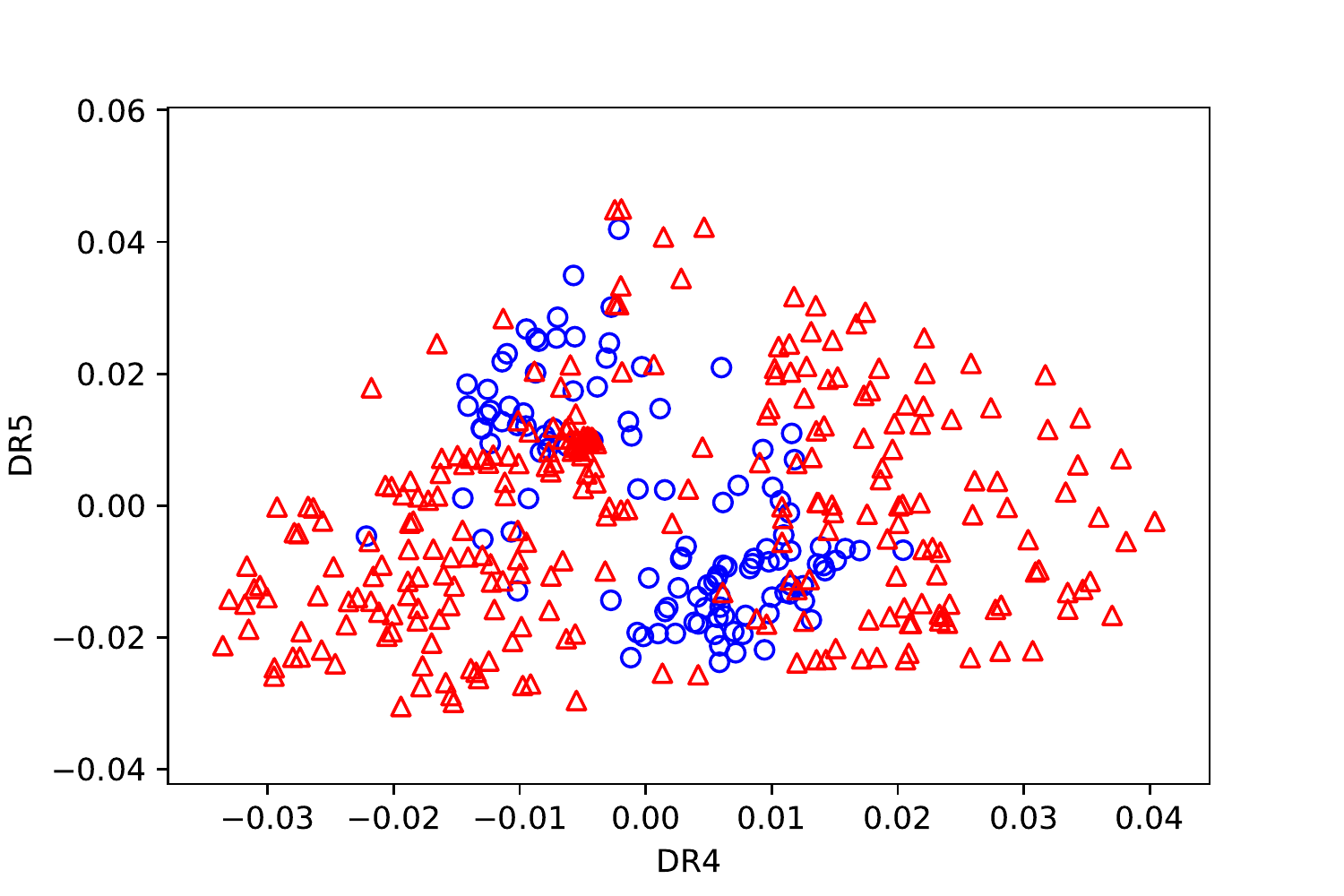}
		\end{minipage}
	\end{minipage}
	\caption{\label{Poldata_drawing}The plots of the transformed data based on each pair of the components of the learned representation by MSRL, GSIR and GSAVE  for the Pole-Telecommunication dataset. Bule '\textup{O}' and red '$\Delta$' means that $Y\ge50$ and $Y\le50$, respectively.}
\end{figure}
\subsection{Intrinsic dimension determination}\label{Intrinsic_dim_est}
To evaluate our proposed intrinsic dimension determination method, we consider a simple model
\begin{equation}\label{Toy_Model_d_est}
	Y=\Phi(X_1)+\Phi(X_2)\epsilon,
\end{equation}
where $\epsilon\perp \!\!\! \perp X=(X_1,X_2,\ldots,X_{10})^\top,\epsilon\sim N(0,1)$ and $X\sim N(\boldsymbol{0},\boldsymbol{\textup{I}}_{10})$. In this example, the true $d_0=2$. The implementation details are given in Section \ref{Descrip_Implement_dest}
in the supplementary material and we calculate the proportions of different estimated values for $d_0$ over 100 replications. The simulation results for two tolerance levels are provided in Table \ref{est_d_proportion}.  It shows that the proposed intrinsic dimension determination method works well for this toy example since it did not estimate $d_0$ as $1$, as $\hat{d}_0=1$ would produce an insufficient representation.

\begin{table}
	\caption{\label{est_d_proportion}The Proportions of Different Estimated Intrinsic Dimensions for Model (\ref{Toy_Model_d_est})}
	\centering
	\begin{tabular}{c|cccccccccc}
			\hline
			\hline
			$\eta$ & 1   & 2    & 3    & 4    & 5   & 6    & 7   & 8   & 9   & 10  \\
			\hline
			0.2   & 0\% & 89\% & 6\%  & 4\%  & 0\% & 1\%  & 0\% & 0\% & 0\% & 0\% \\
			0.1   & 0\% & 21\% & 23\% & 26\% & 5\% & 16\% & 1\% & 2\% & 1\% & 5\%\\
			\hline
		\end{tabular}
\end{table}

\section{Conclusion}\label{Discussion}
In this paper, we have proposed MSRL, a mutual information based deep sufficient representation learning approach.  Under reasonable conditions
on the underlying target representation and the neural network classes, we have established the consistency and derived the non-asymptotic error bounds for MSRL.  Extensive numerical experiments demonstrate that our method outperforms several existing SDR in the simulation models considered.  We have also discussed the extensions to the case of categorical response variables and proposed a dichotomy-based cross-validation algorithm to determine the intrinsic dimension.

Based on the results from our numerical experiments,  MSRL outperforms
the existing nonlinear SDR methods GSIR and GSAVE when the sample size is relatively large.  With relatively small sample sizes such as $n=300$ or $500$, MSRL performs similarly to GSIR and GSAVE.  Moreover,  with big sample sizes, the kernel-based nonlinear SDR methods are computationally prohibitive due to an $n\times n$ Gram matrix involved in the estimation procedure.  
%The proposed 
MSRL is still computationally feasible and performs well.

We have only considered the mutual information criterion
corresponding to the Kullback-Leibler divergence in this work.
In the future, it would be interesting to study the properties of using other mutual information criterions corresponding to divergencies such as Jensen-Shanon and $\chi^2$-divergences.
It would also be interesting to
extend the proposed method to the multi-dataset setting and apply it to problems arising from transfer learning
and domain generalization.

%\iffalse	
\section*{Acknowledgements}
The work of Y. Lin is supported by the Hong Kong Research Grants Council (Grant No.
14306219 and 14306620), the National Natural Science Foundation of China (Grant No.
11961028) and Direct Grants for Research, The Chinese University of Hong Kong.
The work of J. Huang is supported by the grant P0042888 from The Hong Kong Polytechnic University.
%\fi

\newpage
\appendix

\setcounter{equation}{0}  % reset counter
\renewcommand{\theequation}{S.\arabic{equation}}
  % redefine the command that creates the equation no.
\setcounter{table}{0}
\renewcommand{\thetable}{S.\arabic{table}}
\setcounter{figure}{0}
\renewcommand{\thefigure}{S.\arabic{figure}}
\setcounter{equation}{0}  % reset counter
\setcounter{lemma}{0}
     \renewcommand{\thelemma}{\Alph{section}.\arabic{lemma}}
\setcounter{theorem}{0}
     \renewcommand{\thetheorem}{\Alph{section}.\arabic{theorem}}

\begin{center}
\textbf{\Large Supplementary Material}
\end{center}

In this supplementary material, we provide technical details and proofs of the lemmas and the theorems in the paper.  We also include additional numerical results and details for the experiments.

%\section{Proofs and additional numerical results}

\section{Proof of Lemma \ref{f_Data_Processing_Inequality}}
% \ref{f_Data_Processing_Inequality}}
\label{appA}
\begin{proof}[Proof of Lemma \ref{f_Data_Processing_Inequality}]
	By the fact that $Y\to X\to g(X)$ is a Markov chain, Lemma \ref{f_Data_Processing_Inequality} is a direct corollary of Lemma \ref{General_f_Data_Processing_Inequality} below.
\end{proof}
\begin{lemma}\label{Invariant_CD_Y_under_X}
	Suppose that distributions $P_{XY}$ and $Q_{XY}$ have p.d.f's $p_X(x)p_{Y|X}(y|x)$ and $q_X(x)p_{Y|X}(y|x)$, respectively, then
	\[
	\mathbb{D}_f(P_{XY}~||~Q_{XY})=\mathbb{D}_f(P_{X}~||~Q_{X}).
	\]
\end{lemma}
\begin{proof}[Proof of Lemma \ref{Invariant_CD_Y_under_X}]
	By the definition of $f$-divergence,
	\begin{eqnarray*}
		\mathbb{D}_f(P_{XY}~||~Q_{XY})&=&\int_{\mathcal{X}\times\mathcal{Y}}q_X(x)p_{Y|X}(y|x)f\left(\frac{p_{X}(x)p_{Y|X}(y|x)}{q_X(x)p_{Y|X}(y|x)}\right)dxdy\\
		&=&\int_{\mathcal{X}\times\mathcal{Y}}q_X(x)p_{Y|X}(y|x)f\left(\frac{p_{X}(x)}{q_X(x)}\right)dxdy\\
		&=&\int_{\mathcal{X}}q_X(x)f\left(\frac{p_{X}(x)}{q_X(x)}\right)\left\{\int_{\mathcal{Y}}p_{Y|X}(y|x)dy\right\}dx\\
		&=&\int_{\mathcal{X}}q_X(x)f\left(\frac{p_{X}(x)}{q_X(x)}\right)dx\\
		&=&\mathbb{D}_f(P_{X}~||~Q_{X}).
	\end{eqnarray*}
	This completes the proof of Lemma \ref{Invariant_CD_Y_under_X}.
\end{proof}
\begin{lemma}\label{InformationLossOfMarginals}
	For two distributions $P_{XY}$ and $Q_{XY}$, then
	\begin{equation}\label{InformationLossInqOfMarginals}
		\mathbb{D}_f(P_{XY}~||~Q_{XY})\ge \mathbb{D}_f(P_{X}~||~Q_{X}).
	\end{equation}
	If all relevant p.d.f's and conditional p.d.f's are positive, when $f$ is strictly convex, the above equality holds if and only if
	\begin{equation}\label{EqualCDs}
		p_{Y|X}(y|x)=q_{Y|X}(y|x).
	\end{equation}
\end{lemma}
\begin{proof}[Proof of Lemma \ref{InformationLossOfMarginals}]
	\begin{eqnarray}
		\mathbb{D}_f(P_{XY}~||~Q_{XY})&=&\int_{\mathcal{X}\times\mathcal{Y}}q_X(x)q_{Y|X}(y|x)f\left(\frac{p_{X}(x)p_{Y|X}(y|x)}{q_X(x)q_{Y|X}(y|x)}\right)dxdy\nonumber\\
		&=&\int_{\mathcal{X}}q_X(x)\left\{\int_{\mathcal{Y}}q_{Y|X}(y|x)f\left(\frac{p_{X}(x)p_{Y|X}(y|x)}{q_X(x)q_{Y|X}(y|x)}\right)dy\right\}dx\nonumber\\
		&\ge&\int_{\mathcal{X}}q_X(x)f\left(\int_{\mathcal{Y}}q_{Y|X}(y|x)\frac{p_{X}(x)p_{Y|X}(y|x)}{q_X(x)q_{Y|X}(y|x)}dy\right)dx\label{JensenResults}\\
		&=&\int_{\mathcal{X}}q_X(x)f\left(\frac{p_{X}(x)}{q_X(x)}\right)dx\nonumber\\
		&=&\mathbb{D}_f(P_{X}~||~Q_{X})\nonumber,
	\end{eqnarray}
	where the inequality (\ref{JensenResults}) follows from Jensen's inequality. This completes the proof of (\ref{InformationLossInqOfMarginals}). When all relevant densities and conditional densities are positive, by the strict convexity of $f$, the inequality (\ref{JensenResults}) becomes an equality if and only if
	\[
	\frac{p_{X}(x)p_{Y|X}(y|x)}{q_X(x)q_{Y|X}(y|x)}\equiv\frac{p_{X}(x)}{q_X(x)}.
	\]
	This implies the result of (\ref{EqualCDs}).
\end{proof}
\begin{lemma}\label{General_f_Data_Processing_Inequality}
	For a Markov chain $Y\to X\to U$, it holds that
	\[
	\mathbb{I}_f(X;Y)\ge \mathbb{I}_f(U;Y).
	\]
	If all relevant densities and conditional densities are positive, when $f$ is strictly convex, the above equality holds if and only if
	\[
	Y\perp \!\!\! \perp X~|~U.
	\]
\end{lemma}
\begin{proof}[Proof of Lemma \ref{General_f_Data_Processing_Inequality}]
	Note that $(U,X,Y)$ has the joint distribution density $p_{U,X,Y}$ satisfying $p_{U,X,Y}(u,x,y)=p_{U|X,Y}(u|x,y)p_{X,Y}(x,y)$ and $p_{U|X,Y}(u|x,y)=p_{U|X}(u|x)$ by the Markov property. Let $Q_{U,X,Y}$ be the distribution with a joint density
	\begin{equation}\label{Construction_Q}
		q_{U,X,Y}(u,x,y)=p_{U|X,Y}(u|x,y)p_{X}(x)p_{Y}(y).
	\end{equation}
	Then by the result of Lemma \ref{Invariant_CD_Y_under_X}, we have
	\[
	\mathbb{D}_f(P_{U,X,Y}~||~Q_{U,X,Y})=\mathbb{D}_f(P_{X,Y}~||~P_XP_Y)=\mathbb{I}_f(X;Y).
	\]
	By the result of (\ref{InformationLossInqOfMarginals}) in Lemma \ref{InformationLossOfMarginals},
	\[
	\mathbb{D}_f(P_{U,X,Y}~||~Q_{U,X,Y})\ge\mathbb{D}_f(P_{U,Y}~||~Q_{U,Y})=\mathbb{I}_f(U;Y),
	\]
	where the last equality follows from
	\begin{eqnarray*}
		q_{U,Y}(u,y)&=&\int_{\mathcal{X}}q_{U,X,Y}(u,x,y)dx\\
		&=&\int_{\mathcal{X}}p_{U|X,Y}(u|x,y)p_{X}(x)p_{Y}(y)dx\\
		&=&\int_{\mathcal{X}}p_{U|X}(u|x)p_{X}(x)p_{Y}(y)dx\\
		&=&\int_{\mathcal{X}}p_{U,X}(u,x)p_{Y}(y)dx\\
		&=&p_{U}(u)p_{Y}(y).
	\end{eqnarray*}
	By  (\ref{EqualCDs}) in Lemma \ref{InformationLossOfMarginals}, if all relevant densities and conditional densities are positive, when $f$ is strictly convex, then the equality holds if and only if
	\[
	p_{X|U,Y}(x|u,y)=q_{X|U,Y}(x|u,y).
	\]
	As
	\begin{eqnarray}
		&&p_{X|U,Y}(x|u,y)=q_{X|U,Y}(x|u,y)\nonumber\\
		&\Leftrightarrow&\frac{p_{U,X,Y}(u,x,y)}{p_{U,Y}(u,y)}=\frac{q_{U,X,Y}(u,x,y)}{q_{U,Y}(u,y)}\nonumber\\
		&\Leftrightarrow&\frac{p_{U,X,Y}(u,x,y)}{p_{U,Y}(u,y)}=\frac{p_{U|X,Y}(u|x,y)p_{X}(x)p_{Y}(y)}{p_{U}(u)p_{Y}(y)}\label{q_joint_definition}\\
		&\Leftrightarrow&\frac{p_{U,X,Y}(u,x,y)}{p_{U,Y}(u,y)}=\frac{p_{U|X}(u|x)p_{X}(x)}{p_{U}(u)}\label{Con_Independence}\\
		&\Leftrightarrow&\frac{p_{X,Y|U}(x,y|u)}{p_{Y|U}(y|u)}=p_{X|U}(x|u)\nonumber\\
		&\Leftrightarrow&p_{X,Y|U}(x,y|u)=p_{X|U}(x|u)p_{Y|U}(y|u)\nonumber\\
		&\Leftrightarrow&Y\perp \!\!\! \perp X|U,\nonumber
	\end{eqnarray}
	where (\ref{q_joint_definition}) follows from the construction of $q_{U,X,Y}(u,x,y)$ in (\ref{Construction_Q}), and (\ref{Con_Independence}) is due to the Markov property, which means $p_{U|X,Y}(u|x,y)=p_{U|X}(u|x)$, we have under the assumed conditions that all relevant densities and conditional densities are positive, and $f$ is strictly convex,
	\[
	\mathbb{I}_f(X;Y)=\mathbb{I}_f(U;Y)\Leftrightarrow Y\perp \!\!\! \perp X|U.
	\]
	This completes the proof of Lemma \ref{General_f_Data_Processing_Inequality}.
\end{proof}

\section{Proof of Main Results}\label{appB}
In Appendices \ref{appB} and \ref{appC}, let $Z_1,Z_2,\ldots,Z_n$ be i.i.d random vectors in a  measurable space $(\mathcal{Z},\mathbb{P})$, for example, $Z_i=(X_i,Y_i),i=1,2,\ldots,n$,   $\mathcal{Z}=\mathcal{X}\times\mathcal{Y}$. Let $\{Z'_i\}_{i=1}^n$ be an independent copy of $\{Z_i\}_{i=1}^n$.
Let $\epsilon=(\epsilon_1,\epsilon_2,\ldots,\epsilon_n)$ be a Rademacher sequence independent of $\{Z_i\}_{i=1}^n$ and $\{Z'_i\}_{i=1}^n$, and let $\epsilon'$ be an independent copy of $\epsilon$.  We use $E_{Z'}[\cdot]$ and $E_{\epsilon}[\cdot]$ to denote the expectation w.r.t $\{Z'_i\}_{i=1}^n$ and $\{\epsilon_i\}_{i=1}^n$, respectively.

For a class $\mathcal{F}$ of real-valued functions on $\mathcal{Z}$ and a given sequence $z=(z_1,z_2,\cdots,z_m)\in\mathcal{Z}^m$, we define
\[
\mathcal{F}_{|z}=\{(f(z_1),f(z_2),\cdots,f(z_m)):f\in\mathcal{F}\}.
\]
For a positive number $\delta$, let $\mathcal{N}(\delta,\|\cdot\|_{\infty},\mathcal{F}_{|z})$ be the covering number of
$\mathcal{F}_{|z}$ under the norm $\|\cdot\|_{\infty}$ with radius $\delta$. The uniform covering number $\mathcal{N}_m(\delta,\|\cdot\|_{\infty},\mathcal{F})$
is defined by
\[
\mathcal{N}_m(\delta,\|\cdot\|_{\infty},\mathcal{F})=\max\left\{\mathcal{N}(\delta,\|\cdot\|_{\infty},\mathcal{F}_{|z}):z\in\mathcal{Z}^m\right\}.
\]
Thus there exists a $\delta$-uniform covering of $\mathcal{H}$ with centers $h_k,k=1,2,\cdots,\mathcal{N}_{n}$, where $\mathcal{N}_{n}=\mathcal{N}_{2n(n-1)}(\delta,\|\cdot\|_{\infty},\mathcal{H})$, such that for any $h\in\mathcal{H}$, there exists a $k^*$ satisfying $|h(\bar{Z})-h_{k^*}(\bar{Z})|\le \delta$ for any $\bar{Z}\in \{(Z_i,Z_j):i,j=1,2,\cdots,n\text{ and }i\textless j\}\cup \{(Z'_i,Z_j):i,j=1,2,\cdots,n\text{ and }i\neq j\}\cup \{(Z'_i,Z'_j):i,j=1,2,\cdots,n\text{ and }i\textless j\}$. For any $h,\bar{h}\in\mathcal{H}$, we define
\begin{eqnarray*}
	d_{U,n}^2(h,\bar{h})&=&\frac{1}{n(n-1)}\sum_{1\le i\neq j\le n}\big[\left\{h(Z_i,Z_j)-h(Z_i,Z'_j)-h(Z'_i,Z_j)+h(Z'_i,Z'_j)\right\}\\
	&&~~~~~~~~~~~~~~~~~~~~~~~~~~-\left\{\bar{h}(Z_i,Z_j)-\bar{h}(Z_i,Z'_j)-\bar{h}(Z'_i,Z_j)+\bar{h}(Z'_i,Z'_j)\right\}\big]^2.
\end{eqnarray*}
and $N(\delta,\mathcal{H}, d_{U,n})$ is the covering number of $\mathcal{H}$ w.r.t $d_{U,n}$. Some simple calculations shows that
\begin{equation}\label{Covering_Relationship}
	N(\delta, \mathcal{H}, d_{U,n})\le \mathcal{N}_{2n(n-1)}(\delta/2,\|\cdot\|_{\infty},\mathcal{H}).
\end{equation}

In addition, the pseudo dimension $\Pdim(\mathcal{F})$ of a function class $\mathcal{F}$, is the largest integer $N$ satisfying that
there exists $(x_1,x_2,\ldots,x_B,y_1,y_2,\ldots,y_N)\in \mathcal{Z}^N\times \mathbb{R}^N$ such that for any $(r_1,r_2,\ldots,r_N)\in\{0,1\}^N$, there exists an $f\in\mathcal{F}$  satisfying for any $i\in\{1,2,\ldots,N\}:$ $f(x_i)> y_i\Leftrightarrow r_i=1$ \citep{anthony_bartlett_1999, bartlett2019nearly}.

\begin{proof}[Proof of Theorem \ref{Representation_Learnable_Theorem}]
	Since $R_0$ is a $d_0$-dimensional USR, that is,
	\[
	Y \perp \!\!\! \perp X~|~R_0(X),~~R_0(X)\sim \textup{Uniform}[0,1]^{d_0},
	\]
	we have
	\[
	\mathbb{I}_{\textup{KL}}(R_0(X);Y)=\mathbb{I}_{\textup{KL}}(X;Y)\text{ and }\mathbb{D}_{\textup{KL}}(P_{R_0}~||~\gamma_{U})=0.
	\]
	For any $R$ with $\textup{dim}(R)=d_0$, by the data processing inequality (\ref{InformationLossInequality}) and the nonnegativity of KL divergence, we know
	\[
	\mathbb{I}_{\textup{KL}}(X;Y)\ge \mathbb{I}_{\textup{KL}}(R(X);Y)\text{ and }\mathbb{D}_{\textup{KL}}(P_{R}~||~\gamma_{U})\ge 0.
	\]
	Hence, for any $\lambda\ge0$, we have
	\begin{eqnarray*}
		\mathbb{L}(R_0;\lambda)&=&-\mathbb{I}_{\textup{KL}}(Y;R_0(X)) + \lambda \mathbb{D}_{\textup{KL}}(P_{R_0}~||~\gamma_{U})\\
		&\le&-\mathbb{I}_{\textup{KL}}(Y;R(X)) + \lambda \mathbb{D}_{\textup{KL}}(P_{R}~||~\gamma_{U})\\
		&\le&\mathbb{L}(R;\lambda).
	\end{eqnarray*}
	It follows that
	\[
	R_0\in\argmin_{R~:~\textup{dim}(R)=d_0}\{\mathbb{L}(R;\lambda)\}.
	\]
	This completes the proof of Theorem \ref{Representation_Learnable_Theorem}.
	
\end{proof}
\begin{proof}[Proof of Lemma \ref{Error Decomposition}]
	%    Note that
	For any $R\in\mathcal{R}$, we have
	\begin{eqnarray*}
		&&\mathbb{L}(\hat{R}_n^{\lambda};\lambda)-\mathbb{L}(R_0;\lambda)\nonumber\\
		&=&\{\lambda\sup\limits_{Q}\mathbb{L}^{\textup{Push}}(Q,\hat{R}_n^{\lambda})-\sup\limits_{D}\mathbb{L}^{\textup{MI}}(D,\hat{R}_n^{\lambda})\}-\{\lambda\sup\limits_{Q}\mathbb{L}^{\textup{Push}}(Q,R_0)-\sup\limits_{D}\mathbb{L}^{\textup{MI}}(D,R_0)\}\nonumber\\
		&=&\{\lambda\sup\limits_{Q}\mathbb{L}^{\textup{Push}}(Q,\hat{R}_n^{\lambda})-\sup\limits_{D}\mathbb{L}^{\textup{MI}}(D,\hat{R}_n^{\lambda})\}-\{\lambda\sup\limits_{Q\in\mathcal{Q}}\mathbb{L}^{\textup{Push}}(Q,\hat{R}_n^{\lambda})-\sup\limits_{D\in\mathcal{D}}\mathbb{L}^{\textup{MI}}(D,\hat{R}_n^{\lambda})\}\\
		&+&\{\lambda\sup\limits_{Q\in\mathcal{Q}}\mathbb{L}^{\textup{Push}}(Q,\hat{R}_n^{\lambda})-\sup\limits_{D\in\mathcal{D}}\mathbb{L}^{\textup{MI}}(D,\hat{R}_n^{\lambda})\}-\{\lambda\sup\limits_{Q\in\mathcal{Q}}\mathbb{L}_n^{\textup{Push}}(Q,\hat{R}_n^{\lambda})-\sup\limits_{D\in\mathcal{D}}\mathbb{L}_n^{\textup{MI}}(D,\hat{R}_n^{\lambda})\}\nonumber\\
		&+&\{\lambda\sup\limits_{Q\in\mathcal{Q}}\mathbb{L}_n^{\textup{Push}}(Q,\hat{R}_n^{\lambda})-\sup\limits_{D\in\mathcal{D}}\mathbb{L}_n^{\textup{MI}}(D,\hat{R}_n^{\lambda})\}-\{\lambda\sup\limits_{Q\in\mathcal{Q}}\mathbb{L}_n^{\textup{Push}}(Q,R)-\sup\limits_{D\in\mathcal{D}}\mathbb{L}_n^{\textup{MI}}(D,R)\}\\
		&+&\{\lambda\sup\limits_{Q\in\mathcal{Q}}\mathbb{L}_n^{\textup{Push}}(Q,R)-\sup\limits_{D\in\mathcal{D}}\mathbb{L}_n^{\textup{MI}}(D,R)\}-\{\lambda\sup\limits_{Q\in\mathcal{Q}}\mathbb{L}^{\textup{Push}}(Q,R)-\sup\limits_{D\in\mathcal{D}}\mathbb{L}^{\textup{MI}}(D,R)\}\nonumber\\
		&+&\{\lambda\sup\limits_{Q\in\mathcal{Q}}\mathbb{L}^{\textup{Push}}(Q,R)-\sup\limits_{D\in\mathcal{D}}\mathbb{L}^{\textup{MI}}(D,R)\}-\{\lambda\sup\limits_{Q}\mathbb{L}^{\textup{Push}}(Q,R)-\sup\limits_{D}\mathbb{L}^{\textup{MI}}(D,R)\}\\
		&+&\{\lambda\sup\limits_{Q}\mathbb{L}^{\textup{Push}}(Q,R)-\sup\limits_{D}\mathbb{L}^{\textup{MI}}(D,R)\}-\{\lambda\sup\limits_{Q}\mathbb{L}^{\textup{Push}}(Q,R_0)-\sup\limits_{D}\mathbb{L}^{\textup{MI}}(D,R_0)\}\nonumber\\
		&\le&\lambda\{\sup\limits_{Q}\mathbb{L}^{\textup{Push}}(Q,\hat{R}_n^{\lambda})-\sup\limits_{Q\in\mathcal{Q}}\mathbb{L}^{\textup{Push}}(Q,\hat{R}_n^{\lambda})\}\nonumber\\
		&+&\{\lambda\sup\limits_{Q\in\mathcal{Q}}\mathbb{L}^{\textup{Push}}(Q,\hat{R}_n^{\lambda})-\sup\limits_{D\in\mathcal{D}}\mathbb{L}^{\textup{MI}}(D,\hat{R}_n^{\lambda})\}-\{\lambda\sup\limits_{Q\in\mathcal{Q}}\mathbb{L}_n^{\textup{Push}}(Q,\hat{R}_n^{\lambda})-\sup\limits_{D\in\mathcal{D}}\mathbb{L}_n^{\textup{MI}}(D,\hat{R}_n^{\lambda})\}\nonumber\\
		&+&\{\lambda\sup\limits_{Q\in\mathcal{Q}}\mathbb{L}_n^{\textup{Push}}(Q,R)-\sup\limits_{D\in\mathcal{D}}\mathbb{L}_n^{\textup{MI}}(D,R)\}-\{\lambda\sup\limits_{Q\in\mathcal{Q}}\mathbb{L}^{\textup{Push}}(Q,R)-\sup\limits_{D\in\mathcal{D}}\mathbb{L}^{\textup{MI}}(D,R)\}\nonumber\\
		&+&\{\sup\limits_{D}\mathbb{L}^{\textup{MI}}(D,R)-\sup\limits_{D\in\mathcal{D}}\mathbb{L}^{\textup{MI}}(D,R)\}+\mathbb{L}(R;\lambda)-\mathbb{L}(R_0;\lambda)\nonumber\\
		&\le&\lambda\{\sup\limits_{Q}\mathbb{L}^{\textup{Push}}(Q,\hat{R}_n^{\lambda})-\sup\limits_{Q\in\mathcal{Q}}\mathbb{L}^{\textup{Push}}(Q,\hat{R}_n^{\lambda})\}+\{\sup\limits_{D}\mathbb{L}^{\textup{MI}}(D,R)-\sup\limits_{D\in\mathcal{D}}\mathbb{L}^{\textup{MI}}(D,R)\}\\
		&+&2\lambda\sup\limits_{Q\in\mathcal{Q},R\in\mathcal{R}}|\mathbb{L}_n^{\textup{Push}}(Q,R)-\mathbb{L}^{\textup{Push}}(Q,R)|+2\sup\limits_{D\in\mathcal{D},R\in\mathcal{R}}|\mathbb{L}_n^{\textup{MI}}(D,R)-\mathbb{L}^{\textup{MI}}(D,R)|\\
		&+&\mathbb{L}(R;\lambda)-\mathbb{L}(R_0;\lambda).
	\end{eqnarray*}
	The proof is completed by choosing $R=R^*\in\argmin_{R\in \mathcal{R}}\mathbb{L}(R;\lambda)$.
\end{proof}
\begin{proof}[Proofs of Theorem \ref{Weak_Consistency_Results0} and Corollary \ref{Weak_Consistency_Results}]
	By Lemma \ref{Error Decomposition}, we have
	\begin{eqnarray*}
		&&E\{\mathbb{L}(\hat{R}_n^{\lambda};\lambda)-\mathbb{L}(R_0;\lambda)\}\\
		&\le&\lambda E\{\sup\limits_{Q}\mathbb{L}^{\textup{Push}}(Q,\hat{R}_n^{\lambda})-\sup\limits_{Q\in\mathcal{Q}}\mathbb{L}^{\textup{Push}}(Q,\hat{R}_n^{\lambda})\}+\{\sup\limits_{D}\mathbb{L}^{\textup{MI}}(D,R^*)-\sup\limits_{D\in\mathcal{D}}\mathbb{L}^{\textup{MI}}(D,R^*)\}\\
		&+&2\lambda E\sup\limits_{Q\in\mathcal{Q},R\in\mathcal{R}}|\mathbb{L}_n^{\textup{Push}}(Q,R)-\mathbb{L}^{\textup{Push}}(Q,R)|+2E\sup\limits_{D\in\mathcal{D},R\in\mathcal{R}}|\mathbb{L}_n^{\textup{MI}}(D,R)-\mathbb{L}^{\textup{MI}}(D,R)|\\
		&+&\mathbb{L}(R^*;\lambda)-\mathbb{L}(R_0;\lambda),
	\end{eqnarray*}
	where $R^*\in\argmin_{R\in \mathcal{R}}\mathbb{L}(R;\lambda)$. Using part (a) of Lemmas \ref{App_Push_part}-\ref{obj_app_error_upper_bound}, we obtain
	\begin{equation}\label{excess_error_consistency}
		E\{\mathbb{L}(\hat{R}_n^{\lambda};\lambda)-\mathbb{L}(R_0;\lambda)\}\to0,\text{ as } n\to\infty.
	\end{equation}
	This completes the proof of Theorem \ref{Weak_Consistency_Results0}. By Pinsker's inequality, we have
	\begin{eqnarray}
		&&\mathbb{L}(R;\lambda)-\mathbb{L}(R_0;\lambda)\nonumber\\
		&=&\mathbb{I}_{\textup{KL}}(Y;X)+\mathbb{L}(R;\lambda)\nonumber\\
		&=&\mathbb{I}_{\textup{KL}}(Y;X|R(X))+\lambda \mathbb{D}_{\textup{KL}}(P_{R}~||~\gamma_{U})\nonumber\\
		&\ge&\frac{1}{2}(E_{X}\|p_{Y|X}-p_{Y|R}\|_{L_{1}}^2+\lambda\|p_{R}-1\|_{L_1}^2),\label{Pinsker_Results}
	\end{eqnarray}
	for any $R\in\mathcal{R}$, where $\|p_{Y|X}-p_{Y|R}\|_{L_{1}}=\int_{\mathcal{Y}}|p_{Y|X}(y|X=x)-p_{Y|R}(y|R(X)=R(x))|dy$ and $\|p_{R}-1\|_{L_1}=\int_{[0,1]^{d_0}}|p_{R}(r)-1|dr$.  Using (\ref{excess_error_consistency}) and (\ref{Pinsker_Results}), we easily obtain the conclusions of Corollary \ref{Weak_Consistency_Results}.
\end{proof}
\begin{proof}[Proofs of Theorem \ref{Non_Asymp_Results0} and Corollary \ref{Non_Asymp_Results}]
	First, we show that
	\begin{equation}\label{Excess_Risk}
		E\{\mathbb{L}(\hat{R}_n^{\lambda};\lambda)-\mathbb{L}(R_0;\lambda)\}\le C(\lambda \vee2)\{(d_Y+d_0)\vee d_X\}^{2(\lfloor\beta \rfloor+1)}n^{-\frac{\beta(\beta\wedge1)}{2\beta+d_X}\wedge\frac{1}{2\beta+(d_Y+d_0)\vee d_X}},
	\end{equation}
	where $C$ is a constant only depending on $B_2$ and $\beta$, by bounding the terms on the right-hand side of the error decomposition inequality of Lemma \ref{Error Decomposition}.
	
	\begin{enumerate}
		\item  Using part (b) of Lemma \ref{App_Push_part},
		\[
		E\{\sup\limits_{Q}\mathbb{L}^{\textup{Push}}(Q,\hat{R}_n^{\lambda})-\sup\limits_{Q\in\mathcal{Q}}\mathbb{L}^{\textup{Push}}(Q,\hat{R}_n^{\lambda})\}\le36(1+e^{\mathcal{B}_\mathcal{Q}})C_{1,\beta}(d_0,B_{2})n^{-\frac{\beta}{2\beta+d_0}},
		\]
		where $C_{1,\beta}(d,a)=c_2(\lfloor\beta \rfloor)(\lfloor\beta \rfloor+1)^2d^{\lfloor\beta \rfloor+(\beta\vee 1)/2+1}a^{c_1(\lfloor\beta \rfloor)+c_2(\lfloor\beta \rfloor)}$, and $c_1(\cdot)$, $c_2(\cdot)$ are defined in (\ref{log_holder_class}).
		
		\item  Using part (b) of Lemma \ref{App_MI_part},
		\[
		\sup\limits_{D}\mathbb{L}^{\textup{MI}}(D,R^*)-\sup\limits_{D\in\mathcal{D}}\mathbb{L}^{\textup{MI}}(D,R^*)\le72(1+e^{\mathcal{B}_\mathcal{D}})C_{1,\beta}(d_Y+d_0,B_{2})n^{-\frac{\beta}{2\beta+d_Y+d_0}}.
		\]
		
		\item Using part (b) of Lemma \ref{upper_bound_empirical_MI_part},
		\[
		E\sup\limits_{D\in\mathcal{D},R\in\mathcal{R}}|\mathbb{L}_n^{\textup{MI}}(D,R)-\mathbb{L}^{\textup{MI}}(D,R)|\le CC_{2,\beta}((d_Y+d_0)\vee d_X,\mathcal{B}_{\mathcal{D}})n^{-\frac{\beta}{2\beta+(d_Y+d_0)\vee d_X}}.
		\]
		where $C$ is a universal constant and $C_{2,\beta}(d,a)=(1+a)(\lfloor\beta \rfloor+1)^{9/2}d^{\lfloor\beta \rfloor+2}$.
		
		\item  Using part (b) of Lemma \ref{upper_bound_empirical_Push_part},
		\[
		E\sup\limits_{Q\in\mathcal{Q},R\in\mathcal{R}}|\mathbb{L}_n^{\textup{Push}}(Q,R)-\mathbb{L}^{\textup{Push}}(Q,R)|\le CC_{2,\beta}(d_X,\mathcal{B}_{\mathcal{Q}})n^{-\frac{\beta}{2\beta+d_X}}
		\]
		for some universal constant $C$.
		
		\item  Using part (b) of Lemma \ref{obj_app_error_upper_bound},
		\[
		\mathbb{L}(R^*;\lambda)-\mathbb{L}(R_0;\lambda)\le 36\sqrt{d_Y+d_0}C_{3,\beta}(\lambda,B_{2},d_X)n^{-\frac{\beta(\beta\wedge1)}{(2\beta+d_X)}},
		\]
		where $C_{3,\beta}(\lambda,a,d)=(\lfloor\beta \rfloor+1)^2d^{\lfloor\beta \rfloor+(\beta\vee 1)/2}(\lambda \vee2)a^4$.
	\end{enumerate}
	
	Using (\ref{Pinsker_Results}) again, (\ref{conpdf_R_hat_conRate}) and (\ref{pdf_R_hat_conRate}) are the direct corollaries of (\ref{Excess_Risk}).
	This completes the proof.
	%The proof is finished.
\end{proof}
\begin{lemma}\label{App_Push_part}
	(a) Suppose Assumption \ref{Weak_Assumption_on_NetworkR} holds and the network parameters of $\mathcal{Q}$ satisfy (NS\ref{Structure_on_NetworkQ}), we have
	\[
	E\{\sup\limits_{Q}\mathbb{L}^{\textup{Push}}(Q,\hat{R}_n^{\lambda})-\sup\limits_{Q\in\mathcal{Q}}\mathbb{L}^{\textup{Push}}(Q,\hat{R}_n^{\lambda})\}\le 18B_{1}^2d_0^{1/2}(1+e^{\mathcal{B}_\mathcal{Q}})n^{-\frac{1}{2+d_0}}.
	\]
	
	(b) Suppose Assumption \ref{Strong_Assumption_on_NetworkR} holds and the network parameters of $\mathcal{Q}$ satisfy (NS\ref{Strong_Structure_on_NetworkQ}), we have
	\[
	E\{\sup\limits_{Q}\mathbb{L}^{\textup{Push}}(Q,\hat{R}_n^{\lambda})-\sup\limits_{Q\in\mathcal{Q}}\mathbb{L}^{\textup{Push}}(Q,\hat{R}_n^{\lambda})\}\le36(1+e^{\mathcal{B}_\mathcal{Q}})C_{1,\beta}(d_0,B_{2})n^{-\frac{\beta}{2\beta+d_0}},
	\]
	where $C_{1,\beta}(d,a)=c_2(\lfloor\beta \rfloor)(\lfloor\beta \rfloor+1)^2d^{\lfloor\beta \rfloor+(\beta\vee 1)/2+1}a^{c_1(\lfloor\beta \rfloor)+c_2(\lfloor\beta \rfloor)}$, and $c_1(\cdot)$, $c_2(\cdot)$ are defined in (\ref{log_holder_class}).
\end{lemma}
\begin{proof}[Proof of Lemma \ref{App_Push_part}]
	Note that
	\begin{eqnarray}
		&&\sup\limits_{Q}\mathbb{L}^{\textup{Push}}(Q,\hat{R}_n^{\lambda})-\sup\limits_{Q\in\mathcal{Q}}\mathbb{L}^{\textup{Push}}(Q,\hat{R}_n^{\lambda})\nonumber\\
		&\le&\inf\limits_{Q\in\mathcal{Q}}\left\{E\left|Q_{\hat{R}_n^{\lambda}}(\hat{R}_n^{\lambda}(X))-Q(\hat{R}_n^{\lambda}(X))\right|+E\left|e^{Q_{\hat{R}_n^{\lambda}}(U)}-e^{Q(U)}\right|\right\},\label{infsupupperbound}
	\end{eqnarray}
	where the last inequality holds due to  the following facts:
	\begin{eqnarray*}
		&&\sup\limits_{Q}\mathbb{L}^{\textup{Push}}(Q,\hat{R}_n^{\lambda})-\sup\limits_{Q\in\mathcal{Q}}\mathbb{L}^{\textup{Push}}(Q,\hat{R}_n^{\lambda})\\
		&=&EQ_{\hat{R}_n^{\lambda}}(\hat{R}_n^{\lambda}(X))-Ee^{Q_{\hat{R}_n^{\lambda}}(U)}-\sup\limits_{Q\in\mathcal{Q}}\mathbb{L}^{\textup{Push}}(Q,\hat{R}_n^{\lambda})\\
		&=&\inf\limits_{Q\in\mathcal{Q}}\left[EQ_{\hat{R}_n^{\lambda}}(\hat{R}_n^{\lambda}(X))-Ee^{Q_{\hat{R}_n^{\lambda}}(U)}-\left\{EQ(\hat{R}_n^{\lambda}(X))-Ee^{Q(U)}\right\}\right]\\
		&\le&\inf\limits_{Q\in\mathcal{Q}}\left\{
		E\left|Q_{\hat{R}_n^{\lambda}}(\hat{R}_n^{\lambda}(X))-Q(\hat{R}_n^{\lambda}(X))\right|
		+E\left|e^{Q_{\hat{R}_n^{\lambda}}(U)}-e^{Q(U)}\right|\right\},
	\end{eqnarray*}
	and $Q_{\hat{R}_n^{\lambda}}(r)=\log p_{\hat{R}_n^{\lambda}}(r)$ on $[0,1]^{d_0}$.
	
	{\bf Proof of Part (a):} By Assumption \ref{Weak_Assumption_on_NetworkR}, we have that $Q_{\hat{R}_n^{\lambda}}$ is a Lipschitz continuous function with a Lipschitz constant $B_{1}^2$.  By the network structure condition (NS\ref{Structure_on_NetworkQ}), and applying Theorem 2.1 in \cite{Shen2020CICP}
%Lemma \ref{Approximation Error_onlyC}
with $N=\left\lceil n^{\frac{d_0}{2(2+d_0)}} /\log n\right\rceil$ and $M=\lceil\log n\rceil$, there exists a $\psi_{R}\in\mathcal{Q}$ such that
	\[
	\sup\limits_{x\in [0,1]^{d_0}\backslash H_{K,\delta}}|Q_{\hat{R}_n^{\lambda}}-\psi_{R}|\le 18B_{1}^2d_0^{1/2}n^{-\frac{1}{2+d_0}}.
	\]
	where $H_{K,\delta}=\cup_{i=1}^{d_0}\big\{x=[x_1,\ldots,x_{d_0}]:x_i\in\cup_{b=1}^{K-1}\left(b/K-\delta,b/K\right)\big\},K=\lceil (MN)^{2/d_0} \rceil,\delta\in (0,1/(3K)]$.
	
	By (\ref{infsupupperbound}) and the arbitrariness of $\delta$, we have
	\begin{eqnarray}
		&&\sup\limits_{Q}\mathbb{L}^{\textup{Push}}(Q,\hat{R}_n^{\lambda})-\sup\limits_{Q\in\mathcal{Q}}\mathbb{L}^{\textup{Push}}(Q,\hat{R}_n^{\lambda})\nonumber\\
		&\le&\inf\limits_{Q\in\mathcal{Q}}\left\{E\left|Q_{\hat{R}_n^{\lambda}}(\hat{R}_n^{\lambda}(X))-Q(\hat{R}_n^{\lambda}(X))\right|+E\left|e^{Q_{\hat{R}_n^{\lambda}}(U)}-e^{Q(U)}\right|\right\}\nonumber\\
		&\le&E\left|Q_{\hat{R}_n^{\lambda}}(\hat{R}_n^{\lambda}(X))-\psi_{R}(\hat{R}_n^{\lambda}(X))\right|+E\left|e^{Q_{\hat{R}_n^{\lambda}}(U)}-e^{\psi_{R}(U)}\right|\nonumber\\
		&\le&18B_{1}^2d_0^{1/2}(1+e^{\mathcal{B}_\mathcal{Q}})n^{-\frac{1}{2+d_0}}\label{Weak_Push_part_bound},
	\end{eqnarray}
	where the last inequality follows from $\|Q_R\|_\infty\le \mathcal{B}_\mathcal{Q}$ and $\|\psi_R\|_\infty\le \mathcal{B}_\mathcal{Q}$ by Assumption \ref{Weak_Assumption_on_NetworkR} and (NS\ref{Structure_on_NetworkQ}).
	
	{\bf Proof of Part (b):} By Assumption \ref{Strong_Assumption_on_NetworkR} and the same argument as the one for (\ref{log_holder_class}),
	\[
	Q_{R}(r)\in \mathcal{H}^\beta([0,1]^{d_0},2d_0c_2(\lfloor\beta \rfloor)B_{2}^{c_1(\lfloor\beta \rfloor)+c_2(\lfloor\beta \rfloor)}).
	\]
	By the network structure condition (NS\ref{Strong_Structure_on_NetworkQ}), and applying
%Lemma \ref{Approximation Error}
Theorem 3.3 in \citet{jiao2021deep}
 with $N=1$ and $M=\left\lceil n^{\frac{d_0}{2(2+d_0)}}\right\rceil$, we get a $\psi_{R}\in\mathcal{Q}$, such that
	\begin{align*}
		\sup\limits_{x\in [0,1]^{d_0}\backslash H_{K,\delta}}|Q_{R}-\psi_{R}|
		\le 36
		C_{1,\beta}(d_0,B_{2})n^{-\frac{\beta}{2\beta+d_0}},
	\end{align*}
	where $C_{1,\beta}(d,a)=c_2(\lfloor\beta \rfloor)(\lfloor\beta \rfloor+1)^2d^{\lfloor\beta \rfloor+(\beta\vee 1)/2+1}a^{c_1(\lfloor\beta \rfloor)+c_2(\lfloor\beta \rfloor)}$, $H_{K,\delta}=\cup_{i=1}^{d_0}\{x=[x_1,\ldots,x_{d_0}]:x_i\in\cup_{b=1}^{K-1} \left(b/K-\delta,b/K\right)\}$ for $K=\lceil (MN)^{2/d_0} \rceil,\delta\in (0,1/(3K)]$. Due to the arbitrariness of $\delta$, we have
	\begin{equation}\label{Push_app_part1}
		E\left|Q_{R}(R(X))-\psi_{R}(R(X))\right|\le36
		C_{1,\beta}(d_0,B_{2})n^{-\frac{\beta}{2\beta+d_0}}
	\end{equation}
	and
	\begin{equation}\label{Push_app_part2}
		E\left|e^{Q_{R}(U)}-e^{\psi_{R}(U)}\right|\le e^{B_\mathcal{Q}}E\left|Q_{R}(U)-\psi_{R}(U)\right|\le36e^{\mathcal{B}_\mathcal{Q}}		C_{1,\beta}(d_0,B_{2})n^{-\frac{\beta}{2\beta+d_0}},
	\end{equation}
	where the first inequality in (\ref{Push_app_part2}) is due to $\|Q_R\|_\infty\le \mathcal{B}_\mathcal{Q}$ and $\|\psi_R\|_\infty\le \mathcal{B}_\mathcal{Q}$ by Assumption \ref{Strong_Assumption_on_NetworkR} and (NS\ref{Strong_Structure_on_NetworkQ}). Combining (\ref{infsupupperbound}), (\ref{Push_app_part1}), and (\ref{Push_app_part2}), we obtain
	\begin{equation}\label{Push_part_bound}
		E\{\sup\limits_{Q}\mathbb{L}^{\textup{Push}}(Q,\hat{R}_n^{\lambda})-\sup\limits_{Q\in\mathcal{Q}}\mathbb{L}^{\textup{Push}}(Q,\hat{R}_n^{\lambda})\}\le36(1+e^{\mathcal{B}_\mathcal{Q}})C_{1,\beta}(d_0,B_{2})n^{-\frac{\beta}{2\beta+d_0}}.
	\end{equation}
	This completes the proof of Lemma \ref{App_Push_part}.
\end{proof}

\begin{lemma}\label{App_MI_part}
	(a) Suppose Assumption \ref{Weak_Assumption_on_NetworkR} holds and the network parameters of $\mathcal{D}$ satisfy (NS\ref{Structure_on_NetworkD}), we have
	\[
	\sup\limits_{D}\mathbb{L}^{\textup{MI}}(D,R^*)-\sup\limits_{D\in\mathcal{D}}\mathbb{L}^{\textup{MI}}(D,R^*)\le 36B_{1}^2(d_Y+d_0)^{1/2}(1+e^{\mathcal{B}_\mathcal{D}})n^{-\frac{1}{2+d_Y+d_0}}.
	\]
	(b)
	Suppose Assumption \ref{Strong_Assumption_on_NetworkR} holds and the network parameters of $\mathcal{D}$ satisfy (NS\ref{Strong_Structure_on_NetworkD}), we have
	\[
	\sup\limits_{D}\mathbb{L}^{\textup{MI}}(D,R^*)-\sup\limits_{D\in\mathcal{D}}\mathbb{L}^{\textup{MI}}(D,R^*)\le72(1+e^{\mathcal{B}_\mathcal{D}})C_{1,\beta}(d_Y+d_0,B_{2})n^{-\frac{\beta}{2\beta+d_Y+d_0}},
	\]
	where $C_{1,\beta}(d,a)$ is defined in part (b) of Lemma \ref{App_Push_part}.
\end{lemma}
\begin{proof}[Proof of Lemma \ref{App_MI_part}]
	Similar to  Lemma \ref{App_Push_part}, we get
	\begin{eqnarray*}
		&&\sup\limits_{D}\mathbb{L}^{\textup{MI}}(D,R^*)-\sup\limits_{D\in\mathcal{D}}\mathbb{L}^{\textup{MI}}(D,R^*)\\
		&\le&\inf\limits_{D\in\mathcal{D}}\biggl\{E_{P_{XY}}\left|D_{R^*}(Y,R^*(X))-D(Y,R^*(X))\right|+E_{P_{X}P_{Y}}\left|e^{D_{R^*}(Y,R^*(X))}-e^{D(Y,R^*(X))}\right|\biggr\},
	\end{eqnarray*}
	where the last inequality follows from the fact that,  for any $R\in\mathcal{R}$,
	\begin{eqnarray*}
		&&\sup\limits_{D}\mathbb{L}^{\textup{MI}}(D,R)-\sup\limits_{D\in\mathcal{D}}\mathbb{L}^{\textup{MI}}(D,R)\\
		&=&E_{P_{XY}}D_{R}(Y,R(X))-E_{P_{X}P_{Y}}e^{D_{R}(Y,R(X))}-\sup\limits_{D\in\mathcal{D}}\mathbb{L}^{\textup{MI}}(D,R)\\
		&=&\inf\limits_{D\in\mathcal{D}}\bigg[E_{P_{XY}}D_{R}(Y,R(X))-E_{P_{X}P_{Y}}e^{D_{R}(Y,R(X))}-\left\{E_{P_{XY}}D(Y,R(X))-E_{P_{X}P_{Y}}e^{D(Y,R(X))}\right\}\bigg]\\
		&\le&\inf\limits_{D\in\mathcal{D}}\biggl\{E_{P_{XY}}\left|D_{R}(Y,R(X))-D(Y,R(X))\right|+E_{P_{X}P_{Y}}\left|e^{D_{R}(Y,R(X))}-e^{D(Y,R(X))}\right|\biggr\},
	\end{eqnarray*}
	and $D_{R}(y,r)=\log\frac{p_{Y|R}(y|r)}{p_Y(y)}$. By Assumptions \ref{Weak_Assumption_on_NetworkR} and \ref{Strong_Assumption_on_NetworkR}, respectively, a similar argument to (\ref{Weak_Push_part_bound}) and (\ref{Push_part_bound}) can be applied here and we obtain
	\[
	\sup\limits_{D}\mathbb{L}^{\textup{MI}}(D,R^*)-\sup\limits_{D\in\mathcal{D}}\mathbb{L}^{\textup{MI}}(D,R^*)\le 36B_{1}^2(d_Y+d_0)^{1/2}(1+e^{\mathcal{B}_\mathcal{D}})n^{-\frac{1}{2+d_Y+d_0}}
	\]
	for part (a) and
	\[
	\sup\limits_{D}\mathbb{L}^{\textup{MI}}(D,R^*)-\sup\limits_{D\in\mathcal{D}}\mathbb{L}^{\textup{MI}}(D,R^*)\le72(1+e^{\mathcal{B}_\mathcal{D}})C_{1,\beta}(d_Y+d_0,B_{2})n^{-\frac{\beta}{2\beta+d_Y+d_0}},
	\]
	for part (b). The proof is finished.
\end{proof}
%%%%%%%%%%%%%%%%%%%%%%%%%%%%%%%%%%%%%%%%%%%%%%%%%%%%%%%%%%%
\begin{lemma}\label{upper_bound_empirical_MI_part}
	Let $\mathcal{D}\otimes\mathcal{R}=\{h(z)=D(y,R(x)):D\in\mathcal{D},R\in\mathcal{R}\}$, then
	
	(a) If the network parameters of $\mathcal{R}$ and $\mathcal{D}$ satisfy (NS\ref{Structure_on_NetworkR}) and (NS\ref{Structure_on_NetworkD}), then there exists a universal constant $C$, such that for $n~\textgreater~\textup{\Pdim}(\mathcal{D}\otimes\mathcal{R})\vee2$,
	\begin{eqnarray*}
		&&E\sup\limits_{D\in\mathcal{D},R\in\mathcal{R}}|\mathbb{L}_n^{\textup{MI}}(D,R)-\mathbb{L}^{\textup{MI}}(D,R)|\\
		&\le& C(1+\mathcal{B}_{\mathcal{D}})\{(d_Y+d_0)\vee d_X\}^{5/2}n^{-\frac{1}{2+(d_Y+d_0)\vee d_X}}\log n.
	\end{eqnarray*}
	
	(b)
	Suppose the network parameters of $\mathcal{R}$ and $\mathcal{D}$ satisfy (NS\ref{Strong_Structure_on_NetworkR}) and (NS\ref{Strong_Structure_on_NetworkD}), then there exists a universal constant $C$, such that for $n~\textgreater~\textup{\Pdim}(\mathcal{D}\otimes\mathcal{R})\vee2$,
	\[
	E\sup\limits_{D\in\mathcal{D},R\in\mathcal{R}}|\mathbb{L}_n^{\textup{MI}}(D,R)-\mathbb{L}^{\textup{MI}}(D,R)|\le CC_{2,\beta}((d_Y+d_0)\vee d_X,\mathcal{B}_{\mathcal{D}})n^{-\frac{\beta}{2\beta+(d_Y+d_0)\vee d_X}}.
	\]
	where $C_{2,\beta}(d,a)=(1+a)(\lfloor\beta \rfloor+1)^{9/2}d^{\lfloor\beta \rfloor+2}$.
\end{lemma}
\begin{proof}[Proof of Lemma \ref{upper_bound_empirical_MI_part}]
	We have
	\begin{eqnarray*}
		&&\sup\limits_{D\in\mathcal{D},R\in\mathcal{R}}|\mathbb{L}_n^{\textup{MI}}(D,R)-\mathbb{L}^{\textup{MI}}(D,R)|\\
		&=&\sup\limits_{D\in\mathcal{D},R\in\mathcal{R}}\bigg|\frac{1}{n}\sum_{i=1}^{n}D(Y_i,R(X_i))-\frac{1}{n(n-1)}\sum_{i\neq j}e^{D(Y_i,R(X_j))}\\
		&&~~~~~~~~~~~~~-(E_{p_{XY}}D(Y,R(X))-E_{p_{X}p_{Y}}e^{D(Y,R(X))})\bigg|\\
		&\le&\sup\limits_{D\in\mathcal{D},R\in\mathcal{R}}\left\{\frac{1}{n}\sum_{i=1}^{n}D(Y_i,R(X_i))-E_{p_{XY}}D(Y,R(X))\right\}\\
		&+&\sup\limits_{D\in\mathcal{D},R\in\mathcal{R}}\left|\frac{1}{n(n-1)}\sum_{i\neq j}e^{D(Y_i,R(X_j))}-E_{p_{X}p_{Y}}e^{D(Y,R(X))}\right|.
	\end{eqnarray*}
	Let $z_1=(x_1,y_1),z_2=(x_2,y_2)$,
	\[
	\mathcal{H}_{\mathcal{D},\mathcal{R}}=\left\{h_{D,R}(z_1,z_2)=\frac{e^{D(y_1,R(x_2))}+e^{D(y_2,R(x_1))}}{2}:D\in\mathcal{D},R\in\mathcal{R}\right\},
	\]
	and $\mathcal{H}^o_{\mathcal{D},\mathcal{R}}=-\mathcal{H}_{\mathcal{D},\mathcal{R}}\cup \mathcal{H}_{\mathcal{D},\mathcal{R}}$.  Using $Z_i$ to denote $(X_i,Y_i),i=1,2,\ldots,n$, we have
	\begin{eqnarray*}
		&&\sup\limits_{D\in\mathcal{D},R\in\mathcal{R}}\left|\frac{1}{n(n-1)}\sum_{i\neq j}e^{D(Y_i,R(X_j))}-E_{p_{X}p_{Y}}e^{D(Y,R(X))}\right|\\
		&=&\sup\limits_{h\in\mathcal{H}^o_{\mathcal{D},\mathcal{R}}}\left\{\frac{1}{n(n-1)}\sum_{i\neq j}h(Z_i,Z_j)-Eh(Z_i,Z_j)\right\}\\
		&\le&\sup\limits_{h\in\mathcal{H}^o_{\mathcal{D},\mathcal{R}}}\frac{1}{n(n-1)}\sum_{1\le i\neq j\le n}\left\{h(Z_i,Z_j)-\tilde{h}(Z_i)-\tilde{h}(Z_j)+Eh(Z_i,Z_j)\right\}\\
		&+&2\sup\limits_{h\in\mathcal{H}^o_{\mathcal{D},\mathcal{R}}}\left\{\frac{1}{n}\sum_{i=1}^n\tilde{h}(Z_i)-E\tilde{h}(Z_i)\right\},
	\end{eqnarray*}
	where $\tilde{h}(z)=Eh(z,Z)$. Hence to bound $E\sup\limits_{D\in\mathcal{D},R\in\mathcal{R}}|\mathbb{L}_n^{\textup{MI}}(D,R)-\mathbb{L}^{\textup{MI}}(D,R)|$, it suffices to bound
	\begin{equation}\label{MI_U_part}
		E\sup\limits_{h\in\mathcal{H}^o_{\mathcal{D},\mathcal{R}}}\frac{1}{n(n-1)}\sum_{1\le i\neq j\le n}\left\{h(Z_i,Z_j)-\tilde{h}(Z_i)-\tilde{h}(Z_j)+Eh(Z_i,Z_j)\right\},
	\end{equation}
	\begin{equation}\label{MI_U_hoeff_part}
		E\sup\limits_{h\in\mathcal{H}^o_{\mathcal{D},\mathcal{R}}}\left\{\frac{1}{n}\sum_{i=1}^n\tilde{h}(Z_i)-E\tilde{h}(Z_i)\right\},
	\end{equation}
	and
	\begin{equation}\label{MI_mean_part}
		E\sup\limits_{D\in\mathcal{D},R\in\mathcal{R}}\left\{\frac{1}{n}\sum_{i=1}^{n}D(Y_i,R(X_i))-E_{p_{XY}}D(Y,R(X))\right\}.
	\end{equation}
	For (\ref{MI_U_part}), since $\sup\limits_{D\in\mathcal{D}}\|D\|_{\infty}\le B_{\mathcal{D}}$, by Lemma \ref{Symmetrization}, we have
	\begin{eqnarray*}
		&&E\sup\limits_{h\in\mathcal{H}^o_{\mathcal{D},\mathcal{R}}}\frac{1}{n(n-1)}\sum_{1\le i\neq j\le n}\left\{h(Z_i,Z_j)-\tilde{h}(Z_i)-\tilde{h}(Z_j)+Eh(Z_i,Z_j)\right\}\\
		&\le&\inf\limits_{0\textless\delta\textless e^{\mathcal{B}_{\mathcal{D}}}}\left(\delta+\frac{24C\int_{\delta/4}^{e^{\mathcal{B}_{\mathcal{D}}}}\log(\mathcal{N}_{2n(n-1)}(u,\|\cdot\|_{\infty},\mathcal{H}^o_{\mathcal{D},\mathcal{R}}))du}{n}\right)
	\end{eqnarray*}
	for some universal constant $C$. Using Theorem 12.2 in \cite{anthony_bartlett_1999}, we have
	for $2n(n-1)~\textgreater~\Pdim(\mathcal{D}\otimes\mathcal{R})$, where $\mathcal{D}\otimes\mathcal{R}=\{h(z)=D(y,R(x)):D\in\mathcal{D},R\in\mathcal{R}\}$,
	\begin{eqnarray*}
		&&\mathcal{N}_{2n(n-1)}\left(u,\|\cdot\|_{\infty},\mathcal{H}^o_{\mathcal{D},\mathcal{R}}\right)\\
		&\le&2\mathcal{N}_{2n(n-1)}\left(e^{-\mathcal{B}_{\mathcal{D}}}u,\|\cdot\|_{\infty},\mathcal{D}\otimes\mathcal{R}\right)\\
		&\le& 2\left(\frac{4\mathcal{B}_\mathcal{D}e^{1+\mathcal{B}_\mathcal{D}}n(n-1)}{u\Pdim(\mathcal{D}\otimes\mathcal{R})}\right)^{\Pdim(\mathcal{D}\otimes\mathcal{R})}.
	\end{eqnarray*}
	For (\ref{MI_U_hoeff_part}), let $\{Z'_i\}_{i=1}^n$ be an independent copy of $\{Z_i\}_{i=1}^n$, $\epsilon=(\epsilon_1,\epsilon_2,\ldots,\epsilon_n)$ be the Rademacher sequence independent of $\{Z_i\}_{i=1}^n$ and $\{Z'_i\}_{i=1}^n$, and $\|\cdot\|_{g,n}=\left(\sum_{i=1}^{n}h^2(Z'_i,Z_i)\right)^{1/2}$, by Lemma \ref{Classic Symmetrization Inequality}, we have
	\begin{eqnarray*}
		&&E\sup\limits_{h\in\mathcal{H}^o_{\mathcal{D},\mathcal{R}}}\left\{\frac{1}{n}\sum_{i=1}^n\tilde{h}(Z_i)-E\tilde{h}(Z_i)\right\}\\
		&\le&2EE_{\epsilon}\sup\limits_{h\in\mathcal{H}^o_{\mathcal{D},\mathcal{R}}}\frac{1}{n}\sum_{i=1}^n\epsilon_i\tilde{h}(Z_i)\\
		&=&2EE_{\epsilon}\sup\limits_{h\in\mathcal{H}^o_{\mathcal{D},\mathcal{R}}}\frac{1}{n}\sum_{i=1}^n\epsilon_iE_{Z'}h(Z'_i,Z_i)\\
		&\le&2EE_{\epsilon}\sup\limits_{h\in\mathcal{H}^o_{\mathcal{D},\mathcal{R}}}\frac{1}{n}\sum_{i=1}^n\epsilon_ih(Z'_i,Z_i)\\
		&\le&8E\inf\limits_{0\textless\delta\textless e^{B_{\mathcal{D}}}}\left(\delta+\frac{3}{\sqrt{n}}\int_{\delta}^{e^{\mathcal{B}_{\mathcal{D}}}}\sqrt{\log(N(u,\mathcal{H}^o_{\mathcal{D},\mathcal{R}},\|\cdot\|_{g,n}))}du\right),
	\end{eqnarray*}
	where the last inequality follows from Lemma \ref{Classic Dudleys Chaining}. Using Theorem 12.2 in \cite{anthony_bartlett_1999} again, we have for $n~\textgreater~\Pdim(\mathcal{D}\otimes\mathcal{R})$,
	\[
	N(u,\mathcal{H}^o_{\mathcal{D},\mathcal{R}},\|\cdot\|_{g,n})\le\mathcal{N}_{n}\left(u,\|\cdot\|_{\infty},\mathcal{H}^o_{\mathcal{D},\mathcal{R}}\right)\le2 \left(\frac{2\mathcal{B}_\mathcal{D}e^{1+\mathcal{B}_\mathcal{D}}n}{u\Pdim(\mathcal{D}\otimes\mathcal{R})}\right)^{\Pdim(\mathcal{D}\otimes\mathcal{R})}.
	\]
	Similarly, for (\ref{MI_mean_part}), we can get
	\begin{eqnarray*}
		&&E\sup\limits_{D\in\mathcal{D},R\in\mathcal{R}}\left\{\frac{1}{n}\sum_{i=1}^{n}D(Y_i,R(X_i))-E_{p_{XY}}D(Y,R(X))\right\}\\
		&\le&8E\inf\limits_{0\textless\delta\textless \mathcal{B}_\mathcal{D}}\left(\delta+\frac{3}{\sqrt{n}}\int_{\delta}^{\mathcal{B}_\mathcal{D}}\sqrt{\log(N(u,\mathcal{D}\otimes\mathcal{R},\|\cdot\|_{n}))}du\right).
	\end{eqnarray*}
	Again, using Theorem 12.2 in \cite{anthony_bartlett_1999}, we have for $n~\textgreater~\Pdim(\mathcal{D}\otimes\mathcal{R})$,
	\[
	N(u,\mathcal{D}\otimes\mathcal{R},\|\cdot\|_{n})\le\mathcal{N}_{n}\left(u,\|\cdot\|_{\infty},\mathcal{D}\circ\mathcal{R}\right)\le \left(\frac{2\mathcal{B}_\mathcal{D}en}{u\Pdim(\mathcal{D}\otimes\mathcal{R})}\right)^{\Pdim(\mathcal{D}\otimes\mathcal{R})}.
	\]
	Putting together these integrals and covering number upper bounds, we can find a universal constant $C$ such that for $n~\textgreater~\Pdim(\mathcal{D}\otimes\mathcal{R})\vee2$,
	\begin{equation}\label{Empirical_MI_Bound_Pdim}
		E\sup\limits_{D\in\mathcal{D},R\in\mathcal{R}}|\mathbb{L}_n^{\textup{MI}}(D,R)-\mathbb{L}^{\textup{MI}}(D,R)|\le C(1+\mathcal{B}_{\mathcal{D}})\sqrt{\frac{\Pdim(\mathcal{D}\otimes\mathcal{R})\log n}{n}}.
	\end{equation}
	By part (a) of Lemma \ref{Network_Construction}, $\mathcal{D}\otimes\mathcal{R}$ can be implemented by a ReLU network $\mathcal{F}_{\mathcal{D}\otimes\mathcal{R}}$ with width $\mathcal{W}_{\mathcal{F}_{\mathcal{D}\otimes\mathcal{R}}}\le\max(\mathcal{W}_{\mathcal{D}},\mathcal{W}_{\mathcal{R}}+2d_Y,2(d_Y+d_0))\le 4\max(\mathcal{W}_{\mathcal{D}},\mathcal{W}_{\mathcal{R}})$, depth $\mathcal{L}_{\mathcal{F}_{\mathcal{D}\otimes\mathcal{R}}}\le \mathcal{L}_{\mathcal{D}}+\mathcal{L}_{\mathcal{R}}+1\le3\max(\mathcal{L}_{\mathcal{D}},\mathcal{L}_{\mathcal{R}})$, size $\mathcal{S}_{\mathcal{F}_{\mathcal{D}\otimes\mathcal{R}}}\le 48\max(\mathcal{S}_{\mathcal{R}},\mathcal{S}_{\mathcal{D}})$. By Theorem 6 in \citet{bartlett2019nearly}, for any ReLU network $\mathcal{F}$ with depth $\mathcal{L}$ and size $\mathcal{S}$, there exists a universal constant $C_1$, such that
	\[
	\Pdim(\mathcal{F})\le C_1\mathcal{S}\mathcal{L}\log \mathcal{S}.
	\]
	Based on these results, we shall proceed to prove part (a) and (b) separately.
	
	{\bf Proof of Part (a):} Given the settings in (NS\ref{Structure_on_NetworkR}) and (NS\ref{Structure_on_NetworkD}), there exists a universal constant $C_2$ such that
	\[
	\frac{\Pdim(\mathcal{D}\otimes\mathcal{R})\log n}{n}\le C_2 \{(d_Y+d_0)\vee d_X\}^5n^{-\frac{2}{2+(d_Y+d_0)\vee d_X}}\log n.
	\]
	Hence we can easily conclude that there exists a universal constant $C$ such that
	\begin{eqnarray}
		&&E\sup\limits_{D\in\mathcal{D},R\in\mathcal{R}}|\mathbb{L}_n^{\textup{MI}}(D,R)-\mathbb{L}^{\textup{MI}}(D,R)|\nonumber\\
		&\le& C(1+\mathcal{B}_{\mathcal{D}})\{(d_Y+d_0)\vee d_X\}^{5/2}n^{-\frac{1}{2+(d_Y+d_0)\vee d_X}}\log n.\label{Weak_Empirical_MI_Bound_exact}
	\end{eqnarray}
	
	{\bf Proof of Part (b):}
	
	Using the settings in (NS\ref{Strong_Structure_on_NetworkR}) and (NS\ref{Strong_Structure_on_NetworkD}), we can find a universal constant $C_3$, such that
	\[
	\frac{\Pdim(\mathcal{D}\otimes\mathcal{R})\log n}{n}\le C_3(\lfloor\beta \rfloor+1)^9\{(d_Y+d_0)\vee d_X\}^{2\lfloor\beta \rfloor+4}n^{-\frac{2\beta}{2\beta+(d_Y+d_0)\vee d_X}}.
	\]
	Applying this upper bound to (\ref{Empirical_MI_Bound_Pdim}), we conclude that there exists a universal constant $C$ such that
	\begin{eqnarray}
		&&E\sup\limits_{D\in\mathcal{D},R\in\mathcal{R}}|\mathbb{L}_n^{\textup{MI}}(D,R)-\mathbb{L}^{\textup{MI}}(D,R)|\nonumber\\
		&\le& C(1+\mathcal{B}_{\mathcal{D}})(\lfloor\beta \rfloor+1)^{9/2}\{(d_Y+d_0)\vee d_X\}^{\lfloor\beta \rfloor+2}n^{-\frac{\beta}{2\beta+(d_Y+d_0)\vee d_X}}.\label{Empirical_MI_Bound_exact}
	\end{eqnarray}
	This completes the proof of Lemma \ref{upper_bound_empirical_MI_part}.
\end{proof}

\begin{lemma}\label{upper_bound_empirical_Push_part}
	Let $\mathcal{Q}\circ\mathcal{R}=\{Q(R(x)):Q\in\mathcal{Q},R\in\mathcal{R}\}$.
	
	(a) Suppose Assumption \ref{Weak_Assumption_on_NetworkR} holds and the network parameters of $\mathcal{R}$ and $\mathcal{Q}$ satisfy (NS\ref{Structure_on_NetworkR}) and (NS\ref{Structure_on_NetworkQ}), then there exists a universal constant $C$ such that for $n~\textgreater~\textup{\Pdim}(\mathcal{Q}\circ\mathcal{R})\vee2$,
	\[ E\sup\limits_{Q\in\mathcal{Q},R\in\mathcal{R}}
	|\mathbb{L}_n^{\textup{Push}}(Q,R)-\mathbb{L}^{\textup{Push}}(Q,R)|\le C(1+\mathcal{B}_{\mathcal{Q}})d_X^{5/2}n^{-\frac{1}{2+d_X}}.
	\]
	(b) Suppose Assumption \ref{Strong_Assumption_on_NetworkR} holds and the network parameters of $\mathcal{R}$ and $\mathcal{Q}$ satisfy (NS\ref{Strong_Structure_on_NetworkR}) and (NS\ref{Strong_Structure_on_NetworkQ}), then there exists a universal constant $C$ such that for $n~\textgreater~\textup{\Pdim}(\mathcal{Q}\circ\mathcal{R})\vee2$,
	\[ E\sup\limits_{Q\in\mathcal{Q},R\in\mathcal{R}}
	|\mathbb{L}_n^{\textup{Push}}(Q,R)-\mathbb{L}^{\textup{Push}}(Q,R)|\le CC_{2,\beta}(d_X,\mathcal{B}_{\mathcal{Q}})n^{-\frac{\beta}{2\beta+d_X}},
	\]
	where $C_{2,\beta}(d,a)$ is defined in part (b) of Lemma \ref{upper_bound_empirical_MI_part}.
\end{lemma}
\begin{proof}[Proof of Lemma \ref{upper_bound_empirical_Push_part}]
	Note that
	\begin{eqnarray*}
		&&\sup\limits_{Q\in\mathcal{Q},R\in\mathcal{R}}|\mathbb{L}_n^{\textup{Push}}(Q,R)-\mathbb{L}^{\textup{Push}}(Q,R)|\\
		&=&\sup\limits_{Q\in\mathcal{Q},R\in\mathcal{R}}\bigg|\frac{1}{n}\sum_{i=1}^{n}Q(R(X_i))-\frac{1}{n}\sum_{i=1}^ne^{Q(U_i)}-(EQ(R(X))-Ee^{Q(U)})\bigg|\\
		&\le&\sup\limits_{Q\in\mathcal{Q},R\in\mathcal{R}}\left\{\frac{1}{n}\sum_{i=1}^{n}Q(R(X_i))-EQ(R(X))\right\}+\sup\limits_{Q\in\mathcal{Q},R\in\mathcal{R}}\left|\frac{1}{n}\sum_{i=1}^ne^{Q(R(X_j))}-Ee^{Q(U)}\right|.
	\end{eqnarray*}	
	Then the rest of the proof for this lemma is similar to those for (\ref{Weak_Empirical_MI_Bound_exact}) and (\ref{Empirical_MI_Bound_exact}).  This completes the proof of Lemma \ref{upper_bound_empirical_Push_part}.
\end{proof}
%%%%%%%%%%%%%%%%%%%%%%%%%%%%%%%%%%%%%%%%%%%%%%%%%%%%%%%%%%%
\begin{lemma}\label{obj_app_error_upper_bound}
	(a) Suppose Assumptions \ref{Weak_Assumption_on_TrueR} and \ref{Weak_Assumption_on_NetworkR} hold and the network parameters of $\mathcal{R}$ satisfy (NS\ref{Structure_on_NetworkR}), we have
	\[
	\mathbb{L}(R^*;\lambda)-\mathbb{L}(R_0;\lambda)\to0, \text{ as } n\to\infty,
	\]
	for any $\lambda\ge0$.
	
	(b) Suppose Assumptions \ref{Strong_Assumption_on_TrueR} and \ref{Strong_Assumption_on_NetworkR} hold and the network parameters of $\mathcal{R}$ satisfy (NS\ref{Strong_Structure_on_NetworkR}), we have
	\[
	\mathbb{L}(R^*;\lambda)-\mathbb{L}(R_0;\lambda)\le 36\sqrt{d_Y+d_0}C_{3,\beta}(\lambda,B_{2},d_X)n^{-\frac{\beta(\beta\wedge1)}{2\beta+d_X}},
	\]
	where $C_{3,\beta}(\lambda,a,d)=(\lfloor\beta \rfloor+1)^2d^{\lfloor\beta \rfloor+(\beta\vee 1)/2}(\lambda \vee2)a^4$.
\end{lemma}
\begin{proof}[Proof of Lemma \ref{obj_app_error_upper_bound}]
	First note that for any $R\in\mathcal{R}$,
	\begin{eqnarray}
		&&\mathbb{L}(R;\lambda)-\mathbb{L}(R_0;\lambda)\nonumber\\
		&=&\mathbb{I}_{\textup{KL}}(Y;R_0(X))-\mathbb{I}_{\textup{KL}}(Y;R(X)) + \lambda \mathbb{D}_{\textup{KL}}(p_{R}~||~\gamma_{N})\nonumber\\
		&=&\sup\limits_{D}\mathbb{L}^{\textup{MI}}(D,R_0)-\sup\limits_{D}\mathbb{L}^{\textup{MI}}(D,R)+\lambda\sup\limits_{Q}\mathbb{L}^{\textup{Push}}(Q,R)\nonumber\\
		&=&E_{P_{XY}}D_{R_0}(Y,R_0(X))-E_{P_{X}P_{Y}}e^{D_{R_0}(Y,R_0(X))\nonumber}\\
		&-&\sup\limits_{D}\left\{E_{P_{XY}}D(Y,R(X))-E_{P_{X}P_{Y}}e^{D(Y,R(X))}\right\}\nonumber\\
		&+&\lambda\left[\left\{EQ_{R}(R(X))-Ee^{Q_R(U)}\right\}-\sup_{Q}\left\{EQ(R_0(X))-Ee^{Q(U)}\right\}\right]\nonumber\\
		&\le&E_{P_{XY}}\left|D_{R_0}(Y,R_0(X))-D_{R_0}(Y,R(X))\right|\nonumber\\
		&+&E_{P_{X}P_{Y}}\left|e^{D_{R_0}(Y,R_0(X))}-e^{D_{R_0}(Y,R(X))}\right|\nonumber\\
		&+&\lambda E|Q_{R}(R(X))-Q_R(R_0(X))|,
		\label{R_approx_error}
	\end{eqnarray}
	where $D_{R_0}(y,r)=\log\frac{p_{Y|R_0}(y|r)}{p_{Y}(y)}$ and $Q_{R}(r)=\log p_R(r)$.
	
	{\bf Proof of Part (a):} As $R_0(x)=(R_{0,1}(x),\ldots,R_{0,d_0}(x)),R_{0,i}(x)\in\mathcal{F}_{\textup{C}}([0,1]^{d_X})$ for any $i\in\{1,2,\ldots,d_0\}$ by Assumption \ref{Weak_Assumption_on_TrueR}, by applying
%Lemma \ref{Approximation Error_onlyC}
Theorem 2.1 in \cite{Shen2020CICP}
with $N=\left\lceil n^{\frac{d_X}{2(2+d_X)}} /\log n\right\rceil$ and $M=\lceil\log n\rceil$, we find a ReLU network $\bar{R}_{i,n}\in\mathcal{\bar{R}}$ with depth $\mathcal{L}_{\mathcal{\bar{R}}}=L(d_X)$, width $\mathcal{W}_{\mathcal{\bar{R}}}=W(d_X)$ such that
	\[
	\sup\limits_{x\in [0,1]^{d_X}\backslash H_{K,\delta}}|R_{0,i}-\bar{R}_{i,n}|
	\le 18d_X^{1/2}\omega_{R_{0,i}}(n^{-\frac{1}{2+d_X}}),
	\]
	where $H_{K,\delta}=\cup_{i=1}^{d_X}\big\{x=[x_1,\ldots,x_{d_X}]:x_i\in\cup_{b=1}^{K-1}\left(b/K-\delta,b/K\right)\big\},K=\lceil (MN)^{2/d_X} \rceil,\delta\in (0,1/(3K)]$.
	
	Let $T(t)=tI\{0\le t \le 1\}+I\{t\textgreater 1\}$ be the truncation function taking values in $[0,1]$, then $0\le T\circ\bar{R}_{i,n}(x)\le1$ for any $x\in[0,1]^{d_X}$ and it still holds
	\[
	\sup\limits_{x\in [0,1]^{d_X}\backslash H_{K,\delta}}| R_{0,i}-T\circ\bar{R}_{i,n}|
	\le 18d_X^{1/2}\omega_{R_{0,i}}(n^{-\frac{1}{2+d_X}}).
	\]
	Note that $T(t)=\sigma(t)-\sigma(\sigma(t)-1)$, where $\sigma(\cdot)$ is the ReLU activation function. We get $T(\cdot)$ can be implemented by a ReLU network with depth $2$ and width $2$. By part (b) of Lemma \ref{Network_Construction}, the composition function $T\circ\bar{R}_{i,n}$ can be implemented by a ReLU network with width $W(d_X)$, depth $L(d_X)+3$, and size $\le 4S(d_X)$. In summary, when the network parameters of $\mathcal{R}$ satisfy (NS\ref{Structure_on_NetworkR}), we get a $\bar{R}_{n}\in\mathcal{R}$ such that
	\begin{equation}\label{truncated_R_approx}
		\sup\limits_{x\in [0,1]^{d_X}\backslash H_{K,\delta}}\|R_0(x)-\bar{R}_{n}(x)\|_\infty
		\le 18d_X^{1/2}\max_{i=1,2,\ldots,d_0}\omega_{R_{0,i}}(n^{-\frac{1}{2+d_X}}).
	\end{equation}
	By the absolute continuity of $X$ and a similar argument to (\ref{Weak_Push_part_bound}), we have
	\begin{equation}\label{truncated_R_approx0}
		E\|R_0(X)-\bar{R}_{n}(X)\|\le18(d_0d_X)^{1/2}\max_{i=1,2,\ldots,d_0}\omega_{R_{0,i}}(n^{-\frac{1}{2+d_X}})\to0, \text{ as } n\to\infty,
	\end{equation}
	which implies $\bar{R}_{n}(X)\xrightarrow{p}R_0(X)$, where "$\xrightarrow{p}$" means convergence in probability.
	
	By Assumption \ref{Weak_Assumption_on_TrueR}, we get $D_{R_0}(y,r)$ is continuous in $(y,r)$ on $[0,1]^{d_Y+d_0}$.  By Mann-Wald theorem, the continuity of $D_{R_0}(y,r)$ and $\bar{R}_{n}(X)\xrightarrow{p}R_0(X)$, we have
	\[
	D_{R_0}(Y,\bar{R}_{n}(X))\xrightarrow{p} D_{R_0}(Y,R_0(X)),\text{ as }n\to\infty.
	\]
	Using the boundedness of $D_{R_0}(y,r)$ and the Lebesgue's dominated convergence theorem, we obtain
	\[
	E_{P_{XY}}\left|D_{R_0}(Y,R_0(X))-D_{R_0}(Y,\bar{R}_{n}(X))\right|\to0,\text{ as }n\to\infty,
	\]
	and
	\[
	E_{P_{X}P_{Y}}\left|e^{D_{R_0}(Y,R_0(X))}-e^{D_{R_0}(Y,\bar{R}_{n}(X))}\right|\to0,\text{ as }n\to\infty.
	\]
	By Assumption \ref{Weak_Assumption_on_NetworkR}, we know $Q_{R}$ is a Lipschitz continuous and bounded function with a Lipschitz constant $B_{1}^2$. Using (\ref{truncated_R_approx0}), we have
	\[
	E|Q_{\bar{R}_{n}}(\bar{R}_{n}(X))-Q_{\bar{R}_{n}}(R_0(X))|\le B_{1}^2E\|\bar{R}_{n}(X)-R_0(X)\|\to0,\text{ as }n\to\infty.
	\]
	Invoking these results, the definition of $R^*$, that is $R^*\in\argmin_{R\in \mathcal{R}}\mathbb{L}(R;\lambda)$, and (\ref{R_approx_error}), we have
	\[
	0\le\mathbb{L}(R^*;\lambda)-\mathbb{L}(R_0;\lambda)\le\mathbb{L}(\bar{R}_{n};\lambda)-\mathbb{L}(R_0;\lambda)\to0,\text{ as }n\to\infty.
	\]

	{\bf Proof of Part (b):} For part (b), by (\ref{R_approx_error}), we further have for any $R\in\mathcal{R}$,
	\begin{eqnarray}
		&&\mathbb{L}(R;\lambda)-\mathbb{L}(R_0;\lambda)\nonumber\\
		&\le&\begin{cases}
			(2B_{2}^3+\lambda B_{2})	E_X\|R(X)-R_0(X)\|^{\beta}&, \beta\textless 1\\
			(2\sqrt{d_Y+d_0}B_{2}^3+\lambda\sqrt{d_0}B_{2})E_X\|R(X)-R_0(X)\|&, \beta\ge 1
		\end{cases}\label{obj_app_error_cases}\\
		&\le&2(\lambda \vee2)\sqrt{d_Y+d_0}B_{2}^3E_X\|R(X)-R_0(X)\|^{\beta\wedge1},\label{obj_app_error}
	\end{eqnarray}
	
	The first part of (\ref{obj_app_error_cases}) is derived as below.  When $\beta\textless 1$, for any $(y_i,r_i)\in[0,1]^{d_Y}\times[0,1]^{d_0},i=1,2$, we have
	\begin{eqnarray}
		&&|D_{R_0}(y_2,r_2)-D_{R_0}(y_1,r_1)|\nonumber\\
		&\le&|\log p_{Y|R_0}(y_2|r_2)-\log p_{Y|R_0}(y_1|r_1)|+|\log p_{Y}(y_2)-\log p_{Y}(y_1)|\nonumber\\
		&\le&1/c_1'|p_{Y|R_0}(y_2|r_2)-p_{Y|R_0}(y_1|r_1)|+1/c_2'|p_{Y}(y_2)-p_{Y}(y_1)|\label{mean_value_result}\\
		&\le&1/c_1'(\|y_2-y_1\|^2+\|r_2-r_1\|^2)^{\beta/2}+1/c_2'\|y_2-y_1\|^\beta\label{holder_def}\\
		&\le&B_{2}(\|y_2-y_1\|^2+\|r_2-r_1\|^2)^{\beta/2}+B_{2}\|y_2-y_1\|^\beta,\label{Assump3_results}
	\end{eqnarray}
	where $c_1'$ lies between $p_{Y|R_0}(y_2|r_2)$ and $p_{Y|R_0}(y_1|r_1)$, and $c_2'$ lies between $p_{Y}(y_2)$ and $p_{Y}(y_1)$. (\ref{mean_value_result}) is due to the mean value theorem, (\ref{holder_def}) is because $p_{Y|R_0}(y|r)\in\mathcal{H}^\beta([0,1]^{d_Y+d_0},B_{2})$, and (\ref{Assump3_results}) follows from $\inf_{y,r}p_{Y|R_0}(y,r)\ge 1/ B_{2}$ and $\inf_{y}p_{Y}\ge 1/B_{2}$ by Assumption \ref{Strong_Assumption_on_TrueR}.  Similarly, by Assumption \ref{Strong_Assumption_on_NetworkR}, we have
	\begin{equation}\label{Q_beta_bound}
		|Q_{R}(r_2)-Q_{R}(r_1)|\le B_{2}\|r_2-r_1\|^\beta.
	\end{equation}
	Using these facts and that $\|D_{R_0}\|_\infty\le2\log B_{2}$ by Assumption \ref{Strong_Assumption_on_TrueR}, we can easily obtain the first part of (\ref{obj_app_error_cases}).  The proof for the second part of (\ref{obj_app_error_cases}) follows the same line of argument as that for the first part of (\ref{obj_app_error_cases}), with $\beta$ in (\ref{Assump3_results}) and (\ref{Q_beta_bound}), $B_{2}$ in (\ref{Assump3_results}) and $B_{2}$ in (\ref{Q_beta_bound}) replaced by $1,\sqrt{d_Y+d_0}B_{2},\sqrt{d_0}B_{2}$, respectively.
	
	As $R_0(x)=(R_{0,1}(x),R_{0,2}(x),\ldots,R_{0,d_0}(x))$ and $R_{0,i}(x)\in\mathcal{H}^\beta([0,1]^{d_X},B_{2})$ for any $i\in\{1,2,\ldots,d_0\}$ by Assumption \ref{Strong_Assumption_on_TrueR}, by applying
% Lemma \ref{Approximation Error}
Theorem 3.3 in \citet{jiao2021deep}
with $N=1$ and $M=\left\lceil n^{\frac{d_X}{2(2+d_X)}}\right\rceil$, and the same line of argument as that for (\ref{truncated_R_approx}), when the network parameters of $\mathcal{R}$ satisfy (NS\ref{Structure_on_NetworkR}), we find a $\bar{R}_{n}\in\mathcal{R}$ such that
	\[
	\sup\limits_{x\in [0,1]^{d_X}\backslash H_{K,\delta}}\|R_0(x)-\bar{R}_{n}(x)\|_\infty
	\le 18B_{2}(\lfloor\beta \rfloor+1)^2d_X^{\lfloor\beta \rfloor+(\beta\vee 1)/2}n^{-\frac{\beta}{2\beta+d_X}},
	\]
	where $H_{K,\delta}=\cup_{i=1}^{d_X}\big\{x=[x_1,\ldots,x_{d_X}]:x_i\in\cup_{b=1}^{K-1}\left(b/K-\delta,b/K\right)\big\},K=\lceil (MN)^{2/d_X} \rceil,\delta\in (0,1/(3K)]$.  Due to the arbitrariness of $\delta$, we have
	\[
	E_{X}\|\bar{R}_{n}(X)-R_0(X)\|^{\beta\wedge1}\le 18B_{2}(\lfloor\beta \rfloor+1)^2d_X^{\lfloor\beta \rfloor+(\beta\vee 1)/2}n^{-\frac{\beta(\beta\wedge1)}{2\beta+d_X}}.
	\]
	Invoking the definition of $R^*$, that is $R^*\in\argmin_{R\in \mathcal{R}}\mathbb{L}(R;\lambda)$, and (\ref{obj_app_error}), we have
	\begin{eqnarray*}
		&&\mathbb{L}(R^*;\lambda)-\mathbb{L}(R_0;\lambda)\\&\le&\mathbb{L}(\bar{R}_{n};\lambda)-\mathbb{L}(R_0;\lambda)\\
		&\le&36(\lambda \vee2)\sqrt{d_Y+d_0}B_{2}^4(\lfloor\beta \rfloor+1)^2d_X^{\lfloor\beta \rfloor+(\beta\vee 1)/2}n^{-\frac{\beta(\beta\wedge1)}{2\beta+d_X}}.
	\end{eqnarray*}
	This completes the proof of Lemma \ref{obj_app_error_upper_bound}.
\end{proof}

\section{Supporting Lemmas}
% and Some Proofs}
\label{appC}
\begin{lemma}[Example 11.2 in \cite{Villani:2009}]\label{OT_map}
	Let $\mu$ be a probability measure on $\mathbb{R}^d$ and suppose it has finite second moment and is absolutely continuous with respect to the standard Gaussian measure $\gamma_d$ on $\mathbb{R}^d$. Then there exists a unique optimal transportation map $T:\mathbb{R}^d\to\mathbb{R}^d$, which is also injective $\mu$-almost everywhere,
	such that $T_{\#}\mu=\gamma_d\equiv N(0, I_d)$,
	where $T_{\#}\mu(A)=\mu(T^{-1}(A))$ for any Borel measurable set $A$ in $\mathbb{R}^d$.
\end{lemma}
Given $Z_1,Z_2,\ldots,Z_n$ and a class $\mathcal{H}$ of measurable real-valued functions on $\mathcal{Z}$, for any $h\in\mathcal{H}$, we define
\[
\|h\|_n=\left(\sum_{i=1}^{n}h^2(Z_i)\right)^{1/2},
\]
and the empirical Rademacher complexity for $\mathcal{H}$ w.r.t $Z_1,Z_2,\ldots,Z_n$, that is
\[
\hat{R}_n(\mathcal{H})=E_{\epsilon}\sup\limits_{h\in\mathcal{H}}\frac{1}{n}\sum_{i=1}^n\epsilon_ih(Z_i).
\]
\begin{lemma}[
%Classic
Symmetrization Inequality, Lemma 6.3.2 in \cite{vershynin2018}]\label{Classic Symmetrization Inequality}
	Assume $\mathcal{H}$ is a function class of measurable functions, then
	\[
	E_{\epsilon}\sup\limits_{h\in\mathcal{H}}\left\{\frac{1}{n}\sum_{i=1}^nh(Z_i)-Eh(Z)\right\}\le 2E\hat{R}_n(\mathcal{H}).
	\]
\end{lemma}

\begin{lemma}[Classic Dudley's Chaining, Lemma 3 in \cite{farrell2021deep}]\label{Classic Dudleys Chaining}
	Let $N(u,\mathcal{H},\|\cdot\|_n)$ denote the covering number for class $\mathcal{H}$ with covering radius $u$ and w.r.t metric $\|\cdot\|_n$ and assume $\sup\limits_{h\in\mathcal{H}}\|h\|_{n}\le b$, then
	\[
	\hat{R}_n(\mathcal{H})\le\inf\limits_{0\textless\delta\textless b}\left(4\delta+\frac{12}{\sqrt{n}}\int_{\delta}^{b}\sqrt{\log(N(u,\mathcal{H},\|\cdot\|_n))}du\right).
	\]
\end{lemma}
\begin{lemma}\label{Rademacher chaos}
	Let $\mathcal{A}$ be a finite calss of $n\times n$ diagonal-free matrices, and $\epsilon_1,\epsilon_2,\ldots,\epsilon_n$ are i.i.d. Rademacher variables, then there exists a universal constant $C$ such that
	\[
	E_{\epsilon}\max\limits_{A\in\mathcal{A}}\sum_{1\le i\neq j\le n}\epsilon_i\epsilon_ja_{ij}\le C\max\limits_{A\in\mathcal{A}}\|A\|_F\log|\mathcal{A}|,
	\]
	where $\epsilon=(\epsilon_1,\epsilon_2,\ldots,\epsilon_n)^\top$ and $\|A\|_F^2=\sum_{i,j}a_{ij}^2$.
\end{lemma}
\begin{proof}[Proof of Lemma \ref{Rademacher chaos}]
	When $|\mathcal{A}|=1$, the inequality is trivial. Hence, in what follows, we assume $|\mathcal{A}|\textgreater1$. Note that for any $\lambda\textgreater0$,
	\begin{eqnarray}
		E_{\epsilon}\max\limits_{A\in\mathcal{A}}\sum_{1\le i\neq j\le n}\epsilon_i\epsilon_ja_{ij}&=&E_{\epsilon}\max\limits_{A\in\mathcal{A}}\epsilon^\top A\epsilon\nonumber\\
		&=&1/\lambda E_{\epsilon}\log\max\limits_{A\in\mathcal{A}}\exp(\lambda\epsilon^\top A\epsilon)\nonumber\\
		&\le&1/\lambda E_{\epsilon}\log\left(\sum_{A\in\mathcal{A}}\exp(\lambda\epsilon^\top A\epsilon)\right)\nonumber\\
		&\le&1/\lambda \log \left(\sum_{A\in\mathcal{A}}E_{\epsilon}\exp(\lambda\epsilon^\top A\epsilon)\right),\label{intermediate_bound_0ademacher chaos}
	\end{eqnarray}
	where the last inequality follows from Jensen's inequality. For any $n\times n$, diagonal-free matrix $A$, by applying Lemma \ref{Decoupling} with $f(x)=\exp(\lambda x)$, it holds
	\begin{equation}\label{Decoupling_upper_bound}
		E\exp(\lambda\epsilon^\top A\epsilon)\le E\exp(4\lambda\epsilon^\top A\epsilon'),
	\end{equation}
	where $\epsilon'$ is an independent copy of $\epsilon$. As $\|\epsilon_i\|_{\psi_2}=\|\epsilon_i'\|_{\psi_2}=1/\log 2$ for any $i\in [n]$, by Lemma \ref{sub-gaussian-for-Vec}, there exists a universal and positive constant $C_1$ such that
	\[
	\|\epsilon\|_{\psi_2}=\|\epsilon'\|_{\psi_2}\le C_1.
	\]
	Then by Lemma \ref{Comparison to Gaussian chaos} and \ref{Upper Bounds Gaussian chaos}, we have
	\begin{eqnarray}
		E\exp(4\lambda\epsilon^\top A\epsilon')&\le&E\exp(4C_2C_1^2\lambda g^\top Ag')\nonumber\\
		&\le&\exp(C\lambda^2\|A\|_F^2)\label{Final_bound_0ademacher chaos}
	\end{eqnarray}
	for all $\lambda$ satisfying $|\lambda|\le c/\|A\|_F$ and some universal and positive constants $c,C,C_2$.
	
	Combining (\ref{intermediate_bound_0ademacher chaos}), (\ref{Decoupling_upper_bound}), and (\ref{Final_bound_0ademacher chaos}), we have
	\begin{eqnarray}
		E_{\epsilon}\max\limits_{A\in\mathcal{A}}\sum_{1\le i\neq j\le n}\epsilon_i\epsilon_ja_{ij}&\le&1/\lambda \log \left(\sum_{A\in\mathcal{A}}E_{\epsilon}\exp(\lambda\epsilon^\top A\epsilon)\right)\nonumber\\
		&\le&1/\lambda \log \left(\sum_{A\in\mathcal{A}}\exp(C\lambda^2\|A\|_F^2)\right)\nonumber\\
		&\le&1/\lambda \log \left(|\mathcal{A}|\exp(C\lambda^2\max\limits_{A\in\mathcal{A}}\|A\|_F^2)\right)\nonumber\\
		&=&\frac{\log |\mathcal{A}|}{\lambda}+C\lambda \max\limits_{A\in\mathcal{A}}\|A\|_F^2\nonumber
	\end{eqnarray}
	for all $\lambda$ satisfying $|\lambda|\le c/\max\limits_{A\in\mathcal{A}}\|A\|_F$. Setting $\lambda= c/\max\limits_{A\in\mathcal{A}}\|A\|_F$, we have
	\[
	E_{\epsilon}\max\limits_{A\in\mathcal{A}}\sum_{1\le i\neq j\le n}\epsilon_i\epsilon_ja_{ij}\le \frac{\log |\mathcal{A}|\max\limits_{A\in\mathcal{A}}\|A\|_F}{c}+cC\max\limits_{A\in\mathcal{A}}\|A\|_F\le C\max\limits_{A\in\mathcal{A}}\|A\|_F\log|\mathcal{A}|.
	\]
	This completes the proof of Lemma \ref{Rademacher chaos}.
\end{proof}

\begin{lemma}[Theorem 6.1.1 in \cite{vershynin2018}]\label{Decoupling}
	Let $A$ be an $n\times n$, diagonal-free matrix. Let $X=(X_1,X_2,\ldots,X_n)^\top$ be a random vector with independent mean zero coordinates $X_i$. Then, for every convex function $f:\mathbb{R}\to\mathbb{R}$, one has
	\[
	Ef(X^\top AX)\le Ef(4X^\top AX'),
	\]
	where $X'$ is an independent copy of $X$.
\end{lemma}

\begin{lemma}[Lemma 6.2.3 in \cite{vershynin2018}]\label{Comparison to Gaussian chaos}
	Consider independent mean-zero sub-gaussian random vectors $X,X'$ in $\mathbb{R}^n$ with $\|X\|_{\psi_2}\le K$ and $\|X'\|_{\psi_2}\le K$, where
	$\|X\|_{\psi_2}=\sup_{x\in\mathcal{S}^{n-1}}\|\langle X,x\rangle\|_{\psi_2}$ and $\|W\|_{\psi_2}=\inf\{t\textgreater0:E\exp(W^2
	/t^2)\le 2\}$ for a scalar random variable $W$. Consider also independent random vectors $g,g'\sim N(0, I_n)$. Let $A$ be an $n\times n$ matrix. Then there exists a universal constant and positive $C$ such that
	\[
	E\exp(\lambda X^\top AX')\le E\exp(CK^2\lambda g^\top Ag')
	\]
	for any $\lambda\in\mathbb{R}$.
\end{lemma}

\begin{lemma}[Lemma 3.4.2 in \cite{vershynin2018}]\label{sub-gaussian-for-Vec}
	Let $X=(X_1,X_2,\ldots,X_n)^\top$ be a random vector with independent, mean zero, sub-gaussian coordinates $X_i$. Then $X$ is a sub-gaussian random vector, and there exists a universal constant and positive $C$ such that
	\[
	\|X\|_{\psi_2}\le C\max\limits_{1\le i\le n}\|X_i\|_{\psi_2}.
	\]
\end{lemma}

\begin{lemma}[Lemma 6.2.2 in \cite{vershynin2018}]\label{Upper Bounds Gaussian chaos}
	Let $X,X'\sim N(0, I_n)$ be independent and let $A = (a_{ij})$ be an $n\times n$ matrix. Then there exists two universal and positive constants $c,C$ such that
	\[
	E\exp(\lambda X^\top AX')\le \exp(C\lambda^2\|A\|_F^2)
	\]
	for all $\lambda$ satisfying $|\lambda|\le c/\|A\|_F$.
\end{lemma}
\begin{proof}[Proof of Lemma \ref{Symmetrization}] Note that
	\begin{eqnarray*}
		&&h(Z_i,Z_j)-\tilde{h}(Z_i)-\tilde{h}(Z_j)+Eh(Z_i,Z_j)\\
		&=&h(Z_i,Z_j)-E_{Z'}h(Z_i,Z'_j)-E_{Z'}h(Z'_i,Z_j)+E_{Z'}h(Z'_i,Z'_j).
	\end{eqnarray*}
	Then we have
	\begin{eqnarray*}
		&&E\sup\limits_{h\in\mathcal{H}}\mathcal{U}_n^{\textup{Deg}}(h)\\
		&=&E\sup\limits_{h\in\mathcal{H}}\frac{1}{n(n-1)}\sum_{1\le i\neq j\le n}\left\{h(Z_i,Z_j)-\tilde{h}(Z_i)-\tilde{h}(Z_j)+Eh(Z_i,Z_j)\right\}\\
		&=&E\sup\limits_{h\in\mathcal{H}}\frac{1}{n(n-1)}\sum_{1\le i\neq j\le n}\left\{h(Z_i,Z_j)-E_{Z'}h(Z_i,Z'_j)-E_{Z'}h(Z'_i,Z_j)+E_{Z'}h(Z'_i,Z'_j)\right\}\\
		&\le&EE_{Z'}\sup\limits_{h\in\mathcal{H}}\frac{1}{n(n-1)}\sum_{1\le i\neq j\le n}\left\{h(Z_i,Z_j)-h(Z_i,Z'_j)-h(Z'_i,Z_j)+h(Z'_i,Z'_j)\right\}\\
		&=&EE_{\epsilon}\sup\limits_{h\in\mathcal{H}}\frac{1}{n(n-1)}\sum_{1\le i\neq j\le n}\epsilon_i\epsilon_j\left\{h(Z_i,Z_j)-h(Z_i,Z'_j)-h(Z'_i,Z_j)+h(Z'_i,Z'_j)\right\},
	\end{eqnarray*}
	where the inequality follows from Jensen's inequality and the last equality is due to Fubini's Theorem and a direct corollary of the symmetrization inequality in \cite{Sherman1994}.
	
	We prove the second inequality in (\ref{Dudley_Rad_Chaos}) through the chaining technique.  Let $\alpha_0=b$ and for any $t\in \mathbb{N}_+$, let $\alpha_t=2^{-t}b$. For each $t$, let $\mathcal{C}_t$ be a $\alpha_t$-cover of $\mathcal{H}$ w.r.t. $d_{U,n}(\cdot,\cdot)$ such that $|\mathcal{C}_t| = N(\alpha_t, \mathcal{H}, d_{U,n})$. For each $h\in\mathcal{H}$, there exists a function $\hat{h}_t\in \mathcal{C}_t$ such that $d_{U,n}(h,\hat{h}_t)\le\alpha_t$. Let $\hat{h}_0\equiv0$ and for any $T\in\mathbb{N}_+$, we have the following chaining expression for $h$:
	\[
	h=h-\hat{h}_T+\sum_{t=1}^{T}(\hat{h}_t-\hat{h}_{t-1}).
	\]
	Let
	\[
	G(h,Z_i,Z_j,Z'_i,Z'_j)=h(Z_i,Z_j)-h(Z_i,Z'_j)-h(Z'_i,Z_j)+h(Z'_i,Z'_j).
	\]
	Hence for any $T\in\mathbb{N}_+$, we have
	\begin{eqnarray}
		\mathcal{S}_n^{\textup{Deg}}(\mathcal{H})&=&E_{\epsilon}\sup\limits_{h\in\mathcal{H}}\frac{1}{n(n-1)}\sum_{1\le i\neq j\le n}\epsilon_i\epsilon_jG(h,Z_i,Z_j,Z'_i,Z'_j)\nonumber\\
		&\le&E_{\epsilon}\sup\limits_{h\in\mathcal{H}}\frac{1}{n(n-1)}\sum_{1\le i\neq j\le n}\epsilon_i\epsilon_jG(h-\hat{h}_T,Z_i,Z_j,Z'_i,Z'_j)\nonumber\\
		&+&\sum_{t=1}^{T}E_{\epsilon}\sup\limits_{h\in\mathcal{H}}\frac{1}{n(n-1)}\sum_{1\le i\neq j\le n}\epsilon_i\epsilon_jG(\hat{h}_t-\hat{h}_{t-1},Z_i,Z_j,Z'_i,Z'_j)\nonumber\\
		&\le&\alpha_T+\sum_{t=1}^{T}E_{\epsilon}\sup\limits_{h\in\mathcal{H}}\frac{1}{n(n-1)}\sum_{1\le i\neq j\le n}\epsilon_i\epsilon_jG(\hat{h}_t-\hat{h}_{t-1},Z_i,Z_j,Z'_i,Z'_j),\label{Chaining_upper_bound}
	\end{eqnarray}
	where the last inequality is due to
	\begin{eqnarray*}
		&&\frac{1}{n(n-1)}\sum_{1\le i\neq j\le n}\epsilon_i\epsilon_jG(h-\hat{h}_T,Z_i,Z_j,Z'_i,Z'_j)\\
		&\le&\left\{\frac{1}{n(n-1)}\sum_{1\le i\neq j\le n}(\epsilon_i\epsilon_j)^2\right\}^{1/2}\left\{\frac{1}{n(n-1)}\sum_{1\le i\neq j\le n}G^2(h-\hat{h}_T,Z_i,Z_j,Z'_i,Z'_j)\right\}^{1/2}\\
		&=&d_{U,n}(h,\hat{h}_T)\le\alpha_T.
	\end{eqnarray*}
	Now the second term in (\ref{Chaining_upper_bound}) is the summation of empirical Rademacher chaos w.r.t. the function classes $\mathcal{C}_t-\mathcal{C}_{t-1}=\{h_t-h_{t-1}:h_t\in\mathcal{C}_t,h_{t-1}\in\mathcal{C}_{t-1}\},t=1,2,\ldots,T$. For any measurable function $h\in\mathcal{H}$, let
	\[
	M(h)=(a^h_{ij})\in \mathbb{R}^{n\times n}~\text{with}~a^h_{ij}=0,~\text{if}~i=j~\text{and otherwise}~a^h_{ij}=G(\hat{h}_t-\hat{h}_{t-1},Z_i,Z_j,Z'_i,Z'_j).
	\]
	Then for any $h_t\in\mathcal{C}_t,h_{t-1}\in\mathcal{C}_{t-1}$,
	\begin{eqnarray*}
		&&1/\sqrt{n(n-1)}\|M(\hat{h}_t-\hat{h}_{t-1})\|_F\\
		&=&\left\{\frac{1}{n(n-1)}\sum_{1\le i\neq j\le n}G^2(\hat{h}_t-\hat{h}_{t-1},Z_i,Z_j,Z'_i,Z'_j)\right\}^{1/2}\\
		&=&d_{U,n}(\hat{h}_t,\hat{h}_{t-1})\\
		&\le&d_{U,n}(h,\hat{h}_{t})+d_{U,n}(h,\hat{h}_{t-1})\\
		&\le&\alpha_t+\alpha_{t-1}\\
		&=&3\alpha_t.
	\end{eqnarray*}
	Applying Lemma \ref{Rademacher chaos} to $\mathcal{C}_t-\mathcal{C}_{t-1},t=1,2,\ldots,T$, we have there exists a universal constant $C$ such that
	\begin{eqnarray*}
		&&\mathcal{S}_n^{\textup{Deg}}(\mathcal{H})\\
		&\le&\alpha_T+\sum_{t=1}^{T}E_{\epsilon}\sup\limits_{h\in\mathcal{H}}\frac{1}{n(n-1)}\sum_{1\le i\neq j\le n}\epsilon_i\epsilon_jG(\hat{h}_t-\hat{h}_{t-1},Z_i,Z_j,Z'_i,Z'_j)\\
		&\le&\alpha_T+\frac{3C\sum_{t=1}^{T}\alpha_t\log(|\mathcal{C}_t|\cdot|\mathcal{C}_{t-1}|)}{\sqrt{n(n-1)}}\\
		&\le&\alpha_T+\frac{6C\sum_{t=1}^{T}\alpha_t\log(|\mathcal{C}_t|)}{\sqrt{n(n-1)}}\\
		&=&\alpha_T+\frac{12C\sum_{t=1}^{T}(\alpha_t-\alpha_{t+1})\log(N(\alpha_t, \mathcal{H}, d_{U,n}))}{\sqrt{n(n-1)}}\\
		&\le&\alpha_T+\frac{12C\int_{\alpha_{T+1}}^{\alpha_{0}}\log(N(u, \mathcal{H}, d_{U,n}))du}{\sqrt{n(n-1)}}\\
		&\le&\alpha_T+\frac{24C\int_{\alpha_{T+1}}^{\alpha_{0}}\log(N(u, \mathcal{H}, d_{U,n}))du}{n}.
	\end{eqnarray*}
	For any $0\textless\delta\textless b$, let $T=\min\{K\in\mathbb{N}_+:2^{-K}b\le\delta\}$, then
	%\[
	\begin{eqnarray*}
		\mathcal{S}_n^{\textup{Deg}}(\mathcal{H})&\le&\delta+\frac{24C\int_{\delta/2}^{b}\log(N(u, \mathcal{H}, d_{U,n}))du}{n} \nonumber  \\
		& \le& \delta+\frac{24C\int_{\delta/4}^{b}\log(\mathcal{N}_{2n(n-1)}(u,\|\cdot\|_{\infty},\mathcal{H}))du}{n}
	\end{eqnarray*}
	%\]
	where the last inequality follows from (\ref{Covering_Relationship}). This completes the proof of Lemma \ref{Symmetrization}.
\end{proof}

\begin{lemma}\label{Network_Construction}
	(a) Let $\mathcal{D}\otimes\mathcal{R}=\{h(z)=D(y,R(x)):D:\mathbb{R}^{d_Y+d_0}\to\mathbb{R}\in\mathcal{D},R:\mathbb{R}^{d_X}\to\mathbb{R}^{d_0}\in\mathcal{R}\}$. If $\mathcal{D}$ is a class of functions implemented by a ReLU network with width $\mathcal{W}_{\mathcal{D}}$ and depth $\mathcal{L}_{\mathcal{D}}$, and $\mathcal{R}$ is a class of functions implemented by a ReLU network with width $\mathcal{W}_{\mathcal{R}}$ and depth $\mathcal{L}_{\mathcal{R}}$, then each element of $\mathcal{D}\otimes\mathcal{R}$ can be implemented by a ReLU network with width $\max\{\mathcal{W}_{\mathcal{D}},\mathcal{W}_{\mathcal{R}}+2d_Y,2(d_Y+d_0)\}$ and depth $\mathcal{L}_{\mathcal{D}}+\mathcal{L}_{\mathcal{R}}+1$.
	
	(b) Let $\mathcal{Q}\circ\mathcal{R}=\{h(z)=Q(R(x)):Q:\mathbb{R}^{d_0}\to\mathbb{R}^{d_{\textup{out}}}\in\mathcal{Q},R:\mathbb{R}^{d_X}\to\mathbb{R}^{d_0}\in\mathcal{R}\}$. If $\mathcal{Q}$ is a class of functions implemented by a ReLU network with width $\mathcal{W}_{\mathcal{Q}}$ and depth $\mathcal{L}_{\mathcal{Q}}$, and $\mathcal{R}$ is a class of functions implemented by a ReLU network with width $\mathcal{W}_{\mathcal{R}}$ and depth $\mathcal{L}_{\mathcal{R}}$, then each element of $\mathcal{Q}\circ\mathcal{R}$ can be implemented by a ReLU network with width $\max\{\mathcal{W}_{\mathcal{Q}},\mathcal{W}_{\mathcal{R}},2d_0\}$ and depth $\mathcal{L}_{\mathcal{Q}}+\mathcal{L}_{\mathcal{R}}+1$.
\end{lemma}
\begin{proof}[Proof of Lemma \ref{Network_Construction}]
	The proof follows from the method in Subsection B.1.1 in \cite{nakada2020adaptive}, which shows how to construct a neural network from sub-neural networks through concatenation, parallelization, and realization of identity function.
\end{proof}

\section{Additional numerical studies and experiment details}\label{Additional_Simulation}
This Appendix contains Table \ref{Table_for_p30} which includes some additional simulation results.

\begin{table}[H]
	\caption{\label{Table_for_p30}Distance correlation (DC), average prediction errors (APE), and their standard errors (based on 6-fold validation and best hyperparameters selected by validation method) for $p=30$}
	\centering
	 \resizebox{\textwidth}{!}{
		\begin{tabular}{cccccccccc}
			\hline
			\hline
			&        & \multicolumn{2}{c}{Model \textup{\uppercase\expandafter{\romannumeral1}}} & \multicolumn{2}{c}{Model \textup{\uppercase\expandafter{\romannumeral2}}} & \multicolumn{2}{c}{Model \textup{\uppercase\expandafter{\romannumeral3}}} & \multicolumn{2}{c}{Model \textup{\uppercase\expandafter{\romannumeral4}}} \\
			\hline
			& Method & DC           & APE          & DC           & APE          & DC           & APE          & DC           & APE          \\
			\hline
			\multirow{5}{*}{S(\romannumeral1)} & MSRL  & .95(.00)     & 0.30(.00)    & .86(.01)     & \textbf{0.30}(.01)    & .92(0.01)    & \textbf{0.54}(.05)    & .94(.00)     & \textbf{0.31}(.01)    \\
			& SIR    & .97(.00)     & \textbf{0.25}(.01)    & .00(.00)     & 0.79(.01)    & .20(0.00)    & 1.70(.02)    & .09(.02)     & 1.22(.01)    \\
			& SAVE   & .96(.00)     & \textbf{0.25}(.01)    & .00(.00)     & 0.79(.01)    & .20(0.01)    & 1.71(.04)    & .08(.02)     & 1.23(.01)    \\
			& GSIR   & .94(.00)     & 0.32(.01)    & .84(.01)     & 0.32(.02)    & .74(0.01)    & 0.72(.01)    & .65(.04)     & 0.34(.01)    \\
			& GSAVE  & .96(.00)     & 0.27(.01)    & .35(.10)     & 0.63(.05)    & .42(0.04)    & 1.29(.06)    & .32(.04)     & 0.86(.09)\\
			\hline
			\multirow{5}{*}{S(\romannumeral2)} & MSRL  & .93(.00)     & 0.29(.00)    & .88(.01)     & \textbf{0.31}(.01)    & .87(0.03)    & \textbf{0.60}(.08)    & .92(.01)     & 0.43(.05)    \\
			& SIR    & .95(.00)     & \textbf{0.25}(.00)    & .00(.00)     & 0.90(.04)    & .16(0.02)    & 1.51(.06)    & .05(.02)     & 1.44(.10)    \\
			& SAVE   & .95(.00)     & \textbf{0.25}(.00)    & .00(.00)     & 0.90(.04)    & .15(0.02)    & 1.54(.06)    & .06(.02)     & 1.44(.10)    \\
			& GSIR   & .93(.01)     & 0.30(.01)    & .87(.01)     & 0.32(.01)    & .72(0.03)    & 0.68(.03)    & .80(.03)     & \textbf{0.38}(.04)    \\
			& GSAVE  & .95(.00)     & 0.26(.00)    & .42(.02)     & 0.68(.04)    & .33(0.04)    & 1.32(.04)    & .30(.08)     & 0.99(.13)    \\
			\hline
			\multirow{5}{*}{S(\romannumeral3)} & MSRL  & .99(.00)     & 0.30(.01)    & .97(.00)     & \textbf{0.38}(.01)    & .94(0.01)    & \textbf{2.01}(.21)    & .97(.01)     & \textbf{0.66}(.08)    \\
			& SIR    & .99(.00)     & \textbf{0.25}(.00)    & .00(.00)     & 2.03(.04)    & .33(0.02)    & 6.36(.23)    & .02(.01)     & 3.65(.07)    \\
			& SAVE   & .99(.00)     & 0.27(.00)    & .00(.00)     & 2.04(.03)    & .32(0.02)    & 6.41(.25)    & .03(.01)     & 3.63(.06)    \\
			& GSIR   & .84(.04)     & 0.97(.14)    & .93(.02)     & 0.52(.06)    & .71(0.04)    & 2.35(.31)    & .72(.01)     & 0.96(.21)    \\
			& GSAVE  & .99(.00)     & 0.34(.03)    & .66(.13)     & 1.16(.24)    & .64(0.01)    & 3.85(.42)    & .36(.02)     & 2.47(.09)    \\
			\hline
			\multirow{5}{*}{S(\romannumeral4)} & MSRL  & .96(.00)     & 0.28(.01)    & .89(.01)     & \textbf{0.33}(.01)    & .90(0.02)    & \textbf{0.87}(.08)    & .93(.01)     & 0.41(.02)    \\
			& SIR    & .97(.00)     & \textbf{0.25}(.00)    & .00(.00)     & 0.99(.05)    & .19(0.02)    & 2.38(.20)    & .09(.01)     & 1.44(.09)    \\
			& SAVE   & .97(.00)     & \textbf{0.25}(.00)    & .00(.00)     & 0.99(.06)    & .18(0.03)    & 2.39(.21)    & .09(.01)     & 1.44(.09)    \\
			& GSIR   & .95(.00)     & 0.29(.01)    & .89(.01)     & \textbf{0.33}(.01)    & .68(0.06)    & 0.93(.20)    & .75(.04)     & \textbf{0.36}(.05)    \\
			& GSAVE  & .95(.00)     & 0.29(.01)    & .70(.02)     & 0.54(.03)    & .53(0.06)    & 1.52(.22)    & .50(.06)     & 0.81(.06)  \\
			\hline
		\end{tabular}}
\end{table}

\subsection{Experimental Details}\label{Network_Struc}
Python codes will be made publicly available on a GitHub repository, we still
Here we provide the details for the numerical experimental details.  In all the experiments, we set $\lambda=2$ and use the Adam optimizer \citep{kingma2014adam} in Pytorch with learning rate $\textit{lr}$ and weight decay parameter $\textit{wd}$ specified in each experiment.  The activation function for the neural networks is the LeakyReLU.  The batch size in all experiments is $512$.

\subsection{Simulation for Model (\ref{toy_model})}\label{Descrip_Implement_toy}
The training data sample size is $4000$ and no ES criterion is applied. The maximum number of epoches is $3000$. When the reference distribution is $\gamma_{U}$, $\mathcal{R}$ has $2$ hidden layer with width $(64,64)$. When the reference distribution is $\gamma_{S}$, $\mathcal{R}$ has $3$ hidden layer with width $(64,64,64)$. Both $\mathcal{D}$ and $\mathcal{Q}$ has $2$ hidden layers with widths $(32,32)$. $\textit{lr}$ is $0.003$ and $\textit{wd}$ is $0$. For GSIR and GSAVE, we use some default tuning parameters. And the kernel estimates are based on $10000$ new samples.

\subsection{Simulation Study in Subsection \ref{Simulation_Study}}\label{Descrip_Implement_Sim} In this experiment, the total $6000$ data points are first divided into $5000$ training-validation data points and $1000$ testing data points and then $5000$ training-validation data points are divided into $4000$ training data points and $1000$ validation data points.

In MSRL training, ES criterion is applied and it is the empirical distance correlation (DC) calculated based on the validation data. The maximum number of epochs is $1000$ and $\textit{wd}$ is $0.0001$. The \textit{patience} is $200$, where \textit{patience} is the number of epochs until termination if no improvement is made on the validation dataset. The learning procedure is replicated $10$ times with different initial weights for the parameters of neural networks and we choose the one with the largest empirical DC on the validation dataset among the ten representation estimate as the ultimate representation estimate. $\mathcal{R}$ has $3$ hidden layers with widths $(32, 16, 8)$. Both $\mathcal{D}$ and  $\mathcal{Q}$ has $2$ hidden layers with widths $(16,8)$. For Model \textup{\uppercase\expandafter{\romannumeral1}} \& \textup{\uppercase\expandafter{\romannumeral2}}, $\textit{lr}=0.001$, $d_0=1$. For Model \textup{\uppercase\expandafter{\romannumeral3}} $\textit{lr}=0.0003$, $d_0=2$. For Model \textup{\uppercase\expandafter{\romannumeral4}}, $\textit{lr}=0.001$, $d_0=2$.

For SIR and SAVE, the tuning parameter is the slice number $\textit{sn}$. For each $\textit{sn}\in\{5, 10, 15, 20, 25, 30\}$, we use the validation data to calculate the empirical DC between the response and the representation learned from training data. We use the slice number with the largest empirical DC as the ultimate choice for the tuning parameter $\textit{sn}$.  For GSIR and GSAVE, we apply the tuning parameter selection approach provided in \cite{Li2018book} to the validation data.

\subsection{Real Datasets}\label{Descrip_Implement_Real} For each data split, $20\%$ of data is used as the testing data and then from the other $80\%$ of the data, we randomly select $3000$ data points as the validation data and the rest part is the training data.

For GSIR and GSAVE, limited by our computational capability, we only subsample $4000$, $2000$, $3000$ data points from the above training, validation, and testing data as the new training, validation, and testing data, and conduct the estimation procedure we have stated in subsection \ref{Descrip_Implement_Sim}

In the superconductivty dataset, for our MSRL, $\mathcal{R}$ has $3$ hidden layers with widths $(128, 128, 128)$. Both $\mathcal{D}$ and  $\mathcal{Q}$ has $2$ hidden layers with widths $(64, 64)$.  The maximum number of epochs is $2000$, $\textit{lr}=0.0001,\textit{wd}=0.001$ and the \textit{patience} is $400$.  The number of different initial weights is still $10$.  For the goal of data visualization, we randomly sample $800$ data points from the testing data for one fold to conduct plotting.

In the pole-telecommunication dataset, for MSRL, $\mathcal{R}$ has $2$ hidden layers with widths $(30, 25)$. Both $\mathcal{D}$ and  $\mathcal{Q}$ has $2$ hidden layers with widths $(16, 8)$.  The maximum number of epochs is $1000$, $\textit{lr}=0.001,\textit{wd}=0.0001$ and the \textit{patience} is $200$.  The number of different initial weights is still $10$. For the goal of data visualization, we randomly sample $400$ data points from the testing data for one fold to conduct plotting.

\subsection{The extra plots for the superconductivity and pole-telecommunication datasets}\label{extra-plots}
In Figure \ref{SuperConduct_drawingA}, we plot the critical temperature against each component of the learned representation with $d_0=5$ for SIR and SAVE.  We fit a quadratic model to check how the components impact the response.
\begin{figure}[htbp]
	\centering
	\begin{minipage}[t]{0.9\linewidth}
		\centering
		\rotatebox{90}{\quad\tiny SIR}
		\includegraphics[width=0.10\textheight]{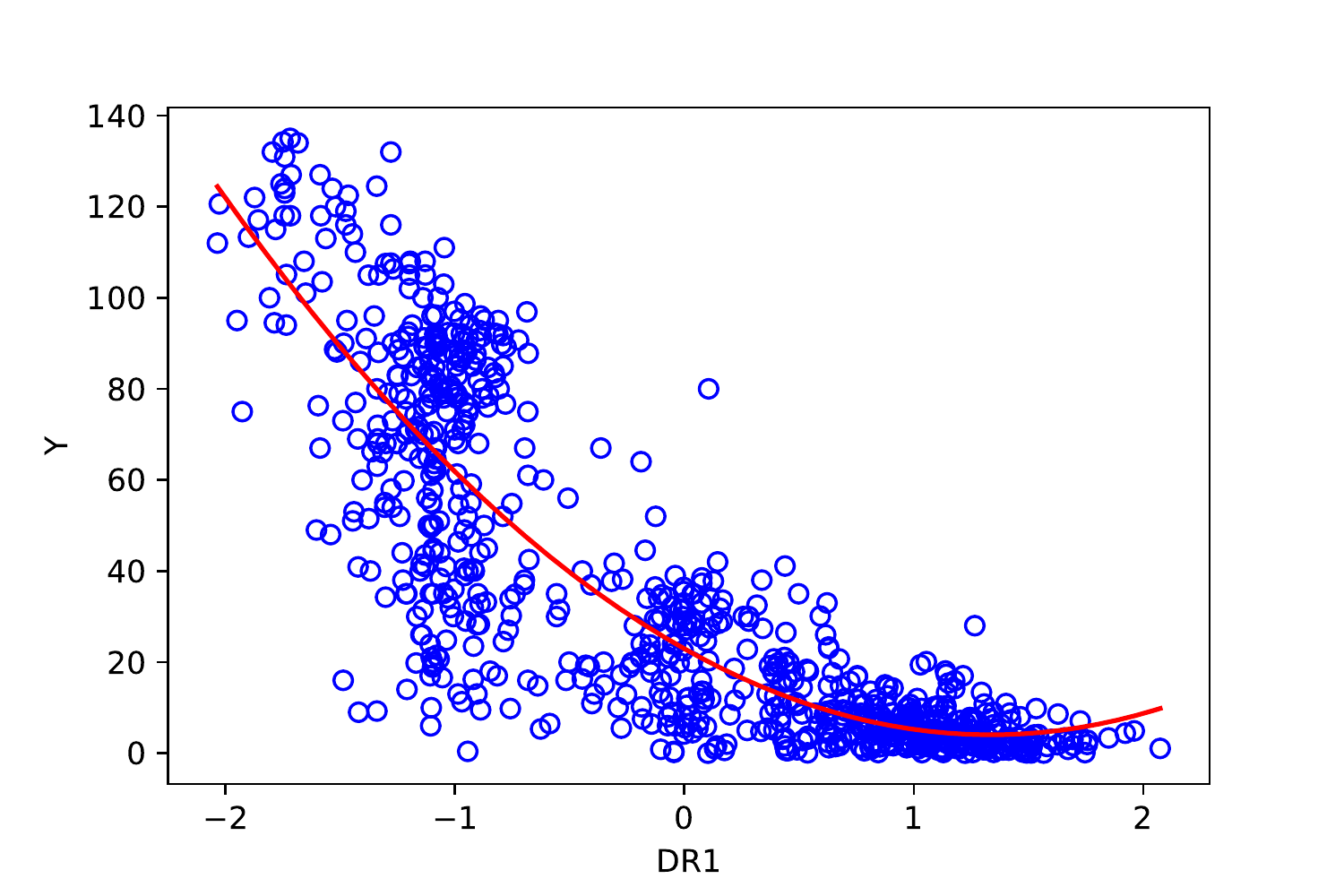} \includegraphics[width=0.10\textheight]{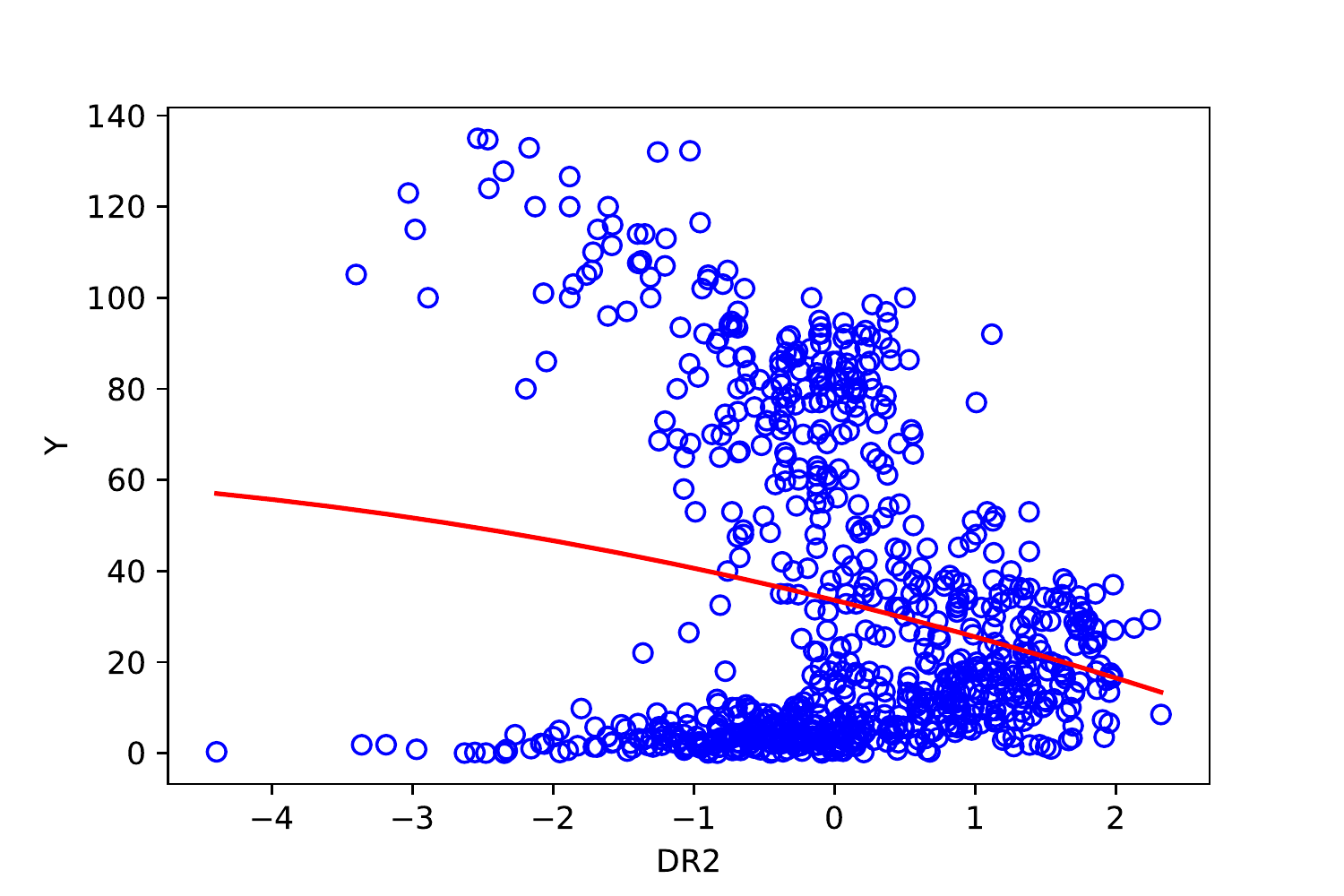}
		\includegraphics[width=0.10\textheight]{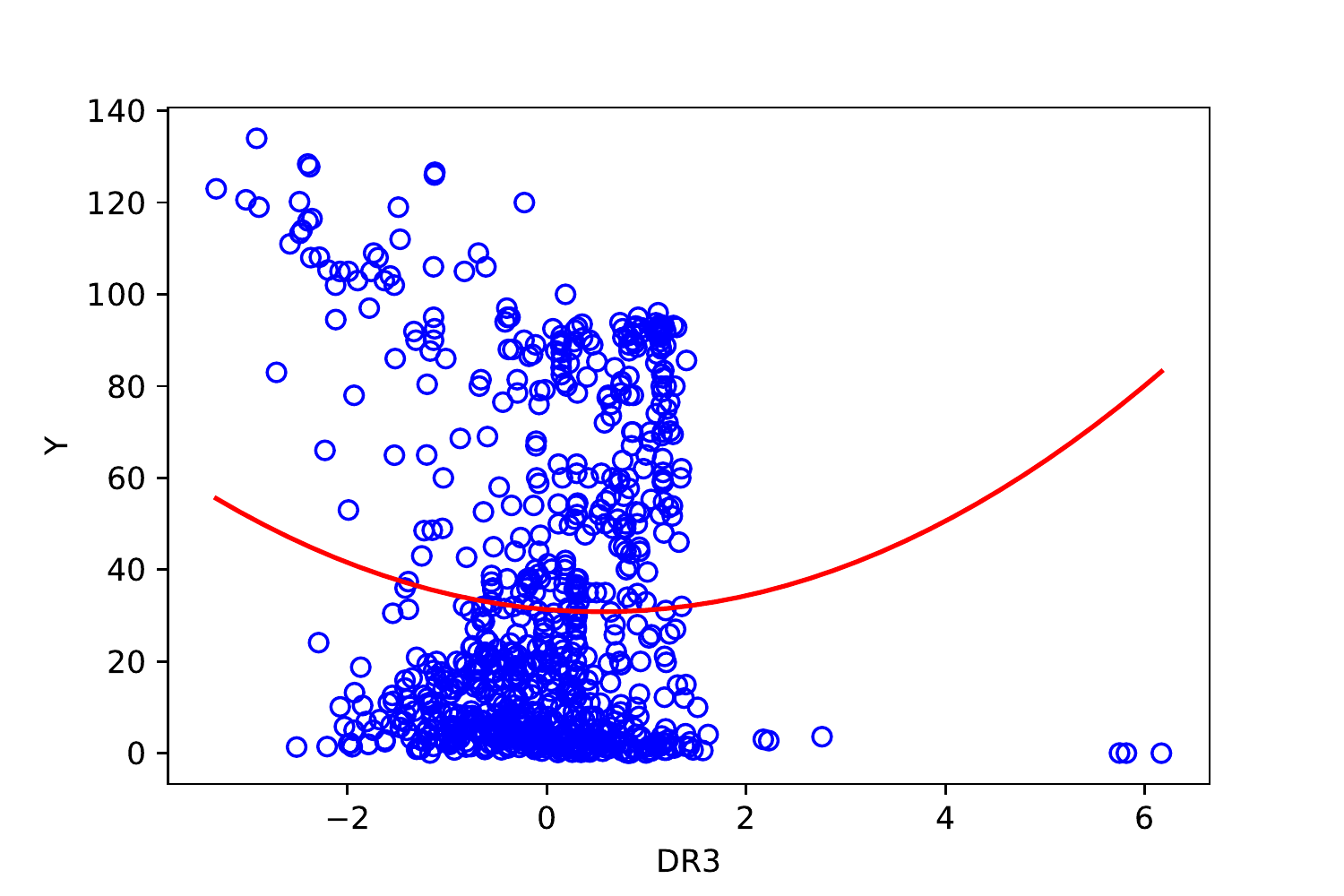} \includegraphics[width=0.10\textheight]{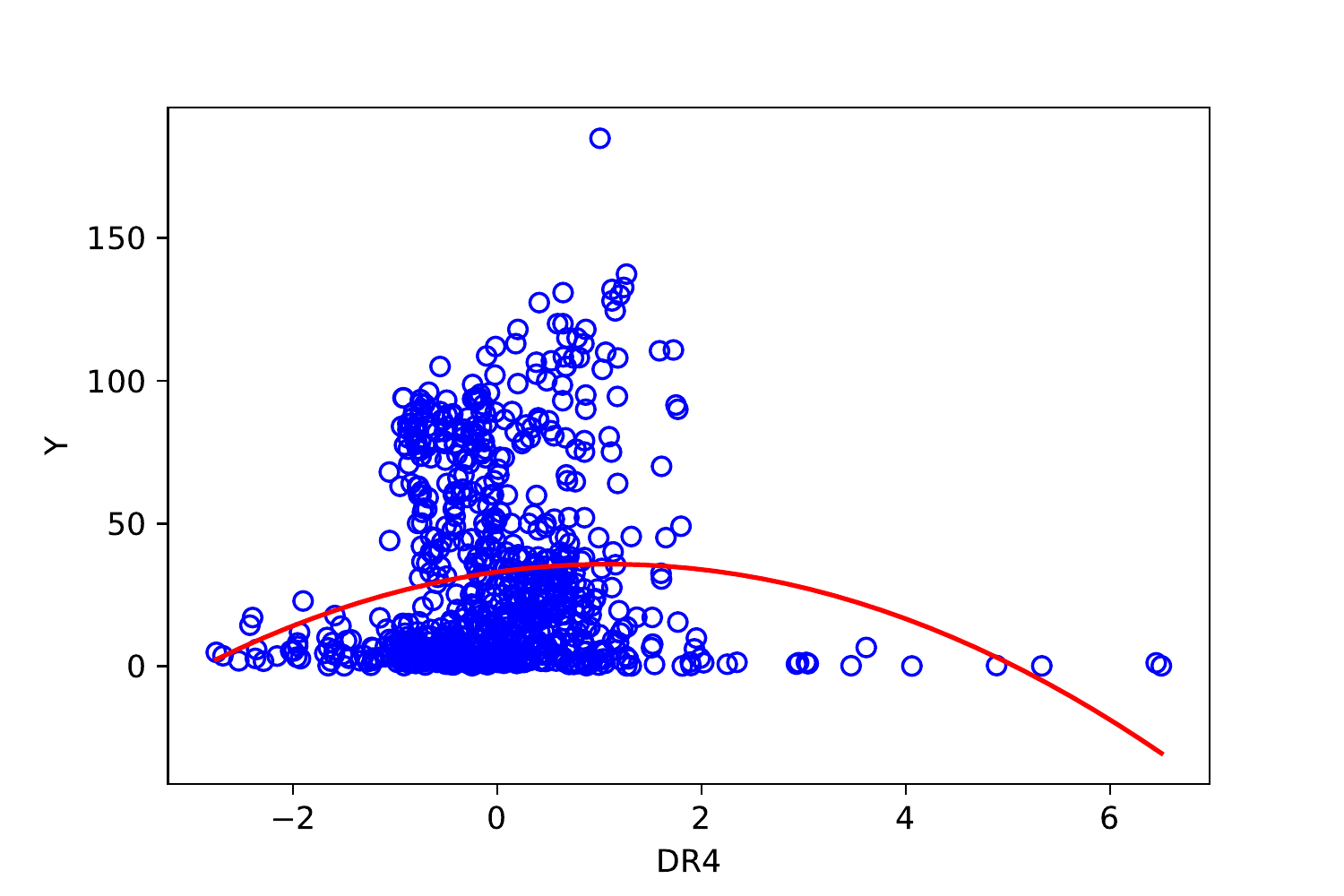}
		\includegraphics[width=0.10\textheight]{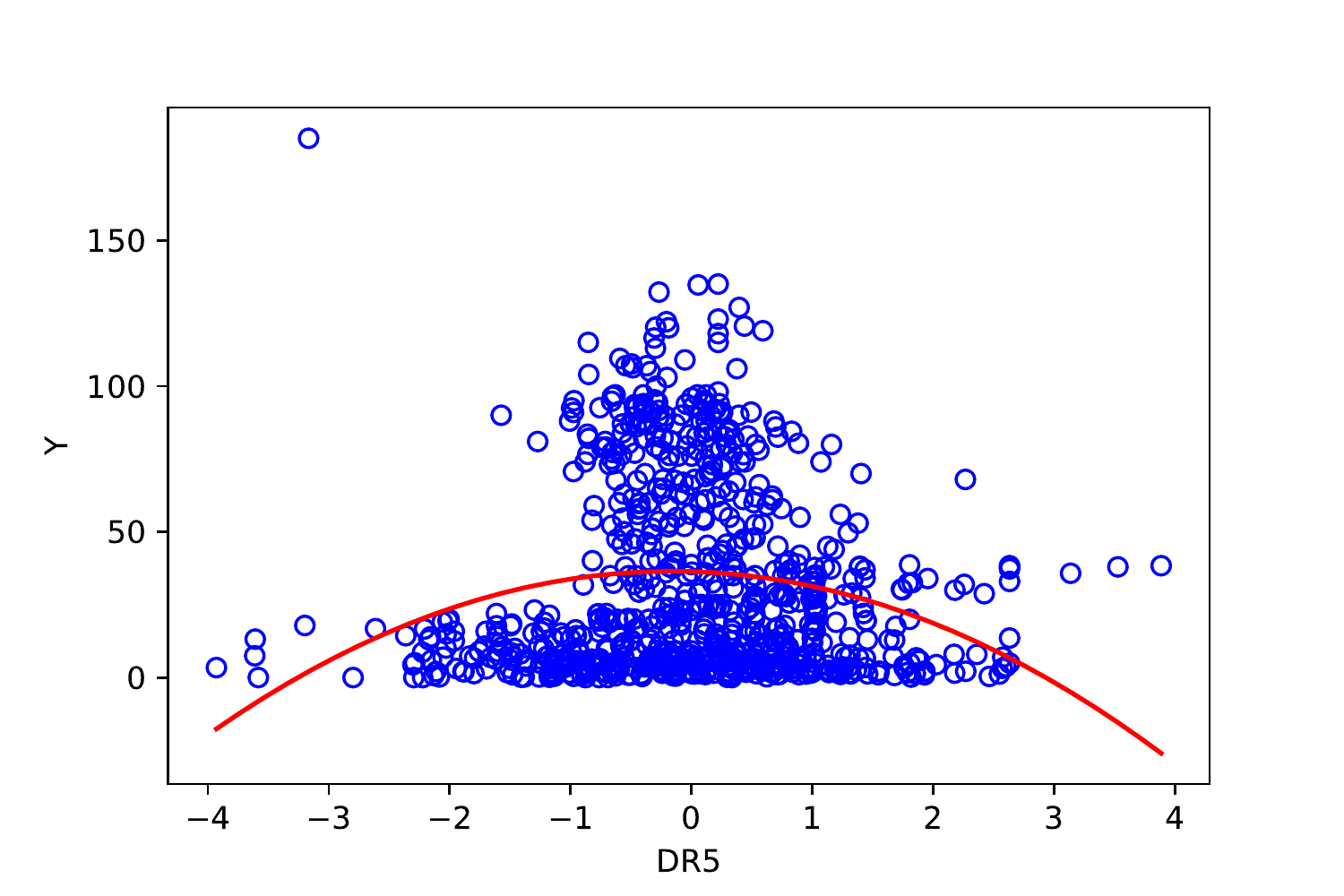}
	\end{minipage}
	\\
	\centering
	\begin{minipage}[t]{0.9\linewidth}
		\centering
		\rotatebox{90}{\quad\tiny SAVE}
		\includegraphics[width=0.10\textheight]{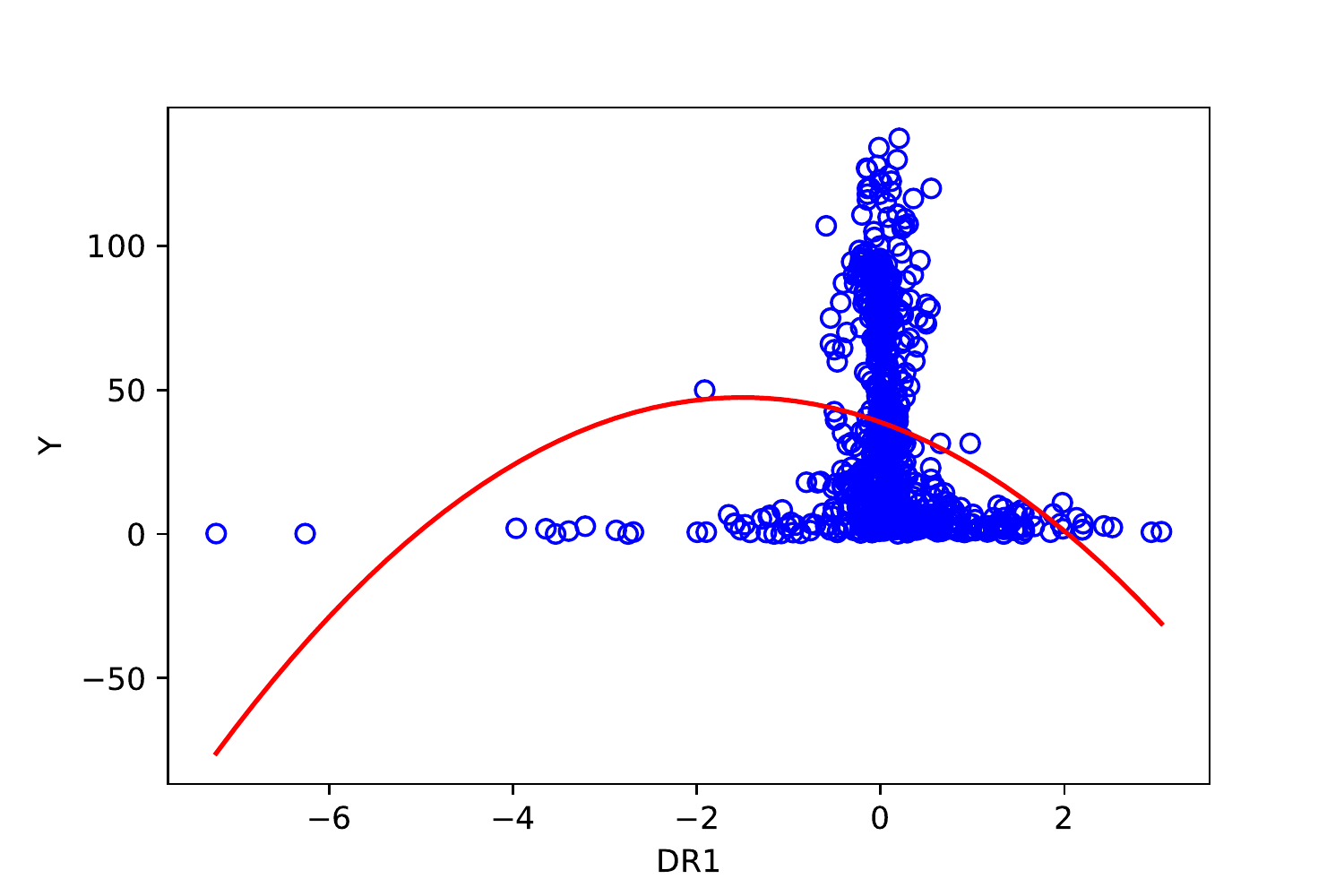} \includegraphics[width=0.10\textheight]{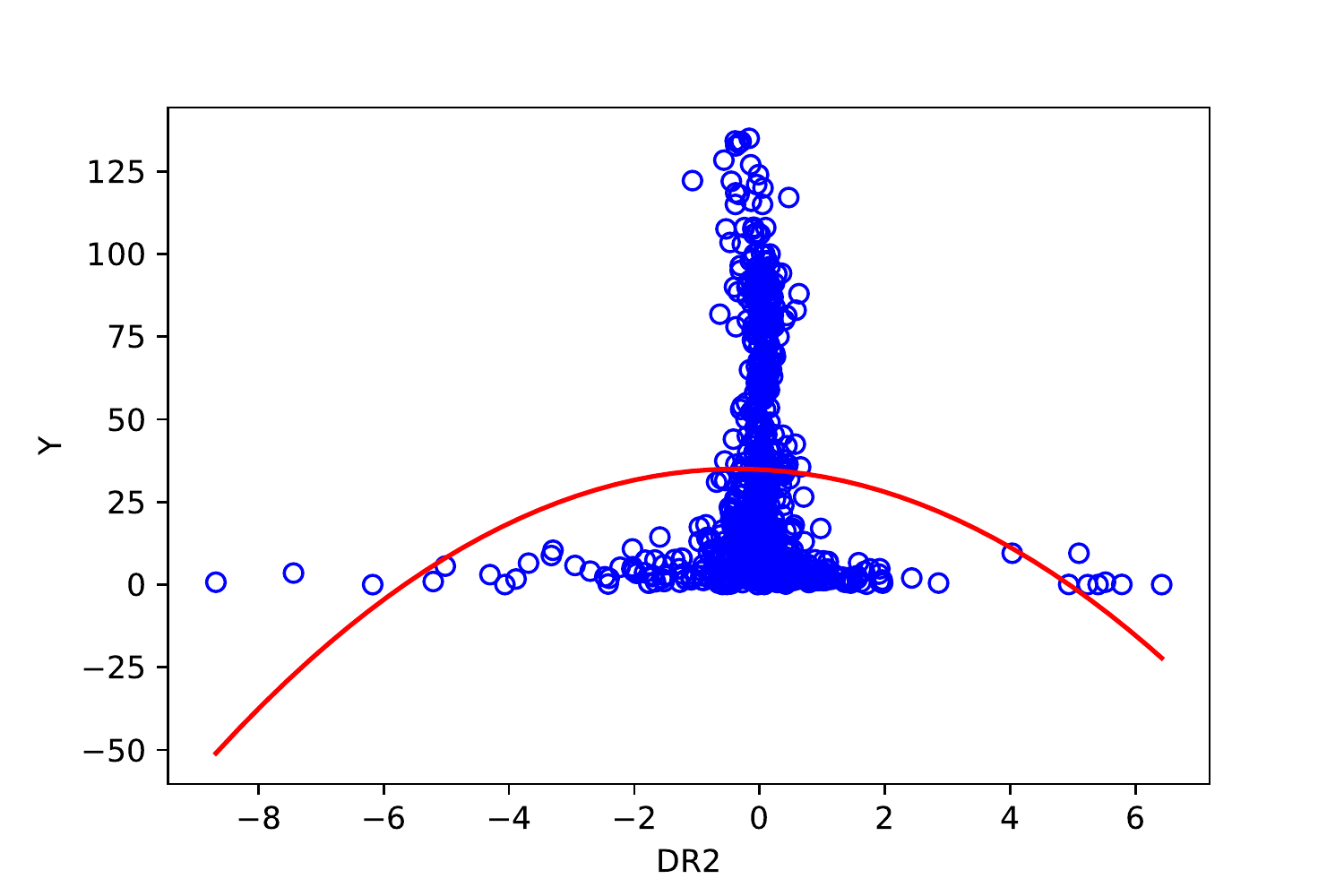}
		\includegraphics[width=0.10\textheight]{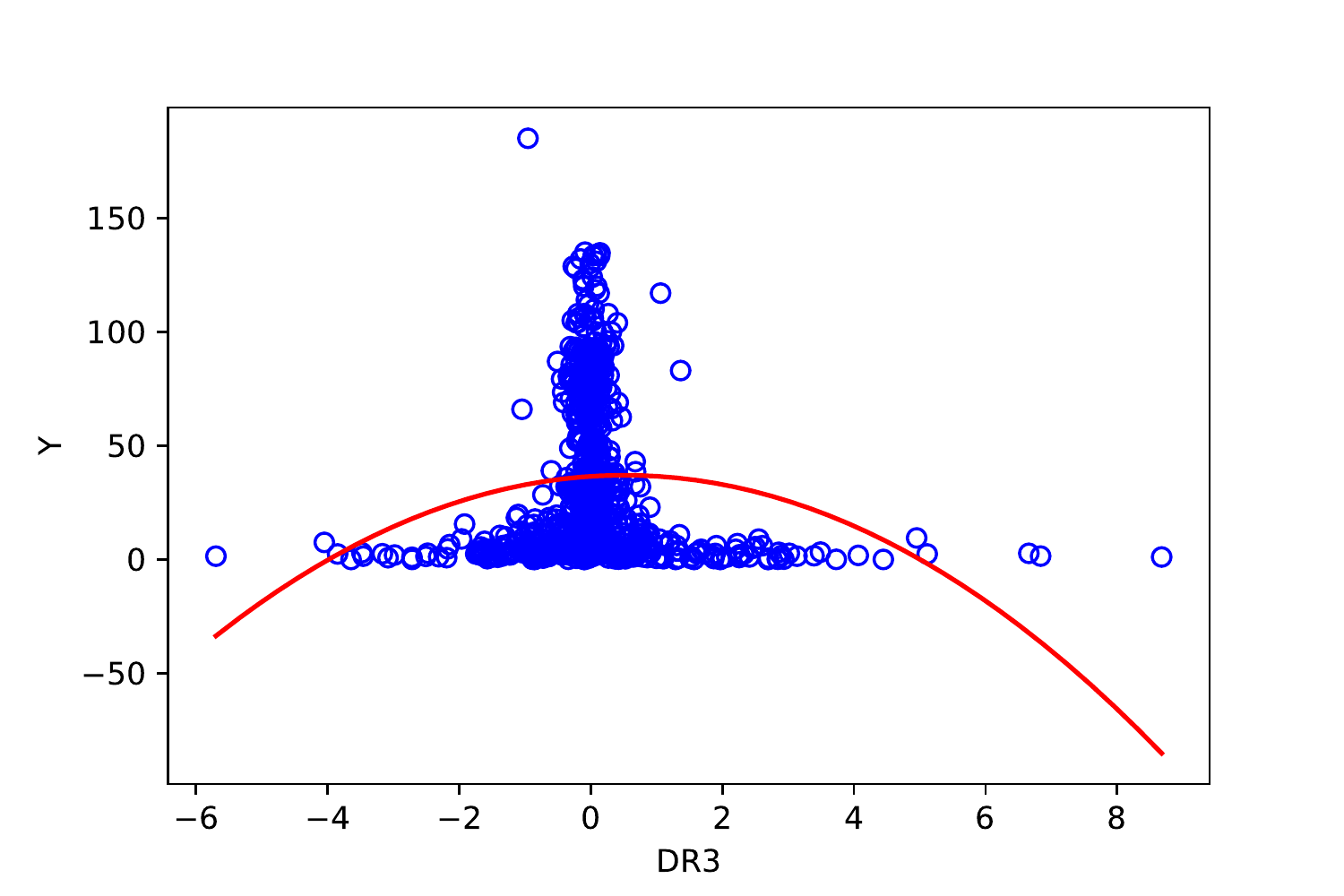} \includegraphics[width=0.10\textheight]{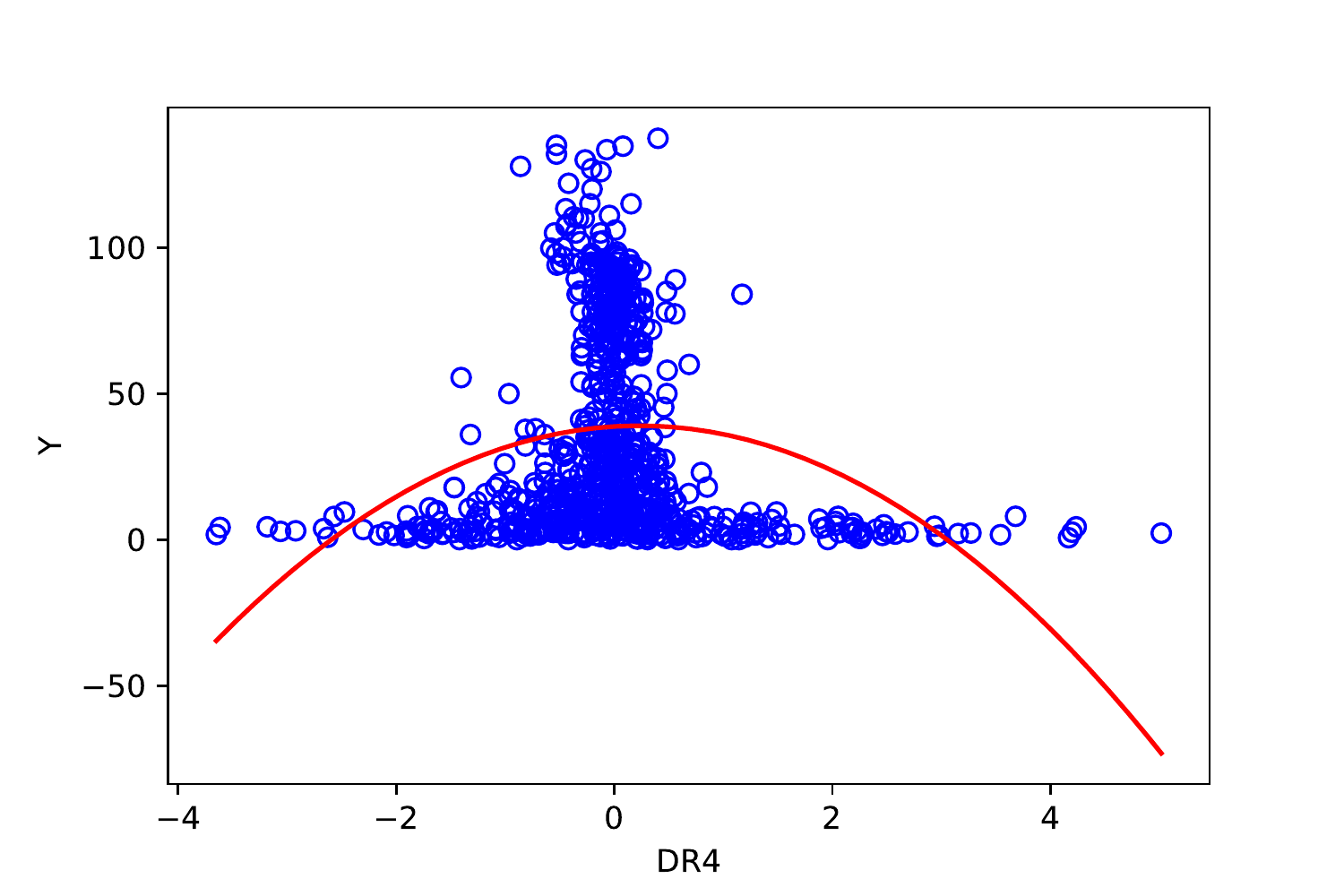}
		\includegraphics[width=0.10\textheight]{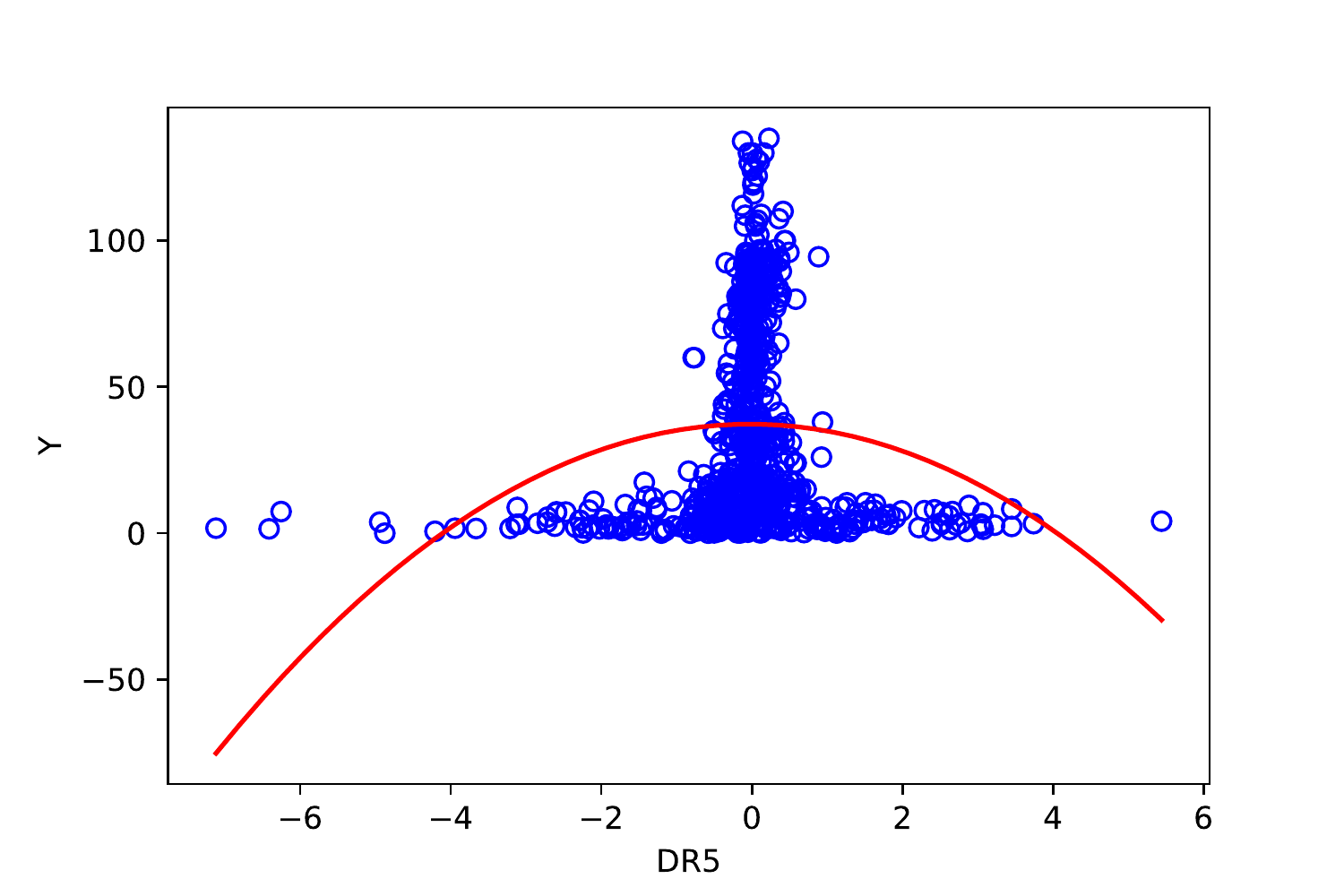}
	\end{minipage}
	\caption{\label{SuperConduct_drawingA}SIR and SAVE: The plots of the critical temperature against each component of the learned representation by SIR and SAVE for the superconductivity dataset.  The red line is a fitting quadratic model.}
\end{figure}

In Figure \ref{Poldata_drawingA},
we plot the transformed data based on each pair of the components of the learned representation with $d_0=5$ for SIR and SAVE.
\begin{figure}[htbp]
	\centering
	\begin{minipage}[t]{0.9\linewidth}
		\parbox[c][0.5cm]{0.18cm}{\rotatebox{90}{\tiny SIR}}
		\begin{minipage}[t]{\linewidth}
			\includegraphics[width=0.10\textheight]{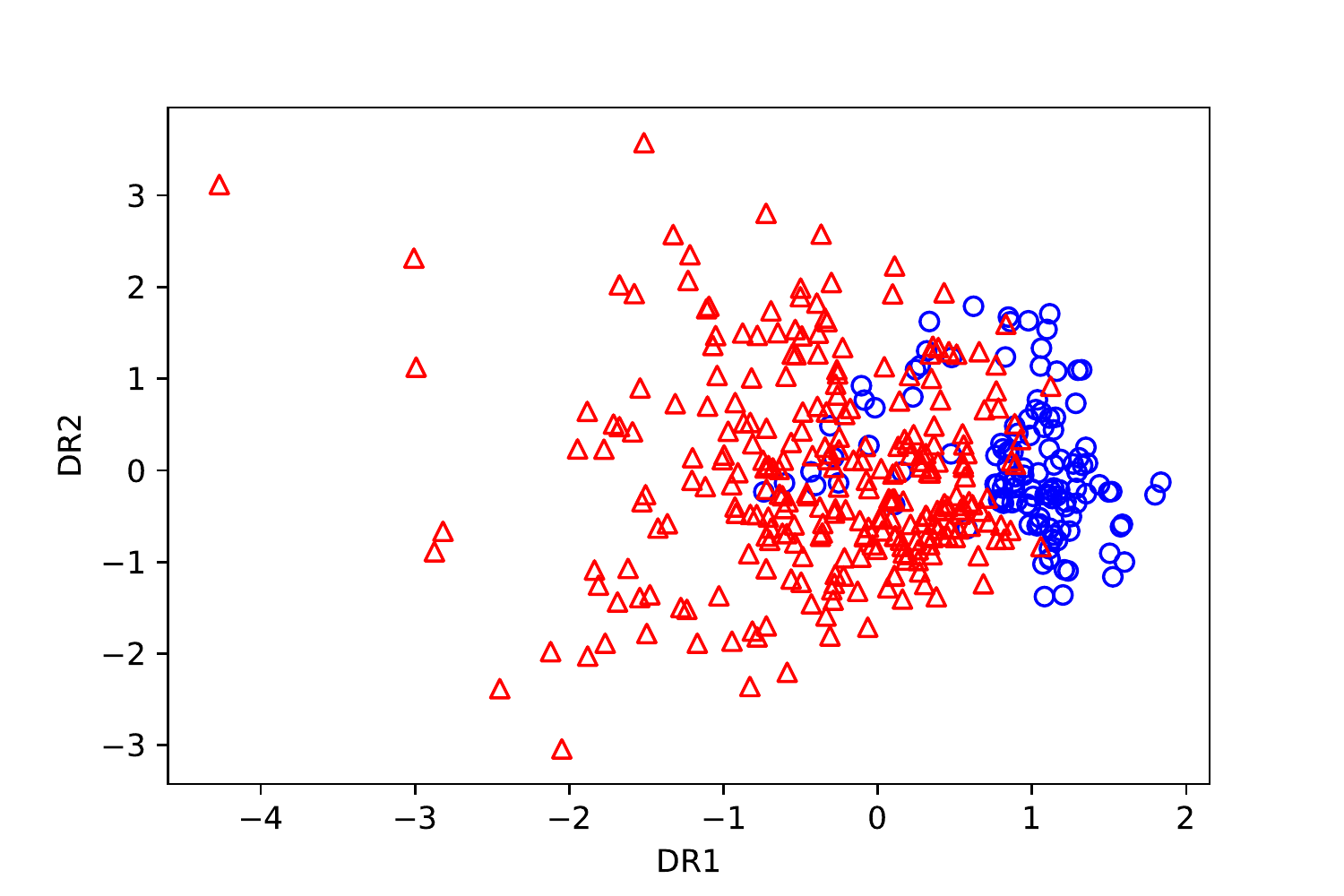}
			\includegraphics[width=0.10\textheight]{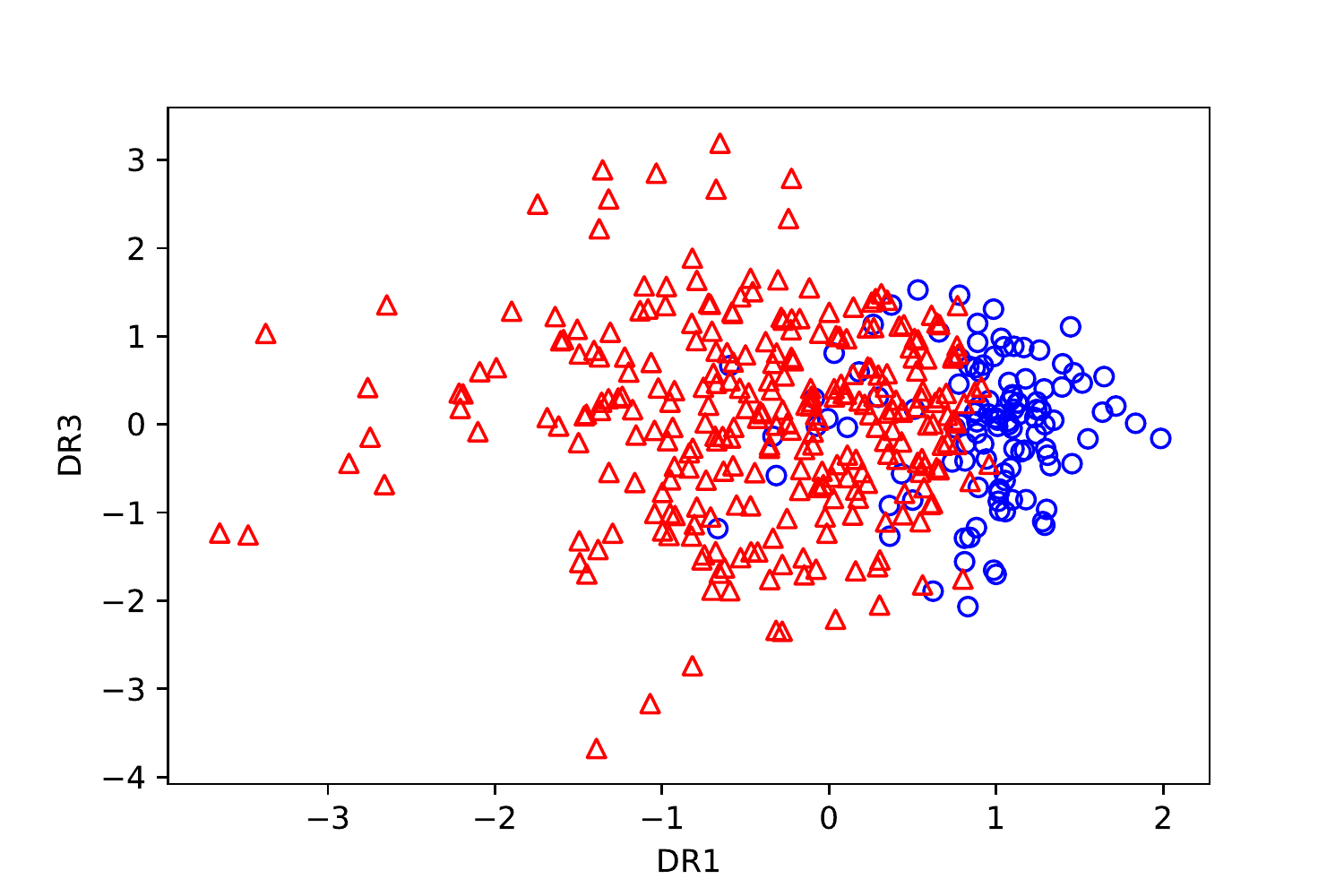}
			\includegraphics[width=0.10\textheight]{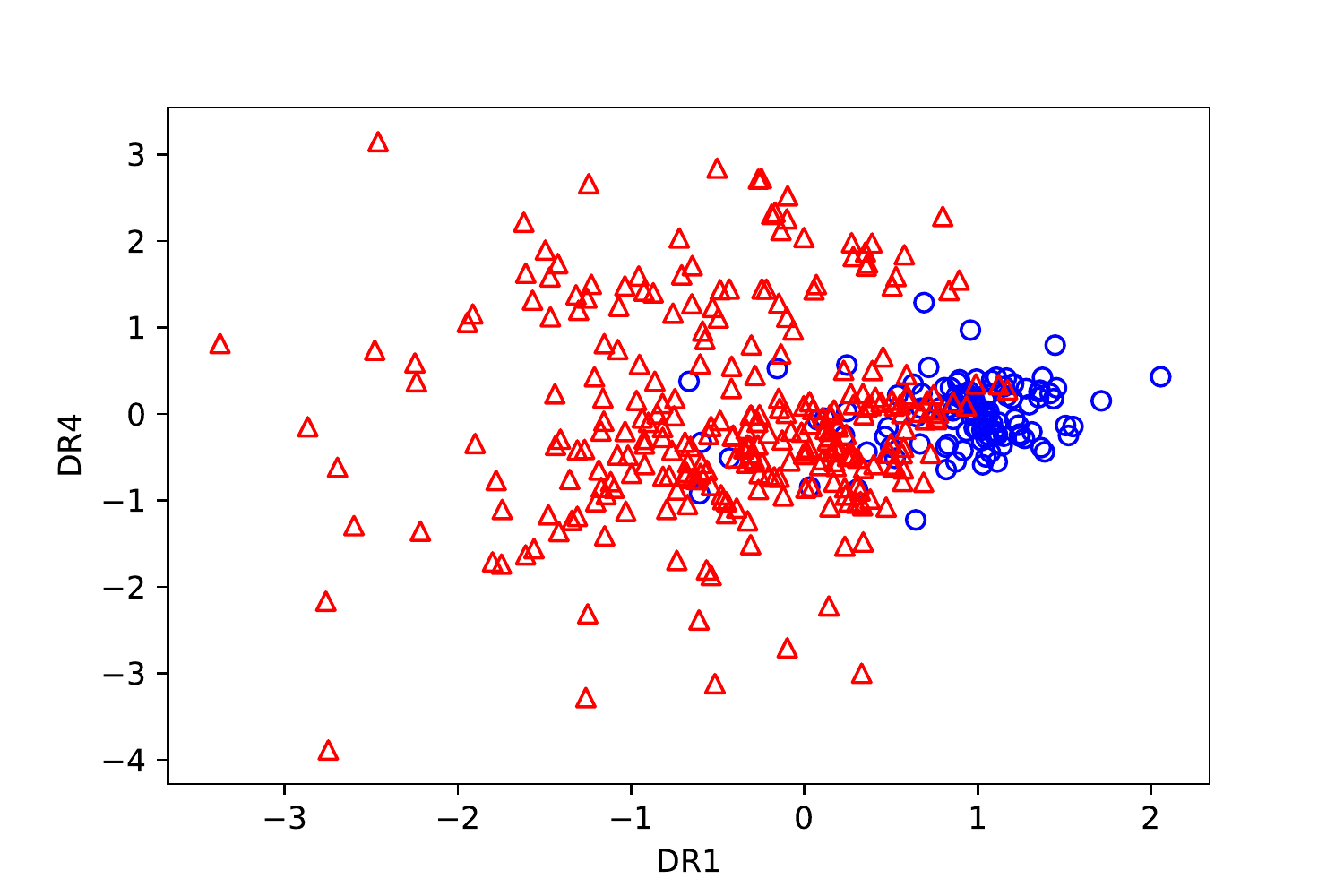}
			\includegraphics[width=0.10\textheight]{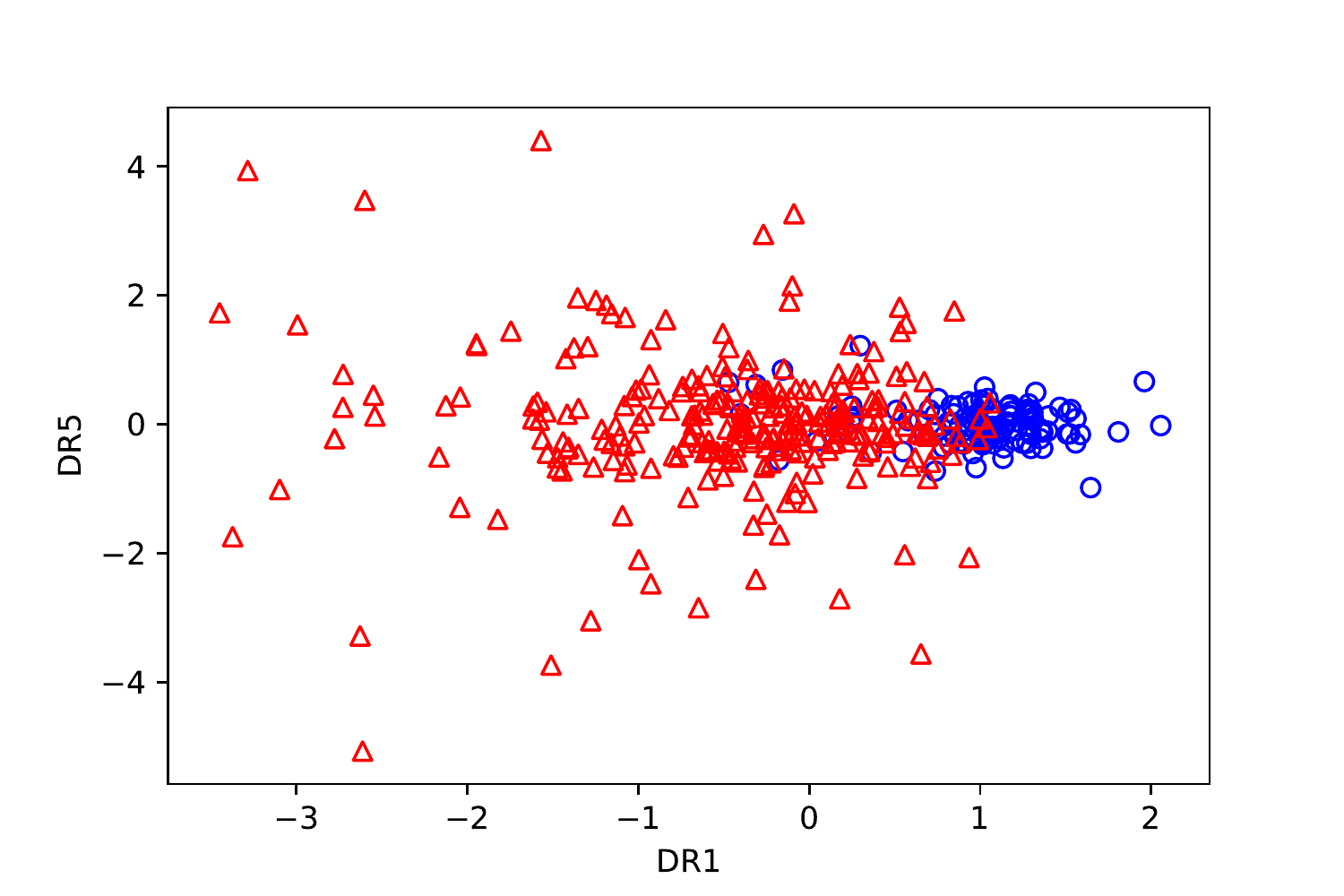}
			\includegraphics[width=0.10\textheight]{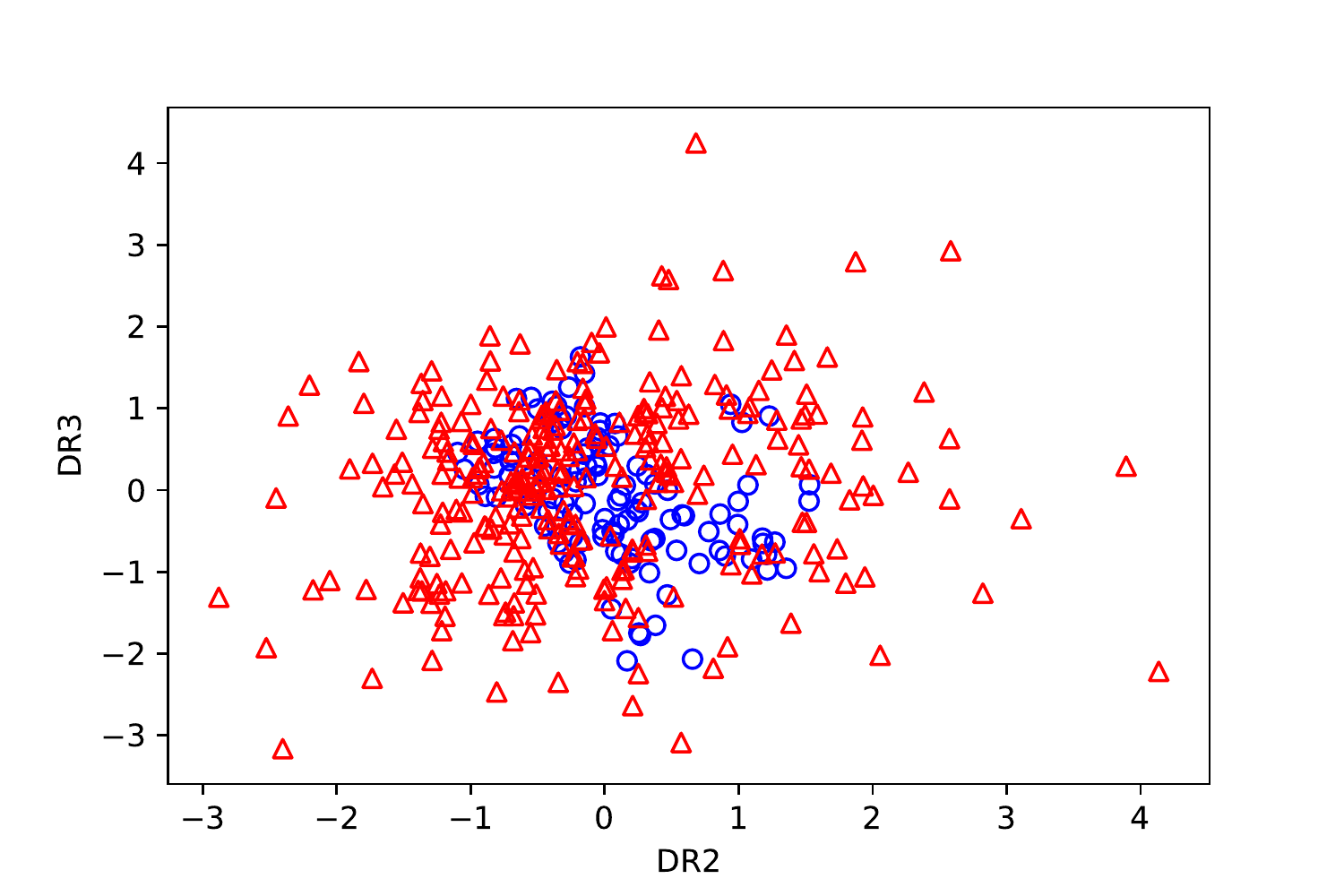}\\
			\includegraphics[width=0.10\textheight]{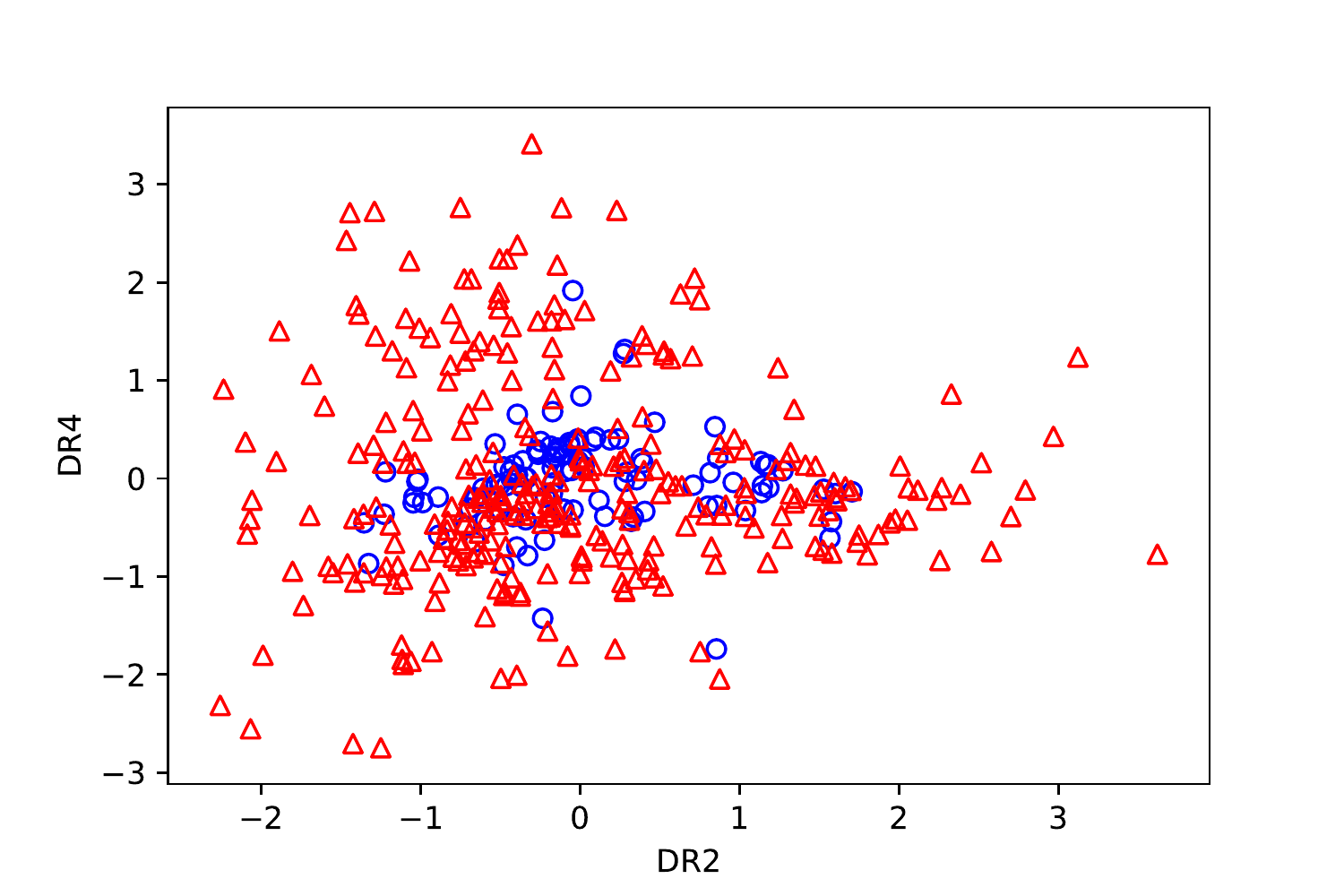}
			\includegraphics[width=0.10\textheight]{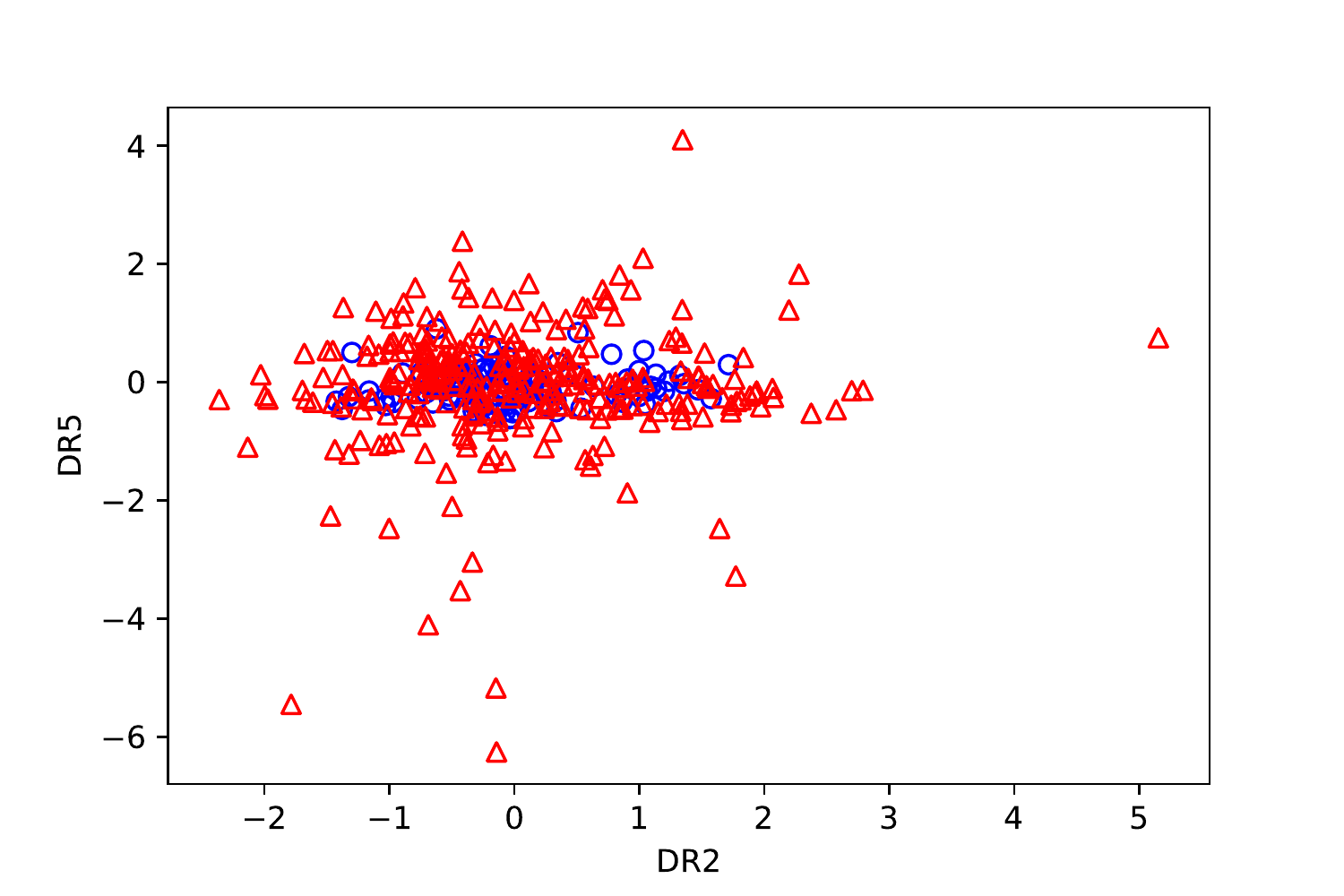}
			\includegraphics[width=0.10\textheight]{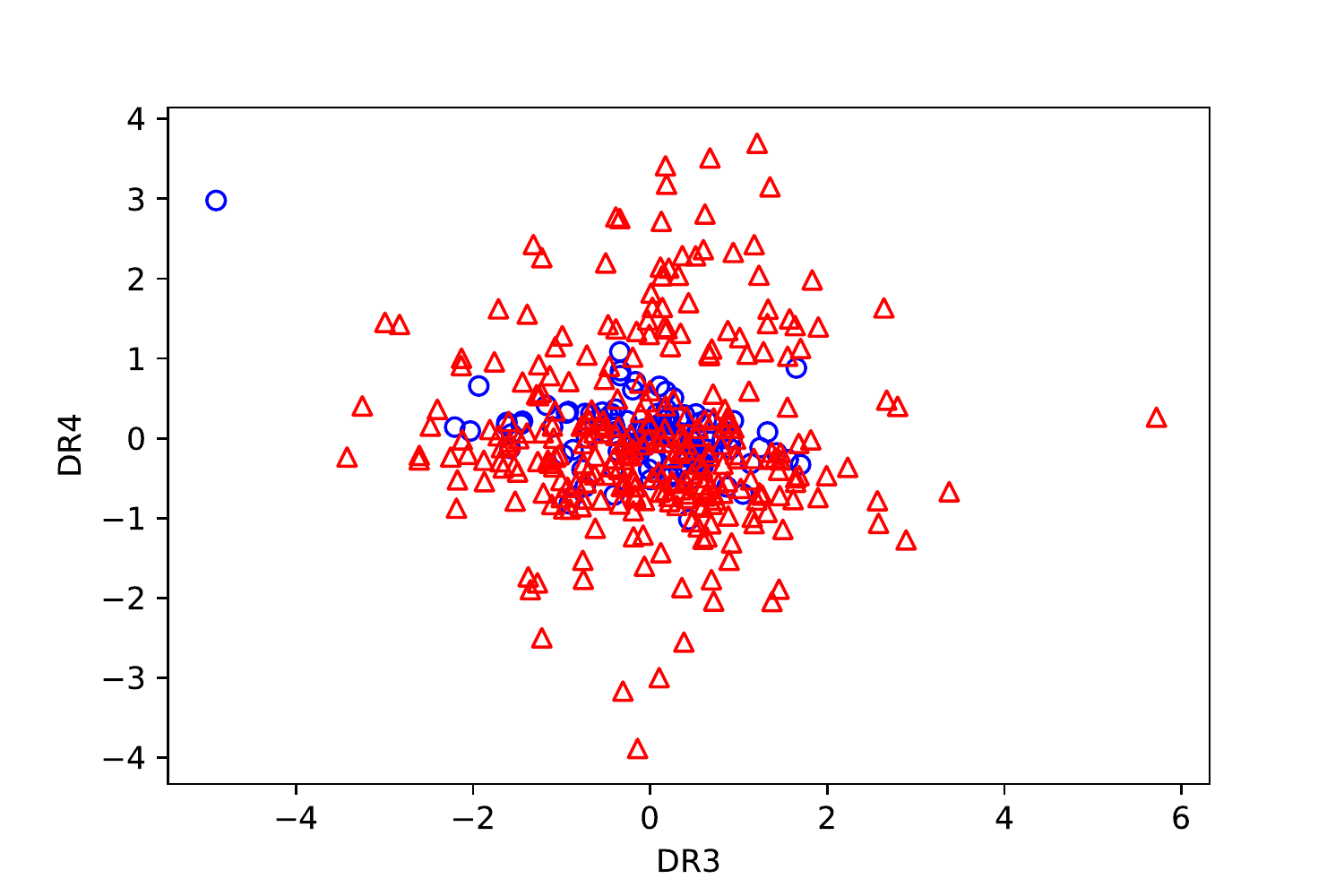}
			\includegraphics[width=0.10\textheight]{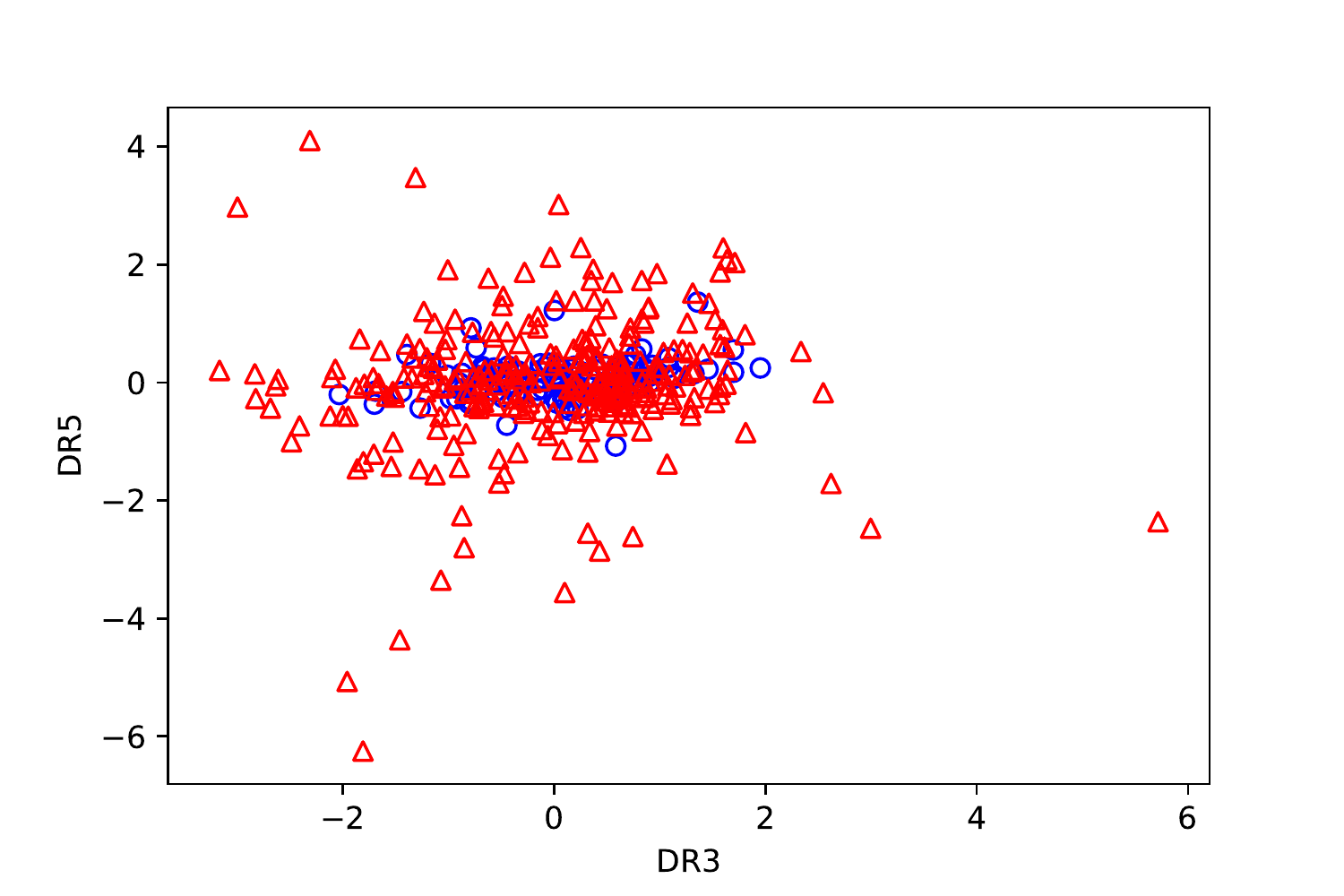}
			\includegraphics[width=0.10\textheight]{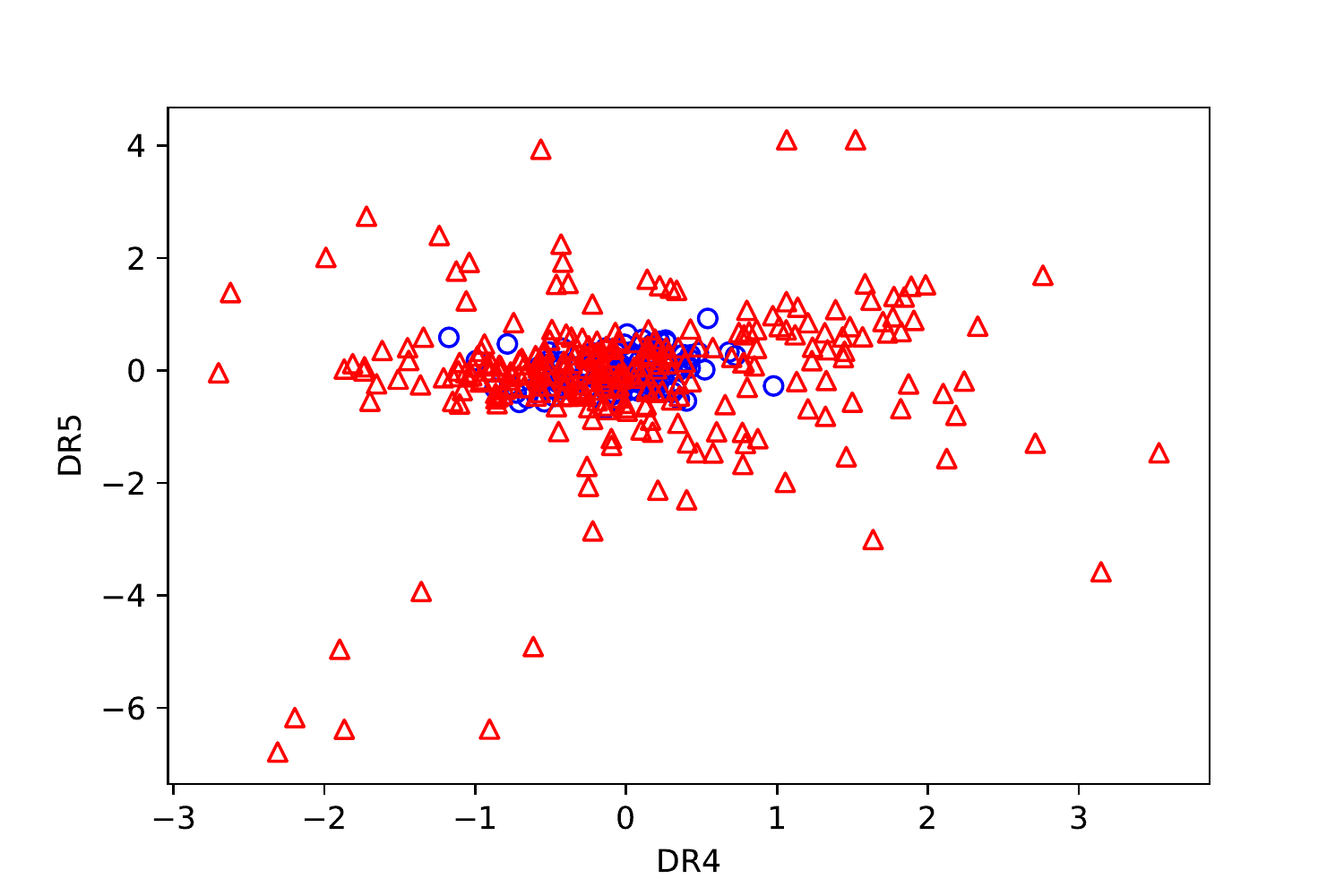}
		\end{minipage}
	\end{minipage}
	\\
	\begin{minipage}[t]{0.9\linewidth}
		\parbox[c][0.5cm]{0.18cm}{\rotatebox{90}{\tiny SAVE}}
		\begin{minipage}[t]{\linewidth}
			\includegraphics[width=0.10\textheight]{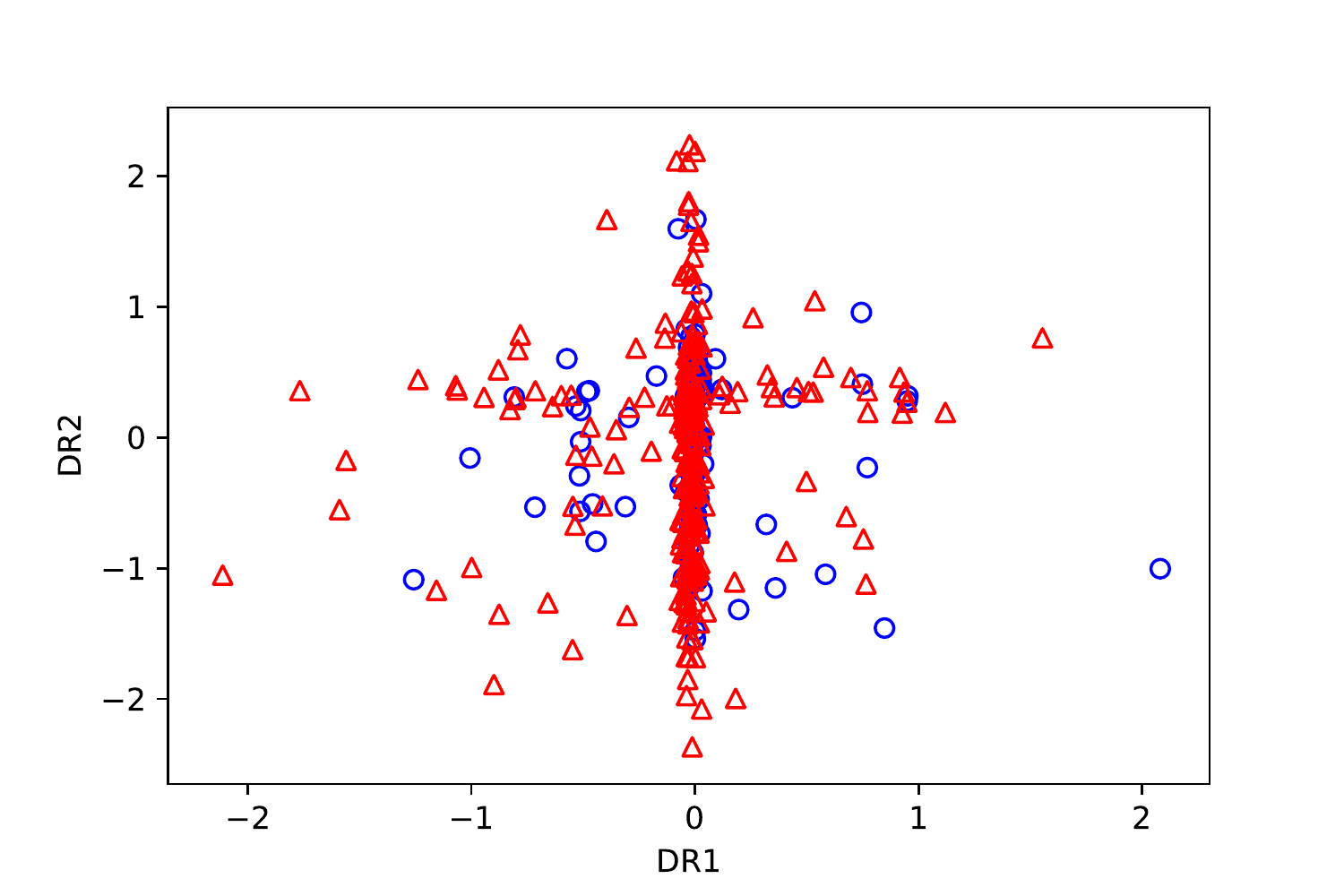}
			\includegraphics[width=0.10\textheight]{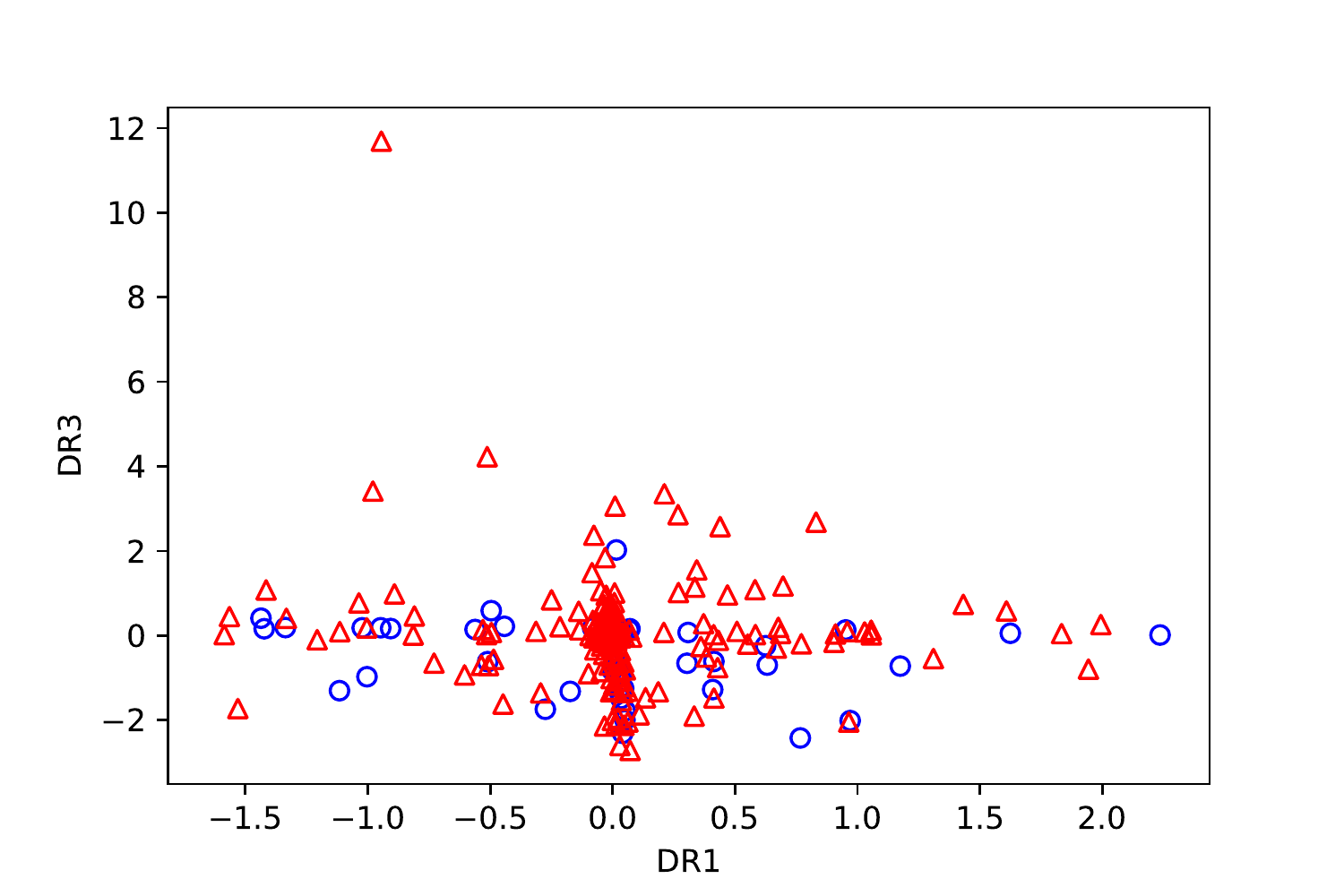}
			\includegraphics[width=0.10\textheight]{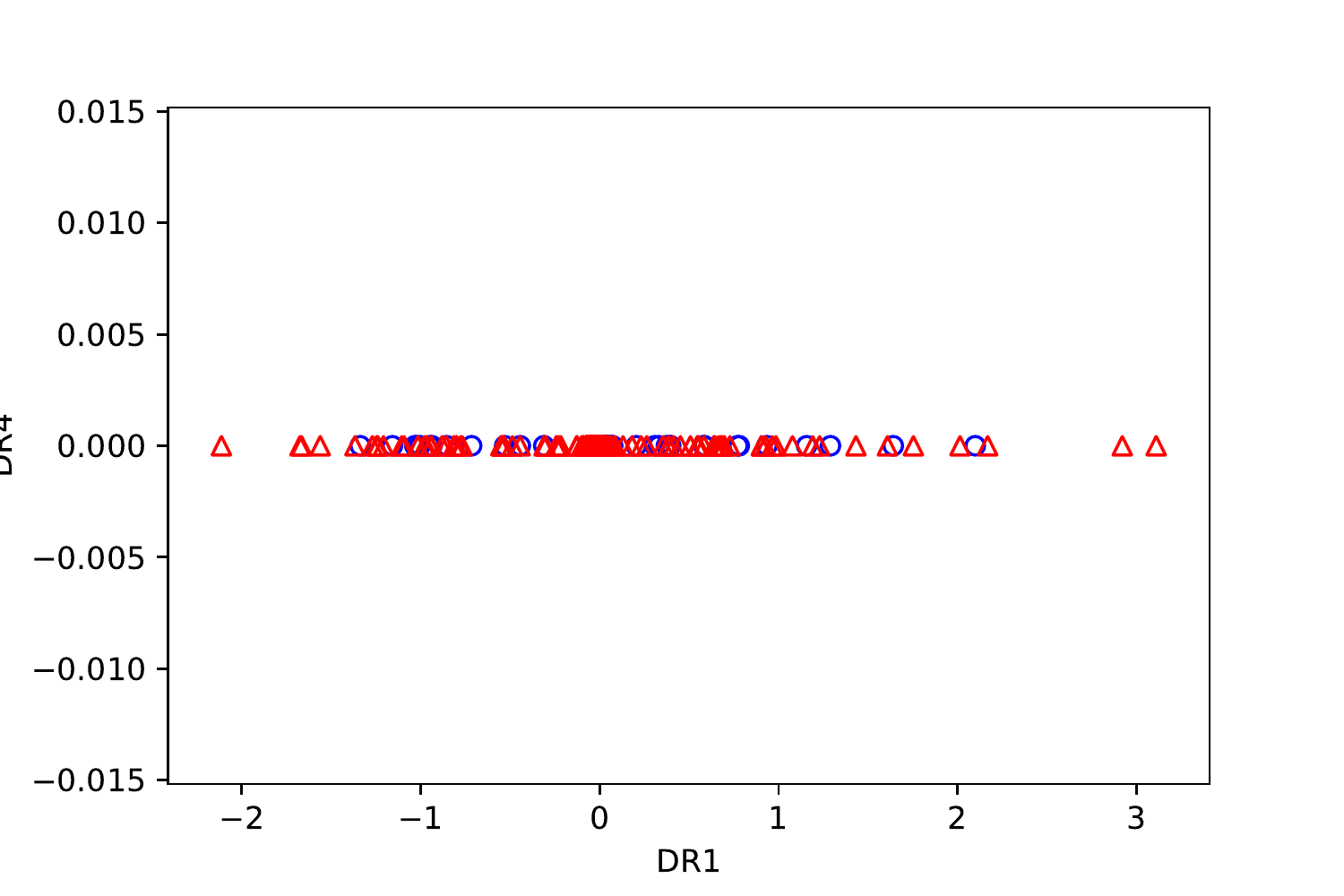}
			\includegraphics[width=0.10\textheight]{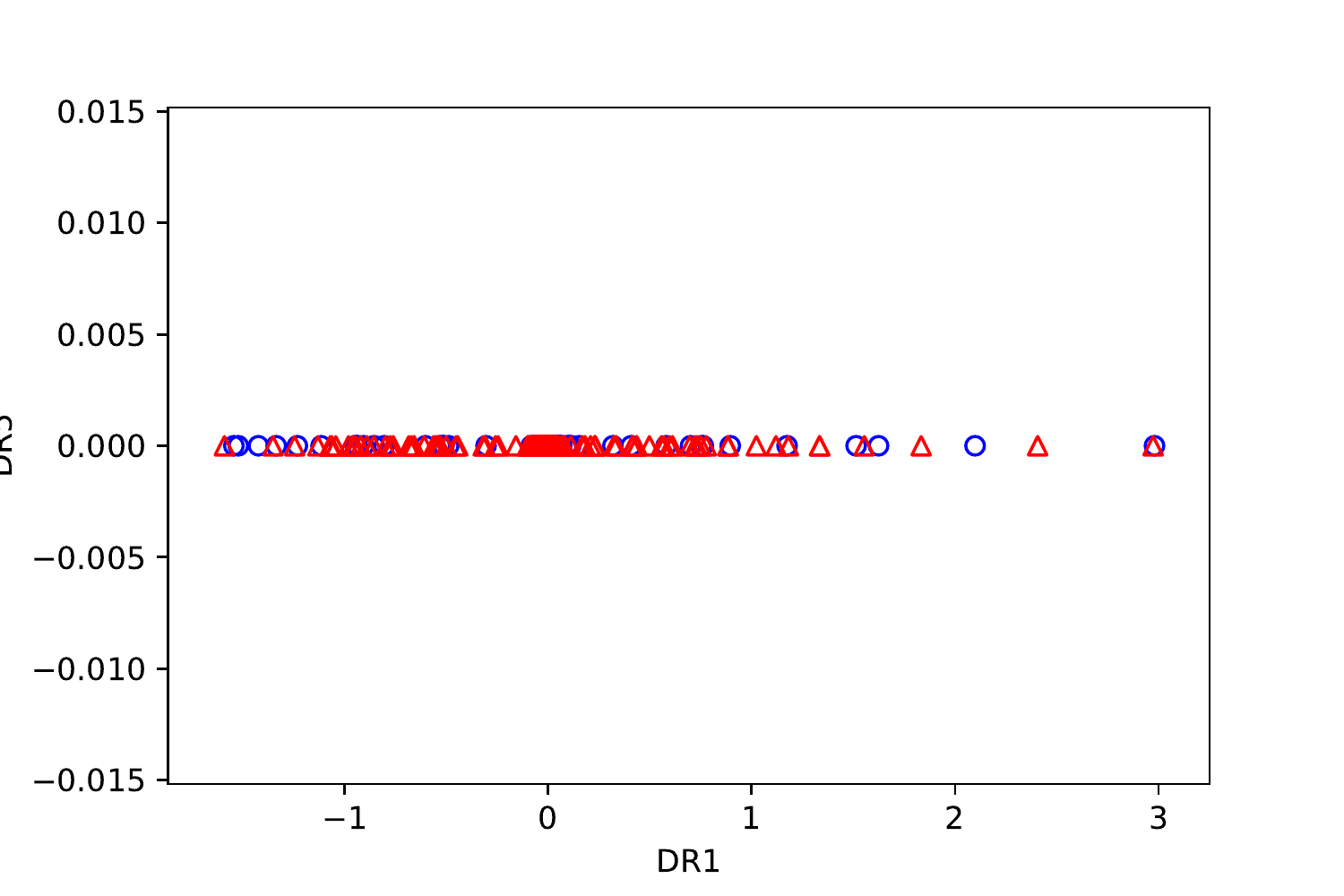}
			\includegraphics[width=0.10\textheight]{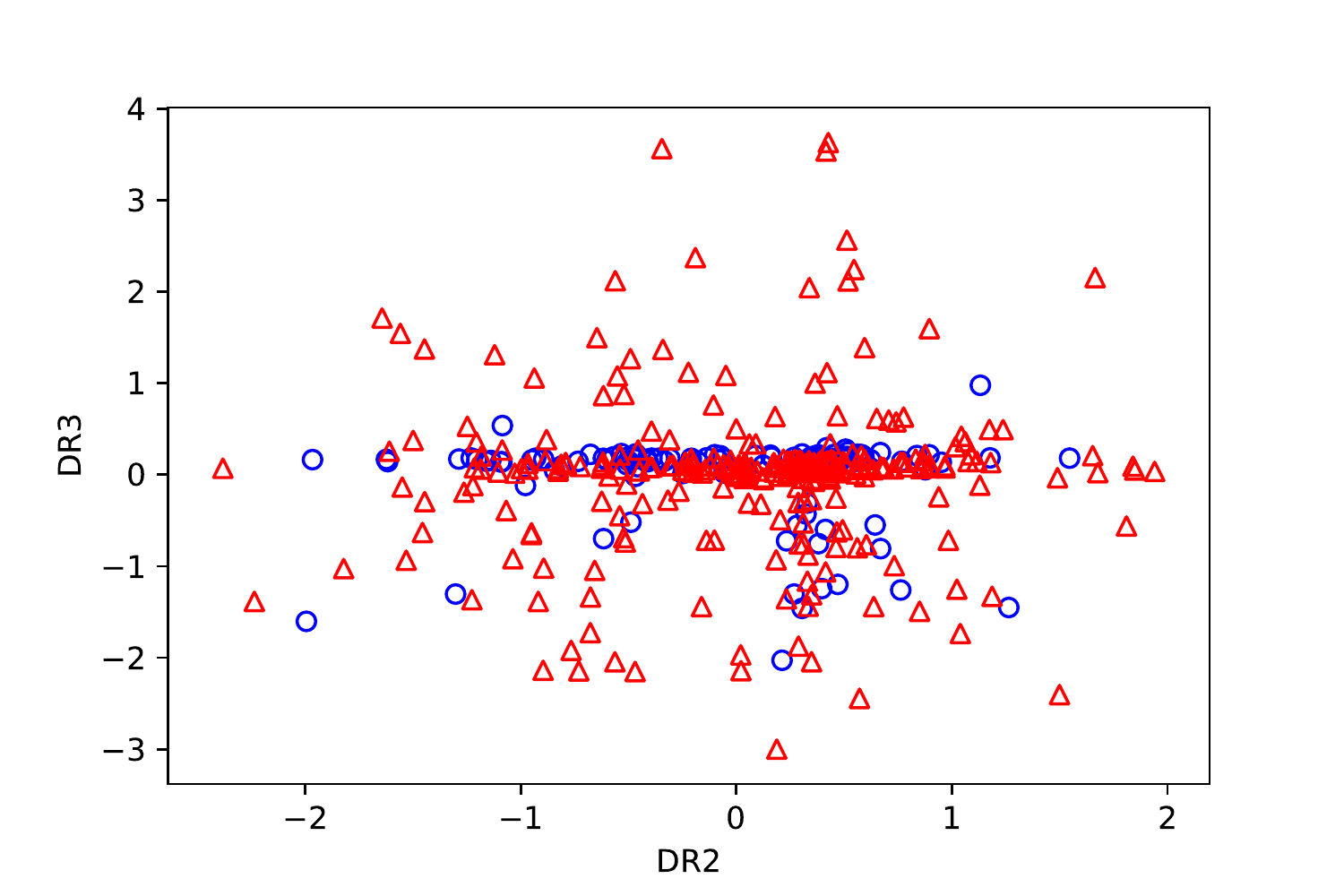}\\
			\includegraphics[width=0.10\textheight]{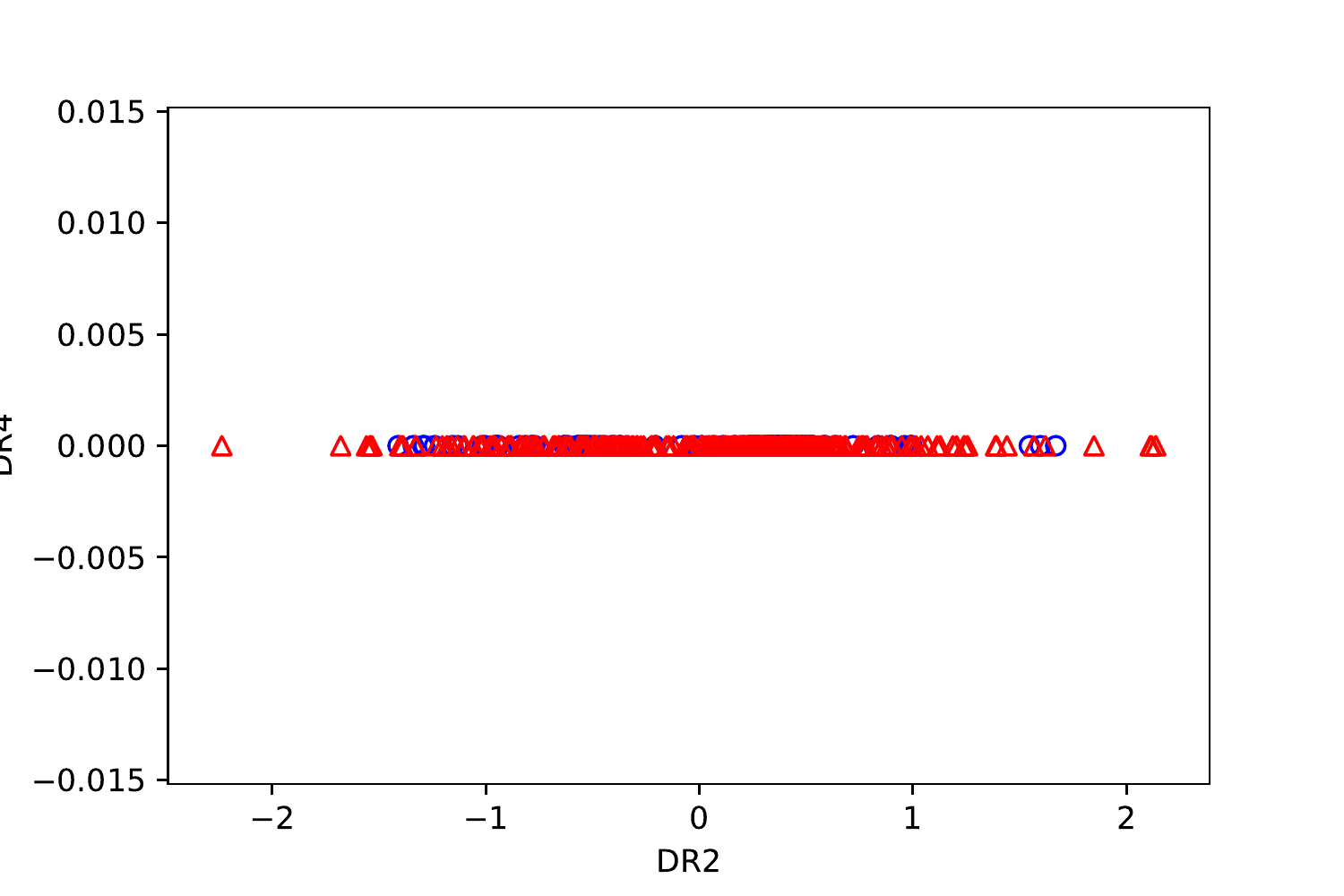}
			\includegraphics[width=0.10\textheight]{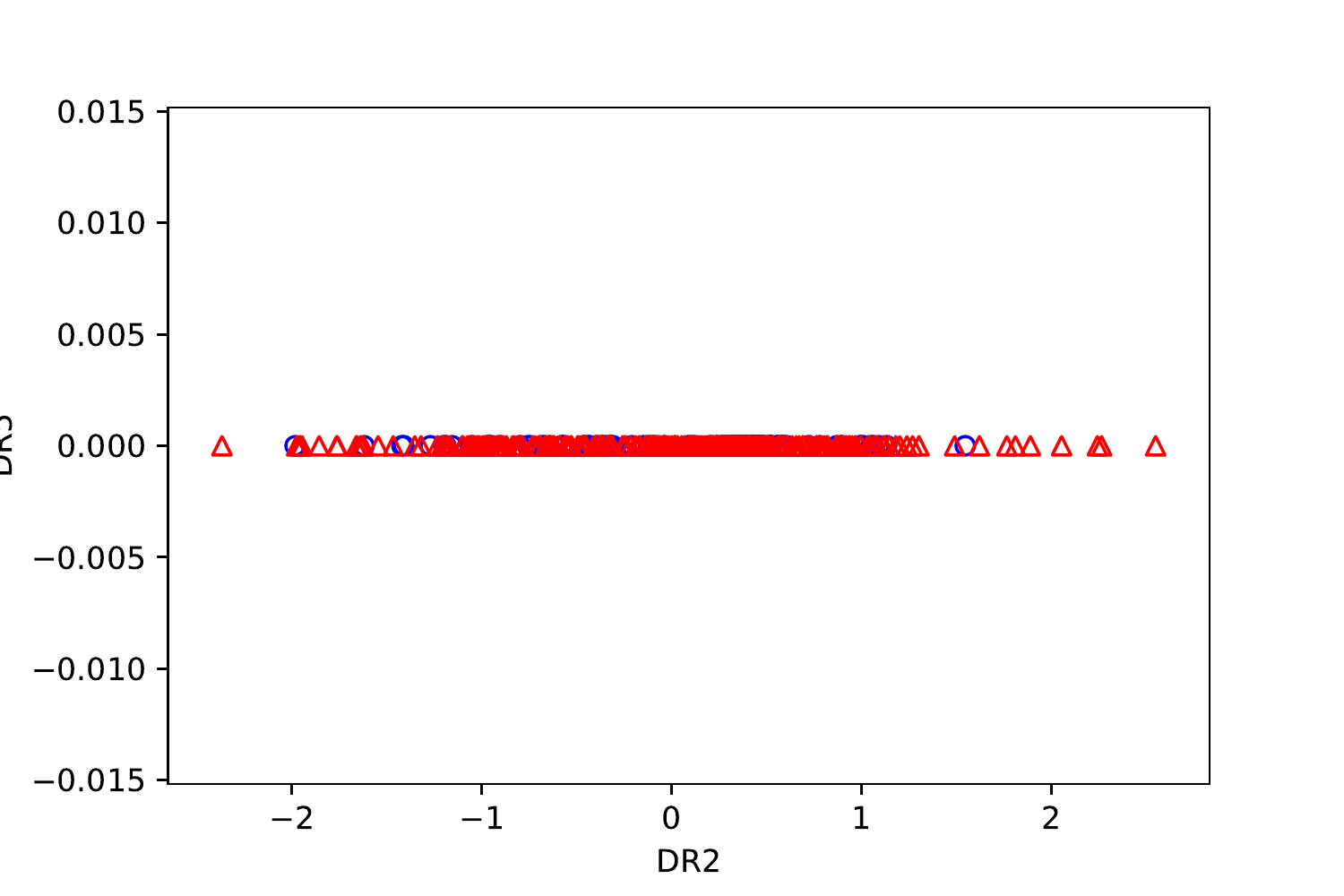}
			\includegraphics[width=0.10\textheight]{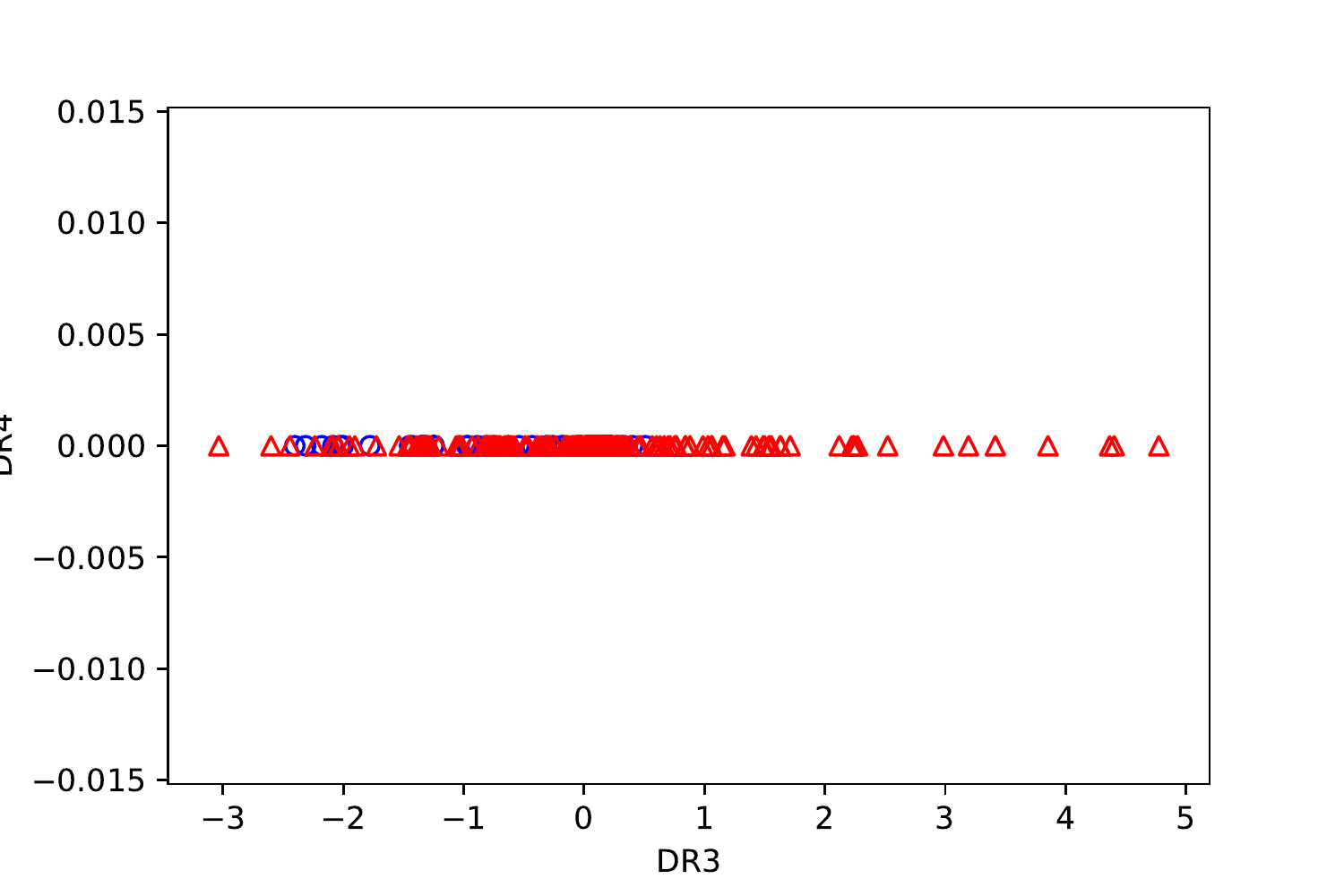}
			\includegraphics[width=0.10\textheight]{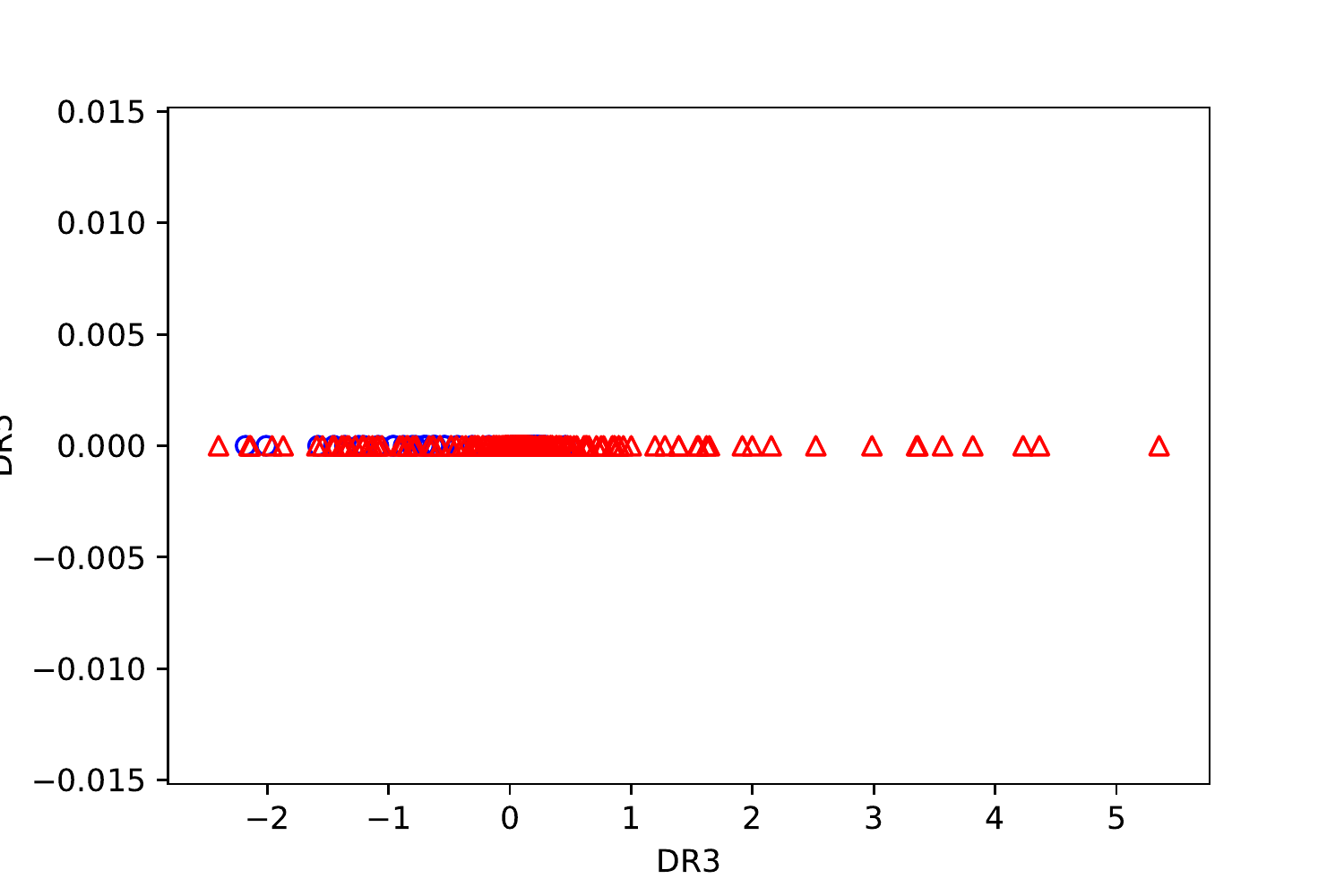}
			\includegraphics[width=0.10\textheight]{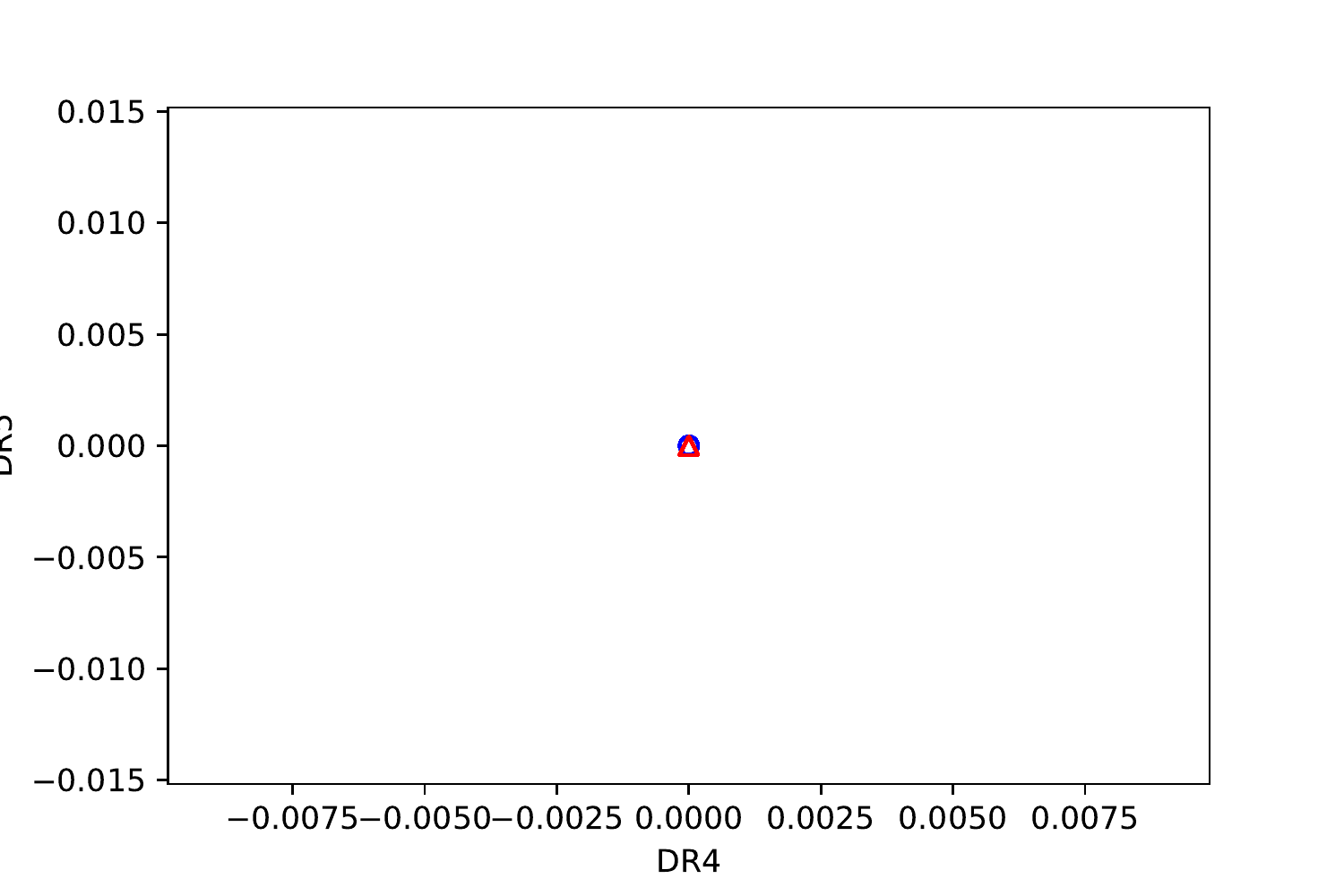}
		\end{minipage}
	\end{minipage}
	\caption{\label{Poldata_drawingA}The plots of the transformed data based on each pair of the components of the learned representation by SIR and SAVE for the Pole-Telecommunication dataset. Bule 'O' and red '$\Delta$' means that $Y\ge50$ and $Y\textless50$, respectively.}
\end{figure}

\subsection{Intrinsic Dimensionality Determination of Model (\ref{Toy_Model_d_est})}\label{Descrip_Implement_dest}
The number of total data points is $8000$. $8000$ data points are first divided into $6000$ training-validation data points and $2000$ testing sdata points and then $6000$ training-validation data points are divided into $4000$ training data points and $2000$ validation data points. $\mathcal{R}$ has $1$ hidden layer with width $128$. Both $\mathcal{D}$ and  $\mathcal{Q}$ also has $1$ hidden layer with width $64$.  The maximum number of epochs is $2000$, $\textit{lr}=0.001,\textit{wd}=0.0001$ and the \textit{patience} is $200$.

\newpage
\bibliographystyle{imsart-nameyear}
\bibliography{RLMI_bib_new}

\end{document}